%% file: main.tex
\documentclass[runningheads]{llncs}
\usepackage[T1]{fontenc}
\usepackage{graphicx}
\usepackage[utf8]{inputenc} % allow utf-8 input
\usepackage[T1]{fontenc}    % use 8-bit T1 fonts
\usepackage{hyperref}       % hyperlinks
\usepackage{url}            % simple URL typesetting
\usepackage{booktabs}       % professional-quality tables
\usepackage{amsfonts}       % blackboard math symbols
\usepackage{nicefrac}       % compact symbols for 1/2, etc.
\usepackage{microtype}      % microtypography
\usepackage{xcolor}         % colors
\usepackage[numbers,sort&compress]{natbib} 
\usepackage{algorithm,algorithmic}
\usepackage{bbm}
\usepackage{amsmath,amssymb,amsfonts, thmtools}
\usepackage{mathtools}

\usepackage{bm}
\newcommand{\mb}{\mathbb}
\newcommand{\mc}{\mathcal}

\DeclareMathOperator{\supp}{supp}

\newcommand{\simplex}{
  \mathchoice
    {\includegraphics[height=1.4ex, trim= 0pt 30pt 0pt 0pt]{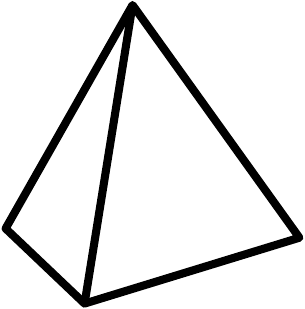}} % \displaystyle
    {\includegraphics[height=1.4ex, trim= 0pt 25pt 0pt 0pt]{simplex.pdf}} % \textstyle
    {\includegraphics[height=1.1ex, trim=0pt 30pt 0pt 0pt]{simplex.pdf}} % \scriptstyle
    {\includegraphics[height=.8ex, trim=0pt 30pt 0pt 0pt]{simplex.pdf}} % \scriptscriptstyle
}
\usepackage{graphicx}
\newcommand*\R[0]{\mathbb{R}}
\usepackage{caption}
\usepackage{subcaption}
\usepackage{url}
\usepackage{cleveref}
\usepackage{macros}
\usepackage{pifont}
\newcommand \yes {\ding{51}}
\newcommand \no  {\ding{55}}
\begin{document}
\title{On the Limitations and Possibilities of Nash Regret Minimization in Zero-Sum Matrix Games under Noisy Feedback}
\titlerunning{Limitations and Possibilities of Nash Regret Minimization in Matrix Games}
\author{Arnab Maiti\orcidID{0000-0002-9142-6255} \and
Kevin Jamieson\orcidID{0000-0003-2054-2985} \and
Lillian J. Ratliff\orcidID{0000-0001-8936-0229}}
\authorrunning{A. Maiti et al.}
\institute{University of Washington\\
\email{arnabm2@uw.edu}\\
\email{jamieson@cs.washington.edu}\\
\email{ratliffl@uw.edu}\\
}
\maketitle              % typeset the header of the contribution
\begin{abstract}
This paper studies a variant of two-player zero-sum matrix games, where, at each timestep, the row player selects row $i$, the column player selects column $j$, and the row player receives a noisy reward with expected value $A_{i,j}$, along with noisy feedback on the input matrix $A$. The row player’s goal is to maximize their total reward against an adversarial column player. Nash regret, defined as the difference between the player’s total reward and the game’s Nash equilibrium value scaled by the time horizon $T$, is often used to evaluate algorithmic performance in zero-sum games.

We begin by studying the limitations of existing algorithms for minimizing Nash regret in zero-sum games. We show that standard algorithms—including Hedge, FTRL, and OMD—as well as the strategy of playing the Nash equilibrium of the empirical matrix—all incur $\Omega(\sqrt{T})$ Nash regret, even when the row player receives noisy feedback on the entire matrix $A$. Furthermore, we show that UCB for matrix games, a natural adaptation of the well-known bandit algorithm, also suffers $\Omega(\sqrt{T})$ Nash regret under bandit feedback. Notably, these lower bounds hold even in the simplest case of $2 \times 2$ matrix games, where the instance-dependent matrix parameters are constant. This highlights a fundamental limitation of existing methods: they fail to leverage favorable problem structure to achieve lower regret.

Motivated by this limitation, we ask whether instance-dependent $\operatorname{polylog}(T)$ Nash regret is achievable against adversarial opponents. We answer this affirmatively. In the full-information setting, we present the first algorithm for general $n \times m$ matrix games that achieves instance-dependent $\operatorname{polylog}(T)$ Nash regret. In the bandit feedback setting, we design an algorithm with similar guarantees for the special case of $2 \times 2$ games—the same regime in which existing algorithms provably suffer $\Omega(\sqrt{T})$ regret despite the simplicity of the instance. Finally, we validate our theoretical results with empirical evidence.

\end{abstract}
\newpage
\input{introduction}

\input{results}

%\input{experiments}
\input{discussion}

\bibliography{refs.bib}
\bibliographystyle{plainnat}
\newpage
\appendix
\section*{Appendix / supplemental material}
\input{appendix}

\end{document}

%% file: introduction.tex
\section{Introduction}
%\textcolor{red}{Make changes in intro to remove UCB and maybe account for MARL papers.}

This paper considers a two-player zero-sum repeated game where the environment, or game matrix, is unknown before the start of play. At the beginning of each round, each player simultaneously chooses from a finite set of actions and then receives their respective rewards, which are equal and opposite, along with stochastic feedback on the environment. For example, consider two firms in a duopoly competing for customers. On any given week, prospective clients submit bids to both firms, and the first firm may offer lower prices while the second firm may offer a more valuable warranty. At the end of the week, both firms know which actions the other took and the state of the market. This game is challenging for two reasons: first, neither firm knows precisely what outcome will occur for any given pair of actions due to uncertainty and stochasticity in the market, and second, they are competing against each other where one firm's gain is the other firm's loss.

When interacting in a stochastic environment, each player is motivated to play the high-reward actions in the environment. However, unlike in a multi-armed bandit game where agents can choose from which arm they receive a reward, in this competitive game, an opponent can actually prevent the agent from receiving a reward from part of the game matrix by simply refusing to play certain actions. For example, in a game of rock-paper-scissors, if the second player never plays "rock," the first player can never receive a reward from the "rock" column of the game matrix. Therefore, instead of a notion of best entry of the game matrix (analogous to the best arm in bandits), the notion of Nash equilibrium is used to assess the performance of each player. For instance, in the same rock-paper-scissors game, the second player not playing "rock" is bad for them, as the first player can play a randomized strategy of $(1/3,2/3)$ over ("rock", "scissors") to get a higher expected reward than they would by playing the equilibrium strategy $(1/3,1/3,1/3)$ over ("rock," "paper," "scissors").
%\ljr{Somethign here doesnt make sense. both of these sentences are saying the second player never plays rock, but there is a "However, on the other hand..."}\arn{Made it clear.}
%\todo{\tiny Should we be so explicit with the values here. We didn't define yet what exact the payoffs of the rock-paper-scissors game is and we didn't define the nash regret definition yet. Maybe a better statement would be- "the first player can play a randomized strategy of $(2/3,1/3)$ over (``rock'',``scissors'') to get an higher expected reward than it gets by playing the equilibrium strategy $(1/3,1/3,1/3)$."}

This paper formally studies situations like those illustrated above in a problem setting at the intersection of multi-armed bandits and game theory.
Specifically, we are interested in the limits of an agent acting in a stochastic environment against an adversarial opponent trying to obtain as much reward as possible. 
Similar questions have been posed before and so-called ``no-regret'' algorithms have played a vital role in understanding game dynamics and computing various equilibrium concepts in multi-player games (see \cite{daskalakis2011near,chen2020hedging,cai2011minmax,freund1999adaptive,rakhlin2013optimization,syrgkanis2015fast,weilinear}). 
A key feature of these algorithms in the repeated-play setting is that the players adapt based on the feedback they receive. 
This paper diverges from this work by assuming that the game matrix is unknown and only noisy observations of its entries are available.
Outside of game theory, multi-armed bandits for regret minimization has been widely studied in scenarios where we observe noisy samples of the input (see \cite{auer2002finite,agrawal2017near,kaufmann2012thompson,kaufmann2016complexity}). 
% Multi-armed bandits can be used to model real world problems such as clinical trial (\cite{zelen1969play}), online-advertisements (\cite{pandey2007bandits}), online auctions (\cite{kleinberg2003value}) and strategic pricing (\cite{bergemann1996learning}). 
This paper works at the intersection of the repeated-play setting and the noisy observation setting to focus on a regret minimization problem in a two-player zero-sum matrix game, where players observe the elements of the matrix with some additional sub-Gaussian noise. 
Our work is inspired by the a recent work of \citet{o2021matrix} that introduced this setting.

% In multi-player learning, the repeated-play setting captures many real world problems like traffic routing (\cite{leblanc1975algorithm}), market prediction (\cite{fainmesser2012community}) and social network dynamics (\cite{skyrms2009dynamic}). 

\subsection{Problem Setting}
Let us first establish some notion for  zero-sum matrix games. 
Consider a matrix $A\in [-1,1]^{n\times m}$. 
Let $\simplex_{m}$ denote the $m$-dimensional simplex, and  let $e_i$ denote the usual canonical vector whose $i$-th coordinate is one and rest of the coordinates are zero.
A pair of strategies $(x^*,y^*)\in \simplex_{n}\times\simplex_{m}$ is a Nash equilibrium if the following inequalities hold:
\begin{equation*}
    \langle x,Ay^*\rangle\leq  \langle x^*,Ay^*\rangle \leq  \langle x^*,Ay\rangle\quad \forall (x,y)\in \simplex_{n}\times\simplex_{m}.
\end{equation*}
For the sake of clarity and technical novelty, we focus on scenarios where the Nash equilibrium is unique within the main body. Instances where the equilibrium is non-unique merely represent degenerate cases of the unique scenario. Given space limitations, the details for the non-unique case are presented in the appendix.
%, a common assumption in the literature (see \cite{daskalakis2019last,weilinear,bailey2018multiplicative}) as the set of matrices with non-unique Nash equilibrium is closed and has a Lebesgue measure of zero in the space of all matrices (see \cite{van1991stability}). 

When the Nash equilibrium $(x^*,y^*)$ is unique, we have $|\supp(x^*)|=|\supp(y^*)|$ where $\supp(v)=\{i:v_i>0\}$ is the support of a vector $v$. Let the value of the game defined by $A$ be given by
\begin{align*}
    V_A^*:= x^{*\top} A y^* =\max_{x\in\simplex_n}\min_{y\in\simplex_m}\langle x,Ay\rangle = \min_{y\in\simplex_m} \max_{x\in\simplex_n}\langle x,Ay\rangle, 
\end{align*}
%denote the value of matrix game on $A$, 
where the exchange of the min and max is guaranteed by Von Neumann's minimax theorem.

% At the beginning of each round $t =1,\dots,T$  the row and column players simultaneously choose strategies $x_t\in \simplex_{n}$ and $y_t\in\simplex_{m}$, respectively.
% Next, $i_t \sim x_t$ and $j_t \sim y_t$ and the row and columns players play actions $i_t$ and $j_t$ respectively.   
% At the end of the round, the row player receives a reward of $A_{i_t,j_t}$ and the column player receives a reward of $-A_{i_t,j_t}$.
% Note that $\E[ A_{i_t,j_t}  | x_t,y_t] = x_t^\top A y_t$.

For any fixed $A \in [-1,1]^{n \times m}$ with a unique Nash eqiulibrium $(x^*,y^*)$, consider the following two-player repeated-game:
In each round $t=1,\dots,T$, the row player chooses a strategy $x_t \in \simplex_{n}$ against an adversarial column player that chooses a strategy $y_t \in \simplex_m$.
Then, a random matrix $\mathbf{A}_t \in [-1,1]^{n \times m}$ is drawn IID where $\mathbb{E}[\mathbf{A}_t] = A$. 
 Next, a row index $i_t$ is sampled from the distribution $x_t$, and a column index $j_t$ is sampled from the distribution $y_t$. 
 At the end of round $t$, the row player observes the column index $j_t$ and receives one of the following forms of feedback:
\begin{itemize}
    \item \textbf{Full information feedback:} The row player observes $\mathbf{A}_t$. 
    \item \textbf{Bandit feedback:} The row player observes only the entry of $\mathbf{A}_t$ corresponding to row and column indices played: $[\mathbf{A}_t]_{i_t,j_t}$.
\end{itemize}

Now under the above setting, we are interested in only controlling the row-player and aim to maximize its expected cumulative reward $\mathbb{E}[ \sum_{t=1}^T [ \mathbf{A}_t ]_{i_t,j_t} ] = \mathbb{E}[ \sum_{t=1}^T x_t^\top A y_t ]$ while playing against a potentially adversarial column-player.
Before the start of the game, the matrix $A$ is unknown to the row-player. 
We make no assumptions about the adversarial column-player, and even allow them to have knowledge of $x_t$ and the matrix $A$. 
After $T$ time periods, we consider two measures of performance for the row-player: the \emph{external regret}, namely
    \begin{align*}
        R^E(A,T) &= \mathbb{E}\left[\max_{x \in \simplex_n} \sum_{t=1}^T \left(x^\top A y_t - x_t^\top A y_t\right)\right],
        %\\
        %&\text{(External Regret)} 
    \end{align*}
    and the \emph{Nash regret} defined with respect to a mixed strategy Nash equilibrium $(x^*,y^*)$, namely
    \begin{align*}
        R^N(A,T) &= \mathbb{E}\left[\sum_{t=1}^T \left(x^{*\top} A y^* - x_t^\top A y_t\right)\right]\\
        &=\mathbb{E}\left[\sum_{t=1}^T \left(V_A^* - x_t^\top A y_t\right)\right].
        %\\
        %&\text{(Nash Regret)} 
    \end{align*}
    By adding and subtracting $x^{*\top} A y_t$ in each term of $R^N(A,T)$ and applying  properties of Nash equilibrium, it follows that $R^N(A,T) \leq R^E(A,T)$.
    Moreover, it is known that there exists a row-player strategy (e.g., Hedge and its bandit variants) such that $\sup_A R^E(A,T) \leq \sqrt{T\cdot\text{poly}(n)}$ against any column-player adversary.
    However, we argue next that external regret does not adequately capture incentives of the two players. 
    % \subsection{On Nash regret versus external regret}
    % While external regret is common in the game theory literature, Nash regret is a concept less frequently studied, with a major exception being \cite{o2021matrix} that is in an identical setting to our own.  
    % For zero-sum games, we argue that Nash regret is a much more natural quantity of interest than external regret.  
    \begin{claim}[Informal; see~Appendix \ref{appendix:lower:ext-nash}] \label{claim:main-nash-external}
        Fix $A=[ 0, -1, 1; 1, 0, -1; -1, 1, 0]$ to encode rock-paper-scissors so that $V_A^* = 0$. 
        For any row-player that enjoys $R^E(A,T) \leq c_1T$ against any adversary, there exists an adversarial column-player such that $R^E(A,T) \geq c_2\sqrt{T}$ and $\mathbb{E}[\sum_{t=1}^T \left( x_t^\top A y_t-x^{*\top} A y^* \right)] \geq c_3T$, where  $c_1,c_2,c_3$ are absolute constants.
    \end{claim}
    %\todo{\tiny It would be better if we hide the constants in order to avoid inconsistency between the main body and the appendix. Yes  $c_1,c_2$ for absolute constants}
    
    On the one hand, given that the Hedge algorithm achieves $R^E(A,T) \leq \sqrt{T}$ on this instance, the adversary of the claim appears to be applying maximum pain to the row-player. 
   On the other hand,  we observe that the column-player is suffering \emph{linear} Nash regret. 
    This is the equivalent of two firms competing for market share in a game where, if playing optimally at the Nash equilibrium, both players would split the market 50/50. 
    But instead, the column-player that aims to maximize the external regret of the row-player is ceding market share and splitting the market 75/25 in the row-player's favor!   
    This example shows that for two strategic agents who are self-interested and solely focused on maximizing their own sum of rewards without any cooperation, Nash regret is a far more representative notion of regret. Moreover, the notion of the best fixed strategy in hindsight, used in external regret, is a relative measure, as it does not serve as a fixed baseline when the column player is allowed to be adaptive. This is not the case with Nash regret, as we are comparing the performance of an algorithm against the fixed value of the game $V_A^*$.
    
    % Since it is known that there exists a row-player strategy such that $\sup_A R^E(A,T) \leq 1/\sqrt{T}$ against any column-player adversary, it appears the column-player alluded to in the claim is applying maximum pain to the row-player in terms of external regret. 
    % However, 

%     The Nash regret is a natural objective when both players are aiming to suffer as little total loss as possible and play close to optimally (on average).
%     On the other hand, the external regret may be more reasonable if the $\mc{Y}$ player is deadset on maximizing the total loss of the $\mc{X}$ player, even if that means accumulating large loss itself. 
%     We follow the lead of \cite{o2021matrix} and focus on Nash regret. 

% \begin{equation*}
%     R(A,T)=\sum_{t=1}^T\mathbb{E}\left[\max_{x\in\Delta_n}\min_{y\in\Delta_m}\langle x,Ay\rangle-\langle x_t,Ay_t\rangle\right]=T\cdot V_A^*-\sum_{t=1}^T\mathbb{E}[\langle x_t,Ay_t\rangle]
% \end{equation*}

In addition to our argument for studying Nash regret, we refer the reader to \citet{o2021matrix,tian2021online,jin2022power}, who have likewise adopted Nash regret as their metric of interest. They analyzed algorithms that play the Nash equilibrium of empirical matrices (or variants thereof). While their results demonstrate a Nash regret of $\sqrt{T}$, analogous instance-dependent logarithmic regret, as seen in Stochastic Multi-Armed Bandits, has not been established for matrix games. This naturally leads us to the following question:
\begin{quote}
\emph{Are the current algorithmic designs fundamentally limited to $\sqrt{T}$ Nash regret, and is there an algorithm that can achieve instance-dependent poly-logarithmic Nash regret against an adversary with respect to the time horizon $T$?}
\end{quote}

\subsection{Our Contributions}
\begin{table}[h]
\centering
\begin{tabular}{|l|l|c|c|l|}
\hline
\textbf{Feedback type} & \textbf{Algorithm} & \textbf{polylog($T$) Regret} & \textbf{2×2 Case} & \textbf{Section} \\
\hline
Full Information & Hedge / FTRL / OMD & \no & \no & Sec.~\ref{sec:myopic-online} \\
\hline
Full Information & Emp. Nash Equilibrium & \no & \no & Sec.~\ref{sec:nasheq-failure} \\
\hline
Bandit & UCB (from \cite{o2021matrix}) & \no & \no & Sec.~\ref{sec:ucb-failure} \\
\hline
Full Information & Our Algorithm $(n\times m)$ & \yes & \yes & Sec.~\ref{sec:algo} \\
\hline
Bandit & Our Algorithm $(2\times2)$ & \yes & \yes & Sec.~\ref{appendix:bandit-feedback} \\
\hline
\end{tabular}
\caption{Summary of Nash regret guarantees under full information and bandit feedback in zero-sum matrix games. The table presents which algorithms can—and cannot—achieve polylog($T$) Regret, even in the 2×2 case.}
\label{tab:main-results}
\end{table}

We summarize our main contributions below, addressing the question posed at the end of the previous section:
\begin{itemize}
\item In \Cref{sec:myopic-online}, we show that Hedge, FTRL, OMD, and any other myopic online learning algorithm cannot achieve better than $\sqrt{T}$ Nash regret, and in particular, cannot attain instance-dependent $\operatorname{polylog}(T)$ Nash regret. This limitation holds even in the special case of $2 \times 2$ matrix games under full-information feedback.
\item In \Cref{sec:nasheq-failure}, we show that the algorithm which plays the Nash equilibrium of the empirical matrix in each round also fails to achieve better than \(\sqrt{T}\) Nash regret, and cannot attain instance-dependent \(\operatorname{polylog}(T)\) regret. This again holds even for \(2 \times 2\) matrices under full-information feedback.
\item In \Cref{sec:ucb-failure}, we show that the UCB algorithm proposed by \citet{o2021matrix} for matrix games under bandit feedback cannot outperform \(\sqrt{T}\) Nash regret or achieve instance-dependent \(\operatorname{polylog}(T)\) regret, even in the \(2 \times 2\) case.
\item In \Cref{sec:algo}, we present a novel algorithm for general \(n \times m\) matrix games that achieves instance-dependent \(\operatorname{polylog}(T)\) Nash regret against any adversary under full-information feedback.
\item In \Cref{appendix:bandit-feedback}, we show that for the important special case of $2 \times 2$ matrix games, there exists an algorithm that achieves instance-dependent $\operatorname{polylog}(T)$ Nash regret under bandit feedback.
\end{itemize}

Additionally, in \Cref{sec:experiments-full,sec:bandit-experiments}, we provide empirical evidence supporting the theoretical guarantees above. Finally, in \Cref{sec:instance-dep-external}, we show that no algorithm can achieve instance-dependent $\operatorname{polylog}(T)$ external regret, even in $2 \times 2$ games, thereby highlighting a fundamental separation between external and Nash regret in this setting.

% We answer this question in the affirmative for $n\times m$ zero-sum games with noisy matrix feedback. We also show that certain algorithmic designs including EXP3 and the UCB-based algorithm necessarily incur a $\sqrt{T}$ Nash regret on those $2\times 2$ matrices for which our algorithm incurs a logarithmic Nash regret. Moreover for these $2\times 2$ matrices, if the row player receives a noisy bandit feedback of $[\textbf{A}_t]_{i_t,j_t}$ instead of the random matrix $\textbf{A}_t$ we still continue to incur logarithmic Nash regret. In the sections that follow, we state our main results, describe the major pieces of the algorithms and analysis, and show our experimental results which supports our theoretical findings.

\input{related_work}

%% file: related_work.tex
\subsection{Related Work}
%\kevin{we need to slim this down}

\textbf{Multi-Armed Bandits.} Instance-dependent regret minimization and best-arm identification has been studied extensively in stochastic multi-armed bandits. \citet{auer2002finite} showed an instance-dependent upper bound of $O(\sum_{i:\mu_i<\mu^*}\frac{\log T}{\Delta_i})$ for the Upper confidence bound algorithm where $\Delta_i=\mu^*-\mu_i$. \citet{agrawal2017near} and \citet{kaufmann2012thompson} subsequently showed a similar instance-dependent upper bound for the Thompson sampling. \citet{jamieson2014lil} showed an instance-dependent upper bound of $O(\sum_{i:\mu_i<\mu^*}\frac{\log\log(1/\Delta_i^2)}{\Delta_i^2})$ for the best-arm identification (BAI) problem. \citet{kaufmann2016complexity} showed instance-dependent lower bounds for the BAI problem. 

%Recently, instance dependent bounds has also been studied in Reinforcement Learning. \citet{simchowitz2019non}, \citet{xu2021fine}, and \citet{dann2021beyond} provide gap based regret bounds for episodic Markov decision processes. Finally, \citet{wagenmaker2022beyond} provided instance dependent sample complexity bounds for PAC tabular Reinforcement learning.

\textbf{Learning in Games.}
\citet{o2021matrix} were the first to study matrix games under the bandit feedback setting, providing a $\sqrt{T}$ cumulative regret upper bound for the UCB algorithm adapted to this setting. There has also been  work on analyzing Nash regret in time-varying zero-sum matrix games (see \cite{cardoso2019competing,zhang2022no}). Recently, \citet{maiti2023instance} initiated a study on instance-dependent sample complexity bounds for computing approximate equilibrium concepts in two-player zero-sum matrix games. 

No-regret learning has also been an area of significant interest in the non-stochastic setting as the average-iterate dynamics converge to a Nash equilibrium. Worst-case external regret bound in repeated play setting is known to be $\Omega(\sqrt{T})$ (see \cite{cesa2006prediction}). However, when the players are not adversarial and instead are competing with each other, there exists strongly uncoupled dynamics that converge at a $O(\frac{\log T}{T})$ rate in two-player zero-sum matrix games (see \cite{daskalakis2011near}). \citet{daskalakis2011near} also achieved a Nash regret of $O\left((\log T)^{3/2}\right)$ when both the players play the same algorithm. 

%\citet{kangarshahi2018let} later designed a no-regret algorithm based on Optimistic mirror descent that achieved a $O(\frac{1}{T})$ convergence rate with respect to the value of game.

%Recently, the last-iterate guarantees of no-regret algorithms have also been studied (see \cite{daskalakis2018training,mokhtari2020unified,mertikopoulos2019optimistic,daskalakis2019last,lei2021last,cen2021fast}). Gradient descent ascent and multiplicative weights update either diverge away from the equilibrium or cycle around the equilibrium even in simple instances (see \cite{mertikopoulos2018cycles,bailey2018multiplicative}). \citet{weilinear} showed that the optimistic versions of these no-regret algorithms exhibit linear last-iterate convergence in two-player zero-sum matrix games. \citet{abe2023last} also showed a similar last-iterate convergence result in noisy-feedback setting.

% The no-regret learning framework has also been studied in the multi-player game setting (see \cite{syrgkanis2015fast,chen2020hedging, blum2007external,hart2000simple,foster1997calibrated}). \citet{daskalakis2021near} showed that a variant of multiplicative weight update achieves an external regret of $O((\log T)^4)$ in multi-player general sum games. \citet{anagnostides2022near} later extended this result to swap regret.
% Recently, \citet{anagnostidesuncoupled} designed an uncoupled dynamics that has a swap regret of $O(\log T)$. 

%\textcolor{red}{Add the MARL papers and cut short the above section.}

\textbf{Multi-agent Reinforcement Learning.} 
Recently, several works have focused on minimizing Nash regret in various Multi-agent Reinforcement Learning (MARL) settings. For example, \citet{tian2021online} proposed a \( T^{2/3} \) Nash regret algorithm for two-player zero-sum Markov games where a player cannot observe the actions of the opponent. They also developed a \( T^{1/2} \) Nash regret algorithm for the case where a player can observe the opponent's actions. In other MARL settings, \citet{jin2022power} demonstrated similar \( T^{1/2} \) Nash regret guarantees. Meanwhile, \citet{liu2022learning} established fundamental limits for minimizing external regret in two-player zero-sum Markov games.

%% file: results.tex
\section{Failure of various algorithmic designs and impossibility results}\label{sec:failure}
In this section, we consider matrices $A = \begin{bmatrix} a & b \\ c & d \end{bmatrix}$ such that $\Delta_{\min} := \min\{|a - b|, |a - c|, |d - b|, |d - c|\}$ is a positive constant. These matrices are considered "easy" instances from a sample complexity perspective (see \cite{maiti2023instance}), and we show in later sections that this is also the case for Nash regret minimization. We first examine why common algorithmic approaches fail to achieve polylogarithmic Nash regret even on such easy instances. We also present an impossibility result showing that instance-dependent logarithmic external regret is not achievable in this setting.
%\todo{(Algorithm 3)$\mapsto$ (see Appendix F)}
\subsection{The failure of myopic online learning algorithms}\label{sec:myopic-online}
Consider playing an myopic online learning algorithm over the $n$ actions in the game as the row player for $T$ rounds, such as Hedge, FTRL, and OMD. By myopic algorithm, we refer to one that updates solely based on the observed reward vector $\mathbf{A}_t e_{j_t}$, without exploiting the structure of the matrix or the information that a specific column $j_t$ was played in round $t$. For the matrix $A = \begin{bmatrix} 
   0.75 & 0.25\\
   0 & 1\\
    \end{bmatrix}$, if the column-player plays $y^{(0)}=(1/2,1/2)^\top$ for all rounds, then from the perspective of the online learning algorithm, this is just a two-armed stochastic multi-armed bandit instance with Bernoulli arms having means $(1/2,1/2)$.
    In this case the algorithm will play each arm an equal number of times in expectation. 
    However, if the column-player alternatively played $y^{(1)}=(\frac{1}{2}+\frac{1}{100\sqrt{T}},\frac{1}{2}-\frac{1}{100\sqrt{T}})$ and the algorithm continued to play both arms equally in expectation, then the algorithm would receive $T/2 - \sqrt{T}/400$ total expected reward, translating to a Nash regret of $\sqrt{T}/400$ since the value of the game $V_A^*=1/2$. 
    Since $T$ rounds are not sufficient to distinguish between $y^{(0)}$ and $y^{(1)}$, we can argue that any myopic online learning algorithm will incur $\sqrt{T}$ Nash regret in expectation.
    However, for this same $A$ matrix, our algorithms in \Cref{sec:algo,appendix:bandit-feedback} have $\log(T)^2$ Nash regret as our matrix-dependent parameters are constants.
    We refer the reader to Appendix \ref{appendix:exp3:lower} for a formal proof that leads to the following result.
    \begin{theorem}\label{thm:main-myopic}
        Under the full-information setting, Hedge, FTRL, OMD or any other myopic online learning algorithm cannot perform better than $\sqrt{T}$ Nash regret and achieve instance-dependent $\text{polylog}(T)$ Nash regret.
    \end{theorem}

\subsection{The failure of action-ignorant algorithms}\label{sec:nasheq-failure}
In the full information setting, we have that the row-player observes both the random reward matrix $\mathbf{A}_t$, and the action $j_t \in [m]$ that the column-player played. However, observing the column-player's action $j_t \in [m]$ is not always available in all settings.
For example, in a duopoly of ride-sharing firms we may observe the outcome of a competitor's actions, but we do not know precisely which actions they took to match riders with drivers on any particular day. 
Unfortunately, we show in the following theorem (stated formally in \Cref{appendix:lower}) that observing the opponent's action is necessary to achieve logarithmic Nash regret under the full information setting.
\begin{theorem}[Informal]\label{thm-impossible-no-action}
    Consider an algorithm that constructs $x_t$ from $\{ \mathbf{A}_s \}_{s < t}$ alone. Then such an algorithm cannot perform better than $\sqrt{T}$ Nash regret and achieve instance-dependent $\text{polylog}(T)$ Nash regret, even for $2\times 2$ games.
\end{theorem}
The main idea is to choose an adversary that always best responds. Since the learner does not observe the adversary’s actions, it behaves similarly on two $2 \times 2$ matrices whose entries differ by $1/\sqrt{T}$, and consequently incurs high regret on one of them. We refer the reader to \Cref{appendix:lower} for the formal proof of the above theorem. The theorem also leads to the following corollary.
\begin{corollary}
    Under the full-information setting, the algorithm that plays the Nash equilibrium of the empirical matrix in each round cannot perform better than $\sqrt{T}$ Nash regret and achieve instance-dependent $\text{polylog}(T)$ Nash regret.
\end{corollary}
\subsection{The failure of the matrix-game UCB algorithm of \citet{o2021matrix}}\label{sec:ucb-failure}
We consider the following version (Algorithm \ref{alg-ucb}) of Upper Confidence Bound (UCB) Algorithm under bandit feedback that was formulated by \citet{o2021matrix}. 
\begin{algorithm}[h!]
\caption{Upper Confidence Bound Algorithm}
\begin{algorithmic}[1]
\STATE For all elements $(i,j)$, $n_{ij}^0\gets 0$ and $\bar A_{ij}\gets 0$
\FOR{round $t=1,2,\ldots,T$}
\STATE For all elements $(i,j)$, $\Delta_{ij}^t\gets \sqrt{\frac{2\log(2T^2nm)}{\max\{1,n_{ij}^{t-1}\}}}$
\STATE For all elements $(i,j)$, $\tilde A_{ij}\gets \bar A_{ij}+\Delta_{ij}^t$
\STATE $x_t\gets \arg\max_{x\in \simplex_n}\min_{y\in \simplex_m}\langle x,\tilde Ay\rangle$
\STATE Let $i_t\sim x_t$ and $j_t\sim y_t$
\STATE Observe the element $(i_t,j_t)$ and update the empirical matrix $\bar A$
\ENDFOR
\end{algorithmic}
\label{alg-ucb}
\end{algorithm}

Consider the matrix $A =\mathrm{diag}(2/3,1/3)$.
%\begin{bmatrix} 
%   2/3 & 0\\
%   0 & 1/3\\
%    \end{bmatrix}$. 
Observe that $A$ has a unique Nash equilibrium $(x^*,y^*)=((1/3,2/3),(1/3,2/3))$. We  claim that UCB incurs a $\Omega(\sqrt{T\log T})$ regret on the matrix game  $A$ even in the non-stochastic setting. We first provide a high-level description of how the column-player chooses $y_t$ in each round $t$. There are two key phases: (1) burn-in phase and (2) best-response phase. In the burn-in phase, the column-player chooses strategies $y_t$ that are close to $y^*$ in order to force UCB to explore all the elements of $A$ sufficiently. This phase lasts for roughly $T/2$ rounds. Next, in the best response phase, the column players chooses the best response $y_t=\arg\min_{y\in\simplex_2}\langle x_t,Ay\rangle$.

    We now provide an intuition behind the $\Omega(\sqrt{T\log T})$ regret incurred by UCB. Let $n_{i,j}^t:=\sum_{s=1}^{t}\mathbbm{1}\{i_s=i,j_s=j\}$ for all $i,j$. Suppose the burn-in phase ends at  time step $t_0$ (which is roughly $T/2$). At the end of the burn-in phase, we have $n_{i,j}^{t}\approx x_i^*y_j^*t_0\approx \frac{x_i^*y_j^*T}{2}$. UCB maintains an upper confidence matrix $\tilde A=\begin{bmatrix} 
   2/3+\Delta_{11} & 0+\Delta_{12}\\
   0+\Delta_{21} & 1/3+\Delta_{22}\\
    \end{bmatrix}$ where $\Delta_{ij}\approx \sqrt{\frac{\log T}{n_{i,j}^t}}$ and plays $x_t=\arg\max_{x\in\simplex_2}\min_{y\in \simplex_2}\langle x,\tilde Ay\rangle$ in each round. In the best response phase, if for $T/c$ rounds (where $c$ is some constant) we have $(x_t)_1\leq 1/3-\sqrt{\frac{\log T}{4T}}$, then the regret incurred by the row-player during these rounds is at least $\sqrt{\frac{\log T}{9T}}\cdot T/c=\sqrt{\frac{T\log T}{9c^2}}$. Next we argue this condition holds.
    
    First, at the time step $t_0$ (i.e., the end of the burn-in phase), we have $(x_{t_0})_1\approx \frac{1/3+\Delta_{22}-\Delta_{21}}{1+\Delta_{22}-\Delta_{21}-\Delta_{12}+\Delta_{11}}\approx 1/3-\sqrt{\frac{\log T}{T}}$.   Now observe that $\sqrt{\frac{\log T}{T}}-\sqrt{\frac{\log T}{T+1}}\approx \sqrt{\frac{\log T}{T^3}} $. This implies that for $t\geq t_0$, we have that $(x_{t+1})_1-(x_t)_1\approx \sqrt{\frac{\log T}{T^3}}$. Hence, starting from $t=t_0$ where $(x_{t_0})_1\approx 1/3-\sqrt{\frac{\log T}{T}}$, it takes $T/c$ (where $c$ is a constant) steps for $x_t$ to become $(x_{t_0})_1\approx 1/3-\sqrt{\frac{\log T}{4T}}$. Hence, UCB incurs a regret of $\Omega(\sqrt{T\log T})$. However, for this matrix $A$, our algorithms in \Cref{sec:algo,appendix:bandit-feedback} achieve $\log(T)^2$ Nash regret.
    
    At a very high-level, UCB may be reacting too slowly by moving only $T^{-3/2}$ each round.
    As we will see in a moment, our algorithms in \Cref{sec:algo,appendix:bandit-feedback} move about $T^{-1}$ per round.
    We refer the reader to Appendix \ref{appendix:ucb:lower} for a more formal analysis.
\subsection{Impossibility result for external regret}\label{sec:instance-dep-external}
When introducing the notion of Nash regret, we pointed out that external regret is an upper bound on Nash regret for any game. 
Given our interest in achieving logarithmic Nash regret, a natural question is whether logarithmic external regret is also attainable.
Our next theorem rules this out.
\begin{theorem}\label{external:thm}
Consider a matrix $A=[a,b;c,d]$ with a unique full support Nash equilibrium. For any algorithm, there exists an adversary that plays $y_t$ in each round $t$ such that %the following holds:
% \begin{equation*}
    $R^E(A,T)=\mathbb{E} \big[\max_{x\in\simplex_2} \sum_{t=1}^T (x^\top Ay_t-x_t^\top Ay_t)\big] \geq c_A\sqrt{T}$, where $c_A$ is a matrix dependent parameter.
% \end{equation*}
\end{theorem}
%\todo{absolute constants, make sure instance dependent cosntants are correct}

Interestingly, the above theorem is in the non-stochastic setting meaning that the row-player has knowledge of $A$ exactly and can use it to choose its actions. 
This is because in the proof (see Appendix \ref{appendix:lower:external}), the external regret comes from the unpredictability of the opponent, not the game matrix.
Nevertheless, for the same matrix, we show in the next section that instance-dependent logarithmic Nash regret is indeed achievable against any adversary.

\section{Algorithm for Logarithmic Nash regret under Full Information Feedback in $n\times m$ Matrix Games}\label{sec:algo}

% Our algorithm begins with a burn-in phase of exploration that lasts until statistics about the input matrix $A$ concentrate and the support of the Nash equilibrium is identified. 
% This burn-in phase can incur substantial regret if continued indefinitely, but since its length depends on problem-dependent constants it contributes at most $\log(T)$ to the overall regret. 

% Once this burn-in phase terminates after $T_0$ rounds, the algorithms then allocates $T_2 \approx T_0$ rounds to a sub-routine which incurs $\log(T)$ regret. When the sub-routine ends, we update the statistics and execute the subroutine again for $T_2 \gets 2 T_2$ rounds. We repeat this process until we exhaust the time horizon.
% Since each execution of the sub-routine incurs $\log(T)$ regret, and it can only be executed at most $\log(T)$ times due to its exponentially growing time-scale, we conclude with a $\log(T)^2$ regret bound.
% Most of the interesting bits of the algorithm occur in the sub-routine, so we begin there. 

%\textcolor{red}{Maybe add some figures for both phases.}

Our algorithm operates in two phases. The first phase, the burn-in phase, focuses on identifying the support of the Nash equilibrium and ensuring that matrix statistics are well-concentrated. The second phase involves running a logarithmic regret subroutine multiple times, doubling the number of timesteps after each execution. Below, we explain the high level ideas for each phase.

\textit{Burn-in Phase}: If row indices within the support of $x^*$ are not played, the incurred regret can be substantial. This regret is determined by the difference $\langle x^*, Ay^* \rangle - \max_{i \notin \text{supp}(x^*)} \langle e_i, Ay^* \rangle = \Delta_x$, where $\Delta_x$ is a positive constant dependent on $A$. Hence, it is crucial to minimize the frequency of playing row indices outside the support. The burn-in phase achieves this by identifying the support.

As only noisy observations of the matrix are available in each round, the entries of the matrix are estimated with an error rate of $\sqrt{\frac{\log T}{t}}$ after $t$ rounds, simultaneously for all $t \in [T]$ with high probability. To ensure accurate support identification and concentration of the matrix statistics, we establish new confidence bounds on $||x^* - x'||_\infty$, where $x'$ is the empirical Nash equilibrium, as well as on the difference between $\Delta_x$ and its empirical estimate. Using these confidence bounds, we establish appropriate stopping conditions to ensure the goals of the burn-in phase are met. Throughout the burn-in phase, we play the Nash equilibrium of the empirical matrix, incurring only a small regret, as the stopping conditions guarantee a quick termination of this phase.

\textit{Second Phase}: The second phase begins after the burn-in phase accurately identifies the support and ensures concentration of the matrix statistics. In this phase, we invoke a subroutine using the standard doubling trick.  

Simple approaches, such as playing the Nash equilibrium of the empirical matrix or applying multiplicative/exponential weights, result in a regret of $\sqrt{T}$, as demonstrated in \Cref{sec:failure}. By leveraging the actions $j_t$ played by the column player and the structure of the matrix game, our subroutine achieves $\text{poly} \log T$ regret. The subroutine starts by playing a strategy $x_t$ close to $x^*$ and adjusts using small descent steps informed by the actions $j_t$ of the column player. This approach ensures low regret: either we move closer to $x^*$ when the adversary's actions aim to induce positive regret, or we secure a reward of at least $V_A^*$ when the adversary plays naively.

% \textit{Doubling Trick and Logarithmic Regret}: The subroutine is executed for $T_2$ timesteps per run, with $T_1 \approx \frac{t}{\log T}$ and $T_2 = t$, where $t$ denotes the timestep of invocation. The regret per run is approximately $\log T$ (ignoring instance-dependent terms). To facilitate analysis, the empirical matrix $\widehat{A}$ is not updated during a run. However, to limit regret growth, the subroutine’s runtime is capped at $T_2$ timesteps, followed by an update to $\widehat{A}$ and doubling of $T_2$. This process ensures that necessary conditions are preserved across runs and that the subroutine is executed only $\log T$ times. The total regret is therefore $\log^2 T$ (ignoring instance-dependent terms).

In the following sections, we provide a detailed explanation of both phases.
\subsection{Burn-in Phase for $n\times m$ matrix games}\label{sec-main:algo-n*m}
If $(x^*,y^*)$ is the unique Nash equilibrium of $A$, recall that $|\supp(x^*)| = |\supp(y^*)|$. Let $k := |\supp(x^*)|$. The burn-in phase consists of two major steps: $(i)$ identifying the $k \times k$ sub-matrix of $A$ associated with the support of $(x^*,y^*)$, and $(ii)$ confirming that the statistics have been estimated accurately enough to initiate the subroutine used in the second phase. Note that during this phase, we play the Nash equilibrium of the running empirical average matrix $\bar{A}_t:=\frac{1}{t}\sum_{s=1}^t\bfA_s$. This phase lasts at most $\log(T)$ steps, scaled by some problem-dependent constants. 

\textbf{Step 1: Identifying the $k\times k$ sub-matrix.}
In the Stochastic Multi-Armed Bandit problem, one approach to achieving logarithmic regret is to identify the best arm as quickly as possible. If the gap between the means of the best arm and the second-best arm is $\Delta$, then sampling each arm approximately $O\left(\frac{\log(T)}{\Delta^2}\right)$ times suffices to identify the best arm. This sufficiency is typically ensured by constructing confidence intervals for the means of each arm using Hoeffding's inequality. Similarly, in matrix games, we aim to achieve logarithmic regret by identifying the $k \times k$ sub-matrix of $A$ corresponding to the support of $(x^*, y^*)$ as quickly as possible. To this end, we define appropriate gap parameters and establish confidence intervals for them using new concentration inequalities. 

Define $\Delta_{g,1}:=V_A^*-\max_{i\notin \supp(x^*)} e_i^\top A y^*$,  $\Delta_{g,2}:=\min_{j\notin \supp(y^*)} x^{*\top} A e_j -V_A^*$, and $\Delta_g:=\min\{\Delta_{g,1},\Delta_{g,2}\}$.  If $k=n$, then $\Delta_{g,1}:=\infty$, and if $k=m$, then $\Delta_{g,2}:=\infty$. \citet{bohnenblust1950solutions} showed that $\Delta_g>0$. 

Assume that $k < \min\{n,m\}$; the other case can be handled analogously. We now present a method to identify $\supp(x^*)$ and $\supp(y^*)$. For any matrix $M$ and its submatrix $N$ with a unique Nash equilibrium, define  $\Delta_r(M,N) := V_{N}^* - \max_{i\notin \supp(\tilde x)}\sum_{j\in \supp(\tilde y)} M_{i,j} \tilde y_j$
and  $\Delta_c(M,N) := \min_{j\notin \supp(\tilde y)} \sum_{i\in \supp(\tilde x)} M_{i,j} \tilde x_i - V_N^*,$ where $(\tilde x, \tilde y)$ is the Nash equilibrium of $N$.  We now state an important lemma, whose proof is provided in Appendix \ref{appendix:nxm}.  
\begin{lemma}[Informal]\label{lem:main-k*k-check}
    Consider a $k\times k$ submatrix $B$ of $A$. Let $B$ have a unique full support Nash Equilibrium $(x_\star,y_\star)$. If $\Delta_r(A,B)>0$ and $\Delta_c(A,B)>0$, then $(x_\star,y_\star)$ is the unique Nash Equilibrium of $A$.
\end{lemma}

Intuitively, we proceed as follows. Suppose $t$ rounds of the game have elapsed. Let $(\widehat x,\widehat y)$ be the Nash equilibrium of $\bar{A}_t$. Let $\bar{B}_t$ be the submatrix of $\bar{A}_t$ with row indices $\supp(\widehat x)$ and column incides $\supp(\widehat y)$. We now check whether the submatrix $B$ of $A$, defined by the same row and column indices, is the desired $k \times k$ submatrix associated with $(x^*,y^*)$. If $|\supp(\widehat x)| \neq |\supp(\widehat y)|$, we conclude that $B$ is not the required matrix and proceed to play the Nash equilibrium of $\bar{A}_t$ in round $t+1$. If $|\supp(\widehat x)| = |\supp(\widehat y)|$, we can easily verify whether we have sampled sufficiently to ensure that $B$ has a unique full-support Nash equilibrium $(x_\star, y_\star)$. The details of this verification are deferred to Appendix \ref{appendix:nxm}. It remains to check whether $\Delta_r(A, B) > 0$ and $\Delta_c(A, B) > 0$. If these inequalities hold, then by \Cref{lem:main-k*k-check}, we can conclude that $B$ is the required matrix. We now informally state the following lemma, which will help us verify these inequalities.
\begin{lemma}[Informal]\label{lem:main-gap-concentration}
    Let us assume that for all pairs $(i,j)$, $|A_{i,j}-{(\bar A_t)}_{i,j}|\leq \Delta$. Then we have the following:
    \begin{align*}
        |\Delta_r(A,B)-\Delta_r(\bar{A}_t,\bar{B}_t)|\leq \beta_{\widehat k}\cdot \Delta \\
        |\Delta_c(A,B)-\Delta_c(\bar{A}_t,\bar{B}_t)|\leq \beta_{\widehat k}\cdot \Delta 
    \end{align*}
    where $\beta_{\widehat k}$ depends only on $\widehat k:=|\supp(\widehat x)|$.
\end{lemma}

If we observe that $\Delta_r(\bar{A}_t, \bar{B}_t) > \beta_{\widehat k} \cdot \Delta$ and $\Delta_c(\bar{A}_t, \bar{B}_t) > \beta_{\widehat k} \cdot \Delta$, we declare $B$ as the required matrix. It can be shown that identifying the required matrix takes at most $O(\log T)$ steps, scaled by problem-dependent constants. For formal details, we refer the reader to Appendix \ref{appendix:nxm}. %\arn{I didn't get the last line. What does "verifying that the coefficients are positive" refer to precisely? }

%\textcolor{red}{Maybe informally mention the confidence intervals and add more technical details.}

\textbf{Step 2: Confirm concentration of matrix estimates.}
Assume that we have identified the optimal $k \times k$ submatrix of $A$ corresponding to the support of the Nash equilibrium, denoted as $B$. Let $\bar{A}_t$ and $\bar{B}_t$ be the running empirical averages of $A$ and $B$, respectively. Our goal is to ensure that the empirical estimates of key matrix parameters, including the Nash equilibrium, have sufficiently concentrated, allowing us to proceed to the second phase.

Intuitively, we proceed as follows.
If $(x^*,y^*)$ is the unique Nash equilibrium of $B$ then
$e_1^\top B y^* = \max_{i=1,\dots,k} e_i^\top B y^* = V_B^*$.
Similarly, $B^\top x^* = V_B^* \mathbf{1}$
which implies $x^* \propto (B^\top)^{-1} \mathbf{1}$. Since $x^*$ sums to $1$ we conclude that $x^* = \frac{(B^\top)^{-1}\mathbf{1}}{\mathbf{1}^\top (B^\top)^{-1} \mathbf{1}} = \frac{\text{adj}(B^\top)\mathbf{1}}{\mathbf{1}^\top\text{adj}(B^\top)\mathbf{1}} $ and $y^* = \frac{\text{adj}(B)\mathbf{1}}{\mathbf{1}^\top\text{adj}(B)\mathbf{1}}$. Similarly, if $(\widehat x,\widehat y)$ is the unique Nash equilibrium of $\bar B_t$ then $\widehat x = \frac{\text{adj}(\bar B_t^\top)\mathbf{1}}{\mathbf{1}^\top\text{adj}(\bar B_t^\top)\mathbf{1}} $ and $\widehat y = \frac{\text{adj}(\bar B_t)\mathbf{1}}{\mathbf{1}^\top\text{adj}(\bar B_t)\mathbf{1}}$. 

Define $D=\mathbf{1}^\top\text{adj}(B)\mathbf{1}$ and $\widehat{D} = \mathbf{1}^\top\text{adj}(\bar{B}_t)\mathbf{1}$. 
% \begin{align*}
%     &\Delta_{\min}=\min\Big\{\min_{i\in[k]} |e_i^\top \text{adj}(B)\mathbf{1}|,\min_{j\in[k]}|e_j^\top \text{adj}(B^\top)\mathbf{1}| \Big\}\\
%     &\text{ and }\quad D=\mathbf{1}^\top\text{adj}(B)\mathbf{1}.
% \end{align*}
 If $\max_{i,j}|(\bar{B}_t)_{i,j} - B_{i,j}| \leq \epsilon$, then there exists an $\alpha_k > 0$ that depends only on $k$ such that $\|\text{adj}(\bar{B}_t^\top)\mathbf{1} - \text{adj}(B^\top)\mathbf{1} \|_\infty \leq \alpha_k \epsilon$.
Hence, if $k \alpha_k \epsilon / |D| \leq 1/2$ then
\begin{align*}
    \widehat{x}_i=e_i^\top \widehat{x} &= \frac{|e_i^\top\text{adj}(\bar{B}_t^\top)\mathbf{1}|}{|\mathbf{1}^\top\text{adj}(\bar{B}_t^\top)\mathbf{1}|} 
    \leq \frac{ |e_i^\top \text{adj}(B^\top)\mathbf{1}| + \alpha_k \epsilon}{|\mathbf{1}^\top\text{adj}(B^\top)\mathbf{1}| - k \alpha_k \epsilon}\\ 
    &\leq \frac{ e_i^\top x^* + \alpha_k \epsilon/|D|}{1 - k \alpha_k \epsilon / |D|} 
    \leq x^*_i + 3 k \alpha_k\epsilon / |D|
    % &\leq e_i^\top x^* (1 + \epsilon n / D) + 2 \epsilon/ D \\
    % &\leq e_i^\top x^* + 3 \epsilon n / D
\end{align*}
Analogously, one can show that if $k \alpha_k \epsilon / |D| \leq 1/2$, then $\widehat{x}_i \geq x^*_i - c \cdot k \alpha_k \epsilon / |D|$, where $c$ is an absolute constant. Since  $\|\text{adj}(\bar{B}_t^\top)\mathbf{1} - \text{adj}(B^\top)\mathbf{1} \|_\infty \leq \alpha_k \epsilon$, it follows that $|\widehat{D} - D| \leq k \alpha_k \epsilon$. We define one of our stopping conditions for the burn-in phase as $\frac{k \alpha_k \epsilon}{|\widehat{D}| - k \alpha_k \epsilon} \leq \frac{1}{2}$. If this condition holds, then $\frac{k \alpha_k \epsilon}{|D|} \leq \frac{k \alpha_k \epsilon}{|\widehat{D}| - k \alpha_k \epsilon} \leq \frac{1}{2}$, which ensures concentration around the Nash equilibrium as shown above. Similarly, we impose additional stopping conditions using the submatrix $B$ to guarantee $\min\limits_{i\in\supp(x^*)}\{x_i^*,1-x_i^*\} \geq \frac{k^2\alpha_k\epsilon}{|D|}$.  

Once all stopping conditions are met, we terminate the burn-in phase and proceed to the second phase. It can be shown that the burn-in phase terminates in $O(\log T)$ steps, scaled by problem-dependent constants. For formal details, we refer the reader to Appendix \ref{appendix:nxn}.
%\arn{As I wrote earlier, $T_2$ and $T_1$ differ by a logarithmic factor. That is, $T_2=t_1$ and $T_1=\frac{t_1}{\log T}$}

%\textcolor{red}{Reword this section. Also explain in detail when the concentration is ensured and how the burn-in phase terminates.}

\subsection{Subroutine for $n\times m$ matrix games with full row-support}\label{sec:subroutine}
%\textcolor{red}{One motivating paragraph here.....}

%\textcolor{red}{This is where our work diverges from MAB. While they can play the best-arm for the rest of the rounds once it is identified, we cannot play the Nash equilibrium of the running of the $k\times k$ sub-matrix that we identify as we can still incur a regret of $\sqrt{T}$. We demonstrate the later fact in \Cref{sec:failure}.}

In the Stochastic Multi-Armed Bandit problem, once the best arm is identified, it can be played for the remaining rounds, resulting in zero regret. In contrast, in matrix games, merely identifying the support of the Nash equilibrium $(x^*, y^*)$ is not sufficient to achieve logarithmic regret. It is also necessary to play strategies $x_t$ that are close to the Nash equilibrium strategy $x^*$. One potential approach is to play the Nash equilibrium of the running empirical matrix at each round. However, as we demonstrate in \Cref{sec:nasheq-failure}, this approach leads to $\sqrt{T}$ regret. The limitation of this approach lies in its static nature, as it fails to account for the actions $j_t$ chosen by the column player. We overcome the $\sqrt{T}$ regret barrier by designing a dynamic subroutine that adapts to the actions $j_t$ and appropriately selects strategies close to the Nash equilibrium strategy $x^*$.

We now describe the subroutine that we plan to invoke repeatedly in the second phase. Assume that the input matrix $A \in [-1,1]^{n \times m}$ has a full row-support Nash equilibrium, i.e., $|\supp(x^*)| = n$. If this is not the case, we can remove the rows not in the support after the burn-in phase and then invoke the subroutine.  The subroutine takes as input a fixed matrix $\widehat{A} \in [-1,1]^{n\times m}$, a vector $x' \in \simplex_n$, parameters $D_1$ and $T_1$, and a time period $T_2$ for its execution. The choice of these parameters when invoking our subroutine is described in the next section. Additionally, it is ensured that the matrix statistics are well-concentrated as follows:
\begin{align}
    &\max_{i,j} | A_{i,j}- \widehat{A}_{i,j} | \leq \tfrac{1}{\sqrt{T_1}},\, \max_{i=1,\dots,n-1}| x'_i-x^*_i | \leq \tfrac{1}{D_1\sqrt{T_1}},\notag\\ 
    &\,\text{ and} \,\min_{i=1,\dots,n} \min\{x_i',1-x_i'\} \geq \tfrac{n-1}{D_1\sqrt{T_1}}. \label{eqn:sub-conditions}
    % & \frac{n-1}{D_1\sqrt{T_1}}\leq x_n'\leq 1-\frac{n-1}{D_1\sqrt{T_1}}\quad\text{ so that $x_t\in \Delta_n$}.
\end{align}
%where $(x^*,y^*)$ is the unique Nash equilibrium of the matrix $A=\mathbb{E}[\bar A_t]$ such that $|\supp(x^*)|=n$ (full row support).
%\arn{I had previously written the condition $x_n'\geq \frac{n-1}{D_1\sqrt{T_1}}$ without which $x_t$ is not guaranteed to be in the simplex.}

% \textbf{Remark 1:} The manner in which the input parameters $x', D_1,T_1$ and $T_2$ are chosen while invoking our subroutine is described in \Cref{sec-main:algo-n*m}.

% \textcolor{red}{Maybe use a toy example of $2\times 2$ matrix to motivate your algorithm design.}

The pseudo-code for the subroutine is given in Algorithm~\ref{subroutine-nxn}. The subroutine runs for $T_2$ rounds.  At a high level, it operates on a sequence of reward vectors $\{ g_t \}_{t=1}^{T_2} \subset [-2,2]^{n-1}$, where $[g_t]_i = \widehat{A}_{i,j_t} - \widehat{A}_{n,j_t}$. The use of $\widehat{A}_{i,j_t} - \widehat{A}_{n,j_t}$ instead of just $\widehat{A}_{i,j_t}$ ensures that $x_t$ remains in the simplex, as only its first $n-1$ components are explicitly updated.  At each time step $t$, the subroutine plays $x_t = x' + \vec{\delta}_t$, where $\vec{\delta}_t(i) \in \left[-\frac{1}{D_1\sqrt{T_1}}, \frac{1}{D_1\sqrt{T_1}}\right]$ for all $i \in [n-1]$ and $\vec{\delta}_t(n) = -\sum_{i\in[n-1]} \vec{\delta}_t(i)$. The update for $\vec{\delta}_t$ is based on $g_t$ with a step size of $\eta = \frac{1}{D_1T_1}$, which is near-optimal for responding to the adversarial column player.  In contrast, myopic online learning algorithms like Hedge respond at a suboptimal rate, and as shown in \Cref{sec:myopic-online}, they necessarily incur a $\sqrt{T}$ regret.

Geometrically, the updates can be visualized as movements within an $(n-1)$-dimensional hypercube centered around the vector obtained by excluding the last component of $x^*$. The algorithm moves in the direction of the reward vector $g_t$ with a step size $\eta$. If an update causes the iterate to leave the hypercube, it is projected back onto it. Note that $\widehat{A}$ remains fixed throughout the execution of the subroutine. These properties play a crucial role in our nontrivial analysis, which we defer to Appendix~\ref{appendix:sub:nxn} due to space constraints. The subroutine satisfies the following guarantee.

%\arn{I replaced the constants with $c_1,c_2$ to avoid inconsistency in case we have to make some changes in the appendix.}
%\textcolor{red}{Discard the proof and instead provide better intuition}
\begin{theorem}\label{sub-nxn-thm}
Assume the conditions of \eqref{eqn:sub-conditions} hold for Algorithm \ref{subroutine-nxn}. 
Then for any sequence of actions $j_1,\ldots,j_{T_2}$ we have
% \begin{equation*}
    $T_2\cdot V_A^*-\sum\limits_{t=1}^{T_2}\langle x_t, A e_{j_t}\rangle\leq  \frac{c_1n}{D_1}+\frac{c_2nT_2}{D_1T_1}$ where $c_1,c_2$ are some absolute constants.
\end{theorem}
\begin{proof}[Proof sketch.]
%\textcolor{red}{See if this proof sketch adds an value to the paper or not. Maybe add some toy example instead of this or cut short the following and remove the term proof sketch.}

Through a series of algebraic manipulations we show that
\begin{align*}
    &\quad T_2\cdot V_A^*-\textstyle\sum_{t=1}^{T_2}\langle x_t, A e_{j_t}\rangle\\
    &\leq \textstyle\frac{c_1n}{D_1}+ \textstyle\sum_{t=1}^{T_2} \langle x^* - x' - \vec\delta_t,\widehat Ae_{j_t}\rangle \\
    % &= \frac{c_1n}{D_1}+ \sum_{t=1}^{T_2} \sum_{i\in[n-1]}\Big((x_{*,i} - x_i')(\widehat A_{i,j_t}-\widehat A_{n,j_t})-\delta_t(i)(\widehat A_{i,j_t}-\widehat A_{n,j_t})\Big) \\
    &= \textstyle\frac{c_1n}{D_1}+ \sum_{i\in[n-1]}\sum_{t=1}^{T_2} \big( -a_{i,j_t}\delta_t(i)+a_{i,j_t}(x_i^* - x_i') \big)
\end{align*}
where $a_{i,j_t} := \widehat A_{i,j_t}-\widehat A_{n,j_t}$. Now we aim to analyse the inner sum above for a fixed $i$.
If we partition the time period as $[T_2] = \cup_{k=1}^\ell [t_k, t_{k+1})$ such that within an interval we have $\delta_{t+1/2}(i) \in [-\frac{1}{D_1\sqrt{T_1}},\frac{1}{D_1\sqrt{T_1}}]$ so that no projection is necessary, then on the interval $[t_k, t_{k+1})=\{t_k,t_k+1,\ldots,t_{k+1}-1\}$ we can derive a closed-form expression for $\sum_{t = t_k}^{t_{k+1}-1} a_{i,j_t} \delta_t(i)$. 
It turns out this sum can be quite large in magnitude. But by exploiting the closed-form expression we can show that such sums can be both very positive and very negative, and these large magnitude pieces cancel out. 
To formalize this, we analyze each interval $[t_k, t_{k+1})$ in one of four cases based on the trajectory of $\delta_t(i)$.

In the first case, we have $\delta_{t_k}(i) = \delta_{t_{k+1}}(i)$. In this case the large positive and negative terms of the sum $\sum_{t=t_k}^{t_{k+1}-1}-a_{i,j_t}\delta_t(i)+a_{i,j_t}(x_i^* - x_i')$ start cancelling out each other and at the end the sum will scale like $\frac{t_{k+1}-t_k}{D_1T_1}$.
In the second case, the trajectory of $\delta_t(i)$ starts from the left-most point of the interval $[-\frac{1}{D_1\sqrt{T_1}},\frac{1}{D_1\sqrt{T_1}}]$ and travels to the right-most point of this interval. The third case is the opposite direction. 
Again, the sum in each case is not small in magnitude, but by pairing one left-to-right trajectory with one right-to-left trajectory, we can simulate a round trip. Similar to the first case, the large negative terms of the sum from one of the trajectory starts cancelling out the large positive terms of the sum from the other trajectory and hence the combined sum of both trajectory will scale like $\frac{t_{k_1+1}-t_{k_1}+t_{k_2+1}-t_{k_2}}{D_1T_1}$.
The final case that ends in the middle at time $T_2$ contributes just constant regret.
We conclude that 
\begin{align*}
    &\quad\textstyle\sum_{t=1}^{T_2} \sum_{i}\big(-a_{i,j_t}\delta_t(i)+a_{i,j_t}(x_{i}^* - x_i') \big)\\
    &\leq \textstyle\sum_{i} \sum_{k=1}^\ell \sum_{t=t_k}^{t_{k+1}-1} \big(-a_{i,j_t}\delta_t(i)+a_{i,j_t}(x_{i}^* - x_i') \big) \\
    &\lesssim  \textstyle\sum_{i\in[n-1]} \sum_{k=1}^\ell \frac{t_{k+1}-t_k}{D_1T_1} \leq \frac{(n-1)T_2}{D_1T_1}.
\end{align*}
\end{proof}
\textbf{Remark 1:}  Our subroutine incurs a logarithmic regret whenever $T_2\lesssim  T_1 \cdot \log(nT)$.

%The full proof of Theorem~\ref{sub-nxn-thm} is found in Appendix~\ref{appendix:sub:nxn}.
% The same intuition behind this proof and result also hold for the bandit case.
% Mechanically, the only difference is that instead of specifying how many iterations to run the sub-routine for, namely $T_2$, the sub-routine is still given $T_2$ as input, but told to terminate only after every entry of the matrix has been observed at least $T_2$ times.
% Consequently, in the analysis found in Appendix~\ref{appendix:sub:2x2} of the bandit case we have to bound the regret by considering two cases: either the time to achieve this minimum sampling number is small, or we receive negative regret. Either way the total regret is bounded. 

\begin{algorithm}[t]
\caption{Subroutine for $n\times m$ games}
\begin{algorithmic}[1]
\STATE \textbf{Input Parameters:} $x'\in \simplex_n$, $D_1\in \mathbb{R}_{+}$, $T_1\in \mathbb{R}_{+}$, $T_2\in \mathbb{N}$, $\widehat A\in [-1,1]^{n\times m}$
\STATE Set $\eta\gets \frac{1}{D_1T_1}$, $\delta_1\gets (-\frac{1}{D_1\sqrt{T_1}},\ldots,-\frac{1}{D_1\sqrt{T_1}})^\top \in \R^{n-1}$
%\STATE Let $N_{i,j}(t)$ denote \# of times $(i,j)$ entry observed up to time $t$
\FOR{round $t=1,2,\dots,T_2$} 
\STATE Play $x_t\gets x' + \vec\delta_t \quad \text{ where } \quad \vec\delta_t=\left(\delta_t(1),\ldots,\delta_t(n-1),-\sum_{i=1}^{n-1} \delta_t(i)\right)^\top $.
%\STATE \textbf{If} \texttt{bandit-feedback} observe $[\mathbf{A}_t]_{i_t,j_t}$, \textbf{else if} \texttt{matrix-feedback} observe $[\mathbf{A}_t]_{i,j} \, \forall i,j$.
\STATE Observe $y$-player's action $j_t$ and set $\delta_{t+1/2}(i) \gets \delta_t(i) + \eta (\widehat A_{i,j_t}-\widehat A_{n,j_t})$ $\forall i\in[n-1]$
\STATE Set $\delta_{t+1}(i) = \arg\min_{z \in [-\frac{1}{D_1\sqrt{T_1}},\frac{1}{D_1\sqrt{T_1}}]} |\delta_{t+1/2}(i) -z |$ $\forall i\in[n-1]$
%\STATE \textbf{if} $\min_{i,j} N_{i,j}(t) \geq T_2$ \textbf{break}
\ENDFOR
\end{algorithmic}
\label{subroutine-nxn}
\end{algorithm}

\subsection{Second Phase: Invoking the sub-routine using the standard doubling trick} \label{sec:invoke-subroutine}
%\textcolor{red}{Invoke properly. Maybe not technically but use the inequalities from the earlier section on concentration.}
Having described our subroutine in the previous section, we now provide an intuitive explanation of how to invoke it after the burn-in phase ends. Recall that $\bar{A}_t$ and $\bar{B}_t$ are the running empirical averages of $A$ and $B$, respectively. Suppose we decide to invoke our subroutine at time step $t_0+1$. We set  $T_1 \approx \frac{t_0}{\log(nT)}, \quad D_1 \approx \frac{|\mathbf{1}^\top\text{adj}(\bar{B}_{t_0})\mathbf{1}|}{k \alpha_k}, x' =\frac{\text{adj}(\bar{B}_{t_0}^\top)\mathbf{1}}{\mathbf{1}^\top\text{adj}(\bar{B}_{t_0}^\top)\mathbf{1}}$, and $ T_2 =\min\{t_0,T-t_0\}$. 
We set $\widehat{A}$ as the $k \times m$ submatrix of $\bar{A}_t$ that contains $\bar{B}_t$. The concentration of our matrix estimates ensures that these parameters satisfy the conditions of \eqref{eqn:sub-conditions}, thereby guaranteeing logarithmic regret after the subroutine runs for $T_2$ rounds.  

Let $t_\star$ be the timestep at which the burn-in phase ends. We invoke our subroutine for the first time at timestep $t_\star+1$, then at timestep $2t_\star+1$, and for the $i$-th invocation at timestep $2^{i-1} \cdot t_\star+1$. The second phase terminates at timestep $T$. Since the subroutine is invoked at most $\log T$ times, we obtain the following regret bound.

%After repeatedly playing the sub-routine and setting $T_1 \gets 2T_1$ and $T_2 \gets 2 T_2$ after each invocation, we obtain the following final regret bound.
% \begin{theorem}\label{thm:nxm:full-feedback}
%     Fix a matrix game on $A\in [-1,1]^{n\times m}$ with a unique Nash equilibrium $(x^*,y^*)$ of support size $k \leq \min\{n,m\}$. 
%     Let $\Delta_{\min}$ be defined with respect to the sub-matrix of $A$ associated the support of $(x^*,y^*)$.
%     Then the Nash regret is upper bounded as follows
%     \begin{align*}
%         R^N(A,T)&=T\cdot V_A^*-\sum_{t=1}^T\mathbb{E}[\langle x_t,Ay_t\rangle]\\
%         &\leq \frac{\alpha_{1,k}\cdot\log(nmT)}{\Delta_{\min}\cdot\min\{1,\Delta_g\}}+\frac{\alpha_{2,k}\cdot(\log(nmT))^2}{|D|}.
%     \end{align*}
%     where $\alpha_{1,k},\alpha_{2,k}$ depend only on $k$.
%     % Moreover, if $n=m$ and the support size $k=n$ then the Nash regret is upper bounded as follows
%     % \begin{align*}
%     %     R^N(A,T)&=T\cdot V_A^*-\sum_{t=1}^T\mathbb{E}[\langle x_t,Ay_t\rangle]\\
%     %     &\leq \frac{\alpha_{1,k}\cdot\log(nmT)}{\Delta_{\min}}+\frac{\alpha_{2,k}\cdot(\log(nmT))^2}{|D|},
%     % \end{align*}
% \end{theorem}

\begin{theorem}[Informal]\label{main-theorem-informal}
    Fix a matrix game on $A\in [-1,1]^{n\times m}$ with a unique Nash equilibrium $(x^*,y^*)$ of support size $k \leq \min\{n,m\}$. 
    Then the Nash regret is upper bounded as follows:
    \begin{align*}
        R^N(A,T)&=T\cdot V_A^*-\sum_{t=1}^T\mathbb{E}[\langle x_t,Ay_t\rangle]\leq \frac{(\log T)^2}{\Delta_A}\\
    \end{align*}
    where $\Delta_A$ is some parameter that only depends on $A$
    % Moreover, if $n=m$ and the support size $k=n$ then the Nash regret is upper bounded as follows
    % \begin{align*}
    %     R^N(A,T)&=T\cdot V_A^*-\sum_{t=1}^T\mathbb{E}[\langle x_t,Ay_t\rangle]\\
    %     &\leq \frac{\alpha_{1,k}\cdot\log(nmT)}{\Delta_{\min}}+\frac{\alpha_{2,k}\cdot(\log(nmT))^2}{|D|},
    % \end{align*}
\end{theorem}
Due to space constraints, we defer the formal description of the algorithm, as well as the formal statement and proof of the above theorem, to Appendix \ref{appendix:matrix-feedback}.

\textbf{Remark 2:} $\Delta_A$ includes matrix-dependent parameters (e.g., $\Delta_g$) that naturally arise in sample complexity analyses and have been shown to be necessary in that context (see \cite{maiti2023instance}).

\textbf{Remark 3:} While the regret bound stated in the theorem achieves instance-dependent $\text{polylog}(T)$ scaling, it exhibits exponential dependence on the support size $k$ in the worst case. However, for well-conditioned matrices—which frequently arise in practice—there is potential to improve this dependence. Our experiments in Section~\ref{sec:experiments-full} support this observation: even for games with support sizes up to 100, we do not observe exponential growth in regret. We refer the reader to \Cref{appendix:condition-number-guarantee} for a brief discussion on how this dependence may be mitigated.

\input{experiments}

\section{Algorithm for Logarithmic Nash regret under Bandit Feedback in $2\times 2$ Matrix Games}\label{appendix:bandit-feedback}
Recall that in \Cref{sec:failure}, we showed that various algorithmic designs fail to perform better than $\sqrt{T}$ Nash regret even on the class of $2\times 2$ matrices $A = \begin{bmatrix}
a & b\\
c & d\\
\end{bmatrix}$ with $\Delta_{\min}=\min\{|a-b|,|a-c|,|d-c|,|d-b|\}$ being a positive constant. This suggests that even the $2\times 2$ matrix game case is non-trivial. For such $2\times2$ matrices, our algorithm from the previous section incurs a Nash regret of at most $O\left(\frac{(\log T)^2}{\Delta_{\min}}\right)$ as the parameter $|D|=|\mathbf{1}^\top\text{adj}(A)\mathbf{1}|=|a-b-c+d|\geq 2\Delta_{\min}$. Interestingly, if we restrict the feedback to bandit feedback, where the row player receives noisy feedback of the entry $A_{i_t,j_t}$ corresponding to the actions $i_t,j_t$ chosen by the row and column player in round $t$, then we still get a logarithmic Nash regret as shown below.

\begin{theorem}\label{thm:2x2:bandit}
    Fix a game matrix $A\in [-1,1]^{2\times 2}$ with a unique Nash equilibrium.
    There exists an algorithm (see Appendix \ref{appendix:2x2}) that under noisy bandit feedback incurs a Nash regret $R^N(A,T)$ of at most $\frac{c_1\cdot\log(T)}{\Delta_{\min}^3}+\frac{c_2\cdot (\log(T))^2}{\Delta_{\min}^2},$ where $c_1$ and $c_2$ are absolute constants and $\Delta_{\min}=\min_i \min\{|A_{i,1}-A_{i,2}|,|A_{1,i}-A_{2,i}|\}$.
\end{theorem}

We defer the full algorithm for bandit feedback and the proof of Theorem~\ref{thm:2x2:bandit} to Appendix \ref{appendix:2x2}. On a high level, in each round $t$, we play a strategy $x_t$ such that either every element would be explored sufficiently or we incur a negative regret. After the statistics concentrate, we call a subroutine (\Cref{subroutine-nxn-bandit}) similar to \Cref{subroutine-nxn} and end up incurring logarithmic Nash regret.

\textbf{Remark 4:} Although Theorem~\ref{thm:2x2:bandit} applies only to $2 \times 2$ games, its primary purpose is to demonstrate the existence of an algorithm that achieves $\operatorname{polylog}(T)$ Nash regret even under the more restrictive bandit feedback setting. Given the fundamental limitations of existing approaches, this result serves as an important proof of concept, highlighting that instance-dependent regret guarantees are indeed possible under bandit feedback.

\subsection{Experiments for Bandit feedback}\label{sec:bandit-experiments}

In this section, we compare the empirical performance of our algorithm under bandit feedback against the \textsc{UCB} algorithm from \citet{o2021matrix}  and \textsc{Exp3} from \citet{auer2002nonstochastic}. We consider the input matrix $[2/3,0\;;\;0,1/3]$ and three different adversaries for the experiments. For each adversary, we run the three algorithms for eight different time horizons ($T=10^1,10^2,\ldots, 10^7,10^8)$ and plot the logarithm of total Nash Regret incurred against the logarithm of each time horizon (also referred to as the log-log plot). Code to replicate these experiments is available at \url{https://github.com/zero-sum-matrix-regret/code}.

First in figure \ref{subfig:a-bandit}, we consider an adversary that for a given horizon $T$ plays the equilibrium strategy for the first $T/2$ steps and plays the best-response strategy for the next $T/2$ steps. Next in figure \ref{subfig:b-bandit}, we consider an adversary that plays only the best-response strategy for all the $T$ steps. Finally in figure \ref{subfig:c-bandit}, we consider an adversary that for a given horizon $T$ plays either the equilibrium strategy or the best-response strategy depending on the strategy played by the algorithm for the first $T/2$ steps and plays the best-response strategy for the next $T/2$ steps. 

In all three log-log plots we observe that our algorithm performs better than UCB and EXP3 as the time horizon grows. Moreover, in all three log-log plots we observe that our algorithm has a decreasing slope suggesting sub-polynomial regret consistent with polylog($T$), whereas the slopes of UCB and EXP3 approach $1/2$, suggesting a $T^{1/2}$ regret. 

\begin{figure*}
     \centering
     \begin{subfigure}[h!]{0.3\textwidth}
         \centering
         \includegraphics[width=\textwidth]{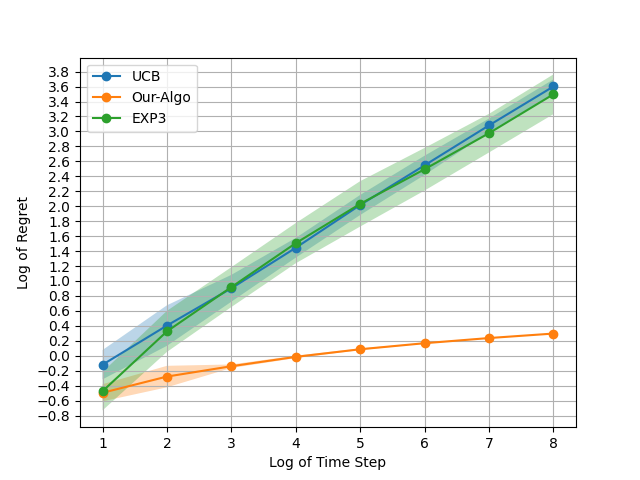}
         \caption{Adversary 1}
         \label{subfig:a-bandit}
     \end{subfigure}
     \hfill
     \begin{subfigure}[h!]{0.3\textwidth}
         \centering
         \includegraphics[width=\textwidth]{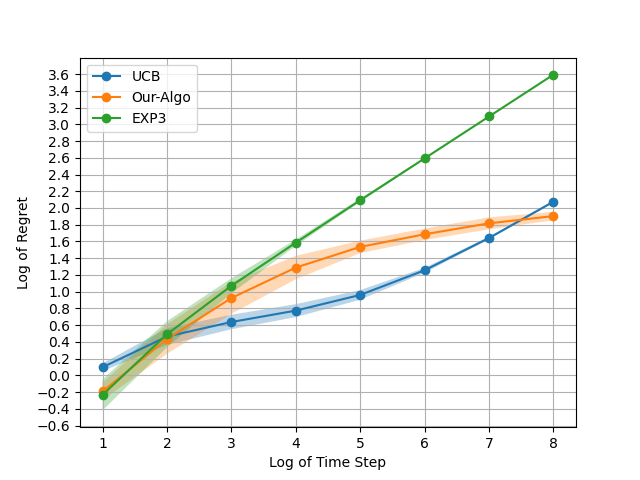}
         \caption{Adversary 2}
         \label{subfig:b-bandit}
     \end{subfigure}
     \hfill
     \begin{subfigure}[h!]{0.3\textwidth}
         \centering
         \includegraphics[width=\textwidth]{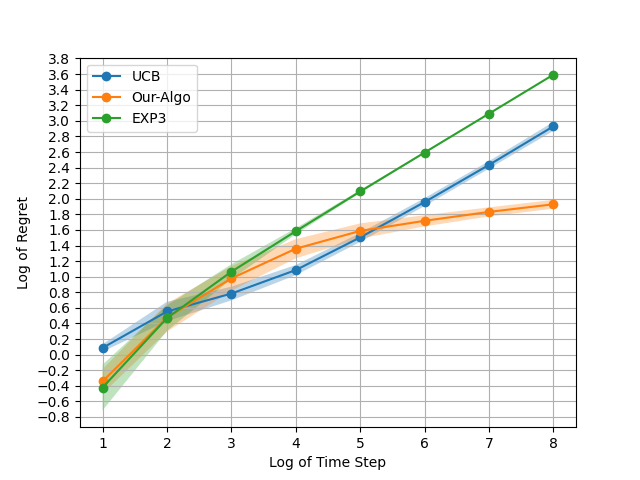}
         \caption{Adversary 3}
         \label{subfig:c-bandit}
     \end{subfigure}
        \caption{Comparing the performance of \textsc{UCB}, \textsc{Exp3} and our algorithm with $N=128$ trials}
        \label{figs}
\end{figure*}

%% file: experiments.tex
\subsection{Experiments for full information feedback}\label{sec:experiments-full}

In this section, we compare the empirical performance of our algorithm under full information feedback against the Nash algorithm and the Hedge algorithm from \citet{freund1997decision}. In the Nash algorithm, the learner plays Nash equilibrium of the empirical matrix in each round. For the experiments, we consider four different input matrices and a single adversary that plays the best-response strategy for all the $T$ steps. For each matrix, we run the three algorithms for seven different time horizons ($T=10^1,10^2,\ldots, 10^7$) and plot the logarithm of total Nash Regret incurred against the logarithm of each time horizon (also referred to as the log-log plot). Code to replicate these experiments is available at \url{https://github.com/zero-sum-matrix-regret/code}.

First in figure \ref{subfig:a}, we consider a $10\times 10$ matrix with $|\supp(x^*)|=|\supp(y^*)|=10$. Next in figure \ref{subfig:a2}, we consider a $20\times 20$ matrix with $|\supp(x^*)|=|\supp(y^*)|=20$. Next in figure \ref{subfig:b}, we consider a $50\times 50$ matrix with $|\supp(x^*)|=|\supp(y^*)|=50$. Finally in figure \ref{subfig:c}, we consider a $100\times 100$ matrix with $|\supp(x^*)|=|\supp(y^*)|=100$.

In all the four log-log plots we observe that our algorithm performs better than the Nash algorithm and the Hedge algorithm as the time horizon grows. Moreover, in all the four log-log plots we observe that our algorithm has a decreasing slope suggesting sub-polynomial regret consistent with our theoretical polylog($T$) regret, whereas the slopes of the Nash algorithm and the Hedge algorithm approach $1/2$, suggesting a $T^{1/2}$ regret. 
%\begin{figure*}
%Do this if you want double column figures.
%\end{figure*}
\begin{figure}
     \centering
     \begin{subfigure}[h!]{0.35\textwidth}
         \centering
         \includegraphics[width=\textwidth]{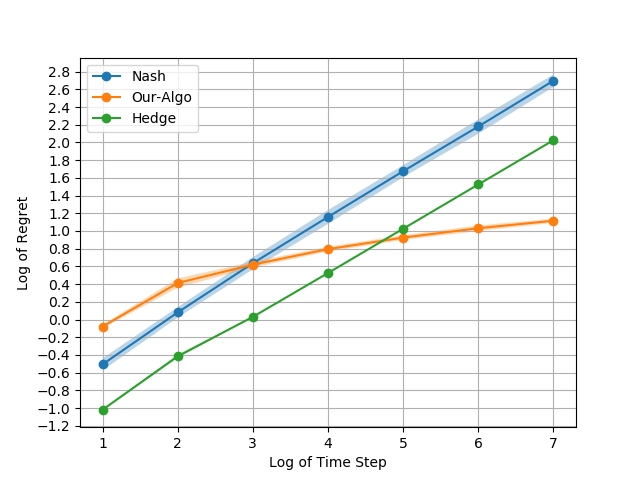}
         \caption{$10\times 10$ matrix}
         \label{subfig:a}
     \end{subfigure}
     \hfill
     \begin{subfigure}[h!]{0.35\textwidth}
         \centering
         \includegraphics[width=\textwidth]{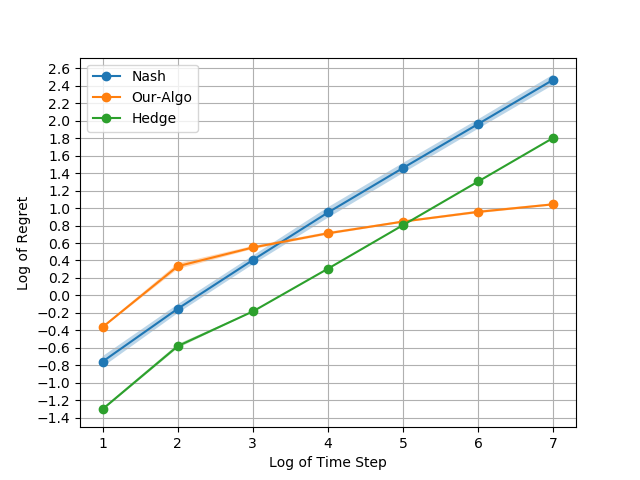}
         \caption{$20\times 20$ matrix}
         \label{subfig:a2}
     \end{subfigure}
     \hfill
     \begin{subfigure}[h!]{0.35\textwidth}
         \centering
         \includegraphics[width=\textwidth]{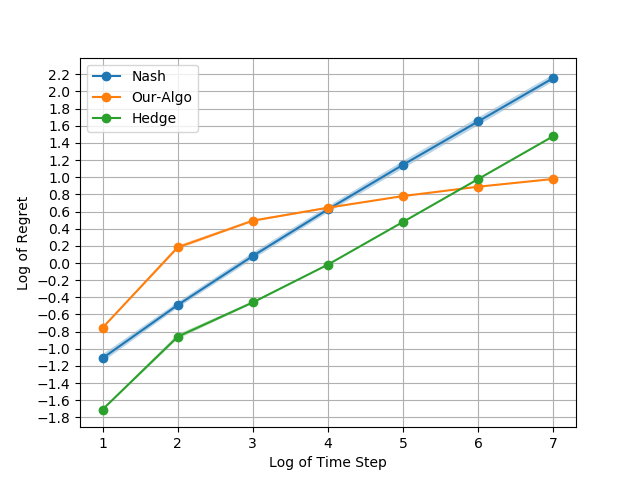}
         \caption{$50\times 50$ matrix}
         \label{subfig:b}
     \end{subfigure}
     \hfill
     \begin{subfigure}[h!]{0.35\textwidth}
         \centering
         \includegraphics[width=\textwidth]{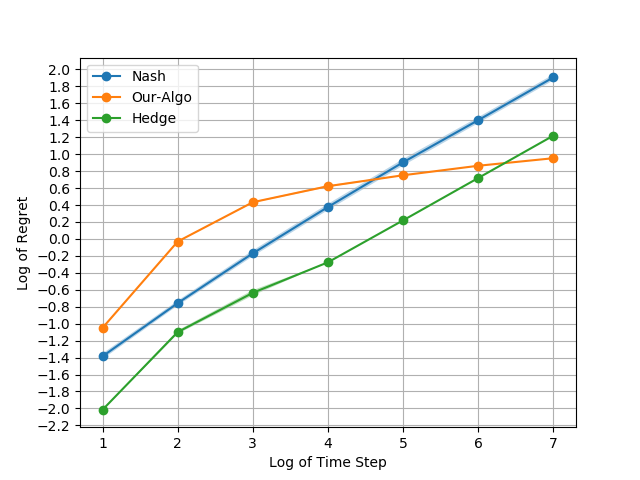}
         \caption{$100\times 100$ matrix}
         \label{subfig:c}
     \end{subfigure}
        \caption{Comparing the performance of the Nash algorithm, the Hedge algorithm and our algorithm with $N=100$ trials. The matrices considered are $n\times n$ diagonal matrices with $i$-th diagonal entry equal to $0.4+\frac{0.2(i-1)}{n-1}$.}
        \label{figs}
\end{figure}

% , which indicates that our algorithm will perform better even for large time horizons. The slopes also indicate that UCB and EXP3 incur $\Omega(\sqrt{T})$ regret whereas our algorithm incurs polylog($T$) regret.
%Next in Figure \ref{subfig:b}, we run the three algorithms for eight different time horizons ($T=10^1, 10^2,\ldots,10^8)$ and plot the logarithm of total Nash Regret incurred against $\log(T)$. We observe that our algorithm has a lower slope than UCB and EXP3, which indicates that our algorithm might perform better even for large time horizons.

%Finally in Figure \ref{subfig:b}, we run the three algorithms for $T=10^6$ time steps and observe for each intermediate time step $t$ the KL-divergence between the strategy $x_t$ played by the algorithms and the equilibrium strategy $x^*$. The plot captures the dynamics of all the algorithms that was proved in the previous sections. \textsc{Exp3} incurs large regret because $x_t$ is not sufficiently close to $x^*$ during the time steps $\frac{10^6}{2}\leq t\leq 10^6$. Relative to the other algorithms, \textsc{UCB} does adapts much slower to the change in adversary's behaviour around the time step $\frac{10^6}{2}$ and hence, incurs a huge regret. Our algorithm has a low regret as it is sufficiently close to $x^*$ and reacts quickly to the change in adversary's behaviour. We refer the reader to Appendix \ref{appendix:experiments} for additional plots.

%% file: discussion.tex
%\section{Discussion \& Extensions}%discussion & extensions.
\section{Conclusion \& Future Works}
This paper presents a study of the Nash regret minimization problem in matrix games under noisy feedback. Given that the existing literature only establishes $\sqrt{T}$ regret bounds, our main objective was to understand the fundamental limitations and possibilities of achieving instance-dependent polylog$(T)$ regret. We answer this fundamental question by showing that existing approaches fail to achieve polylogarithmic Nash regret even in the full-information setting, while a new algorithm succeeds in this setting. However, several questions remain open. A natural direction for future work is to establish instance-dependent lower bounds, analogous to the sample complexity lower bounds for zero-sum games in the noisy feedback setting (see \cite{maiti2023instance}). The main challenge lies in constructing hard alternative instances of the input matrix and characterizing the adversarial behavior of the column player.

Another important direction is to improve our upper bounds in the full-information setting, particularly with respect to the support size $k$. Our algorithms estimate the determinant of certain matrices up to a constant factor, which results in exponential dependence on $k$. Improving this dependence requires alternative techniques that do not heavily rely on determinant estimation. We outline one possible approach to address this in \Cref{appendix:condition-number-guarantee}.

The existence of instance-dependent poly-logarithmic regret for arbitrary $n\times m$ matrix games in the bandit feedback setting also remains an unanswered question, and the major challenge is to sufficiently explore all the elements while incurring low regret. We believe that the insights in our paper provide a template for any future work which makes an attempt to address the $n\times m$ case. 

No-regret algorithms like Hedge have an external regret of $\sqrt T$. As the Nash regret is upper bounded by external regret, no-regret algorithms like Hedge also have a Nash Regret of $\sqrt T$. However, we showed that myopic online learning algorithms like Hedge cannot provably achieve instance-dependent polylog($T$) Nash regret. Whether it is possible to have an algorithm with polylog($T$) Nash regret and $\sqrt{T}$ external regret is an interesting open question.

% Several key questions remain open in two-player zero-sum Markov games. If a player cannot observe the opponent's actions, can one design an algorithm that incurs a Nash regret of \( T^{1/2} \) against any adversary? If a player can observe the opponent's actions, what is the tight minimax Nash regret bound? Furthermore, if a player can observe the opponent's actions, can one design an algorithm that incurs an instance-dependent \(\text{polylog}(T)\) Nash regret?

%\textcolor{red}{Check whether all the references are added in the main body.}

%In the full-feedback setting, we provide instance-dependent poly-logarithmic regret for any $n \times m$ matrix game. In the bandit-feedback setting, we provide instance-dependent poly-logarithmic regret for any $2 \times 2$ matrix game. To the best of our knowledge, these are the first instance-dependent guarantees in these settings. 

%In addition to zero-sum matrix games, various important problems such as general sum games and zero-sum Markov games can be studied in the noisy setting , and instance-dependent bounds can be developed for appropriate regret minimization problems from the perspective of a single player. Moreover, one might extend our problem setting to cooperative games where players form a coalition to achieve a desirable outcome.

%% file: appendix.tex
%\textcolor{red}{Organize appendix, hide n!, make it more readable and do something about bandit feedback. Essentially clean up the mess. Also add all the formal lemmas and theorems to the appendix as they have been removed from the main body. Talk about condition number. Maybe talk about stopping conditions}

\section{Technical results}
\subsection{Equivalent version of the regret}
Consider a matrix $A\in [0,1]^{n\times m}$. In this section we provide an equivalent version of the regret $T\cdot V_A^*-\mathbb E[\sum\limits_{t=1}^T\langle x_t, A y_t\rangle] = T\cdot V_A^*-\mathbb E[\sum\limits_{t=1}^T\mathbb E[\langle x_t, A y_t \rangle | \mc{F}_{t-1}]]$.
Let $\mathbb{F}:=(\mathcal{F}_t)_{t\in [T]}$ be the filtration of the probability space. Let $j_t$ denote the column played in round $t$. Let $e_i$ denote a vector whose $i$-th coordinate is $1$ and rest of the coordinates are zero .  In this paper we will be dealing with algorithms that decide $x_t$ based on $\mathcal{F}_{t-1}$ only.  For such class of algorithms, we have the following:
\begin{align*}
    \mathbb{E}[\langle x_t, Ae_{j_t} \rangle | \mc{F}_{t-1} ]&=\sum_{i=1}^n\sum_{j=1}^m\mathbb{E}[(x_t)_i\cdot\mathbbm{1}\left\{j_t=j\right\}|\mathcal{F}_{t-1}]\cdot A_{i,j}\\
    &=\sum_{i=1}^n\sum_{j=1}^m(x_t)_i\cdot\mathbb{E}[\mathbbm{1}\left\{j_t=j\right\}|\mathcal{F}_{t-1}]\cdot A_{i,j}\tag{as $x_t$ is fully determined by $\mathcal{F}_{t-1}$}\\
    &=\sum_{i=1}^n\sum_{j=1}^m(x_t)_i\cdot(y_t)_j\cdot A_{i,j}\\
    &=\langle x_t, Ay_t\rangle
\end{align*}
Hence we have 
\begin{align*}
    \mathbb{E}[T\cdot V_A^*-\sum\limits_{t=1}^T\langle x_t, A e_{j_t}\rangle] &= T\cdot V_A^*- \sum\limits_{t=1}^T\mathbb{E}[ \langle x_t, A e_{j_t}\rangle  ] \\
    &= T\cdot V_A^*-\sum\limits_{t=1}^T  \mathbb{E}[ \mathbb{E}[ \langle x_t, A e_{j_t} \rangle | \mc{F}_{t-1} ] ] \\
    &= T\cdot V_A^*-\sum\limits_{t=1}^T  \mathbb{E}[ \langle x_t, A y_t\rangle  ].
\end{align*}
Now if under an event $G$,  we have $T\cdot V_A^*-\sum\limits_{t=1}^T\langle x_t, A e_{j_t}\rangle\leq f(n,m,T)$ for any sequence of columns $j_1,j_2,\ldots, j_T$ played by the column player, then we have the following:
\begin{align}
    \mathbb{E}[T\cdot V_A^*-\sum\limits_{t=1}^T\langle x_t, A y_t\rangle] 
 &= \mathbb{E}[T\cdot V_A^*-\sum\limits_{t=1}^T\langle x_t, A e_{j_t}\rangle] \nonumber\\
    &=\mathbb{E}[T\cdot V_A^*-\sum\limits_{t=1}^T\langle x_t, A e_{j_t}\rangle|G]\cdot \mathbb P[G]+\mathbb{E}[T\cdot V_A^*-\sum\limits_{t=1}^T\langle x_t, A e_{j_t}\rangle|\bar G]\cdot \mathbb P[\bar G]\nonumber\\
    &\leq f(m,n,T)+T\cdot \mb{P}[\bar G]. \label{equivalent:regret}
\end{align}
Thus, it suffices to show that $f(m,n,T)$ is low and the good event $G$ holds with high probability.
\subsection{Technical Lemmas}\label{sec:technical_lemmas}
\begin{lemma}[sub-Gaussian tail bound]
Let $X_1,X_2,\ldots,X_n$ be i.i.d samples  from a $1$-sub-Gaussian distribution with mean $\mu$. Then we have the following:
\begin{equation*}
    \mathbb{P}\left[\left|\frac{1}{n}\cdot\sum_{i=1}^nX_i-\mu\right|\geq \sqrt{\frac{2\log(2/\delta)}{n}}\right]\leq \delta
\end{equation*}
\end{lemma}
\begin{lemma}[Chernoff Bound]
    Consider $n$ i.i.d bernoulli random variables $X_1,\ldots,X_n$ such that $\mathbb{E}[X_i]=\mu$. Let $\bar X:=\frac{1}{n}\sum_{i=1}^nX_i$. Then for any $0<\delta\leq 1$, we have the following: 
    \begin{align*}
        \mathbb{P}[\bar X\leq (1-\delta)\mu]\leq \exp\left(-\frac{n\delta^2\mu}{2}\right)\\
        \mathbb{P}[\bar X\geq (1+\delta)\mu]\leq \exp\left(-\frac{n\delta^2\mu}{3}\right)
    \end{align*} 
\end{lemma}
\begin{proposition}\label{prop1}
    Consider $\varepsilon>0$, $\varepsilon< c_0\leq \frac{1}{2}\leq d_0$ such that  $\frac{\varepsilon}{c_0}\leq\frac{1}{2}$.  Then we have the following:
    \begin{equation*}
        c_0\ln\left(\frac{c_0-\varepsilon}{c_0+\varepsilon}\right)+d_0\ln\left(\frac{d_0+\varepsilon}{d_0-\varepsilon}\right)\leq \frac{8\varepsilon^2}{c_0}
    \end{equation*}
\end{proposition}
\begin{proof}
    We use the inequality $\ln(1+x)\leq x$ for all $x>-1$ to get the following:
    \begin{align*}
        c_0\ln\left(\frac{c_0-\varepsilon}{c_0+\varepsilon}\right)+d_0\ln\left(\frac{d_0+\varepsilon}{d_0-\varepsilon}\right)&=c_0\ln\left(1+\frac{-2\varepsilon/c_0}{1+\varepsilon/c_0}\right)+d_0\ln\left(1+\frac{2\varepsilon/d_0}{1-\varepsilon/d_0}\right)\\
        &\leq c_0\cdot \frac{-2\varepsilon/c_0}{1+\varepsilon/c_0}+d_0\cdot \frac{2\varepsilon/d_0}{1-\varepsilon/d_0}\\
        &=2\varepsilon \cdot \frac{-1+\varepsilon/d_0+1+\varepsilon/c_0}{(1+\varepsilon/c_0)(1-\varepsilon/d_0)}\\
        &=2\varepsilon^2 \cdot \frac{1/c_0+1/d_0}{(1+\varepsilon/c_0)(1-\varepsilon/d_0)}\\
        &\leq\frac{4\varepsilon^2}{c_0}\cdot\frac{1}{(1+\varepsilon/c_0)(1-\varepsilon/d_0)}\\
        &\leq\frac{8\varepsilon^2}{c_0}
    \end{align*}
    We get the first inequality as $ \frac{-2\varepsilon/c_0}{1+\varepsilon/c_0}>-1$ and $\frac{2\varepsilon/d_0}{1-\varepsilon/d_0}>-1$. We get the second inequality as $1/c_0+1/d_0\leq 2/c_0$. We get the third inequality as $(1+\varepsilon/c_0)\geq 1$ and $(1-\varepsilon/d_0)\geq 1/2$.
\end{proof}

\begin{proposition}\label{prop2}
    Consider $\varepsilon>0$, $\varepsilon< c_0\leq \frac{1}{2}\leq d_0$ such that  $\frac{\varepsilon}{c_0}\leq\frac{1}{2}$.  Then we have the following:
    \begin{equation*}
        c_0\ln\left(\frac{c_0+\varepsilon}{c_0-\varepsilon}\right)+d_0\ln\left(\frac{d_0-\varepsilon}{d_0+\varepsilon}\right)\leq \frac{8\varepsilon^2}{c_0}
    \end{equation*}
\end{proposition}
\begin{proof}
    We use the inequality $\ln(1+x)\leq x$ for all $x>-1$ to get the following:
    \begin{align*}
        c_0\ln\left(\frac{c_0+\varepsilon}{c_0-\varepsilon}\right)+d_0\ln\left(\frac{d_0-\varepsilon}{d_0+\varepsilon}\right)&=c_0\ln\left(1+\frac{2\varepsilon/c_0}{1-\varepsilon/c_0}\right)+d_0\ln\left(1+\frac{-2\varepsilon/d_0}{1+\varepsilon/d_0}\right)\\
        &\leq c_0\cdot \frac{2\varepsilon/c_0}{1-\varepsilon/c_0}+d_0\cdot \frac{-2\varepsilon/d_0}{1+\varepsilon/d_0}\\
        &=2\varepsilon \cdot \frac{1+\varepsilon/d_0-1+\varepsilon/c_0}{(1-\varepsilon/c_0)(1+\varepsilon/d_0)}\\
        &=2\varepsilon^2 \cdot \frac{1/c_0+1/d_0}{(1-\varepsilon/c_0)(1+\varepsilon/d_0)}\\
        &\leq\frac{4\varepsilon^2}{c_0}\cdot\frac{1}{(1-\varepsilon/c_0)(1+\varepsilon/d_0)}\\
        &\leq\frac{8\varepsilon^2}{c_0}
    \end{align*}
    We get the first inequality as $ \frac{2\varepsilon/c_0}{1-\varepsilon/c_0}>-1$ and $\frac{-2\varepsilon/d_0}{1+\varepsilon/d_0}>-1$. We get the second inequality as $1/c_0+1/d_0\leq 2/c_0$. We get the third inequality as $(1-\varepsilon/c_0)\geq 1/2$ and $(1+\varepsilon/d_0)\geq 1$.
\end{proof}

\begin{proposition}\label{prop3}
    Consider $\varepsilon>0$, $\varepsilon< c_0\leq \frac{1}{2}$ such that  $\frac{\varepsilon}{c_0}\leq\frac{1}{2}$. Then we have the following:
    \begin{equation*}
        \varepsilon\cdot\left( \ln\left(\frac{c_0+\varepsilon}{c_0-\varepsilon}\right)+\ln\left(\frac{1-c_0+\varepsilon}{1-c_0-\varepsilon}\right)\right)\leq \frac{8\varepsilon^2}{c_0}
    \end{equation*}
\end{proposition}
\begin{proof}
    Since $\ln(\frac{x+\varepsilon}{x-\varepsilon})$ is decreasing when $x\in(\varepsilon,\infty)$, we deduce that 
    \begin{align*}
        \varepsilon\cdot\left( \ln\left(\frac{c_0+\varepsilon}{c_0-\varepsilon}\right)+\ln\left(\frac{1-c_0+\varepsilon}{1-c_0-\varepsilon}\right)\right)&\leq 2\varepsilon\cdot\ln\left(\frac{c_0+\varepsilon}{c_0-\varepsilon}\right),\\
        &=2\varepsilon\cdot \ln\left(1+\frac{2\varepsilon/c_0}{1-\varepsilon/c_0}\right),\\
        &\leq \frac{4\varepsilon^2/c_0}{1-\varepsilon/c_0},\\
        &\leq \frac{8\varepsilon^2}{c_0},
    \end{align*}
    where the second to last inequality follows from the fact that $\ln(1+x)\leq x$ for $x>-1$, and the last inequality follows from the fact that $\varepsilon/c_0\leq \frac{1}{2}$.
\end{proof}

% \begin{proposition}\label{prop1}
% For any $x\in [0,1)$, $\ln\left(\frac{1-x}{1+x}\right)+\frac{1}{1-x}-\frac{1}{1+x}\geq 0$
% \end{proposition}
% \begin{proof}
%     Let $f(x)=\ln\left(\frac{1-x}{1+x}\right)+\frac{1}{1-x}-\frac{1}{1+x}$. Observe that $f'(x)=\frac{4x^2}{(1-x^2)^2}\geq 0$ for any $x\in [0,1)$. This implies that $f(x)$ is increasing in the range $x\in[0,1)$. As $f(0)=0$, we have $f(x)\geq 0$ for any $x\in[0,1)$.
% \end{proof}

% \begin{proposition}\label{prop2}
%     Consider $0<\varepsilon<1$. The function $f(x)=x\ln\left(\frac{x-\varepsilon}{x+\varepsilon}\right)$ is increasing in the range $(\varepsilon,1]$
% \end{proposition}
% \begin{proof}
%     First observe that $f'(x)=\ln\left(\frac{x-\varepsilon}{x+\varepsilon}\right)+\frac{x}{x-\varepsilon}-\frac{x}{x+\varepsilon}$. Next observe that $f''(x)=\frac{-4\varepsilon^3}{(x-\varepsilon)^2(x+\varepsilon)^2}<0$ for any $x\in (\varepsilon,1]$. Hence, $f'(x)$ is a decreasing function in the range $(\varepsilon,1]$. Hence, for any $x\in(\varepsilon,1]$, $f'(x)\geq f'(1)$. Due to proposition \ref{prop1}, $f'(1)\geq 0$. Hence $f'(x)>0$ for all $x\in(\varepsilon,1]$. Hence, $f(x)$ is increasing in the range $(\varepsilon,1]$.
% \end{proof}
\begin{lemma}[KL-divergence of Bernoulli]\label{kl:bernoulli}
    Consider $\varepsilon>0$, $\varepsilon< c_0\leq \frac{1}{2}$ such that  $\frac{\varepsilon}{c_0}\leq\frac{1}{2}$. Let $P$ and $Q$ be bernoulli distributions with means $c_0-\varepsilon$ and $c_0+\varepsilon$. Then we have the following:
    \begin{equation*}
        KL(P,Q)\leq \frac{16\varepsilon^2}{c_0} \quad \text{and}\quad KL(Q,P)\leq \frac{16\varepsilon^2}{c_0} 
    \end{equation*}
\end{lemma}
\begin{proof}
 The KL divergence satisfies
    \begin{align*}
        KL(P,Q)&=(c_0-\varepsilon)\cdot \ln\left(\frac{c_0-\varepsilon}{c_0+\varepsilon}\right)+(1-c_0+\varepsilon)\cdot \ln\left(\frac{1-c_0+\varepsilon}{1-c_0-\varepsilon}\right)\\
        &=c_0\ln\left(\frac{c_0-\varepsilon}{c_0+\varepsilon}\right)+(1-c_0)\ln\left(\frac{1-c_0+\varepsilon}{1-c_0-\varepsilon}\right)+\varepsilon\cdot\left( \ln\left(\frac{c_0+\varepsilon}{c_0-\varepsilon}\right)+\ln\left(\frac{1-c_0+\varepsilon}{1-c_0-\varepsilon}\right)\right)\\
        &\leq \frac{8\varepsilon^2}{c_0}+\varepsilon\cdot\left( \ln\left(\frac{c_0+\varepsilon}{c_0-\varepsilon}\right)+\ln\left(\frac{1-c_0+\varepsilon}{1-c_0-\varepsilon}\right)\right)\tag{due to Proposition \ref{prop1}}\\
        &\leq \frac{8\varepsilon^2}{c_0}+\frac{8\varepsilon^2}{c_0} \tag{due to Proposition \ref{prop3}}\\
        &=\frac{16\varepsilon^2}{c_0}
        \end{align*}

    Similarly the KL divergence satisfies
    \begin{align*}
        KL(Q,P)&=(c_0+\varepsilon)\cdot \ln\left(\frac{c_0+\varepsilon}{c_0-\varepsilon}\right)+(1-c_0-\varepsilon)\cdot \ln\left(\frac{1-c_0-\varepsilon}{1-c_0+\varepsilon}\right)\\
        &=c_0\ln\left(\frac{c_0+\varepsilon}{c_0-\varepsilon}\right)+(1-c_0)\ln\left(\frac{1-c_0-\varepsilon}{1-c_0+\varepsilon}\right)+\varepsilon\cdot\left( \ln\left(\frac{c_0+\varepsilon}{c_0-\varepsilon}\right)+\ln\left(\frac{1-c_0+\varepsilon}{1-c_0-\varepsilon}\right)\right)\\
        &\leq \frac{8\varepsilon^2}{c_0}+\varepsilon\cdot\left( \ln\left(\frac{c_0+\varepsilon}{c_0-\varepsilon}\right)+\ln\left(\frac{1-c_0+\varepsilon}{1-c_0-\varepsilon}\right)\right)\tag{due to Proposition \ref{prop2}}\\
        &\leq \frac{8\varepsilon^2}{c_0}+\frac{8\varepsilon^2}{c_0} \tag{due to Proposition \ref{prop3}}\\
        &=\frac{16\varepsilon^2}{c_0}
        \end{align*}
    %     where the last inequality follows from Proposition \ref{prop2} and the fact that $c_0\leq 1-c_0$.
    %     Since $\ln(\frac{x+\varepsilon}{x-\varepsilon})$ is decreasing when $x\in(\varepsilon,\infty)$, we deduce that 
    %     \begin{align*}
    %      KL(P,Q) &\leq 2\varepsilon\cdot\ln\left(\frac{c_0+\varepsilon}{c_0-\varepsilon}\right),\\
    %     &=2\varepsilon\cdot \ln\left(1+\frac{2\varepsilon/c_0}{1-\varepsilon/c_0}\right),\\
    %     &\leq \frac{4\varepsilon^2/c_0}{1-\varepsilon/c_0},\\
    %     &\leq \frac{8\varepsilon^2}{c_0},
    % \end{align*}
    % where the second to last inequality follows from the fact that $\ln(1+x)\leq x$ for $x\geq 0$, and the last inequality follows from the fact that $\varepsilon/c_0\leq \frac{1}{2}$.
\end{proof}
% \begin{corollary}
%     Consider $\varepsilon>0$, $\varepsilon< c_0\leq \frac{1}{2}$ such that  $\frac{\varepsilon}{c_0}\leq\frac{1}{2}$. Let $P$ and $Q$ be bernoulli distributions with means $1-c_0+\varepsilon$ and $1-c_0-\varepsilon$. Then $KL(P,Q)\leq \frac{8\varepsilon^2}{c_0}$.
% \end{corollary}\label{kl:bernoulli:cor}
% \begin{proof}
% Consider two bernoulli distributions $\tilde P$ and $\tilde Q$ with means $c_0-\varepsilon$ and $c_0+\varepsilon$. Then due to the definition of KL-divergence and the lemma \ref{kl:bernoulli}, we have $KL(P,Q)=KL(\tilde P,\tilde Q)\leq \frac{8\varepsilon^2}{c_0}$. 
% \end{proof}

\begin{lemma}\label{lem:swap}
    Consider  an arbitrary sequence of numbers  $S=y_1,y_2,\ldots,y_\ell$ such that for all $i\in[\ell]:=\{1,\ldots,\ell\}$, $y_i\in\{1,2\}$. Let $n_{i,1}$ denote the number of ones in the subsequence $y_1,y_2,\ldots,y_i$ and $n_{i,2}$ denote the number of twos in the subsequence $y_1,y_2,\ldots,y_i$. Then, we have that
    \begin{equation*}
        \sum_{i\in[\ell]:y_i=2}n_{i,1}+\sum_{i\in[\ell]:y_i=1}n_{i,2}=n_1\cdot n_2,
    \end{equation*}
    where $n_1$ denotes the number of ones in the subsequence $S$ and $n_2$ denotes the number of twos in the subsequence $S$.
\end{lemma}
\begin{proof}
Let $[W]_{i,j} = \mathbbm{1}\{ y_j = 1\} \mathbbm{1}\{ y_i = 2\} + \mathbbm{1}\{ y_j = 2 \} \mathbbm{1}\{ y_i = 1 \}$, and observe that 
 $W = W^\top$ and $[W]_{i,i} = 0 \quad \forall i$. Then, we have that 
    \begin{align*}
\sum_{i \in [\ell]: y_i = 2} n_{i,1} + \sum_{i \in [\ell]: y_i = 1} n_{i,2} &= \sum_{i \in [\ell]} n_{i,1}  \mathbbm{1}\{ y_i = 2\} + \sum_{i \in [\ell]} n_{i,2}  \mathbbm{1}\{ y_i = 1 \}, \\
&= \sum_{i \in [\ell]} \sum_{j=1}^i  \mathbbm{1}\{ y_j = 1\}  \mathbbm{1}\{ y_i = 2\} + \sum_{i \in [\ell]} \sum_{j=1}^i  \mathbbm{1}\{ y_j = 2 \}  \mathbbm{1}\{ y_i = 1 \}, \\ 
&= \sum_{i=1}^\ell \sum_{j=1}^i  \mathbbm{1}\{ y_j = 1\}  \mathbbm{1}\{ y_i = 2\} +  \mathbbm{1}\{ y_j = 2 \}  \mathbbm{1}\{ y_i = 1 \}, \\
&= \sum_{i=1}^\ell \sum_{j=1}^i W_{i,j},\\
%&= \sum_{j=1}^\ell \sum_{i=1}^j W_{i,j}, \\
&= \frac{1}{2} \sum_{i=1}^\ell \sum_{j=1}^\ell  W_{i,j},\\
&= n_1 \cdot n_2,
    \end{align*}
    which concludes the proof. 
\end{proof}
\begin{lemma}\label{det:lem}
    Consider two $n\times n$ matrices $B, \bar B$ such that $\max_{i,j}|B_{i,j}-\bar B_{i,j}|\leq \Delta$, $\max_{i,j}|B_{i,j}|\leq 1$ and $\max_{i,j}|\bar B_{i,j}|\leq 1$. Then we have $|det(B)-det(\bar B)|\leq n\cdot n!\cdot \Delta$.
\end{lemma}
\begin{proof}
    For all $i,j$, let $\Delta_{i,j}:=\bar B_{i,j}-B_{i,j}$ Let $S_n$ be the set of all permutations of $\{1,2,\ldots,n\}$. Let $sgn(\sigma)$ denote the sign of the permutation $\sigma\in S_n$. Observe that for any matrix $M\in \mathbb{R}^{n \times n}$, $det(M)=\sum\limits_{\sigma\in S_n}sgn(\sigma)\prod\limits_{i=1}^nM_{i,\sigma_i}$. For any $\sigma\in S_n$, we have the following:
    \begin{align*}
        \prod_{i=1}^n\bar B_{i,\sigma_i}&=(B_{1,\sigma_1}+\Delta_{1,\sigma_1})\prod_{i=2}^n\bar B_{i,\sigma_i}\\
        &=B_{1,\sigma_1}\cdot(B_{2,\sigma_2}+\Delta_{2,\sigma_2})\cdot\prod_{i=3}^n\bar B_{i,\sigma_i}+\Delta_{1,\sigma_1}\cdot\prod_{i=2}^n\bar B_{i,\sigma_i}\\
        &=B_{1,\sigma_1}\cdot B_{2,\sigma_2}\cdot (B_{3,\sigma_3}+\Delta_{3,\sigma_3})\cdot\prod_{i=4}^n\bar B_{i,\sigma_i}+\Delta_{1,\sigma_1}\cdot\prod_{i=2}^n\bar B_{i,\sigma_i}+\Delta_{2,\sigma_2}\cdot B_{1,\sigma_1} \cdot\prod_{i=3}^n\bar B_{i,\sigma_i}\\
        & \vdots\\
        &=  \prod_{i=1}^n B_{i,\sigma_i}+\sum_{i=1}^n \left(\Delta_{i,\sigma_i}\cdot\prod_{\ell=1}^{i-1}B_{i,\sigma_i}\cdot\prod_{\ell=i+1}^n\bar B_{i,\sigma_i}\right)\\
    \end{align*}
    Due to the above analysis, we have the following:
    \begin{align*}
        &|det(B)-det(\bar B)|\\
        =&\left|\sum\limits_{\sigma\in S_n}sgn(\sigma)\prod\limits_{i=1}^nB_{i,\sigma_i}-\sum\limits_{\sigma\in S_n}sgn(\sigma)\prod\limits_{i=1}^n\bar B_{i,\sigma_i}\right|\\
        =&\left|\sum\limits_{\sigma\in S_n}sgn(\sigma)\prod\limits_{i=1}^nB_{i,\sigma_i}-\sum\limits_{\sigma\in S_n}sgn(\sigma)\prod\limits_{i=1}^n\ B_{i,\sigma_i}-\sum\limits_{\sigma\in S_n}sgn(\sigma)\sum_{i=1}^n \left(\Delta_{i,\sigma_i}\cdot\prod_{\ell=1}^{i-1}B_{i,\sigma_i}\cdot\prod_{\ell=i+1}^n\bar B_{i,\sigma_i}\right)\right|\\
        =&\left|\sum\limits_{\sigma\in S_n}sgn(\sigma)\sum_{i=1}^n \left(\Delta_{i,\sigma_i}\cdot\prod_{\ell=1}^{i-1}B_{i,\sigma_i}\cdot\prod_{\ell=i+1}^n\bar B_{i,\sigma_i}\right)\right|\\
        \leq& \sum_{\sigma\in S_n}\sum_{i=1}^n \left|\Delta_{i,\sigma_i}\cdot\prod_{\ell=1}^{i-1}B_{i,\sigma_i}\cdot\prod_{\ell=i+1}^n\bar B_{i,\sigma_i}\right|\\
        \leq& \sum_{\sigma\in S_n}\sum_{i=1}^n \Delta\\
        =&n\cdot n!\cdot \Delta
    \end{align*}   
\end{proof}
\section{Fundamental limitations of various algorithmic strategies}\label{appendix:fundamential-limitation}
\subsection{External Regret versus Nash Regret (Proof of \Cref{claim:main-nash-external})}\label{appendix:lower:ext-nash}
In this section, we show an example where the row player incurs an external regret of $\Omega(\sqrt{T})$ and the column player incurs a Nash regret of $\Omega(T)$. Consider the following matrix $A$:
\[
A = \begin{bmatrix} 
   0 & -1 & 1 \\
   1 & 0 & -1 \\
   -1 & 1 & 0
    \end{bmatrix}
\]
In this setting, the matrix $A$ is known to the row player and the column player. In each round the row player plays a strategy $x_t$ and the column player plays a strategy $y_t$. At the end of each round $t$, the row player only observes a column $j_t$ that is sampled from the distribution $y_t$.

In this section, we prove the following theorem, the formal version of \Cref{claim:main-nash-external}.
\begin{theorem}[Formal version of \Cref{claim:main-nash-external}]\label{external-nash:thm}
Consider the matrix $A$. For any algorithm that achieves an external regret of at most $\frac{T}{10}$, there exists an adversary that plays a fixed strategy $\tilde y$ in each round $t$ such the following inequalities hold:
\begin{align*}
    \max_{i\in\{1,2,3\}}\sum_{t=1}^T \langle e_i,A\tilde y\rangle-\sum_{t=1}^T\mathbb{E}[\langle x_t,A\tilde y\rangle] \geq \frac{\sqrt{T}}{128},\\
    \sum_{t=1}^T\langle x_t,A\tilde y\rangle-T\cdot V_A^*\geq \frac{7T}{30}.
\end{align*}
\end{theorem}

Let $y^{(1)}=(\frac{1}{3}+\frac{1}{16\sqrt{T}},\frac{2}{3}-\frac{1}{16\sqrt{T}},0)$ and $y^{(2)}=(\frac{1}{3}-\frac{1}{16\sqrt{T}},\frac{2}{3}+\frac{1}{16\sqrt{T}},0)$ be two possible strategies of the column player. We need the following technical lemma.
\begin{lemma}\label{external:lem3}
    Consider any $x\in \simplex_3$. There exists $\tilde y\in\{y^{(1)},y^{(2)}\}$ such that \[\max_{i\in\{1,2,3\}}\langle e_i,A\tilde y\rangle-\langle x,A\tilde y\rangle\geq \frac{1}{16\sqrt{T}}.\]
\end{lemma}
\begin{proof}
    If $x_2\leq \frac{2}{3}$, then $\max_{i\in\{1,2,3\}}\langle e_i,Ay^{(1)}\rangle-\langle x,A y^{(1)}\rangle\geq \frac{1}{3}+\frac{1}{16\sqrt{T}}-\frac{1}{3}=\frac{1}{16\sqrt{T}}$. This is because $\langle x,Ay^{(1)}\rangle \leq \langle (0,2/3,1/3), y^{(1)}\rangle=1/3$. 
    
    Similarly if $x_2\geq \frac{2}{3}$, then $\max_{i\in\{1,2,3\}}\langle e_i,Ay^{(2)}\rangle-\langle x,A y^{(2)}\rangle\geq \frac{1}{3}+\frac{2}{16\sqrt{T}}-\frac{1}{3}=\frac{2}{16\sqrt{T}}$. This is because $\langle x,Ay^{(2)}\rangle \leq \langle (0,2/3,1/3), y^{(2)}\rangle=1/3$. 
    
    %\kevin{This has typos, should be half of what these calculations say (i.e,, $\frac{1}{16\sqrt{T}}$)}
\end{proof}

\begin{proof}[Proof of Theorem~\ref{external-nash:thm}]
 Observe that $\simplex_3=S_1\cup S_2$, where 
 \begin{align*}
     S_1&:=\left\{x\in \simplex_3: \max_{i\in\{1,2,3\}}\langle e_i,Ay^{(1)}\rangle-\langle x,Ay^{(1)}\rangle\geq \frac{1}{16\sqrt{T}}\right\},\\
     S_2&:=\left\{x\in \simplex_3: \max_{i\in\{1,2,3\}}\langle e_i,Ay^{(2)}\rangle-\langle x,Ay^{(2)}\rangle\geq \frac{1}{16\sqrt{T}}\right\}. 
 \end{align*}

Let $\Omega=\{1,2\}^{T}$ be the sample space of the column-player's plays. Consider any realization $\omega\in \Omega$ of the sample space. Then, the column $j_t$ sampled at the end of round $t$ is $\omega_{t}$.

Let us now analyse an algorithm $\mathcal{A}$ that aims to minimize the external regret. We have the following two adversarial strategies:
\begin{enumerate}
    \item $y_t=y^{(1)}$ for all $t\in [T]$
    \item $y_t=y^{(2)}$ for all $t\in [T]$
\end{enumerate}
Let $P_1$ and $P_2$ correspond to the probability distributions of the first and second adversarial strategy. For any $\tilde y\in \{y^{(1)},y^{(2)}\}$, let the external regret be denoted \[R(\tilde y)=\max_{i\in\{1,2,3\}}\sum_{t=1}^T\langle e_i,A\tilde y\rangle-\sum_{t=1}^T\langle x_t,A\tilde y\rangle.\] 

Now observe that for any realization $\omega\in \Omega$, the algorithm $\mathcal{A}$ behaves exactly the same for both the adversarial strategies. Let $G\subset \Omega$ denote the event such that $|\{t\in[T]:x_t\in S_1\}|<\frac{T}{2}$. As $S_1\cup S_2=\simplex_2$, we have $|t\in[T]:x_t\in S_1|+|t\in[T]:x_t\in S_2|\geq T$. Hence if event $G$ holds then $|t\in[T]:x_t\in S_2|\geq T/2$ and if event $G$ does not hold then $|t\in[T]:x_t\in S_1|\geq T/2$.

Note that the row player cannot incur a negative regret in any round. Now we begin our regret analysis.

If $P_1(G)<\frac{3}{4}$, then due to Lemma \ref{external:lem3} we have that\[R(y^{(1)})\geq (1-P_1(G))\cdot \frac{T}{2}\cdot \frac{1}{16\sqrt{T}}\geq \frac{\sqrt{T}}{128}.\] 

On the other hand, if $P_1(G)\geq \frac{3}{4}$, then  Pinsker's inequality implies that %\kevin{must cascade all constant mistakes throughout}
\begin{align*}
    2(P_1(G)-P_2(G))^2&\leq KL(P_2,P_1),\\
    &=T\cdot KL\left(y^{(1)},y^{(2)}\right),\\
    &\leq 48T\cdot \left(\frac{1}{16\sqrt{T}}\right)^2,\\
    &<\frac{1}{4},
\end{align*}
where the second inequality follows from  Lemma \ref{kl:bernoulli}.

Hence we have $P_2(G)\geq P_1(G)-\frac{1}{2\sqrt{2}}> \frac{1}{3}$. Due to the definition of $S_2$, we have $R(y^{(2)})\geq P_2(G)\cdot \frac{T}{2}\cdot \frac{1}{16\sqrt{T}}\geq \frac{\sqrt{T}}{96}$.

Now observe that $V_A^*=0$. For any $\tilde y\in\{y^{(1)},y^{(2)}\}$, we have $\max_{i\in\{1,2,3\}}\langle e_i,A\tilde y\rangle\geq \frac{1}{3}$. Now if there is an algorithm with external regret at most $\frac{T}{10}$, we have $\sum_{t=1}^T\mathbb{E}[\langle x_t,Ay_t\rangle]\geq \frac{T}{3}-\frac{T}{10}=\frac{7T}{30}$. 
This implies that $\sum_{t=1}^T\mathbb{E}[\langle x_t,Ay_t\rangle]-T\cdot V_A^*=\frac{7T}{30}$.
\end{proof}

\subsection{Myopic online learning algorithms incur $\sqrt{T}$ regret (Proof of \Cref{thm:main-myopic}) }\label{appendix:exp3:lower}
Let $T$ denote the time horizon. Consider a bandit instance $\nu$ with two bernoulli arms with means $(\frac{1}{2},\frac{1}{2})$. Let $n_i$ denote the number of times arm $i$ is pulled by a mirror descent algorithm. In this section, we analyse a myopic online learning algorithm $\mathcal{A}$ for which $\mathbb{E}_\nu[n_1]=\mathbb{E}_{\nu}[n_2]=T/2$. Now consider the bandit instance $\nu'$ with two bernoulli arms with means $(\frac{1}{2}+\frac{1}{200\sqrt{T}},\frac{1}{2}-\frac{1}{100\sqrt{T}})$. We now prove the following technical lemma.
\begin{lemma}\label{lem:bandit:similar}
     Consider a myopic online learning algorithm $\mathcal{A}$. Then for the bandit instance $\nu'$ we have $\mathbb{E}_{\nu'}[n_1]\geq 0.38T$ and $\mathbb{E}_{\nu'}[n_2]\geq 0.38T$
\end{lemma}
\begin{proof}
Let $\Omega=\{0,1\}^{2\times T}$ be the sample space of the rewards. Consider any realization $\omega\in \Omega$ of the sample space. Then the reward received when we sample arm $i$ for the $t$-th time is is $\omega_{i,t}$. Note that we have the same sample space for both the bandit instances $\nu$ and $\nu'$.

For any realization $\omega\in \Omega$ of the sample space, let $X_i=\sum_{t=1}^T\omega_{i,t}$. Now due to Chernoff bound, for bandit instance $\nu$ we have with probability at least $1-10^{-4}$ that $\frac{T}{2}-4\sqrt{T}\leq X_1\leq \frac{T}{2}+4\sqrt{T}$ and $\frac{T}{2}-4\sqrt{T}\leq X_2\leq \frac{T}{2}+4\sqrt{T}$. Also note that for any $\omega\in\Omega$, we have $\mathbb{P}_\nu[\omega]=\frac{1}{2^{2T}}$.

Now we do few calculations in order to establish a lower bound on $\mathbb{P}_{\nu'}[\omega]$. First we have the following:
\begin{align*}
    &\left(\frac{1}{2}+\frac{1}{200\sqrt{T}}\right)^{\frac{T}{2}-4\sqrt{T}}\left(\frac{1}{2}-\frac{1}{200\sqrt{T}}\right)^{\frac{T}{2}+4\sqrt{T}}\\
    =& \left(\frac{1}{4}-\frac{1}{40000T}\right)^{\frac{T}{2}-4\sqrt{T}}\left(\frac{1}{2}-\frac{1}{200\sqrt{T}}\right)^{8\sqrt{T}}\\
    =&\left(\frac{1}{4}\right)^{\frac{T}{2}-4\sqrt{T}}\left(1-\frac{1}{10000T}\right)^{\frac{T}{2}-4\sqrt{T}}\left(\frac{1}{2}\right)^{8\sqrt{T}}\left(1-\frac{1}{100\sqrt{T}}\right)^{8\sqrt{T}}\\
    \geq&\left(\frac{1}{2}\right)^{T}( 0.99995)( 0.92)
\end{align*}
For the last inequality we used the fact that $(1+x)^r\geq 1+rx$ for all $x\geq -1$ and $r\geq 2$ (Bernoulli's inequality).

Similarly we have the following:
\begin{align*}
    &\left(\frac{1}{2}+\frac{1}{100\sqrt{T}}\right)^{\frac{T}{2}-4\sqrt{T}}\left(\frac{1}{2}-\frac{1}{100\sqrt{T}}\right)^{\frac{T}{2}+4\sqrt{T}}\\
    =& \left(\frac{1}{4}-\frac{1}{10000T}\right)^{\frac{T}{2}-4\sqrt{T}}\left(\frac{1}{2}-\frac{1}{100\sqrt{T}}\right)^{8\sqrt{T}}\\
    =&\left(\frac{1}{4}\right)^{\frac{T}{2}-4\sqrt{T}}\left(1-\frac{1}{2500T}\right)^{\frac{T}{2}-4\sqrt{T}}\left(\frac{1}{2}\right)^{8\sqrt{T}}\left(1-\frac{1}{50\sqrt{T}}\right)^{8\sqrt{T}}\\
    \geq&\left(\frac{1}{2}\right)^{T}( 0.9998)( 0.84) \tag{due to Bernoulli's inequality}
\end{align*}

Let $G\subset \Omega$ such for any $\omega\in G$ we have  $\frac{T}{2}-4\sqrt{T}\leq X_i\leq \frac{T}{2}+4\sqrt{T}$ for all $i$. For the bandit instance $\nu'$ and for any $\omega\in G$, we have the following:
\begin{align*}
    \mathbb{P}_{\nu'}[\omega]&\geq \left(\frac{1}{2}+\frac{1}{200\sqrt{T}}\right)^{\frac{T}{2}-4\sqrt{T}}\left(\frac{1}{2}-\frac{1}{200\sqrt{T}}\right)^{\frac{T}{2}+4\sqrt{T}}\left(\frac{1}{2}+\frac{1}{100\sqrt{T}}\right)^{\frac{T}{2}-4\sqrt{T}}\left(\frac{1}{2}-\frac{1}{100\sqrt{T}}\right)^{\frac{T}{2}+4\sqrt{T}}\\
    &\geq \left(\frac{1}{2}\right)^{2T}( 0.99995)( 0.92)( 0.9998)( 0.84)\\
    &\geq \mathbb{P}_{\nu}[\omega](0.77)\\
\end{align*}

Let $n_{i,\omega}$ denote the number of times arm $i$ has been pulled by the mirror descent algorithm $\mathcal{A}$ under the realization $\omega\in \Omega$. 
For any $i$, we have the following:
\begin{align*}
    T/2&=\mathbb{E}_{\nu}[n_i]\\
    &=\sum_{\omega\in G}\mathbb{P}_{\nu}[\omega]\cdot n_{i,\omega}+\sum_{\omega\in \bar G}\mathbb{P}_{\nu}[\omega]\cdot n_{i,\omega}\\
    &\leq \sum_{\omega\in G}\mathbb{P}_{\nu}[\omega]\cdot n_{i,\omega}+T\cdot \sum_{\omega\in \bar G}\mathbb{P}_{\nu}[\omega]\\
    &\leq \sum_{\omega\in G}\mathbb{P}_{\nu}[\omega]\cdot n_{i,\omega}+10^{-4} \cdot T\tag{as $\mathbb{P}[\bar G]\leq 10^{-4}$}\\
\end{align*}
Hence we have $\sum_{\omega\in G}\mathbb{P}_{\nu}[\omega]\cdot n_{i,\omega}\geq 0.4999 T$. 

Now we lower bound $\mathbb{E}_{\nu'}[n_i]$ as follows:
\begin{align*}
    \mathbb{E}_{\nu'}[n_i]&\geq \sum_{\omega\in G}\mathbb{P}_{\nu'}[\omega]\cdot n_{i,\omega}\\
    &\geq (0.77)\sum_{\omega\in G}\mathbb{P}_{\nu}[\omega]\cdot n_{i,\omega}\\
    &= (0.77)(0.4999)T\\
    &>0.38T
\end{align*}
\end{proof}

%$(\frac{1}{2}+\frac{1}{200\sqrt{T}},\frac{1}{2}-\frac{1}{100\sqrt{T}})$

Now consider the following matrix.

\[
A = \begin{bmatrix} 
   0.75 & 0.25\\
   0 & 1\\
    \end{bmatrix}
\]
Recall that in each round $t$, a random matrix $\mathbf{A}_t \in [-1,1]^{n \times m}$ is drawn IID where $\mathbb{E}[\mathbf{A}_t] = A$. Let $[\mathbf{A}_t]_{i,j}\sim Bernoulli(A_{i,j})$ for all $i,j$.

Let $y=(\frac{1}{2}+\frac{1}{100\sqrt{T}},\frac{1}{2}-\frac{1}{100\sqrt{T}})^\top$. Then we have the following:
\[
Ay = \begin{bmatrix} 
   \frac{1}{2}+\frac{1}{200\sqrt{T}}\\
   \frac{1}{2}-\frac{1}{100\sqrt{T}}\\
    \end{bmatrix}
\]
If the row player plays the mirror descent algorithm $\mathcal{A}$, it is equivalent to running the mirror descent algorithm on the bandit instance $\nu'$. Hence due to lemma \ref{lem:bandit:similar}, the total rewards collected in expectation is $(\frac{1}{2}+\frac{1}{200\sqrt{T}})\cdot \mathbb{E}_{\nu'}[n_1]+(\frac{1}{2}-\frac{1}{100\sqrt{T}})\cdot \mathbb{E}_{\nu'}[n_2]\geq 0.62T\cdot (\frac{1}{2}+\frac{1}{200\sqrt{T}})+0.38T\cdot (\frac{1}{2}-\frac{1}{100\sqrt{T}})$. Observe that $V_A^*=\frac{1}{2}$. Hence we have the following:
\begin{equation*}
    T\cdot V_A^*-\sum_{t=1}^T\mathbb{E}[\langle x_t,Ay_t\rangle]=T/2-\left(\frac{1}{2}+\frac{1}{200\sqrt{T}}\right)\cdot \mathbb{E}_{\nu'}[n_1]-\left(\frac{1}{2}-\frac{1}{100\sqrt{T}}\right)\cdot \mathbb{E}_{\nu'}[n_2]\geq \frac{7\sqrt{T}}{1000}
\end{equation*}

\textbf{Remark:} Note that in general, if $\frac{\mathbb{E}_\nu[n_1]}{\mathbb{E}_\nu[n_1]+\mathbb{E}_\nu[n_2]}= c_1$ , we can always construct an instance such that regret is roughly $\sqrt{T}$ by ensuring that $|x_1^*-c_1|\geq c_2$, where $c_1,c_2$ are some absolute constants.

\subsection{Action-ignorant strategies incur $\sqrt{T}$ regret (Proof of Theorem \ref{thm-impossible-no-action})}\label{appendix:lower}
Fix the following matrices:
\begin{align}
A_1 = \begin{bmatrix} 
   a-\frac{\sqrt{\delta_{\min}}(a-c)}{32\sqrt{T}} & b+\frac{\sqrt{\delta_{\min}}(d-b)}{32\sqrt{T}} \\
   c-\frac{\sqrt{\delta_{\min}}(a-c)}{32\sqrt{T}} & d+\frac{\sqrt{\delta_{\min}}(d-b)}{32\sqrt{T}} \\
    \end{bmatrix}\ \ \text{and}\ \ 
A_2 = \begin{bmatrix} 
   a+\frac{\sqrt{\delta_{\min}}(a-c)}{32\sqrt{T}} & b-\frac{\sqrt{\delta_{\min}}(d-b)}{32\sqrt{T}} \\
   c+\frac{\sqrt{\delta_{\min}}(a-c)}{32\sqrt{T}} & d-\frac{\sqrt{\delta_{\min}}(d-b)}{32\sqrt{T}} \\
    \end{bmatrix}
\label{eqn:thm1_instances}
\end{align}
where $\frac{1}{\sqrt{T}}\ll a,b,c,d\leq \frac{1}{2}$, $a>b,a>c,d>b,d>c$ and $\delta_{\min}=\min\{b,c\}$,
and consider the following theorem, the formal version of \Cref{thm-impossible-no-action}.
\begin{theorem}[Formal version of \Cref{thm-impossible-no-action}]\label{ucb:lower}
    With $A_1$ and $A_2$ defined in \eqref{eqn:thm1_instances}, let an instance $B \in \{A_1,A_2\}$ be fixed before play starts. At each time $t$, the row- and column-players choose $x_t \in \simplex_2$ and $y_t \in \simplex_2$, respectively, actions $i_t \sim x_t$ and $j_t \sim y_t$ are drawn, and the row-player receives reward $[\mathbf{B}_t]_{i_t, j_t}$ and observes $\mathbf{B}_t$ where $[ \mathbf{B}_t ]_{i,j} \sim \text{Bernoulli}(B_{i,j})$.
    If an algorithm constructs $x_t \in \simplex_2$ from $\{ \mathbf{B}_s \}_{s < t}$ alone (and not $\{j_s\}_{s <t}$) then the best response strategy $y_t = \arg\min_{y \in \simplex_2} x_t^\top B y$ satisfies $\displaystyle \max_{B\in\{A_1,A_2\}} \mathbb{E} \left[\sum_{t=1}^T V_B^*-x_t^\top B y_t  \right] \geq \textstyle\frac{\Delta_{\min}\sqrt{\delta_{\min}T}}{512}$ where $\Delta_{\min}:=\min\{|a-b|,|a-c|,|d-b|,|d-c|\}$.
\end{theorem}

We begin by proving the following key lemma.
\begin{lemma}\label{lower:lem1}
    Consider any $x\in \simplex_2$. Then, there exists  $B\in\{A_1,A_2\}$ such that \[V_B^*-\min_{j\in\{1,2\}}\langle x,Be_j\rangle\geq \frac{\Delta_{\min}\sqrt{\delta_{\min}}}{32\sqrt{T}}.\]
\end{lemma}
\begin{proof}
    Let $D:=a-b-c+d$. Observe that $V_{A_1}^*=V_{A_2}^*=\frac{ad-bc}{D}$. If $x_1\leq \frac{d-c}{D}$, then \[V_{A_1}^*-\langle x,A_1e_1\rangle\geq \frac{ad-bc}{D}-\left(\frac{ad-bc}{D}-\frac{\sqrt{\delta_{\min}}(a-c)}{32\sqrt{T}}\right)\geq \frac{\Delta_{\min}\sqrt{\delta_{\min}}}{32\sqrt{T}}.\] Analogously, if $x_1\geq \frac{d-c}{D}$, then \[V_{A_1}^*-\langle x,A_2e_2\rangle\geq \frac{ad-bc}{D}-\left(\frac{ad-bc}{D}-\frac{\sqrt{\delta_{\min}}(d-b)}{32\sqrt{T}}\right)\geq \frac{\Delta_{\min}\sqrt{\delta_{\min}}}{32\sqrt{T}},\] which concludes the proof. 
\end{proof}

Given the above technical lemma, we prove Theorem~\ref{ucb:lower}.
\begin{proof}[Proof of Theorem~\ref{ucb:lower}]
 Due to Lemma \ref{lower:lem1}, we have $\simplex_2=S_1\cup S_2$ where
\begin{align*}
    S_1&:=\{x\in \simplex_2: V_{A_1}^*-\min_{j\in\{1,2\}}\langle x,A_1e_j\rangle\geq\frac{\Delta_{\min}\sqrt{\delta_{\min}}}{32\sqrt{T}}\},\\
    S_2&:=\{x\in \simplex_2: V_{A_2}^*-\min_{j\in\{1,2\}}\langle x,A_2e_j\rangle\geq \frac{\Delta_{\min}\sqrt{\delta_{\min}}}{32\sqrt{T}}\}.
\end{align*}

Let $\Omega=\mathbb{R}^{4\times T}$ be the sample space. Consider any realization $\omega\in \Omega$ of the sample space. Then $[\mathbf{B}_t ]_{i,j} =\omega_{2i+j,t}$. Recall that $[\mathbf{B}_t ]_{i,j}\sim \text{Bernoulli}(B_{i,j}) $. Let $P_1$ and $P_2$ correspond to the probability distributions of input instances $A_1$ and $A_2$ respectively. 

Let us now analyse an algorithm $\mathcal{A}$ that constructs $x_t \in \simplex_2$ from $\{ \mathbf{B}_s \}_{s < t}$ alone (and not $\{j_s\}_{s <t}$). We choose an adversarial column player that for any input matrix $B$ plays the column $\arg\max_{j\in\{1,2\}}V_B^*-\min\langle x,Be_j\rangle$ (breaking ties arbitrarily but consistently). Now observe that for any realization $\omega\in \Omega$, the algorithm $\mathcal{A}$ behaves exactly the same for both the input instances $A_1$ and $A_2$. Let $G\subset \Omega$ denote the event such that $|t\in[T]:x_t\in S_1|<\frac{T}{2}$.  As $S_1\cup S_2=\simplex_2$, we have $|t\in[T]:x_t\in S_1|+|t\in[T]:x_t\in S_2|\geq T$. Hence if event $G$ holds then $|t\in[T]:x_t\in S_2|\geq T/2$ and if event $G$ does not hold then $|t\in[T]:x_t\in S_1|\geq T/2$.

Note that the row player cannot incur a negative regret in any round. Now we begin our regret analysis.

If $P_1(G)<\frac{3}{4}$, then by the definition of $S_1$, we have that \[R(A_1,T)\geq (1-P_1(G))\cdot \frac{T}{2}\cdot \frac{\Delta_{\min}\sqrt{\delta_{\min}}}{32\sqrt{T}}\geq \frac{\Delta_{\min}\sqrt{\delta_{\min}T}}{256}.\] 
On the other hand, if $P_1(G)\geq \frac{3}{4}$, then due to Pinsker's inequality, we have that
\begin{align*}
    2(P_1(G)-P_2(G))^2&\leq KL(P_1,P_2)\\
    %&=2\cdot T\cdot KL\left(\mathcal{N}\left(\frac{1}{2}-\frac{1}{4\sqrt{T}},1\right),\mathcal{N}\left(\frac{1}{2}+\frac{1}{4\sqrt{T}},1\right)\right)\\
    &\leq \frac{16T}{a}\cdot \left(\frac{\sqrt{\delta_{\min}}(a-c)}{32\sqrt{T}}\right)^2+\frac{16T}{c}\cdot \left(\frac{\sqrt{\delta_{\min}}(a-c)}{32\sqrt{T}}\right)^2\\
    &+\frac{16T}{b}\cdot \left(\frac{\sqrt{\delta_{\min}}(d-b)}{32\sqrt{T}}\right)^2+\frac{16T}{d}\cdot \left(\frac{\sqrt{\delta_{\min}}(d-b)}{32\sqrt{T}}\right)^2\\
    &\leq 4\cdot 16T\cdot \left(\frac{1}{32\sqrt{T}}\right)^2\\
    &<\frac{1}{8}
\end{align*}
where 
the second inequality follows from lemma \ref{kl:bernoulli}, and the third inequality follows from the definition of $\delta_{\min}$ and assumption that $0<a,b,c,d<1$. Hence, we have that $P_2(G)\geq P_1(G)-\frac{1}{4}\geq  \frac{1}{2}$.  Due to the definition of $S_2$, we have \[R(A_2,T)\geq P_2(G)\cdot \frac{T}{2}\cdot \frac{\Delta_{\min}\sqrt{\delta_{\min}}}{32\sqrt{T}}\geq \frac{\Delta_{\min}\sqrt{\delta_{\min} T}}{128}.\]
\end{proof}

\subsection{Regret lower bound for UCB under bandit feedback}\label{appendix:ucb:lower}

Now consider the following matrix $A$:
\[
A = \begin{bmatrix} 
   2/3 & 0 \\
   0 & 1/3 \\
    \end{bmatrix}
\]
Observe that $(x^*,y^*)=((1/3,2/3),(1/3,2/3))$ is the Nash Equilibrium of $A$. We do not add any noise to the elements of $A$. In this section, we prove the following theorem.
\begin{theorem}
    UCB incurs a regret of $c_0\sqrt{T\log T}$ on the input instance $A$ where $c_0$ is an absolute constant. 
\end{theorem}

We describe the behavior of the column player in three different stages.

\textbf{Stage 1.} Let $y^{(1)}=(1,0)$ and $y^{(2)}=(0,1)$. In this stage the column player behaves as follows:
\begin{itemize}
    \item If $(x_t)_1<0.33$, then $y_t=y^{(1)}$
    \item If $(x_t)_1>0.34$, then $y_t=y^{(2)}$
    \item If $0.33\leq(x_t)_1\leq 0.34$, then $y_t=y^*$
\end{itemize}

This stage lasts for $t_0:=\frac{T}{2}+\sqrt{T\log T}$ rounds. If $|t\in[t_0]:0.33\leq(x_t)_1\leq 0.34|<T/2$, then the regret incurred by the row player is $\Omega(\sqrt{T\log T})$. Hence, let us assume that $|t\in[t_0]:0.33\leq(x_t)_1\leq 0.34|\geq T/2$. Let $n_{i,j}^t$ denote the number of times the element $(i,j)$ has been sampled till the end of round $t$. Let $G$ be the event such that the following holds at the end of round $t$:
\begin{align*}
    0.054T&\leq n_{1,1}^{t_0}\leq 0.057T\\
    0.108T&\leq n_{1,2}^{t_0}\leq 0.114T\\
    0.108T&\leq n_{2,1}^{t_0}\leq 0.114T\\
    0.216T&\leq n_{2,2}^{t_0}\leq 0.228T\\
\end{align*}
Due to Chernoff bound, event $G$ holds with probability at least $1-\frac{1}{T}$. Let now assume for the rest of the section that event $G$ holds. 

\textbf{Stage 2.} In this stage the column player plays $y_t=\arg\min_{y\in\simplex_2}\langle x_t,Ay\rangle$ in each round $t$. Recall that $\Delta_{ij}^{t}=\sqrt{\frac{\log(64T^4)}{n_{ij}^{t-1}}}$. Now for each $t\in \mathcal{T}:=\{t_0,t_0+1,\ldots, t_0+T/1000-1\}$ we have the following:

\begin{align*}
    0.054T&\leq n_{1,1}^{t}\leq 0.058T\\
    0.108T&\leq n_{1,2}^{t}\leq 0.115T\\
    0.108T&\leq n_{2,1}^{t}\leq 0.115T\\
    0.216T&\leq n_{2,2}^{t}\leq 0.229T\\
\end{align*}

Now for each $t\in\mathcal{T}$, we have the following:
\[\Delta_{22}-\Delta_{21}=\sqrt{\frac{\log(64T^4)}{n_{2,2}^{t-1}}}-\sqrt{\frac{\log(64T^4)}{n_{2,1}^{t-1}}}\leq \sqrt{\frac{\log(64T^4)}{0.216T}}-\sqrt{\frac{\log(64T^4)}{0.115T}}<-0.5\sqrt{\frac{\log(64T^4)}{T}}\] 

Next we have the following:
\begin{align*}
    \Delta_{22}-\Delta_{21}-\Delta_{12}+\Delta_{11}&=\sqrt{\frac{\log(64T^4)}{n_{2,2}^{t-1}}}-\sqrt{\frac{\log(64T^4)}{n_{2,1}^{t-1}}}-\sqrt{\frac{\log(64T^4)}{n_{1,2}^{t-1}}}+\sqrt{\frac{\log(64T^4)}{n_{1,1}^{t-1}}}\\
    &\geq \sqrt{\frac{\log(64T^4)}{0.229T}}-\sqrt{\frac{\log(64T^4)}{0.108T}}-\sqrt{\frac{\log(64T^4)}{0.108T}}+\sqrt{\frac{\log(64T^4)}{0.058T}}\\
    &>0
\end{align*}

Now for each $t\in\mathcal T$, we have the following:
\begin{align*}
    (x_t)_1&=\frac{1/3+\Delta_{22}-\Delta_{21}}{1+\Delta_{22}-\Delta_{21}-\Delta_{12}+\Delta_{11}}\\
    &\leq 1/3+\Delta_{22}-\Delta_{21}\\
    &\leq 1/3-0.5\sqrt{\frac{\log(64T^4)}{T}}
\end{align*}

Hence, each round $t\in \mathcal{T}$ the column player plays the strategy $y_t=(1,0)$. Recall that $V_A^*=2/9$. Therefore, the regret incurred by the row player in each round $t\in \mathcal{T}$ is $V_A^*-\langle x_t,Ay_t\rangle\geq 2/9-2/9+\sqrt{\frac{\log(64T^4)}{9T}}=\sqrt{\frac{\log(64T^4)}{9T}}$. Hence the total regret incurred is at least $\frac{T}{1000}\cdot \sqrt{\frac{\log(64T^4)}{9T}} =\Omega(\sqrt{T\log T})$.

\textbf{Remark:} The row player always incurs a non-negative regret in each round. Hence, just the analysis of the case when event $G$ holds suffices and therefore UCB incurs a regret of $\Omega(\sqrt{T\log T})$.
\subsection{Impossibility result for instance-dependent external regret (Proof of Theorem~\ref{external:thm})}\label{appendix:lower:external}
Consider the following matrix $A$:
\[
A = \begin{bmatrix} 
   a & b \\
   c & d \\
    \end{bmatrix}
\]
Without loss of generality, let $a>b,a>c,d>c,d<b$ and $d-b\leq a-c$. Let $D:=a-b-c+d$, $V_A^*=\frac{ad-bc}{D}$ and $\Delta_{\min}=\min\{|a-b|,|a-c|,|d-b|,|d-c|\}$.

In this setting, the matrix $A$ is known to the row player and the column player. In each round the row player plays a strategy $x_t$ and the column player plays a strategy $y_t$. At the end of each round $t$, the row player only observes a column $j_t$ that is sampled from the distribution $y_t$.

% \begin{theorem}
% Consider the matrix $A$. For any algorithm, there exists an adversary that plays $y_t$ in each round $t$ such that the following holds:
% \begin{equation*}
%     \max_{i\in\{1,2\}}\sum_{t=1}^T \langle e_i,Ay_t\rangle-\sum_{t=1}^T\langle \mathbb{E}[x_t,Ay_t]\rangle \geq \frac{1}{64}\sqrt{\frac{T\Delta_{\min}^3}{D}}
% \end{equation*}
% \end{theorem}

Let $y^{(1)}=(\frac{d-b}{D}+\sqrt{\frac{\Delta_{\min}}{64DT}},\frac{a-c}{D}-\sqrt{\frac{\Delta_{\min}}{64DT}})$ and $y^{(2)}=(\frac{d-b}{D}-\sqrt{\frac{\Delta_{\min}}{64DT}},\frac{a-c}{D}+\sqrt{\frac{\Delta_{\min}}{64DT}})$ be two possible strategies of the column player. We need the following technical lemma.
\begin{lemma}\label{external:lem1}
    Consider any $x\in \simplex_2$. There exists $\tilde y\in\{y^{(1)},y^{(2)}\}$ such that \[\max_{i\in\{1,2\}}\langle e_i,A\tilde y\rangle-\langle x,A\tilde y\rangle\geq \sqrt{\frac{\Delta_{\min}^3}{64DT}}.\]
\end{lemma}
\begin{proof}
    If $x_1\leq \frac{d-c}{D}$, then \[\max_{i\in\{1,2\}}\langle e_i,Ay^{(1)}\rangle-\langle x,A y^{(1)}\rangle\geq V_A^*+(a-b)\cdot \sqrt{\frac{\Delta_{\min}}{64DT}}-V_A^*\geq\sqrt{\frac{\Delta_{\min}^3}{64DT}}.\] Analogously, if $x_1\geq \frac{d-c}{D}$, then \[\max_{i\in\{1,2\}}\langle e_i,Ay^{(2)}\rangle-\langle x,A y^{(2)}\rangle\geq V_A^*+(d-c)\cdot \sqrt{\frac{\Delta_{\min}}{64DT}}-V_A^*\geq\sqrt{\frac{\Delta_{\min}^3}{64DT}},\] which concludes the proof. 
\end{proof}

Given the above technical lemma, we prove Theorem~\ref{external:thm}
\begin{proof}[Proof of Theorem~\ref{external:thm}]
   Observe that $\simplex_2=S_1\cup S_2$ by Lemma~\ref{external:lem1}, where 
   \begin{align*}
     S_1&:=\left\{x\in \simplex_2: \max_{i\in\{1,2\}}\langle e_i,Ay^{(1)}\rangle-\langle x,Ay^{(1)}\rangle\geq \sqrt{\frac{\Delta_{\min}^3}{64DT}}\right\},\\
     S_2&:=\left\{x\in \simplex_2: \max_{i\in\{1,2\}}\langle e_i,Ay^{(2)}\rangle-\langle x,Ay^{(2)}\rangle\geq \sqrt{\frac{\Delta_{\min}^3}{64DT}}\right\}.
   \end{align*}

Let $\Omega=\{1,2\}^{T}$ be the sample space. Consider any realization $\omega\in \Omega$ of the sample space. Then, the column $j_t$ sampled at the end of round $t$ is $\omega_{t}$.

Let us now analyse an algorithm $\mathcal{A}$ that aims to minimize the external regret. We have the following two adversarial strategies:
\begin{enumerate}
    \item $y_t=y^{(1)}$ for all $t\in [T]$
    \item $y_t=y^{(2)}$ for all $t\in [T]$
\end{enumerate}
Let $P_1$ and $P_2$ correspond to the probability distributions of the first and second adversarial strategy. For any $\tilde y\in \{y^{(1)},y^{(2)}\}$, let $R(\tilde y)$ denote the external regret \[\max_{i\in\{1,2\}}\sum_{t=1}^T\langle e_i,A\tilde y\rangle-\sum_{t=1}^T\mathbb{E}[\langle x_t,A\tilde y\rangle].\] 

Now observe that for any realization $\omega\in \Omega$, the algorithm $\mathcal{A}$ behaves exactly the same for both the adversarial strategies. Let $G\subset \Omega$ denote the event such that $|\{t\in[T]:x_t\in S_1\}|<\frac{T}{2}$.  As $S_1\cup S_2=\simplex_2$, we have $|t\in[T]:x_t\in S_1|+|t\in[T]:x_t\in S_2|\geq T$. Hence if event $G$ holds then $|t\in[T]:x_t\in S_2|\geq T/2$ and if event $G$ does not hold then $|t\in[T]:x_t\in S_1|\geq T/2$.

Note that the row player cannot incur a negative regret in any round. Now we begin our regret analysis.

If $P_1(G)<\frac{3}{4}$, then Lemma \ref{external:lem1} implies that  \[R(y^{(1)})\geq (1-P_1(G))\cdot \frac{T}{2}\cdot \sqrt{\frac{\Delta_{\min}^3}{64DT}}\geq \frac{1}{64}\sqrt{\frac{T\Delta_{\min}^3}{D}}.\] On the other hand, if $P_1(G)\geq \frac{3}{4}$, then Pinsker's inequality implies that 
\begin{align*}
    2(P_1(G)-P_2(G))^2&\leq KL(P_1,P_2),\\
    &=T\cdot KL\left(y^{(1)},y^{(2)}\right),\\
    &\leq T\cdot \frac{16D}{d-b}\cdot \left(\sqrt{\frac{\Delta_{\min}}{64DT}}\right)^2,\\
    &< \frac{1}{2},
\end{align*}
where the second inequality follows from Lemma \ref{kl:bernoulli}.
Hence, we have that $P_2(G)\geq P_1(G)-\frac{1}{2}\geq \frac{1}{4}$. Due to the definition of $S_2$, we have \[R(y^{(2)})\geq P_2(G)\cdot \frac{T}{2}\cdot \sqrt{\frac{\Delta_{\min}^3}{64DT}}\geq \frac{1}{64}\sqrt{\frac{T\Delta_{\min}^3}{D}}.\]
% This inequality holds since $|\{t\in[T]:x_t\in S_2\}| + |\{t\in[T]:x_t\in S_1\}| \geq T$ due to $\Delta_2 = S_1 \cup S_2$. Thus, $|\{t\in[T]:x_t\in S_2\}| \geq T/2$ on $G$.
\end{proof}

%\textcolor{red}{Revert back this section to $n\times n$ matrix with full support.}

\section{Analysis of various subroutines }\label{appendix:subroutine}
Recall that $A\in[-1,1]^{n\times m}$ and $\widehat A\in [-1,1]^{n\times m}$. We translate and re-scale these matrices by adding $1$ to every entry and then dividing every entry by $2$. Hence for the rest of the section, we assume that  $A\in[0,1]^{n\times m}$ and $\widehat A\in [0,1]^{n\times m}$. Note that 
the regret only changes by a factor of $2$.
\subsection{Analysis of the sub-routine (Algorithm \ref{subroutine-nxn}) for noisy matrix feedback  }\label{appendix:sub:nxn}
In this section, we now prove Theorem \ref{sub-nxn-thm}.

Recall that at each time $t=1,\ldots,T_2$ we define $x_t = x' + \vec\delta_t$ where \[\vec\delta_t= \left(\delta_t(1),\ldots,\delta_t(n-1),-\sum_{i=1}^{n-1} \delta_t(i)\right)^\top\] with $\max_{i=1,\dots,n-1} | \delta_t(i) | \leq  \frac{1}{D_1\sqrt{T_1}}$, by construction. Also observe that the step $\delta_{t+1}(i) = \arg\min_{z \in [-\frac{1}{D_1\sqrt{T_1}},\frac{1}{D_1\sqrt{T_1}}]} |\delta_{t+1/2}(i) -z |$ in our algorithm is equivalent to the following:
\begin{itemize}
    \item If $\delta_{t+1/2}(i)\in [-\frac{1}{D_1\sqrt{T_1}},\frac{1}{D_1\sqrt{T_1}}]$ then $\delta_{t+1}(i)=\delta_{t+1/2}(i)$;
    \item If $\delta_{t+1/2}(i)<-\frac{1}{D_1\sqrt{T_1}}$ then $\delta_{t+1}(i)=-\frac{1}{D_1\sqrt{T_1}}$;
    \item If $\delta_{t+1/2}(i)>\frac{1}{D_1\sqrt{T_1}}$ then $\delta_{t+1}(i)=\frac{1}{D_1\sqrt{T_1}}$.
\end{itemize}

By assumption we have that 
\begin{align*}
    &\max_{i,j} | A_{i,j}- \widehat{A}_{i,j} | \leq \frac{1}{\sqrt{T_1}},\quad \max_{i=1,\dots,n-1}| x'_i-x^*_i | \leq \frac{1}{D_1\sqrt{T_1}},\quad \text{ and}\\
    & \min_{i=1,\dots,n} \min\{x_i',1-x_i'\} \geq \tfrac{n-1}{D_1\sqrt{T_1}}\quad\text{ so that $x_t\in \simplex_n$}.
\end{align*}
Now consider the following decomposition:
\begin{align*}
    \langle x_t,Ae_{j_t}\rangle&=\langle x'+\vec\delta_t,Ae_{j_t}\rangle\\
    % &=\langle x^* + x' - x_* +\delta_t,Ae_{j_t}\rangle\\
    &=\langle x^*, Ae_{j_t}\rangle+\langle x' - x_* + \vec\delta_t , Ae_{j_t}\rangle\\
    &\geq V_A^*+\langle x' - x^* + \vec\delta_t, Ae_{j_t}-\widehat Ae_{j_t}\rangle+\langle x' - x_* + \vec\delta_t, \widehat Ae_{j_t}\rangle\\
    &\geq V_A^*-\frac{4(n-1)}{D_1T_1}+\langle x' - x^* + \vec\delta_t, \widehat Ae_{j_t}\rangle  \tag{due to our assumptions}
    \end{align*}
So it now remains to show that $\sum_{t=1}^{T_2}\langle x^*-x' - \vec\delta_t, \widehat Ae_{j_t}\rangle \leq c\cdot n\cdot\lceil\frac{T_2}{D_1T_1}\rceil$ where $c$ is some absolute constant.

For all $i\in[n-1]$ and $j\in[m]$, let $a_{i,j}:=\widehat A_{i,j}-\widehat A_{n,j}$. %\textcolor{red}{LJR: this sentence doesnt make sense..  I guess we mean let $\theta\in \mb{R}^n$ such that $\theta_i=x_i^\ast-x_i'$ for each $i\in[n]$...:} 
Let $\theta\in\mathbb{R}^n$ such that for all $i\in[n]$, $\theta_i=x^*_i-x'_i$. Observe that $\theta_n=-\sum_{i\in[n-1]}\theta_i$. Now we have the following:
\begin{align*}
    \langle x^*-x' - \vec\delta_t, \widehat Ae_{j_t}\rangle&= \langle \theta,\widehat Ae_{j_t}\rangle-\langle\vec\delta_t,\widehat Ae_{j_t}\rangle\\
    &=\sum_{i\in[n-1]}(\theta_i\widehat A_{i,j_t}-\delta_t(i)\widehat A_{i,j_t})-\sum_{i\in[n-1]}(\theta_i\widehat A_{n,j_t}-\delta_t(i)\widehat A_{n,j_t})\\
    &=\sum_{i\in[n-1]}(\theta_i(\widehat A_{i,j_t}-\widehat A_{n,j_t})-\delta_t(i)(\widehat A_{i,j_t}-\widehat A_{n,j_t}))\\
    &=\sum_{i\in[n-1]}(-a_{i,j_t}\delta_t(i)+a_{i,j_t}\theta_i).
\end{align*}
Fix an index $i\in[n-1]$. Now we analyse the sum $\sum_{t=1}^{T_2}-a_{i,j_t}\delta_t(i)+a_{i,j_t}\theta_i$.  We will be using the update rule $\delta_{t+1/2}(i)=\delta_t(i)+\frac{a_{i,j_t}}{D_1T_1}$ repeatedly to analyse the sum.

Let $\{\delta_t(i)\}_{t=1}^{T_2}$ denote the trajectory of the row-player in the interval $[-\frac{1}{D_1\sqrt{T_1}},\frac{1}{D_1\sqrt{T_1}}]$. As we show below, we can obtain a meaningful expression for the sum $\sum_{t=1}^{T_2}-a_{i,j_t}\delta_t(i)+a_{i,j_t}\theta_i$ if for every $t\in [T_2-1]$, we have $\delta_{t+1/2}(i)=\delta_{t+1}(i)$. We obtain this expression using the rule $\delta_{t+1}=\delta_{t+1/2}(i)=\delta_t(i)+\frac{a_{i,j_t}}{D_1T_1}$. However, there can exist a timestep $s$ such that $\delta_{s+1/2}(i)\neq\delta_{s+1}(i)$. Hence, we aim to break the trajectory  $\{\delta_t(i)\}_{t=1}^{T_2}$  into smaller trajectories that have the property that $\delta_{t+1/2}(i)=\delta_{t+1}(i)$ and analyse the sum separately for each of the smaller trajectories.

Let $\mathcal{T}=\{t\in[T_2]:\delta_t(i)\in \{-\frac{1}{D_1\sqrt{T_1}},\frac{1}{D_1\sqrt{T_1}}\}\}$. 
Let $t_1\leq t_2\leq\ldots\leq t_\ell$ be the ordering of the elements of $\mathcal{T}$. Here $\ell$ is a finite number and it depends on the sequence of columns played by the column player. %\textcolor{red}{LJR: I think I am missing somethign here what is $\ell$? I mean I see its the number of elements in $\mc{T}$ but do we know what it is, or is it just some finite number? }\arn{It is a finite number and it depends on the sequence of columns played by the column player.}\ljr{Yes that part I understand. Its just that we dont say that and i tmight be helpful to the reader to give some context}\arn{Ok I will add a line about $\ell$.}
Observe that $t_1=1$. Let us also define $t_{\ell+1}:=T_2+1$. For any $k\in\{1,\ldots,\ell\}$, we now analyse the sum $\sum_{t=t_k}^{t_{k+1}-1}-a_{i,j_t}\delta_t(i)+a_{i,j_t}\theta_i$ for the trajectory $\{\delta_t(i)\}_{t=t_k}^{t_{k+1}-1}$

%\kevin{So if $\mc{T} = [T_2+1]$ then this sum is just a single element since $t_{k+1}-1=t_k$? As I keep reading, I realize I  have no idea what $t_k$ represents. Something is inconsistent}\textcolor{blue}{ARN: Ok I will add more clarity.}

Recall that for any time step $t\in\{t_k,t_k+1,\ldots,t_{k+1}-1\}$, we have
\begin{align*}
    \delta_{t}(i)&=\delta_t(i)+\frac{a_{i,j_t}}{D_1T_1}-\frac{a_{i,j_t}}{D_1T_1} \\
    &=\delta_{t-1}(i)+\frac{a_{i,j_{t-1}}}{D_1T_1}+\frac{a_{i,j_t}}{D_1T_1}- \frac{a_{i,j_t}}{D_1T_1} \\
    &\ \ \vdots\\
    &= \delta_{t_k}(i) + \sum_{j\in[m]}n_{t,j}^{(k)} \frac{a_{i,j}}{D_1T_1}-\frac{a_{i,j_t}}{D_1T_1}\\
    &= \delta_{t_k}(i)+(n_{t,j_t}^{(k)}-1)\cdot\frac{a_{i,j_t}}{D_1T_1}+\sum_{j\in[m]\setminus\{j_t\}}n_{t,j}^{(k)}\cdot\frac{a_{i,j}}{D_1T_1},
\end{align*}
where for any $j\in[m]$, we define $n_{t,j}^{(k)}:=\sum\limits_{s=t_k}^{t}\mathbbm{1}\{j_{s}=j\}$ and $n_{j}^{(k)}:=\sum\limits_{t=t_k}^{t_{k+1}-1}\mathbbm{1}\{j_{t}=j\}$. 
Now we have the following:
\begin{align}
\sum_{t=t_k}^{t_{k+1}-1}a_{i,j_t}\delta_t(i)&=\sum_{t=t_k}^{t_{k+1}-1}a_{i,j_t}\left(\delta_{t_k}(i)+(n_{t,j_t}^{(k)}-1)\cdot\frac{a_{i,j_t}}{D_1T_1}+\sum_{j\in[m]\setminus\{j_t\}}n_{t,j}^{(k)}\cdot\frac{a_{i,j}}{D_1T_1}\right)\nonumber\\ 
    &=\sum_{t=t_k}^{t_{k+1}-1}a_{i,j_t}\delta_{t_k}(i)+\sum_{t=t_k}^{t_{k+1}-1} (n_{t,j_t}^{(k)}-1) \frac{a_{i,j_t}^2}{D_1 T_1} + \sum_{t=t_k}^{t_{k+1}-1} \sum_{j\in[m]\setminus\{j_t\}}n_{t,j}^{(k)}\cdot\frac{a_{i,j_t} a_{i,j}}{D_1T_1}  \nonumber.
\end{align}
The first term can be written as $(\sum_{j\in[m]}n_j^{(k)}a_{i,j})\delta_{t_k}(i)$.
The second term simplifies to
\begin{align*}
    \sum_{t=t_k}^{t_{k+1}-1} (n_{t,j_t}^{(k)}-1) \frac{a_{i,j_t}^2}{D_1 T_1} &= \sum_{j\in[m]}\sum_{\ell=1}^{n_{j}^{(k)}}(\ell-1)\cdot \frac{a_{i,j}^2}{D_1T_1} = \sum_{j\in[m]} \frac{n_{j}^{(k)}(n_{j}^{(k)}-1)}{2} \cdot \frac{a_{i,j}^2}{D_1T_1} 
\end{align*}
and the third term simplifies to
\begin{align*}
    \sum_{t=t_k}^{t_{k+1}-1} \sum_{j\in[m]\setminus\{j_t\}}n_{t,j}^{(k)}\cdot\frac{a_{i,j_t} a_{i,j}}{D_1T_1} 
 &= \sum_{t=t_k}^{t_{k+1}-1} \sum_{s\in[m]} \sum_{j\in[m]\setminus\{s\}} n_{t,j}^{(k)}\cdot\frac{a_{i,s} a_{i,j}}{D_1T_1} \mathbbm{1}\{j_t=s\}\\
 %&= 2  \sum_{\ell \in [n]} \sum_{j  < \ell} \frac{a_{i,\ell} a_{i,j}}{D_1T_1} \sum_{t=t_k}^{t_{k+1}-1} n_{t,j}^{(k)} \mathbbm{1}\{j_t=\ell\}\\
 &= \sum_{s\in[m]} \sum_{j  < s} \frac{a_{i,s} a_{i,j}}{D_1T_1} \left( \sum_{t=t_k}^{t_{k+1}-1} n_{t,j}^{(k)} \mathbbm{1}\{j_t=s\} + n_{t,\ell}^{(k)} \mathbbm{1}\{j_t=j\}\right)\\
 &= \sum_{s\in[m]} \sum_{j  < s} \frac{a_{i,s} a_{i,j}}{D_1T_1} n_{j}^{(k)} n_{s}^{(k)} 
 % \sum_{t=t_k}^{t_{k+1}-1}a_{i,j_t}\delta_{t_k}(i)-\sum_{j\in[m]}\frac{n_j^{(k)}a_{i,j}^2}{D_1T_1}+\sum_{j\in[m]}\sum_{i'=1}^{n_{j}^{(k)}}i'\cdot \frac{a_{i,j}^2}{D_1T_1}\nonumber\\
    % &=\sum_{j_1',j_2'\in[n]:j_{1}'< j_{2}'}\frac{a_{i,j_1}a_{i,j_2}}{D_1T_1}\sum_{t=t_k}^{t_{k+1}-1}n_{t,j_1'}^{(k)}\cdot\mathbbm{1}\{j_t=j_2'\}+n_{t,j_2'}^{(k)}\cdot\mathbbm{1}\{j_t=j_1'\}\nonumber\\
    % &=(\sum_{j\in[m]}n_j^{(k)}a_{i,j})\delta_{t_k}(i)-\sum_{j\in[m]}\frac{n_j^{(k)}a_{i,j}^2}{2D_1T_1}+\sum_{j\in[m]}\frac{(a_{i,j}n_j^{(k)})^2}{2D_1T_1}+\sum_{j_1',j_2'\in[n]:j_{1}'< j_{2}'}\frac{a_{i,j_{1}'}a_{i,j_2'}}{D_1T_1}\cdot n_{j_1'}^{(k)}n_{j_2'}^{(k)}\tag{due to Lemma \ref{lem:swap}}\\
    % &=(\sum_{j\in[m]}n_j^{(k)}a_{i,j})\delta_{t_k}(i)-\sum_{j\in[m]}\frac{n_j^{(k)}a_{i,j}^2}{2D_1T_1}+\frac{1}{2D_1}\cdot\left(\frac{\sum_{j\in[m]}a_{i,j}n_j^{(k)}}{\sqrt{T_1}}\right)^2\label{subroutine:equality}
\end{align*}
where the last line follows from Lemma \ref{lem:swap} in Section~\ref{sec:technical_lemmas}. 
Putting the pieces together, we obtain the identity
\begin{align}
    \sum_{t=t_k}^{t_{k+1}-1}a_{i,j_t}\delta_t(i)&= \delta_{t_k}(i) \sum_{j\in[m]}n_j^{(k)}a_{i,j} + \sum_{j\in[m]} \frac{n_{j}^{(k)}(n_{j}^{(k)}-1)}{2} \cdot \frac{a_{i,j}^2}{D_1T_1} + \sum_{s\in[m]} \sum_{j  < s} \frac{a_{i,\ell} a_{i,j}}{D_1T_1} n_{j}^{(k)} n_{s}^{(k)} \nonumber\\
    &=\delta_{t_k}(i) \sum_{j\in[m]}n_j^{(k)}a_{i,j}  -\sum_{j\in[m]}\frac{n_j^{(k)}a_{i,j}^2}{2D_1T_1}+\frac{1}{2D_1}\cdot\left(\frac{\sum_{j\in[m]}a_{i,j}n_j^{(k)}}{\sqrt{T_1}}\right)^2.\label{subroutine:equality}
\end{align}
The expression in  \eqref{subroutine:equality} is the key identity that we will exploit to bound the overall regret since we recall from above that
\begin{align*}
    T_2 V_A^* - \sum_{t=1}^{T_2} \langle x_t,Ae_{j_t}\rangle &\leq \frac{4(n-1) T_2}{D_1T_1} + \sum_{t=1}^{T_2} \langle x^* - x' - \vec\delta_t, \widehat Ae_{j_t}\rangle \\
    &= \frac{4(n-1) T_2}{D_1T_1} + \sum_{t=1}^{T_2} \sum_{i\in[n-1]}(-a_{i,j_t}\delta_t(i)+a_{i,j_t}\theta_i) \\
    &= \frac{4(n-1) T_2}{D_1T_1} + \sum_{k=1}^\ell \sum_{i\in[n-1]} \sum_{t=t_k}^{t_{k+1}-1}  (-a_{i,j_t}\delta_t(i)+a_{i,j_t}\theta_i)
\end{align*}
For any $k \in [\ell]$ and $i \in [n-1]$ consider a trajectory $\{ \delta_{t}(i) \}_{t=t_k}^{t_{k+1}-1}$.
By definition, $\delta_{t_k}(i), \delta_{t_{k+1}}(i) \in \{ -\frac{1}{D_1 \sqrt{T_1}}, \frac{1}{D_1 \sqrt{T_1}} \}$. %and our analysis considers three major cases: %\textcolor{red}{LJR: you mean 4 cases right? }\arn{kevin wanted to imply that case 0 is not that important I guess}\ljr{I would then be clear about that. For example, say case zero is not important so we focus on the other three. Otherwise its a bit confusing.... Actually maybe like you did in C.2 wher eyou say three cases and then a "side case" or somethign} 
%(1) round-trip $\delta_{t_k}(i)=\delta_{t_{k+1}}$, (2) left-to-right $\delta_{t_k}(i)=-\frac{1}{D_1\sqrt{T_1}}, \delta_{t_{k+1}}(i)=\frac{1}{D_1\sqrt{T_1}}$, and (3) right-to-left $\delta_{t_k}(i)=\frac{1}{D_1\sqrt{T_1}}, \delta_{t_{k+1}}(i)=-\frac{1}{D_1\sqrt{T_1}}$
The analysis considers three major cases: 
\begin{enumerate}
    \item[(1)] round-trip $\delta_{t_k}(i)=\delta_{t_{k+1}}(i)$;
    \item[(2)] left-to-right $\delta_{t_k}(i)=-\frac{1}{D_1\sqrt{T_1}}, \delta_{t_{k+1}}(i)=\frac{1}{D_1\sqrt{T_1}}$;
    \item[(3)] right-to-left $\delta_{t_k}(i)=\frac{1}{D_1\sqrt{T_1}}, \delta_{t_{k+1}}(i)=-\frac{1}{D_1\sqrt{T_1}}$;
\end{enumerate}
and, a side case: last-trajectory $t_{k+1} = T_2 + 1$.
Define the following sets:
\begin{align*}
   \mathcal{T}_0&:=\{k\in[\ell]:\delta_{t_k}(i)=\delta_{t_{k+1}}(i)\},\\
   \mathcal{T}_{1}&:=\left\{k\in[\ell]:\delta_{t_k}(i)=-\tfrac{1}{D_1\sqrt{T_1}},\ \  \delta_{t_{k+1}}(i)=\tfrac{1}{D_1\sqrt{T_1}}\right\},\\
   \mathcal{T}_{-1}&:=\left\{k\in[\ell]:\delta_{t_k}(i)=\tfrac{1}{D_1\sqrt{T_1}},\ \  \delta_{t_{k+1}}(i)=-\tfrac{1}{D_1\sqrt{T_1}}\right\}.
\end{align*}

Before we begin our formal analysis, let us try to understand intuitively how the sum $\sum_{t=t_k}^{t_{k+1}-1}-a_{i,j_t}\delta_t(i)+a_{i,j_t}\theta_i$ behaves in each of the three major cases. At any time-step $t$, it is the case that \[-a_{i,j_t}\delta_t(i)+a_{i,j_t}\theta_i\in\left[-\tfrac{2|a_{i,j_t}|}{D_1\sqrt{T_1}},\tfrac{2|a_{i,j_t}|}{D_1\sqrt{T_1}}\right].\] Note that time-steps $t$ when $-a_{i,j_t}\delta_t(i)+a_{i,j_t}\theta_i$ is positive and large, and summation of such large terms can lead to huge regret (potentially $\sqrt{T}$ regret). However, as we show later, we do not incur a huge regret as there are time-steps $t$ when $-a_{i,j_t}\delta_t(i)+a_{i,j_t}\theta_i$ is negative and large in magnitude.

In case 1, the row-player returns to the starting point after the end of the trajectory (hence the name round trip). In such a round trip, the large positive and negative terms of the sum $\sum_{t=t_k}^{t_{k+1}-1}-a_{i,j_t}\delta_t(i)+a_{i,j_t}\theta_i$ start cancelling out each other and at the end the sum will scale linearly with respect to $\frac{t_{k+1}-t_k}{D_1T_1}$.

In case 2, the trajectory starts from the left-most point of the interval $[-\frac{1}{D_1\sqrt{T_1}},\frac{1}{D_1\sqrt{T_1}}]$ and right-most point of this interval (hence the name left-to-right).  Similarly in case 3, the trajectory starts from the right-most point of the interval $[-\frac{1}{D_1\sqrt{T_1}},\frac{1}{D_1\sqrt{T_1}}]$ and left-most point of this interval (hence the name right-to-left). Now by pairing one left-to-right trajectory with one right-to-left trajectory, we can simulate a round trip. Similar to case 1, the large negative terms of the sum from one of the trajectory starts cancelling out the large positive terms of the sum from the other trajectory and hence the combined sum of both trajectory will scale linearly with respect to $\frac{t_{k+1}-t_k}{D_1T_1}$. 

\begin{figure}
			 		\includegraphics[width=1\linewidth]{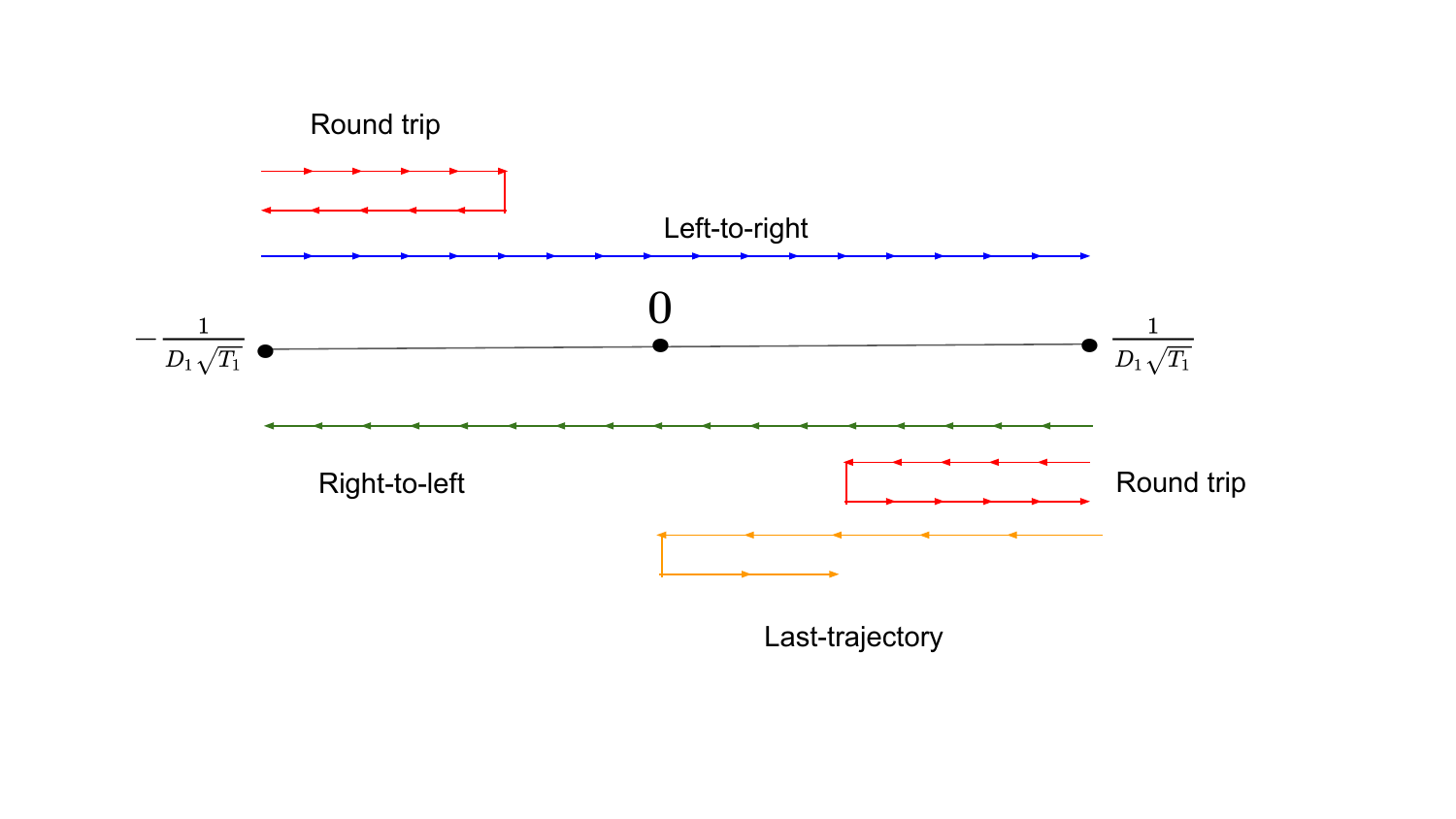}
      \caption{Various trajectories of the row player}
			 	   \end{figure}

Now we formally begin our case-by-case analysis.

\paragraph{Case 0: Last-trajectory.}
Recall that $t_{\ell+1}=T_2+1$. If $\delta_{t_{\ell+1}}(i)\notin \{-\frac{1}{D_1\sqrt{T_1}},\frac{1}{D_1\sqrt{T_1}}\}$, then \[|\delta_{t_{\ell+1}}(i)|=|\delta_{t_\ell}(i)+\tfrac{\sum_{j\in[m]}a_{i,j}n_j^{(\ell)}}{D_1T_1}|<\tfrac{1}{D_1\sqrt{T_1}}.\] This implies that $|\frac{\sum_{j\in[m]}a_{i,j}n_j^{(\ell)}}{D_1T_1}|<\frac{2}{D_1\sqrt{T_1}}$. 
From \eqref{subroutine:equality}, we deduce that
\begin{align*}
    \sum_{t=t_\ell}^{t_{\ell+1}-1}-a_{i,j_t}\delta_t(i)+a_{i,j_t}\theta_i
    &=\left(\sum_{j\in[m]}n_j^{(\ell)}a_{i,j}\right)(\theta_i-\delta_{t_\ell}(i))\\
    &\qquad+\sum_{j\in[m]}\frac{n_j^{(\ell)}a_{i,j}^2}{2D_1T_1}-\frac{1}{2D_1}\bigg(\frac{\sum_{j\in[m]}a_{i,j}n_j^{(\ell)}}{\sqrt{T_1}}\bigg)^2 \\
    &\leq |\sum_{j\in[m]}n_j^{(\ell)}a_{i,j}||\theta_i-\delta_{t_\ell}(i)|+\sum_{j\in[m]}\frac{n_j^{(\ell)}a_{i,j}^2}{2D_1T_1}.
    \end{align*}
Since     $|\sum_{j\in[m]}n_j^{(\ell)}a_{i,j}|\leq 2\sqrt{T_1}$ and  $|\theta_i-\delta_{t_\ell}(i)|\leq \frac{2}{D_1\sqrt{T_1}}$, we further deduce that
    \begin{align*}
      \sum_{t=t_\ell}^{t_{\ell+1}-1}-a_{i,j_t}\delta_t(i)+a_{i,j_t}\theta_i
    &\leq \frac{4}{D_1}+\sum_{j\in[m]}\frac{n_j^{(\ell)}a_{i,j}^2}{2D_1T_1},\\
    &\leq \frac{4}{D_1}+\sum_{j\in[m]}\frac{n_j^{(\ell)}}{2D_1T_1},\\
    &\leq \frac{4}{D_1}+\frac{T_2}{2D_1T_1},
\end{align*}
where the second inequality follows from the fact that $a_{i,j}=\widehat A_{i,j}-\widehat A_{n,j}\in[-1,1]$, and the last inequality follows from the fact that  $\sum_{j\in[m]}n_j^{(\ell)}\leq T_2$.

Note that this case happens exactly once. Also note that if $\delta_{t_{\ell+1}}(i)\in \{-\frac{1}{D_1\sqrt{T_1}},\frac{1}{D_1\sqrt{T_1}}\}$, then $\ell \in \mathcal{T}_0\cup \mathcal{T}_1\cup \mathcal{T}_{-1}$ and the last trajectory is therefore analysed in either case 1, 2 or 3.

\paragraph{Case 1: Round-trip.}
Consider an index $k\in\mathcal{T}_0$. Due to the definition of $\mathcal{T}_0$ we have $\delta_{t_k}(i)=\delta_{t_{k+1}}(i)$. Observe that \[\delta_{t_{k+1}-1+1/2}(i)=\delta_{t_k}(i)+\tfrac{\sum_{j\in[m]}a_{i,j}n_j^{(k)}}{D_1T_1}.\] If $\delta_{t_k}(i)=\delta_{t_{k+1}}(i)=-\frac{1}{D_1\sqrt{T_1}}$, then we have $\frac{\sum_{j\in[m]}a_{i,j}n_j^{(k)}}{D_1T_1}\leq 0$ as we have $\delta_{t_{k+1}-1+1/2}(i)\leq-\frac{1}{D_1\sqrt{T_1}}$ in this case.

Recall that $-\frac{1}{D_1\sqrt{T_1}}\leq \theta_i\leq \frac{1}{D_1\sqrt{T_1}}$. Hence we have $(\sum_{j\in[m]}a_{i,j}n_j^{(k)})(\theta_i-\delta_{t_k}(i))\leq 0$. Similarly, if $\delta_{t_k}(i)=\delta_{t_{k+1}}(i)=\frac{1}{D_1\sqrt{T_1}}$, then we have $\frac{\sum_{j\in[m]}a_{i,j}n_j^{(k)}}{D_1T_1}\geq 0$ as we have $\delta_{t_{k+1}-1+1/2}(i)\geq\frac{1}{D_1\sqrt{T_1}}$ in this case.

Hence we have $(\sum_{j\in[m]}a_{i,j}n_j^{(k)})(\theta_i-\delta_{t_k}(i))\leq 0$. Combining this with \ref{subroutine:equality}, we deduce that
\begin{align*}
    \sum_{t=t_k}^{t_{k+1}-1}-a_{i,j_t}\delta_t(i)+a_{i,j_t}\theta_i
    &=\left(\sum_{j\in[m]}n_j^{(k)}a_{i,j}\right)(\theta_i-\delta_{t_k}(i))\\
    &\qquad+\sum_{j\in[m]}\frac{n_j^{(k)}a_{i,j}^2}{2D_1T_1}-\frac{1}{2D_1}\cdot\left(\frac{\sum_{j\in[m]}a_{i,j}n_j^{(k)}}{\sqrt{T_1}}\right)^2,\\
    &\leq \sum_{j\in[m]}\frac{n_j^{(k)}a_{i,j}^2}{2D_1T_1},\\
    &\leq \sum_{j\in[m]}\frac{n_j^{(k)}}{2D_1T_1}, 
\end{align*}
where the last inequality follows from the fact that $a_{i,j}\in[-1,1]$. %\ljr{should be $a_{ij}\in[0,1]$?}
We conclude that $\sum_{k \in \mc{T}_0} \sum_{t=t_k}^{t_{k+1}-1}-a_{i,j_t}\delta_t(i)+a_{i,j_t}\theta_i \leq \frac{T_2}{2D_1T_1}$.

\paragraph{Case 2: Left-to-right.} %\ljr{same issue here with the implications for the delta's}
Consider an index $k\in\mathcal{T}_{1}$. Due to the definition of $\mathcal{T}_{1}$, we have that $\delta_{t_k}(i)=-\frac{1}{D_1\sqrt{T_1}}$ and $\delta_{t_{k+1}}(i)=\frac{1}{D_1\sqrt{T_1}}$. Observe that \[\delta_{t_{k+1}-1+1/2}(i)=\delta_{t_k}(i)+\frac{\sum_{j\in[m]}a_{i,j}n_j^{(k)}}{D_1T_1}.\] Now, we have $\frac{\sum_{j\in[m]}a_{i,j}n_j^{(k)}}{D_1T_1}\geq \frac{2}{D_1\sqrt{T_1}}$ since $\delta_{t_{k+1}-1+1/2}(i)\geq\frac{1}{D_1\sqrt{T_1}}$ in this case. Since $t_{k+1}-1\notin\mathcal{T}$, we have $\delta_{t_{k+1}-1}(i)<\frac{1}{D_1\sqrt{T_1}}$. Hence, we have \[\delta_{t_{k+1}-1+1/2}(i)\leq \delta_{t_{k+1}-1}(i)+\frac{1}{D_1T_1}<\frac{1}{D_1\sqrt{T_1}}+\frac{1}{D_1T_1},\] which implies 
\[\frac{2}{D_1\sqrt{T_1}} \leq \frac{\sum_{j\in[m]}a_{i,j}n_j^{(k)}}{D_1T_1}< \frac{2}{D_1\sqrt{T_1}}+\frac{1}{D_1T_1}.\]

\paragraph{Case 3: Right-to-left.} The analysis of this case is completely analogous to Case 2. 
Indeed, consider an index $k\in\mathcal{T}_{-1}$. Due to the definition of $\mathcal{T}_{-1}$ we have $\delta_{t_k}(i)=\frac{1}{D_1\sqrt{T_1}}$ and $\delta_{t_{k+1}}(i)=-\frac{1}{D_1\sqrt{T_1}}$. Observe that \[\delta_{t_{k+1}-1+1/2}(i)=\delta_{t_k}(i)+\frac{\sum_{j\in[m]}a_{i,j}n_j^{(k)}}{D_1T_1}.\] Now we have $\frac{\sum_{j\in[m]}a_{i,j}n_j^{(k)}}{D_1T_1}\leq -\frac{2}{D_1\sqrt{T_1}}$ since $\delta_{t_{k+1}-1+1/2}(i)\leq-\frac{1}{D_1\sqrt{T_1}}$ in this case. Further, since $t_{k+1}-1\notin\mathcal{T}$, we have that $\delta_{t_{k+1}-1}(i)>-\frac{1}{D_1\sqrt{T_1}}$. Hence, we deduce that \[\delta_{t_{k+1}-1+1/2}(i)\geq \delta_{t_{k+1}-1}(i)-\frac{1}{D_1T_1}>-\frac{1}{D_1\sqrt{T_1}}-\frac{1}{D_1T_1},\] which implies \[-\frac{2}{D_1\sqrt{T_1}} \geq \frac{\sum_{j\in[m]}a_{i,j}n_j^{(k)}}{D_1T_1}>- \frac{2}{D_1\sqrt{T_1}}-\frac{1}{D_1T_1}.\]

\paragraph{Putting the pieces together.} 
Given \eqref{subroutine:equality}, for any $k\in \mathcal{T}_{1}\cup\mathcal{T}_{-1}$, we have that
\begin{align*}
    \sum_{t=t_k}^{t_{k+1}-1}a_{i,j_t}\delta_t(i)
    &=\frac{1}{2D_1}\cdot\left(\frac{\sum_{j\in[m]}a_{i,j}n_j^{(k)}}{\sqrt{T_1}}\right)^2+\left(\sum_{j\in[m]}a_{i,j}n_j^{(k)}\right)\delta_{t_k}(i)-\sum_{j\in[m]}\frac{n_j^{(k)}a_{i,j}^2}{2D_1T_1},\\
    &=\frac{1}{2D_1}\cdot\left(\frac{\sum_{j\in[m]}a_{i,j}n_j^{(k)}}{\sqrt{T_1}}\right)^2-\frac{1}{D_1}\cdot\left|\frac{\sum_{j\in[m]}a_{i,j}n_j^{(k)}}{\sqrt{T_1}}\right|-\sum_{j\in[m]}\frac{n_j^{(k)}a_{i,j}^2}{2D_1T_1}.
\end{align*}
Since $\frac{x^2}{2}-x\geq 0$ for all $x\geq 2$ and $\left|\frac{\sum_{j\in[m]}a_{i,j}n_j^{(k)}}{\sqrt{T_1}}\right|\geq 2$, we have that
\begin{align*}    
   \sum_{t=t_k}^{t_{k+1}-1}a_{i,j_t}\delta_t(i)  &\geq -\sum_{j\in[m]}\frac{n_j^{(k)}a_{i,j}^2}{2D_1T_1}\geq -\sum_{j\in[m]}\frac{n_j^{(k)}}{2D_1T_1},
\end{align*}
where the last inequality follows from the fact that $a_{i,j}\in[-1,1]$. 
Therefore, we deduce that $\sum_{k \in \mc{T}_1\cup\mathcal{T}_{-1}} \sum_{t=t_k}^{t_{k+1}-1}-a_{i,j_t}\delta_t(i) \leq \frac{T_2}{2D_1T_1}$.

\paragraph{Analysis of the sum $\sum_{t=t_k}^{t_{k+1}-1}a_{i,j_t}\theta_i$ for $k\in \mathcal{T}_{1}\cup\mathcal{T}_{-1}$:} The most critical part of this proof is analysing the sum $\sum_{t=t_k}^{t_{k+1}-1}a_{i,j_t}\theta_i$. Note that $|\theta_i|$ can be as large as $\frac{1}{D_1\sqrt{T_1}}$ and $|\sum_{t=t_k}^{t_{k+1}-1}a_{i,j_t}|$ can be as large as $\sqrt{T_1}$. Hence the sum $\sum_{t=t_k}^{t_{k+1}-1}a_{i,j_t}\theta_i$ can be as large as $\frac{1}{D_1}$. 

Now, if we naively try to upper bound each such sum by $\frac{1}{D_1}$, then we only obtain an upper bound of $\frac{T_2}{D_1\sqrt{T_1}}$ since $|\mathcal T_1|$ and $|\mathcal T_{-1}|$ can be as large as $\frac{T_2}{\sqrt{T_1}}$ (we show this fact below). Hence, we can end up with a huge regret (potentially $\sqrt{T}$ regret). Instead, we show that for $k_1\in\mathcal{T}_1$ and $k_{2}\in \mathcal{T}_{-1}$, the sums $\sum_{t=t_{k_1}}^{t_{k_1+1}-1}a_{i,j_t}\theta_i$ and $\sum_{t=t_{k_{2}}}^{t_{k_{2}+1}-1}a_{i,j_t}\theta_i$ are roughly equal in magnitude and opposite in sign. Hence, these sums nicely cancel out each other out, and we are left with terms that scale linearly with $\frac{(t_{k_1+1}-t_{k_1})+(t_{k_2+1}-t_{k_2})}{D_1T_1}$.

Combining the above results, we observe the following: if $k \in \mc{T}_1$, then $2\sqrt{T_1}\leq \sum_{j\in[m]}a_{i,j}n_j^{(k)}<2\sqrt{T_1}+1$, and if $k \in \mc{T}_{-1}$, we have that $-2\sqrt{T_1}-1<\sum_{j\in[m]}a_{i,j}n_j^{(k)}\leq -2\sqrt{T_1}$.

Define the mapping $\mathsf{s}:\mathbb R\rightarrow \{-1,1\}$  as follows:
\begin{equation*}
    \mathsf{s}(x)=\left\{\begin{array}{ll}
        1, &\ x\geq 0,  \\
        -1, &\ x<0.
    \end{array}\right.
\end{equation*}
Hence, if $k\in \mathcal{T}_{\mathsf{s}(\theta_i)}$, then  $(\sum_{j\in[m]}a_{i,j}n_j^{(k)})\theta_i <(2\sqrt{T_1}+1)|\theta_i|$,  and if $k\in \mathcal{T}_{-\mathsf{s}(\theta_i)}$, then $(\sum_{j\in[m]}a_{i,j}n_j^{(k)})\theta_i <-(2\sqrt{T_1})|\theta_i|$. %\textcolor{blue}{ARN: I used variables $i_1,i_2$ to make the inequalities concise and this will be useful in the final calculations.}

Now, we show that $|\mathcal{T}_1|\leq \frac{T_2}{\sqrt{T_1}}$. Recall that $|\sum_{j\in[m]}a_{i,j}n_j^{(k)}|\geq 2\sqrt{T_1}$ if $k\in \mathcal{T}_1$. Since $|a_{i,j}|\leq 1$, we have that $t_{k+1}-t_{k}=\sum\limits_{j\in[m]}|n_j^{(k)}|\geq |\sum_{j\in[m]}a_{i,j}n_j^{(k)}|\geq 2\sqrt{T_1}$. Hence, the smallest element $s$ of $\mathcal{T}_1$ greater than $k$ satisfies $t_{s}\geq t_{k}+2\sqrt{T_1}$. Since the total time period is $T_2$, there can be at most $\frac{T_2}{2\sqrt{T_1}}+1\leq\frac{T_2}{\sqrt{T_1}}$ elements in $\mathcal{T}_2$. By an identical argument, we can show that $|\mathcal{T}_{-1}|\leq \frac{T_2}{\sqrt{T_1}}$. Also observe that $||\mathcal{T}_1|-|\mathcal{T}_{-1}||\leq 1$.

Now, we analyse the sum $\sum\limits_{k\in \mathcal T_1\cup \mathcal T_{-1}}\sum_{t=t_k}^{t_{k+1}-1}a_{i,j_t}\theta_i$. Due to the arguments above, we have the following:
\begin{align*}
    \sum_{k\in \mathcal T_1\cup \mathcal T_{-1}}\sum_{t=t_k}^{t_{k+1}-1}a_{i,j_t}\theta_i&=\sum\limits_{k\in\mathcal{T}_{\mathsf{s}(\theta_i)}}\sum_{t=t_k}^{t_{k+1}-1}a_{i,j_t}\theta_i+\sum\limits_{k\in\mathcal{T}_{\mathsf{s}(-\theta_i)}}\sum_{t=t_k}^{t_{k+1}-1}a_{i,j_t}\theta_i,\\
    &=\sum\limits_{k\in\mathcal{T}_{\mathsf{s}(\theta_i)}}\left(\sum_{j\in[m]}a_{i,j}n_j^{(k)}\right)\theta_i+\sum\limits_{k\in\mathcal{T}_{\mathsf{s}(-\theta_i)}}\left(\sum_{j\in[m]}a_{i,j}n_j^{(k)}\right)\theta_i,\\
    &\leq |\mathcal{T}_{\mathsf{s}(\theta_i)}|\cdot (2\sqrt{T_1}+1)\cdot|\theta_i|-|\mathcal{T}_{\mathsf{s}(\theta_i)}|\cdot 2\sqrt{T_1}\cdot|\theta_i|,\\
    & = 2\sqrt{T_1}|\theta_i|(|\mathcal{T}_{\mathsf{s}(\theta_i)}|-|\mathcal{T}_{-\mathsf{s}(\theta_i)}|)+|\mathcal T_{\mathsf{s}(\theta_i)}|\cdot|\theta_i|,\\
    & \leq \frac{2}{D_1}+\frac{T_2}{D_1T_1},
\end{align*}
where the last inequality follows from the fact that $|\theta_i|\leq\frac{1}{D_1\sqrt{T_1}}$, $|\mathcal{T}_{\mathsf{s}(\theta_i)}|-|\mathcal{T}_{-\mathsf{s}(\theta_i)}|\leq 1$ and $|\mathcal{T}_{\mathsf{s}(\theta_i)}|\leq \frac{T_{2}}{\sqrt{T_1}}$.

Combining the upper bounds on the sums \[\sum_{k\in \mathcal T_1\cup \mathcal T_{-1}}\sum_{t=t_k}^{t_{k+1}-1}-a_{i,j_t}\delta_t(i)\quad \text{and}\quad  \sum_{k\in \mathcal T_1\cup \mathcal T_{-1}}\sum_{t=t_k}^{t_{k+1}-1}a_{i,j_t}\theta_i,\] we deduce that
\begin{align*}
    \sum\limits_{k\in \mathcal T_1\cup \mathcal T_{-1}}\sum_{t=t_k}^{t_{k+1}-1}-a_{i,j_t}\delta_t(i)+a_{i,j_t}\theta_i&=\sum\limits_{k\in \mathcal T_1\cup \mathcal T_{-1}}\sum_{t=t_k}^{t_{k+1}-1}-a_{i,j_t}\delta_t(i)+\sum\limits_{k\in \mathcal T_1\cup \mathcal T_{-1}}\sum_{t=t_k}^{t_{k+1}-1}a_{i,j_t}\theta_i,\\
    &=\frac{T_2}{2D_1T_1}+\frac{2}{D_1}+\frac{T_2}{D_1T_1},\\
    &=\frac{3T_2}{2D_1T_1}+\frac{2}{D_1}.
\end{align*}

Now, we analyse the sum $\sum_{t=1}^{T_2}-a_{i,j_t}\delta_t(i)+a_{i,j_t}\theta_i$. Due to the arguments above in Cases 0--3, we have the following:
%\kevin{Break this up into multiple displays, highlighting the cancellation and sign-posting where things came from. This is the most important part of the proof so it needs to be crystal crystal clear}. \arn{ok. will make the suggested changes}
\begin{align*}
    \sum\limits_{t=1}^{T_2}-a_{i,j_t}\delta_t(i)+a_{i,j_t}\theta_i&=\sum\limits_{k\in\mathcal{T}_0}\sum\limits_{t=t_k}^{t_{k+1}-1}-a_{i,j_t}\delta_t(i)+a_{i,j_t}\theta_i+\sum\limits_{k\in\mathcal{T}_{1}\cup T_{-1}}\sum\limits_{t=t_k}^{t_{k+1}-1}-a_{i,j_t}\delta_t(i)+a_{i,j_t}\theta_i\\
    &\qquad+\mathbbm{1}\left\{\delta_{t_{\ell+1}}(i)\notin\left\{-\frac{1}{D_1\sqrt{T_1}},\frac{1}{D_1\sqrt{T_1}}\right\}\right\}\sum\limits_{t=t_\ell}^{T_2}-a_{i,j_t}\delta_t(i)+a_{i,j_t}\theta_i,\\
    &\leq \frac{T_2}{2D_1T_1}+\frac{2}{D_1}+\frac{3T_2}{2D_1T_1}+\frac{4}{D_1}+\frac{T_2}{2D_1T_1},\\
    &\leq\frac{6}{D_1}+\frac{3T_2}{D_1T_1}.
\end{align*}
%We get the last inequality as $\sum\limits_{k\in \mathcal{T}}\sum_{j\in[m]}n_j^{(k)}\leq T_2$, $$. 

Using the above analysis, we have the following upper bound on regret:
\begin{align*}
T_2\cdot V_A^* - \sum_{t=1}^{T_2} \langle x_t,Ae_{j_t}\rangle &\leq \sum_{t=1}^{T_2}\frac{4n}{D_1T_1}+\langle x^*-x' - \vec\delta_t, \widehat Ae_{j_t}\rangle, \\
    &=\frac{4nT_2}{D_1T_1}+\sum_{t=1}^{T_2}\sum\limits_{i=1}^{n-1}-a_{i,j_t}\delta_t(i)+a_{i,j_t}\theta_i,\\
    &=\frac{4nT_2}{D_1T_1}+\sum\limits_{i=1}^{n-1}\sum_{t=1}^{T_2}-a_{i,j_t}\delta_t(i)+a_{i,j_t}\theta_i,\\
    &\leq \frac{4nT_2}{D_1T_1}+\sum\limits_{i=1}^{n-1} (\frac{6}{D_1}+\frac{3T_2}{D_1T_1}),\\
    &\leq \frac{6n}{D_1}+\frac{7nT_2}{D_1T_1},
\end{align*}
which concludes the proof.

\subsection{Subroutine for Bandit Feedback}\label{appendix:sub:2x2}
Similar to the previous section, we analyse the subroutine (\Cref{subroutine-nxn-bandit}) for bandit feedback in this section. Our analysis will be limited to $2\times 2$ games. This subroutine will be used in \Cref{appendix:2x2}. Let us assume the conditions in (\ref{eqn:sub-conditions}) hold with $n=2$.

\begin{algorithm}[t]
\caption{Subroutine for bandit feedback}
\begin{algorithmic}[1]
\STATE \textbf{Input Parameters:} $x'\in \simplex_2$, $D_1\in \mathbb{R}_{+}$, $T_1\in \mathbb{R}_{+}$, $T_2\in \mathbb{N}$, $\widehat A\in [-1,1]^{2\times 2}$
\STATE Set $\eta\gets \frac{1}{D_1T_1}$, $\delta_1(1)\gets -\frac{1}{D_1\sqrt{T_1}} $
\STATE Let $N_{i,j}(t)$ denote number of times $(i,j)$ entry observed up to time $t$
\FOR{round $t=1,2,\dots$} 
\STATE Play $x_t\gets x' + \vec\delta_t \quad \text{ where } \quad \vec\delta_t=\left(\delta_t(1),-\delta_t(1)\right)^\top $.
\STATE Observe $y$-player's action $j_t$ and the bandit feedback $[\mathbf{A}_t]_{i_t,j_t}$.
\STATE Set $\delta_{t+1/2}(1) \gets \delta_t(1) + \eta (\widehat A_{i,j_t}-\widehat A_{n,j_t})$ 
\STATE Set $\delta_{t+1}(1) = \arg\min_{z \in [-\frac{1}{D_1\sqrt{T_1}},\frac{1}{D_1\sqrt{T_1}}]} |\delta_{t+1/2}(1) -z |$ 
\STATE \textbf{if} $\min_{i,j} N_{i,j}(t) \geq T_2$ \textbf{break}
\ENDFOR
\end{algorithmic}
\label{subroutine-nxn-bandit}
\end{algorithm}

Consider a sequence of actions $j_1,\ldots,j_{T'}$ and assume that the subroutine terminates after $T'$ steps. In the subroutine at each time $t=1,\dots,T'$ we define $x_t = x' + \vec\delta_t$ where $\vec\delta_t= \left(\delta_t(1),- \delta_t(1)\right)^\top$ with $| \delta_t(1) | \leq  \frac{1}{D_1\sqrt{T_1}}$, by construction.
    By assumption, we have that 
\begin{align*}
    &\max_{i,j} | A_{i,j}- \widehat{A}_{i,j} | \leq \frac{1}{\sqrt{T_1}},\;| x'_1-x^*_1 | \leq \frac{1}{D_1\sqrt{T_1}}, \text{ and}\\
    & \min_{i=1,2} \min\{x_i',1-x_i'\} \geq \tfrac{1}{D_1\sqrt{T_1}}\quad\text{ so that $x_t\in \simplex_2$}.
\end{align*}
    
We now define few parameters as follows: 
\begin{align*}
    n_{j}&=|\{t\in[T']:j_t=j\}|,\\
    a_j&=\widehat A_{1,j}-\widehat A_{2,j},\\
    j_*&:=\arg\min_j n_j,\\
    j_c&:=\{1,2\}\setminus \{j_*\}.
\end{align*}
Without loss of generality, let $a_1>0$, $a_2<0$ and $j_*=2$. 

Consider the time steps $t$ such that $\delta_t(1)=\frac{1}{D_1\sqrt{T_1}}$ and $j_t=1$ or $\delta_t(1)=-\frac{1}{D_1\sqrt{T_1}}$ and $j_t=2$. We can ignore these time steps as we incur non-positive regret in these steps. Hence,  without loss of generality, in the remaining proof we assume throughout that if $\delta_t(1)=\frac{1}{D_1\sqrt{T_1}}$ then $j_t=2$ and if $\delta_t(1)=-\frac{1}{D_1\sqrt{T_1}}$ and $j_t=1$.

    We now \textbf{aim to show} that the time horizon $n_1+n_2$ is upper bounded by $(3\left\lceil\frac{|a_2|}{|a_1|}\right\rceil+1)n_2+\lfloor \frac{2\sqrt{T_1}}{|a_1|}\rfloor+1$.  %We will be using the update rule $\delta_{t+1/2}(1)=\delta_t(1)+\frac{a_{i,j_t}}{D_1T_1}$ repeatedly to analyse the sum.

    Define $\mathcal{T}:=\{t\in[T']:\delta_t(1)\in \{-\frac{1}{D_1\sqrt{T_1}},\frac{1}{D_1\sqrt{T_1}}\}\}$, and  let $t_1\leq t_2\leq\ldots\leq t_\ell$ be the ordering of the elements of $\mathcal{T}$. Observe that $t_1=1$. Let us also define $t_{\ell+1}:=T'+1$ and  $n_j^{(k)}:=|\{t_k\leq t<t_{k+1}: j_t=j\}|$. From our analysis in Section \ref{appendix:sub:nxn}, we have that $\delta_{t_{k+1}-1+1/2}(1)=\delta_{t_k}(1)+\frac{a_1n_1^{(k)}+a_2n_2^{(k)}}{D_1T_1}$. This also implies that \[\delta_{t_{k+1}-1+1/2}(1)=\delta_{t_k}(1)+\frac{|a_1|n_1^{(k)}-|a_2|n_2^{(k)}}{D_1T_1}\] since $a_1> 0$ and $a_2<0$. We will be using this equality repeatedly to upper bound $T'$.
%    \ljr{These $\delta_{t_k}$'s above should have $\delta_{t_k}(1)$ right? Looks like there are a few places you dropped that $(1)$ thing.}\arn{thanks for pointing out the typos}
    
    %Let $\mathcal{T}_1=\{k\in[\ell]:\delta_{t_k}(1)=\delta_{t_{k+1}}(1)\}$, $\mathcal{T}_2=\{k\in[\ell]:\delta_{t_k}(1)=-\frac{1}{D_1\sqrt{T_1}}, \delta_{t_{k+1}}(1)=\frac{1}{D_1\sqrt{T_1}}\}$ and $\mathcal{T}_3=\{k\in[\ell]:\delta_{t_k}(1)=\frac{1}{D_1\sqrt{T_1}}, \delta_{t_{k+1}}(1)=-\frac{1}{D_1\sqrt{T_1}}\}$. 

    For any $k \in [\ell]$ consider a trajectory $\{ \delta_{t}(1) \}_{t=t_k}^{t_{k+1}-1}$.
By definition, $\delta_{t_k}(1),\ \delta_{t_{k+1}}(1) \in \{ -\frac{1}{D_1 \sqrt{T_1}}, \frac{1}{D_1 \sqrt{T_1}} \}$. The analysis considers three major cases: 
\begin{enumerate}
    \item[(1)] round-trip $\delta_{t_k}(1)=\delta_{t_{k+1}}(1)$;
    \item[(2)] left-to-right $\delta_{t_k}(1)=-\frac{1}{D_1\sqrt{T_1}}, \delta_{t_{k+1}}(1)=\frac{1}{D_1\sqrt{T_1}}$;
    \item[(3)] right-to-left $\delta_{t_k}(1)=\frac{1}{D_1\sqrt{T_1}}, \delta_{t_{k+1}}(1)=-\frac{1}{D_1\sqrt{T_1}}$;
\end{enumerate}
and, a side case: last-trajectory $t_{k+1} = T' + 1$.
Define the sets
\begin{align*}
    \mathcal{T}_0&:=\{k\in[\ell]:\delta_{t_k}(1)=\delta_{t_{k+1}}(1)\},\\
    \mathcal{T}_{1}&:=\left\{k\in[\ell]:\delta_{t_k}(1)=-\frac{1}{D_1\sqrt{T_1}},\ \  \delta_{t_{k+1}}(1)=\frac{1}{D_1\sqrt{T_1}}\right\},\\
    \mathcal{T}_{-1}&:=\left\{k\in[\ell]:\delta_{t_k}(1)=\frac{1}{D_1\sqrt{T_1}},\ \  \delta_{t_{k+1}}(1)=-\frac{1}{D_1\sqrt{T_1}}\right\}.
\end{align*}

A key part of our analysis revolves around upper bounding the sum $\sum_{k\in[\ell]}n_1^{(k)}$ by a linear function of $\sum_{k\in[\ell]}n_2^{(k)}$. To intuitively understand why this is a realistic aim, let us consider a simple round-trip trajectory $\{ \delta_{t}(1) \}_{t=t_k}^{t_{k+1}-1}$ where the row player takes steps of size $\frac{|a_1|}{D_1T_1}$ towards right initially before it reverses its direction of movement and comes back to the starting point by taking steps of size $\frac{|a_2|}{D_1T_1}$. Since the net displacement is zero, we have   $\frac{|a_1|n_{1}^{(k)}}{D_1T_1}-\frac{|a_2|n_{2}^{(k)}}{D_1T_1}=0$ (ignoring the boundary cases) and therefore $|n_1^{(k)}|=\frac{|a_2|n_2^{(k)}}{|a_1|}$.

Now we formally begin our case-by-case analysis.

    \paragraph{Case 1: Round-trip.}  Consider an index $k\in\mathcal{T}_0$. Due to the definition of $\mathcal{T}_0$, we have that $\delta_{t_k}(1)=\delta_{t_{k+1}}(1)$. Recall that $\delta_{t_{k+1}-1+1/2}(1)=\delta_{t_k}(1)+\frac{|a_1|n_1^{(k)}-|a_2|n_2^{(k)}}{D_1T_1}$. If $\delta_{t_k}(1)=-\frac{1}{D_1\sqrt{T_1}}$, then we have $\frac{|a_1|n_1^{(k)}-|a_2|n_2^{(k)}}{D_1T_1}\leq 0$ since we have $\delta_{t_{k+1}-1+1/2}(1)\leq-\frac{1}{D_1\sqrt{T_1}}$ in this case. Hence, we deduce that $n_1^{(k)}\leq \frac{n_2^{(k)}\cdot |a_2|}{|a_1|}$. 
    
    If $\delta_{t_k}(1)=\delta_{t_{k+1}}(1)=\frac{1}{D_1\sqrt{T_1}}$, then $j_{t_{k+1}-1}=1$. If $j_{t_{k+1}-1}=2$ instead, then $\delta_{t_{k+1}}(1)=\delta_{t_{k+1}-1}-\frac{|a_2|}{D_1T_1}$ and hence $\delta_{t_{k+1}}(1)\neq \frac{1}{D_1\sqrt{T_1}}$ which is a contradiction to the assumption. %\ljr{This sentence doesnt make sense. I am not sure even what the cases are you are trying to point out. Looking at the next case, I think you are trying to say that "Observe that $j_{t_{k+1}-1}=1$ otherwise $\delta_{t_{k+1}}(1)\neq\frac{1}{D_1\sqrt{T_1}}$."}\arn{No, what I am trying to argue is that when $\delta_{t_k}(1)=\delta_{t_{k+1}}(1)=\frac{1}{D_1\sqrt{T_1}}$, then the first column played in time-step $t_{k+1}-1$ (that is  $j_{t_{k+1}-1}=1$ ). If $j_{t_{k+1}-1}=2$ instead, then the row player moves towards left (that is, $\delta_{t_{k+1}}(1)=\delta_{t_{k+1}-1}-\frac{|a_2|}{D_1T_1}$) and hence cannot finish the trajectory at $\frac{1}{D_1\sqrt{T_1}}$. I will make this clear.}\ljr{Ok yeah maybe jsut saying this explicitly would be good. The orginal sentence just wasnt grammatically making sense.}

    Recall that we have assumed there are not time-steps $t$ such that $\delta_t(1)=\frac{1}{D_1\sqrt{T_1}}$ and $j_t=1$. Hence, we have that \[\delta_{t_{k+1}-1}(1)=\delta_{t_{k}}(1)+\frac{|a_1|(n_1^{(k)}-1)-|a_2|n_2^{(k)}}{D_1T_1}<\frac{1}{D_1\sqrt{T_1}},\] which implies that $n_1^{(k)}-1< \frac{|a_2|n_2^{(k)}}{|a_1|}$. Also, we have that \[\delta_{t_{k+1}-1+1/2}(1)=\delta_{t_k}(1)+\frac{|a_1|n_1^{(k)}-|a_2|n_2^{(k)}}{D_1T_1}\geq \frac{1}{D_1\sqrt{T_1}},\] which  implies that $\frac{|a_2|n_2^{(k)}}{|a_1|}\leq n_1^{(k)}$. Hence, we have $n_1^{(k)}-1< \frac{|a_2|n_2^{(k)}}{|a_1|}\leq n_1^{(k)}$ which implies that $n_1^{(k)}=\lceil \frac{n_2^{(k)}\cdot|a_2|}{|a_1|}\rceil$. Therefore we deduce that \[n_1^{(k)}=\left\lceil \frac{n_2^{(k)}\cdot|a_2|}{|a_1|}\right\rceil\leq \left\lceil n_2^{(k)}\left\lceil\frac{|a_2|}{|a_1|}\right\rceil\right\rceil=n_2^{(k)}\left\lceil\frac{|a_2|}{|a_1|}\right\rceil.\]

    \paragraph{Case 2: Left-to-right.} The analysis of this case is analogous to the previous case.
   Consider an index $k\in\mathcal{T}_1$. Recall that $\delta_{t_k}(1)=-\frac{1}{D_1\sqrt{T_1}}$ and $\delta_{t_{k+1}}(1)=\frac{1}{D_1\sqrt{T_1}}$. Observe that $j_{t_{k+1}-1}=1$ otherwise we get $\delta_{t_{k+1}}(1)\neq\frac{1}{D_1\sqrt{T_1}}$ which is a contradiction. Recall that we assumed that there is no time-step $t$ such that $\delta_t(1)=\frac{1}{D_1\sqrt{T_1}}$ and $j_t=1$. Hence, we have that \[\delta_{t_{k+1}-1}(1)=\delta_{t_{k}}(1)+\frac{|a_1|(n_1^{(k)}-1)-|a_2|n_2^{(k)}}{D_1T_1}<\frac{1}{D_1\sqrt{T_1}},\] which implies that $n_1^{(k)}-1< \frac{2\sqrt{T_1}+|a_2|n_2^{(k)}}{|a_1|}$. Also, we have that \[\delta_{t_{k+1}-1+1/2}(1)=\delta_{t_k}(1)+\frac{|a_1|n_1^{(k)}-|a_2|n_2^{(k)}}{D_1T_1}\geq \frac{1}{D_1\sqrt{T_1}},\] which implies that $\frac{2\sqrt{T_1}+|a_2|n_2^{(k)}}{|a_1|}\leq n_1^{(k)}$. Hence, we have $n_1^{(k)}-1< \frac{2\sqrt{T_1}+|a_2|n_2^{(k)}}{|a_1|}\leq n_1^{(k)}$ which implies that $n_1^{(k)}=\lceil \frac{2\sqrt{T_1}+n_2^{(k)}\cdot|a_2|}{|a_1|}\rceil$. This is equivalent to \[n_1^{(k)}=\left\lfloor \frac{2\sqrt{T_1}}{|a_1|}\right\rfloor+\left\lceil\frac{n_2^{(k)}\cdot |a_2|+\Delta_{a_1}}{|a_1|}\right\rceil,\] where $\Delta_{a_1}=2\sqrt{T_1}-|a_1|\cdot\lfloor \frac{2\sqrt{T_1}}{|a_1|}\rfloor$. 

    \paragraph{Case 3: Right-to-left.} Again, the analysis of this case is analogous to the previous case.
Consider an index $k\in\mathcal{T}_{-1}$. Recall that $\delta_{t_k}(1)=\frac{1}{D_1\sqrt{T_1}}$ and $\delta_{t_{k+1}}(1)=-\frac{1}{D_1\sqrt{T_1}}$. Observe that $j_{t_{k+1}-1}=2$ otherwise we get $\delta_{t_{k+1}}(1)\neq-\frac{1}{D_1\sqrt{T_1}}$ which is a contradiction. Recall that we assumed that there is no time-step $t$ such that $\delta_t(1)=-\frac{1}{D_1\sqrt{T_1}}$ and $j_t=2$. Hence, we have that \[\delta_{t_{k+1}-1}(1)=\delta_{t_{k}}(1)+\frac{|a_1|n_1^{(k)}-|a_2|(n_2^{(k)-1})}{D_1T_1}>-\frac{1}{D_1\sqrt{T_1}},\] which implies that $n_2^{(k)}-1< \frac{2\sqrt{T_1}+|a_1|n_1^{(k)}}{|a_2|}$. Also, we have that \[\delta_{t_{k+1}-1+1/2}(1)=\delta_{t_k}(1)+\frac{|a_1|n_1^{(k)}-|a_2|n_2^{(k)}}{D_1T_1}\leq -\frac{1}{D_1\sqrt{T_1}},\] which implies $\frac{2\sqrt{T_1}+|a_1|n_1^{(k)}}{|a_2|}\leq n_2^{(k)}$. Hence, we have that $n_2^{(k)}-1< \frac{2\sqrt{T_1}+|a_1|n_1^{(k)}}{|a_2|}\leq n_2^{(k)}$ which implies that $n_2^{(k)}=\lceil \frac{2\sqrt{T_1}+n_1^{(k)}\cdot|a_1|}{|a_2|}\rceil$. This is equivalent to \[n_2^{(k)}=\left\lfloor \frac{2\sqrt{T_1}}{|a_2|}\right\rfloor+\left\lceil\frac{n_1^{(k)}\cdot |a_1|+\Delta_{a_2}}{|a_2|}\right\rceil,\] where $\Delta_{a_2}=2\sqrt{T_1}-|a_2|\cdot\lfloor \frac{2\sqrt{T_1}}{|a_2|}\rfloor$.  
    
    \paragraph{Putting the pieces together.} Consider $k_1\in\mathcal{T}_1$ and $k_2\in\mathcal{T}_{-1}$.  First, observe that $n_1^{(k_2)}\leq \frac{|a_2|}{|a_1|}\cdot\lceil\frac{n_1^{(k_2)}\cdot |a_1|+\Delta_{a_2}}{|a_2|}\rceil$ as $\lceil\frac{n_1^{(k_2)}\cdot |a_1|+\Delta_{a_2}}{|a_2|}\rceil\geq \lceil\frac{n_1^{(k_2)}\cdot |a_1|}{|a_2|}\rceil\geq n_1^{(k_2)}\frac{|a_1|}{|a_2|}$. Next, observe that $\lceil\frac{n_2^{(k_1)}\cdot |a_2|+\Delta_{a_1}}{|a_1|}\rceil \leq \lceil\frac{n_2^{(k_1)}\cdot |a_2|}{|a_1|}\rceil +\lceil \frac{\Delta_{a_1}}{|a_1|}\rceil\leq n_2^{(k_1)}\cdot\lceil \frac{|a_2|}{|a_1|}\rceil +1$. Finally, observe that $\lfloor \frac{2\sqrt{T_1}}{|a_1|}\rfloor\leq \lceil \frac{2\sqrt{T_1}}{|a_2|}\cdot \frac{|a_1|}{|a_2|}\rceil\leq \lceil \frac{|a_2|}{|a_1|}\rceil \cdot \lceil \frac{2\sqrt{T_1}}{|a_2|}\rceil\leq \lceil\frac{|a_2|}{|a_1|}\rceil \cdot (\lfloor \frac{2\sqrt{T_1}}{|a_2|}\rfloor+1 ).$ Therefore we have that
    \begin{align*}
        n_1^{(k_1)}+n_1^{(k_2)}
        &=\left\lfloor \frac{2\sqrt{T_1}}{|a_1|}\right\rfloor+\left\lceil\frac{n_2^{(k_1)}\cdot |a_2|+\Delta_{a_1}}{|a_1|}\right\rceil+n_1^{(k_2)}\\
        &\leq \left\lceil\frac{|a_2|}{|a_1|}\right\rceil \cdot \left(\left\lfloor \frac{2\sqrt{T_1}}{|a_2|}\right\rfloor+1 \right)+ n_2^{(k_1)}\cdot\left\lceil \frac{|a_2|}{|a_1|}\right\rceil +1+\left\lceil\frac{|a_2|}{|a_1|}\right\rceil\cdot\left\lceil\frac{n_1^{(k_2)}\cdot |a_1|+\Delta_{a_2}}{|a_2|}\right\rceil\\
        & = \left\lceil\frac{|a_2|}{|a_1|}\right\rceil \cdot \left(\left\lfloor \frac{2\sqrt{T_1}}{|a_2|}\right\rfloor+1 \right)+ n_2^{(k_1)}\cdot\left\lceil \frac{|a_2|}{|a_1|}\right\rceil +1+\left\lceil\frac{|a_2|}{|a_1|}\right\rceil\cdot \left(n_2^{(k_2)}-\left\lfloor \frac{2\sqrt{T_1}}{|a_2|}\right\rfloor\right), \\
        &=\left\lceil\frac{|a_2|}{|a_1|}\right\rceil\cdot (n_2^{(k_1)}+n_2^{(k_2)})+\left\lceil\frac{|a_2|}{|a_1|}\right\rceil+ 1,
            \end{align*}
where the second to last equality follows from the fact that  $n_2^{(k_2)}-\lfloor \frac{2\sqrt{T_1}}{|a_2|}\rfloor=\lceil\frac{n_1^{(k_2)}\cdot |a_1|+\Delta_{a_2}}{|a_2|}\rceil$.
Since $1\leq \left\lceil\frac{|a_2|}{|a_1|}\right\rceil\cdot (n_2^{(k_1)}+n_2^{(k_2)})$ and $\left\lceil\frac{|a_2|}{|a_1|}\right\rceil\leq\left\lceil\frac{|a_2|}{|a_1|}\right\rceil\cdot (n_2^{(k_1)}+n_2^{(k_2)})$, we further deduce that 
\begin{align*}
        n_1^{(k_1)}+n_1^{(k_2)}& \leq 3\left\lceil\frac{|a_2|}{|a_1|}\right\rceil\cdot (n_2^{(k_1)}+n_2^{(k_2)}).
    \end{align*}

    \paragraph{Combining all the three cases.} Note that except one element $k_0$ in $\mathcal{T}_1$, rest of the elements in $\mathcal{T}_1$ can be paired with distinct elements in $\mathcal{T}_{-1}$. Let $i_1,i_2,\ldots,i_\mathsf{k} \in \mathcal{T}_1$ be paired with $j_1,j_2,\ldots,j_\mathsf{k} \in \mathcal{T}_1$, respectively, where $\tau=|\mathcal{T}_{-1}|$. Hence, we have that
    \begin{align*}
        \sum_{k\in\mathcal{T}_0}n_1^{(k)}+\sum_{s=1}^{\tau}(n_1^{(i_s)}+n_1^{(j_s)})= \left\lceil\frac{|a_2|}{|a_1|}\right\rceil\sum_{k\in\mathcal{T}_0}n_2^{(k)}+3\left\lceil\frac{|a_2|}{|a_1|}\right\rceil\sum_{s=1}^{\tau} (n_2^{(i_s)}+n_2^{(j_s)}).
    \end{align*}

    \paragraph{Side case: Last-trajectory.} First, we observe that  the side case occurs at most once. Let $\{x\}$ denote the fractional part of $x$. 
    
    If $\delta_{t_\ell}(1)=-\frac{1}{D_1\sqrt{T_1}}$ and $\delta_{t_{\ell+1}}(1)\notin\{-\frac{1}{D_1\sqrt{T_1}},\frac{1}{D_1\sqrt{T_1}}\}$, then using similar arguments as above we get the following:
    \begin{align*}
        n_1^{(\ell)}&=\left\lfloor \frac{D_1T_1\cdot \delta_{t_{\ell+1}}(1)+\sqrt{T_1}}{|a_1|}\right\rfloor+\left\lceil\frac{n_2^{(\ell)}\cdot |a_2|}{|a_1|}+\left\{\frac{D_1T_1\cdot \delta_{t_{\ell+1}}(1)+\sqrt{T_1}}{|a_1|}\right\}\right\rceil,\\
        &\leq  \left\lfloor\frac{2\sqrt{T_1}}{|a_1|}\right\rfloor+\left\lceil\frac{n_2^{(\ell)}\cdot |a_2|}{|a_1|}\right\rceil+\left\lceil\left\{\frac{D_1T_1\cdot \delta_{t_{\ell+1}}(1)+\sqrt{T_1}}{|a_1|}\right\}\right\rceil, \\
        &\leq \left\lfloor\frac{2\sqrt{T_1}}{|a_1|}\right\rfloor+n_2^{(\ell)}\cdot\left\lceil \frac{|a_2|}{|a_1|}\right\rceil +1, 
    \end{align*}
    where the second to last inequality follows from the fact that $\delta_{t_{\ell+1}}(1)<\frac{1}{D_1\sqrt{T_1}}$.

    On the other hand, if $\delta_{t_\ell}(1)=\frac{1}{D_1\sqrt{T_1}}$ and $\delta_{t_{\ell+1}}(1)\notin\{-\frac{1}{D_1\sqrt{T_1}},\frac{1}{D_1\sqrt{T_1}}\}$, then using  arguments analogous to those above, we get \[n_2^{(\ell)}=\lfloor \frac{\sqrt{T_1}-D_1T_1\cdot \delta_{t_{\ell+1}}(1)}{|a_2|}\rfloor+\lceil\frac{n_1^{(\ell)}\cdot |a_1|}{|a_2|}+\{\frac{-D_1T_1\cdot \delta_{t_{\ell+1}}(1)+\sqrt{T_1}}{|a_2|}\}\rceil\geq n_1^{(\ell)}\cdot\frac{|a_1|}{|a_2|}.\] Since $\delta_{t_\ell}(1)=\frac{1}{D_1\sqrt{T_1}}$ and $\delta_{t_{\ell+1}}(1)\notin\{-\frac{1}{D_1\sqrt{T_1}},\ \frac{1}{D_1\sqrt{T_1}}\}$, there is exactly one index $k_0\in \mathcal{T}_1$ which cannot be paired with an index in $\mathcal{T}_{-1}$. Hence, we deduce that  \[n_1^{(k_0)}=\lfloor \frac{2\sqrt{T_1}}{|a_1|}\rfloor+\lceil\frac{n_2^{(k_0)}\cdot |a_2|+\Delta_{a_1}}{|a_1|}\rceil\leq \lfloor \frac{2\sqrt{T_1}}{|a_1|}\rfloor+ n_2^{(k_0)}\cdot\lceil \frac{|a_2|}{|a_1|}\rceil +1.\]

  By combining the case-by-case analysis and subsequent arguments above, we deduce that
  %finally upper $T'$ as follows
    \begin{align*}
        n_1+n_2\leq \left(3\left\lceil\frac{|a_2|}{|a_1|}\right\rceil+1\right)\cdot n_2+\left\lfloor \frac{2\sqrt{T_1}}{|a_1|}\right\rfloor+1,
    \end{align*}
    and 
    we get the desired regret bound by applying Theorem \ref{sub-nxn-thm}, which concludes the proof. 
\begin{theorem}\label{2x2:subroutine:bandit}
Consider a sequence of actions $j_1,\ldots,j_{T'}$ and assume that the Algorithm \ref{subroutine-nxn-bandit} terminates after $T'$ steps. Assume the conditions of \eqref{eqn:sub-conditions} hold for Algorithm \ref{subroutine-nxn-bandit}. Algorithm \ref{subroutine-nxn-bandit} has the following regret upper bound in the bandit-feedback case:
\begin{equation*}
    T'\cdot V_A^*-\sum_{t=1}^{T'}\langle x_t,Ae_{j_t}\rangle\leq \frac{c_1}{D_1}+\frac{c_2T_3}{D_1T_1}
\end{equation*}
where $T_3=3(\left\lceil \left|\frac{a_{j_*}}{a_{j_c}}\right|\right\rceil+1)\cdot n_{j_*}+\lfloor \frac{2\sqrt{T_1}}{|a_{j_c}|}\rfloor+1$ and $c_1,c_2$ are absolute constants.
\end{theorem}

\textbf{Remark:} The regret upper bound implies that either both columns are played sufficiently by the column player or the row player incurs a very low regret. 

\section{$n\times m$ matrix game under noisy matrix feedback}\label{appendix:matrix-feedback}
In this section, we formally describe the algorithm for $n \times m$ matrix games under noisy matrix feedback. First, in \Cref{appendix:nxm}, we outline how to identify the support of the unique Nash equilibrium. Next, in \Cref{appendix:nxn}, we explain how to invoke the subroutine (Algorithm \ref{subroutine-nxn}) in $n \times m$ games with a unique full row-support Nash equilibrium by appropriately setting the input parameters. Then, in \Cref{appendix:main-thm-guarantee}, we formally state our Nash regret guarantee. In \Cref{appendix:condition-number-guarantee}, we discuss how one might improve the dependence on $n$ and $m$. Finally, in \Cref{appendix:nxm:non-unique}, we address the case of non-unique Nash equilibria. Note that, throughout this section, NE stands for Nash equilibrium.
\subsection{Algorithm to identify the support of $n\times m$ matrix}\label{appendix:nxm}
%\textcolor{red}{Write theorem and its proof sketch. Split this section into two sections. Add this part-- }
In this section, consider an $n\times m$ matrix $A$ with unique Nash Equilibrium $(x^*,y^*)$ and describe an algorithm (Algorithm \ref{alg-nxm-instance}) to identify $\supp(x^*)$ and $\supp(y^*)$ while incurring a very low regret. Recall that $A\in[-1,1]^{n\times m}$ and $ \mathbf{A}_t\in [-1,1]^{n\times m}$. Moreover, for any matrix $B\in [-1,1]^{k\times k}$ we say a pair $(x',y')$ is a full support NE of $B$ if $|\supp(x')|=|\supp(y')|=k$.  We translate and re-scale these matrices by adding $1$ to every entry and then dividing every entry by $2$. Hence for the rest of the section, we assume that  $A\in[0,1]^{n\times m}$ and $\mathbf{A}_t\in [0,1]^{n\times m}$. Note that 
the regret only changes by a factor of $2$.

\begin{algorithm}[t!]
\caption{Algorithm to identify the optimal submatrix}
\begin{algorithmic}[1]
\STATE $x_1\gets (1/n,\ldots,1/n)$.
\FOR{time step $t=1,2,\ldots,T$} 
\STATE Play $x_t$ and update the empirical means $\bar A_{ij}$.
\STATE Let $(x',y')$ be the Nash Equilibirum of $\bar A$ and $x_{t+1}\gets x'$.
\STATE If $|\supp(x')|\neq |\supp(y')|$, then proceed to the next time step.
\STATE $k\gets |\supp(x')|$
\STATE Let $\bar B$ be a $k\times k$ submatrix of $\bar A$ with row indices $\supp(x')$ and column indices $\supp(y')$.
\STATE $\Delta \gets \sqrt{\frac{2\log({nm\cdot T^2)}}{t}}$ and $\tilde\Delta=k\cdot k!\cdot\sqrt{\frac{2\log(nmT^2)}{t}}$
\STATE $\tilde\Delta_{\min}\gets \min\{|det(M_{\bar B^\top})|,|det(M_{\bar B})|,\min_{i\in[n]}|det(M_{\bar B^\top,i})|,\min_{j\in[n]}|det(M_{\bar B,j})|\}$
\IF{$1\leq \frac{\tilde \Delta_{\min}+2k^2\cdot k!\cdot\Delta}{\tilde \Delta_{\min}-2k^2\cdot k!\cdot\Delta}\leq \frac{3}{2}$ }\label{con:sub:line1}
\IF{$k\neq n$}
\STATE $\tilde\Delta_{g,1}\gets V_{\bar A}^*-\max_{i\notin \supp(x')}\sum_{j=1}^my'_j\bar A_{i,j}$ 
\ELSE
\STATE $\tilde\Delta_{g,1}\gets \infty$
\ENDIF
\IF{$k\neq m$}
\STATE $\tilde\Delta_{g,2}\gets \min_{j\notin \supp(y')}\sum_{i=1}^nx'_i\bar A_{i,j}-V_{\bar A}^*$ 
\ELSE
\STATE $\tilde\Delta_{g,2}\gets\infty$
\ENDIF
\STATE $\tilde \Delta_g\gets \min\{\tilde\Delta_{g,1},\tilde\Delta_{g,2}\}$
\STATE $\tilde D \gets \min\{|det(M_{\bar B^\top})|,|det(M_{\bar B})|\}$
\IF{$\tilde\Delta_g\geq \frac{5k\tilde\Delta}{\tilde D}+2\Delta$} \label{con:sub:line2}
\STATE Return $B$ as the optimal sub-matrix
\ENDIF
\ENDIF
\ENDFOR
\end{algorithmic}
\label{alg-nxm-instance}
\end{algorithm}

\textbf{Equivalent closed form solution of NE.} In this section we work with an equivalent closed form solution of NE. First, we define a few matrices. For any $n\times n$ matrix $\tilde A$ with support NE, we define $M_{\tilde A}$ as follows--for all $i\in[n-1]$ and $j\in [n]$ $(M_{\tilde A})_{i,j}=\tilde A_{1,j}-\tilde A_{i+1,j}$ and for all $j\in[n]$ $(M_{\tilde A})_{n,j}=1$. For any $\ell\in[n]$, we also define $M_{\tilde A,\ell}$ as follows-- for $i\in[n]$ and $j\in[n]\setminus\{\ell\}$ $(M_{\tilde A,\ell})_{i,j}=(M_{\tilde A})_{i,j}$ and $(M_{\tilde A,\ell})_{i,\ell}=0$ if $i\neq n$ and $(M_{\tilde A,\ell})_{i,\ell}=1$ if $i=n$. Let $b=(0,\ldots,0,1)^\top$. Observe that we get $M_{\tilde A,i}$ by replacing the $i$-th column of $M_{\tilde A}$ with the column vector $b$.  If $\det(M_{\tilde A})\neq 0$, then by Cramer's rule, the system of linear equations $M_Az=b$ has a unique solution $z^*$ where  $z_i^*=\frac{\det(M_{\tilde A,i})}{\det(M_{\tilde A})}$. 

Now observe that $x^*$ and $y^*$ must satisfy the system of linear equations
$M_{\tilde A^\top}x=b$ and $M_{\tilde A}y=b$ respectively. This implies that $x_i^*=\frac{\det(M_{\tilde A^\top,i})}{\det(M_{\tilde A^\top})}$ and $y_j^*=\frac{\det(M_{\tilde A,j})}{\det(M_{\tilde A})}$. Observe that $\tilde A$ has a unique full support Nash equilibrium $(x^*,y^*)$ if and only if for all $i,j\in[n]$ ${\det(M_{\tilde A^\top,i})}\cdot{\det(M_{\tilde A^\top})}>0$ and ${\det(M_{\tilde A,j})}\cdot {\det(M_{\tilde A})}>0$.

\textbf{Important matrix-dependent parameters.} 
If $k:=|\supp(x^*)|<n$, let $\Delta_{g,1}:=V_A^*-\max_{i\notin \supp(x^*)}\sum_{j=1}^my^*_j A_{i,j}$ otherwise let $\Delta_{g,1}=\infty$. If $|\supp(y^*)|<m$, let $\Delta_{g,2}:=\min_{j\notin \supp(y^*)}\sum_{i=1}^nx'_i A_{i,j}-V_A^*$ otherwise let $\Delta_{g,2}=\infty$. Let $\Delta_g=\min\{\Delta_{g,1},\Delta_{g,2}\}$. We now define the parameters $\Delta_{\min}$ as:
\begin{align*}
    %&\Delta_{\min}=\min\{|det(M_{ A^T})|,|det(M_{ A})|,\min_{i\in[n]}|det(M_{ A^T,i})|,\min_{j\in[n]}|det(M_{ A,j})|\}\\
    &\Delta_{\min}=\min\left\{\min_{i\in[k]}|\det(M_{ B^\top,i})|,\min_{j\in[k]}|\det(M_{ B,j})|\right\}.
\end{align*}
where $B$ is the $k\times k$ optimal sub-matrix with row indices $\supp(x^*)$ and column indices $\supp(y^*)$.

\textbf{Main Result.} Let $\bar A_{ij,t}$ denote the element $(i,j)$ of $\bar A$ at the end of round $t$. Let us now define an event $G$ as follows:
\begin{align*}
    G:=&\bigcap_{t=1}^T\bigcap_{i=1}^n \bigcap_{j=1}^m \{ |A_{ij}-\bar A_{ij,t}|\leq \sqrt{\tfrac{2\log(nmT^2)}{t}} \}.
\end{align*}
Observe that event $G$ holds with probability at least $1-\frac{8}{T}$. We now present our main result.
\begin{theorem}\label{thm:nxm:full}
    Consider a matrix game on $A\in [0,1]^{n\times n}$ that has a unique full support Nash equilibrium $(x^*,y^*)$. If the event $G$ holds, then  Algorithm \ref{alg-nxm-instance} identifies the optimal $k\times k$ sub-matrix in $t'\leq \frac{c_1\cdot(k^2\cdot k!)^2\log(nmT^2)}{\min\{1,\Delta_{\min}^2\}\cdot\min\{1,\Delta_g^2\}}$ time steps where $c_1$ is an absolute constant.
    Moreover if the event $G$ holds, then the regret $R(t')=t'\cdot V_A^*-\sum_{t=1}^{t'} \langle x_t,Ae_{j_t}\rangle$ incurred by Algorithm \ref{alg-nxm-instance} is $\frac{c_2\cdot k^2\cdot k!\cdot\log(nmT)}{\min\{1,\Delta_{\min}\}\cdot\min\{1,\Delta_g\}}$ where $c_2$ is an absolute constant.
\end{theorem}

We now begin analysing the Algorithm \ref{alg-nxm-instance}. \Cref{det:lem} plays a crucial role in this. In section \ref{appendix:game-theory} we present few game theoretic lemmas. We then use these lemmas to prove Theorem \ref{thm:nxm:full} in section \ref{appendix:nxm-lems}
\subsubsection{Game theoretic Lemmas}\label{appendix:game-theory}
First we present the following lemma.
\begin{lemma}[\cite{bohnenblust1950solutions}]\label{shapley:lem1}
    Let the matrix game on $A\in\mathbb{R}^{n\times m}$ have a unique NE $(x^*,y^*)$. Then the following holds:
    \begin{itemize}
        \item $|\supp(x^*)|=|\supp(y^*)|$
        \item If $|\supp(x^*)|<n$, $V_A^*-\max_{i\notin \supp(x^*)}\sum_{j=1}^my^*_j A_{i,j}>0$
        \item If $|\supp(y^*)|<m$, $\min_{j\notin \supp(y^*)}\sum_{i=1}^nx'_i A_{i,j}-V_A^*>0$
    \end{itemize}
\end{lemma}
Let $\simplex_{\mathcal I}$ denote the simplex over the set of indices $\mathcal I$. Now we present the following lemma.
\begin{lemma}[\cite{bohnenblust1950solutions}]\label{shapley:lem2}
    Let the matrix game on $A\in\mathbb{R}^{n\times m}$ have a unique NE $(x^*,y^*)$. Let $\mathcal I=\supp(x^*)$ and $\mathcal J=\supp(y^*)$. Let $(x,y)\in \simplex_{\mathcal{I}}\times \simplex_{\mathcal{J}} $ such that for all $i\in\mathcal{I}$, $x_i=x_i^*$ and for all $j\in\mathcal{J}$, $y_i=y_i^*$. Then $(x,y)$ is the unique NE of the sub-matrix of $A$ with row indices $\mathcal{I}$ and column indices $\mathcal{J}$.
\end{lemma}

In this section we deal with a matrix game on $A\in\mathbb{R}^{n\times m}$ that has a unique NE $(x^*,y^*)$. If $|\supp(x^*)|<n$, recall that $\Delta_{g,1}:=V_A^*-\max_{i\notin \supp(x^*)}\sum_{j=1}^my^*_j A_{i,j}$ otherwise $\Delta_{g,1}=\infty$. If $|\supp(y^*)|<m$, recall that $\Delta_{g,2}:=\min_{j\notin \supp(y^*)}\sum_{i=1}^nx'_i A_{i,j}-V_A^*$ otherwise $\Delta_{g,2}=\infty$. Recall that $\Delta_g=\min\{\Delta_{g,1},\Delta_{g,2}\}$. 

% \begin{lemma}\label{shapley:lem2}
%     Consider a matrix game $A\in\mathbb{R}^{m\times n}$. Let $X=\{x\in \Delta_m:V_A^*=\min_{y\in \Delta_n}\langle x, Ay \rangle\}$ and $Y=\{y\in \Delta_n:V_A^*=\max_{x\in \Delta_m}\langle x, Ay\rangle\}$. Let $I_x=\{j: V_A^*=\langle x, A e_j^n\rangle\}$ and $I_y=\{i: V_A^*=\langle e_i^m, A y\rangle\}$. Then we have the following:
%     \begin{itemize}
%         \item $\bigcup_{x\in X}\supp(x)=\bigcap_{y\in Y}I_y$
%         \item $\bigcup_{y\in Y}\supp(y)=\bigcap_{x\in X}I_x$
%     \end{itemize}
% \end{lemma}

\begin{lemma}\label{nxm:game:suboptimal}
    Consider a matrix game on $A\in\mathbb{R}^{n\times m}$. Let $(x',y')$ be a unique full support NE of a square submatrix $B$ of $A$. If $(x',y')$ is not a Nash Equilibrium of $A$, then either $V_B^*-\max_{i\notin \supp(x')}\sum_{j=1}^my'_j A_{i,j}<0$ or $\min_{j\notin \supp(y')}\sum_{i=1}^nx'_i A_{i,j}-V_B^*<0$.
\end{lemma}
\begin{proof}
    First observe that $\langle x', Ay'\rangle=V_B^*$. For all $i\in \supp(x')$, $\sum_{j=1}^my'_j A_{i,j}=\langle x', Ay'\rangle$. Similarly for all $j\in \supp(y')$, $\sum_{i=1}^nx'_i A_{i,j}=\langle x', Ay'\rangle$.  Now as $(x',y')$ is not a NE of $A$, there exists some indices $i'\notin \supp(x')$, $j'\notin \supp(y')$ such that either $\sum_{j=1}^my'_j A_{i',j}>\langle x', Ay'\rangle$ or $\sum_{i=1}^nx'_i A_{i,j'}<\langle x', Ay'\rangle$.
\end{proof}

\begin{lemma}\label{nxm:game:unique-full}
    Consider a matrix game on $A\in\mathbb{R}^{n\times n}$. Let $A$ have a full support NE $(x^*,y^*)$. Let $det(M_A)\neq 0$ and $det(M_{A^T})\neq 0$. Then $(x^*,y^*)$ is the unique NE of $A$
\end{lemma}
\begin{proof}
    Consider any NE $(x',y')$ of $A$. As $(x',y^*)$ is also a Nash Equilibrium of $A$, for all $j\in [n]$ we have $\sum_{i\in[n]}A_{i,j}x_i'=V_A^*$. Hence, $x'$ is a solution of the system of linear equations $M_{A^T}x=b$ where $b=(0,0,\ldots,1)^T$.Similarly as $(x^*,y')$ is also a Nash Equilibrium of $A$, for all $i\in [n]$ we have $\sum_{j\in[n]}A_{i,j}y_j'=V_A^*$. Hence, $y'$ is a solution of the system of linear equations $M_{A}y=b$ where $b=(0,0,\ldots,1)^T$. However, as $det(M_A)\neq 0$ and $det(M_{A^T})\neq 0$, the systems of linear equations $M_{A^T}x=b$ and $M_{A}y=b$ have unique solutions. Hence we have $(x',y')=(x^*,y^*)$. 
\end{proof}

For the following lemma, we do a small notational abuse. If we talk about vector $v$ from a $p$-dimensional simplex  in the context of $d$-dimensional simplex where $p<d$, we are actually talking about the vector that we get by appending $d-p$ zeroes to $v$. 
\begin{lemma}[Formal version of \Cref{lem:main-k*k-check}]\label{nxm:game:unique-overall}
    Consider a matrix game on $A\in\mathbb{R}^{n\times m}$. Consider a $k\times k$ submatrix $B$ of $A$. Let $B$ have a unique full support NE $(x^*,y^*)$. If $\max_{i\notin \supp(x^*)}\sum_{j\in \supp(y^*)}A_{i,j}y_j^*<V_B^*$ or $k=n$ and $\min_{j\notin \supp(y^*)}\sum_{i\in \supp(x^*)}A_{i,j}x_i^*>V_B^*$ or $k=m$, then $(x^*,y^*)$ is the unique NE of $A$.
\end{lemma}
\begin{proof}
    We prove the lemma for the case when $k<n$ and $k<m$. The proof for the other cases are identical.
    First observe that $(x^*,y^*)$ is a Nash Equilibrium of the matrix game on $A$ as no player can gain by deviating unilaterally. Now consider a Nash Equilibrium $(x,y)$ of the matrix game on $A$. Now observe that $(x,y^*)$ is also a Nash Equilibrium of $A$. This along with the fact that $\max_{i\notin \supp(x^*)}\sum_{j\in \supp(y^*)}A_{i,j}y_j^*<V_B^*=\sum_{j\in \supp(y^*)}A_{i',j}y_j^*$ for any $i'\in \supp(x^*)$ implies that $\supp(x)\subseteq \supp(x^*)$. Similarly, we can show that $\supp(y)\subseteq \supp(y^*)$. Now observe that $(x,y)$ is also a Nash Equilibrium of the matrix game on $B$. Since $B$ has a unique NE, we have $(x,y)=(x^*,y^*)$.
\end{proof}

\begin{lemma}[Formal version of \Cref{lem:main-gap-concentration}]\label{nxm:game:deltag}
    Consider two matrices $A,\bar A\in [0,1]^{n\times m}$ such that for all pairs $(i,j)$, $|A_{ij}-\bar A_{ij}|\leq \Delta_1$. Let $B$ and $\bar B$ be two $k\times k$ sub-matrices of $A$ and $\bar A$ respectively with same row and column indices. Let $(x^*,y^*)$ and $(x',y')$ be the unique full support NE of $B$ and $B'$ respectively. Let $||x^*-x'||_\infty \leq\Delta_2$ and $||y^*-y'||_\infty \leq\Delta_2$. Then the following inequalities hold for any $i'\in[n]$ and $j'\in[m]$:
    \begin{align*}
    V_B^*-\sum_{j\in \supp(y^*)} A_{i',j}y_j^*-(2\Delta_1+k\cdot\Delta_2)&\leq V_{\bar B}^*-\sum_{j\in \supp(y')}\bar A_{i',j}y_j'\\
    &\leq V_B^*-\sum_{j\in \supp(y^*)} A_{i',j}y_j^*+(2\Delta_1+k\cdot \Delta_2)
    \end{align*}
    and
    \begin{align*}
    \sum_{i\in \supp(x^*)}A_{i,j'}x_i^*-V_B^*-(2\Delta_1+k\cdot\Delta_2) &\leq \sum_{i\in \supp(x')}\bar A_{i,j'}x_i'-V_{\bar B}^*\\
    &\leq \sum_{i\in \supp(x^*)}A_{i,j'}x_i^*-V_B^*+(2\Delta_1+k\cdot\Delta_2)
    \end{align*}
\end{lemma}
\begin{proof}
    Let $A_{i}^r$ denote the $i$-th row of the matrix $A$. Now consider an index $i\in \supp(x^*)$ and an index $i'\in [n]$. Now we have the following:
    \begin{align*}
        &V_{\bar B}^*-\sum_{j\in \supp(y')}\bar A_{i',j}y_j'\\
        &=\langle y',\bar A_i^r\rangle -\langle y',\bar A_{i'}^r\rangle\\
        &=\langle y'-y^*,\bar A_i^r\rangle+\langle y^*,\bar A_i^r\rangle-\langle y'-y^*,\bar A_{i'}^r\rangle-\langle y^*,\bar A_{i'}^r\rangle\\
        &=\langle y'-y^*,\bar A_i^r-\bar A_{i'}^r\rangle+\langle y^*,\bar A_i^r-A_i^r\rangle+\langle y^*,A_i^r\rangle-\langle y^*,\bar A_{i'}^r-A_{i'}^r\rangle-\langle y^*,A_{i'}^r\rangle\\
        &=V_B^*-\sum_{j\in \supp(y^*)} A_{i',j}y_j^*+\langle y'-y^*,\bar A_i^r-\bar A_{i'}^r\rangle+\langle y^*,\bar A_i^r-A_i^r-(\bar A_{i'}^r-A_{i'}^r)\rangle
    \end{align*}
    As $-k\cdot \Delta_2\leq \langle y'-y^*,\bar A_i^r-\bar A_{i'}^r\rangle\leq k\cdot \Delta_2$ and $-2\Delta_1\leq \langle y^*,\bar A_i^r-A_i^r-(\bar A_{i'}^r-A_{i'}^r)\rangle\leq 2\Delta_1$, we get the following:
    \begin{align*}
        V_B^*-\sum_{j\in \supp(y^*)} A_{i',j}y_j^*-(2\Delta_1+k\cdot\Delta_2)&\leq V_{\bar B}^*-\sum_{j\in \supp(y')}\bar A_{i',j}y_j'\\
        &\leq V_B^*-\sum_{j\in \supp(y^*)} A_{i',j}y_j^*+(2\Delta_1+k\cdot \Delta_2)
    \end{align*}
    Using a similar analysis, one can show the following:
    \begin{align*}
        \sum_{i\in \supp(x^*)}A_{i,j'}x_i^*-V_B^*-(2\Delta_1+k\cdot\Delta_2)&\leq \sum_{i\in \supp(x')}\bar A_{i,j'}x_i'-V_{\bar B}^*\\
        &\leq \sum_{i\in \supp(x^*)}A_{i,j'}x_i^*-V_B^*+(2\Delta_1+k\cdot\Delta_2)
    \end{align*}
\end{proof}
\subsubsection{Consequential lemmas of Algorithm~\ref{alg-nxm-instance}'s conditional statements}\label{appendix:nxm-lems}
Now we begin the analysis of the Algorithm~\ref{alg-nxm-instance}. For the rest of the section, let us define $\Delta_{\min,B}$ for any square matrix as follows:
\begin{align*}
\Delta_{\min,B}= \min\{|det(M_{B^\top})|,|det(M_{ B})|,\min_{i\in[n]}|det(M_{B^\top,i})|,\min_{j\in[n]}|det(M_{B,j})|\}
\end{align*}
Let $(x^*,y^*)$ denote the unique NE of the matrix $A$. Let $B_*$ be the matrix with row indices $\supp(x^*)$ and column indices $\supp(y^*)$. We refer to this matrix as the optimal sub-matrix of $A$. If $B$ is the optimal-submatrix, then $\Delta_{\min}=\Delta_{\min,B}$.

Now we begin proving the following technical lemmas to analyse the algorithm \ref{alg-nxm-instance}.

\begin{lemma}
    Consider a timestep $t$ when the condition in the line \ref{con:sub:line1} of the algorithm \ref{alg-nxm-instance} is satisfied. Consider the $k\times k$ submatrix $\bar B$ at the timestep $t$. Let $B$ be the submatrix of $A$ with same row indices and column indices as that of $\bar B$. If the event $G$ holds, then $\frac{4}{5}\leq \frac{\Delta_{\min,B}}{\tilde \Delta_{\min}}\leq \frac{6}{5}$.
\end{lemma}
\begin{proof}
    Let $\bar\Delta:=k^2\cdot k!\cdot\Delta$. As the condition $1\leq \frac{\tilde \Delta_{\min}+2\bar\Delta}{\tilde \Delta_{\min}-2\bar\Delta}\leq \frac{3}{2}$ holds true, we have $\bar\Delta\leq \frac{\tilde \Delta_{\min}}{10}$. Now due to \Cref{det:lem}, we have that $\frac{\Delta_{\min,B}}{\tilde \Delta_{\min}}\leq \frac{\tilde \Delta_{\min}+2\bar\Delta}{\tilde \Delta_{\min}}\leq \frac{\tilde \Delta_{\min}+\tilde \Delta_{\min}/5}{\tilde \Delta_{\min}}\leq\frac{6}{5}$. Similarly we have that $\frac{\Delta_{\min,B}}{\tilde \Delta_{\min}}\geq \frac{\tilde \Delta_{\min}-2\bar\Delta}{\tilde \Delta_{\min}}\geq \frac{\tilde \Delta_{\min}-\tilde \Delta_{\min}/5}{\tilde \Delta_{\min}}\geq\frac{4}{5}$. 
\end{proof}

\begin{lemma}\label{nxm:nash}
    Let $j_t$ be the column played by $y$-player in round $t$. Let $(x',y')$ be the NE of the empirical matrix $\bar A$. If event $G$ holds, then $V_A^*-\langle x',Ae_{j_t}\rangle \leq \min\{2\sqrt{\frac{2\log(nmT^2)}{t-1}},1\}$
\end{lemma}
\begin{proof}
    As $V_A^*\leq 1$ and $\langle x',Ae_{j_t}\rangle\geq 0$, we have $V_A^*-\langle x',Ae_{j_t}\rangle \leq 1$. Let $\Delta_t=\sqrt{\frac{2\log(nmT^2)}{t-1}}$. Let $(x^*,y^*)$ be the NE of $A$. As event $G$ holds, we now have the following for any $t\geq 2$: 
    \begin{align*}
        \langle x', Ae_{j_t}\rangle &\geq \langle x', \bar Ae_{j_t}\rangle-\Delta_t\\
        &\geq\langle x',\bar Ay'\rangle-\Delta_t\\
        &\geq \langle x^*, \bar Ay'\rangle-\Delta_t\\
        &\geq \langle x^*, Ay'\rangle-2\Delta_t\\
        &\geq \langle x^*, Ay^*\rangle-2\Delta_t\\
        &= V_A^*-2\Delta_t
    \end{align*}
\end{proof}

\begin{corollary}\label{nxm:cor}
    Let the algorithm \ref{alg-nxn-full} play the NE of the empirical matrix $\bar A$ from round $1$ to round $t'$. If event $G$ holds, then the regret incurred till round $t$ is at most $4\sqrt{2t'\log(nmT^2)}+1$
\end{corollary}
\begin{proof}
    For all $t\in\{2,\ldots,t'\}$, we have $V_A^*-\langle x',Ae_{j_t}\rangle \leq 2\sqrt{\frac{2\log(nmT^2)}{t-1}}$ due to lemma \ref{nxm:nash}. Hence the regret recurred till round $t$ is $t\cdot V_A^*-\sum_{t=1}^{t'} \langle x',Ae_{j_t} \leq 1+2\sum_{t=2}^{t'}\sqrt{\frac{2\log(nmT^2)}{t-1}}\leq 4\sqrt{2t'\log(nmT^2)}+1$.
\end{proof}

\begin{lemma}\label{nxm:deviation}
    Consider a timestep $t$ when the condition in the line \ref{con:sub:line1} of the algorithm \ref{alg-nxm-instance} is satisfied. Consider the $k\times k$ submatrix $\bar B$ at the timestep $t$. Let $B$ be the submatrix of $A$ with same row indices and column indices as that of $\bar B$. Let $(x^*,y^*)$ and $(x',y')$ be the unique full support NE of $B$ and $\bar B$ respectively. If the event $G$ holds, then for any $i\in[n]$, $|x^*_i-x'_i|\leq \frac{5\tilde\Delta}{\tilde D}$ where $\tilde\Delta=k\cdot k!\cdot\sqrt{\frac{2\log(nmT^2)}{t}}$. Similarly if the event $G$ holds, then for any $j\in[m]$, $|y^*_j-y'_j|\leq \frac{5\tilde\Delta}{\tilde D}$ where $\tilde\Delta=k\cdot k!\cdot\sqrt{\frac{2\log(nmT^2)}{t}}$.
\end{lemma}
\begin{proof}
         W.l.o.g let us assume that $det(M_{B^T})>0$. Let $\tilde D_1=|det(M_{\bar B^T})|$ and $\tilde  D_2=|det(M_{\bar B})|$. Let us also assume that event $G$ holds. As the condition $1\leq \frac{\tilde \Delta_{\min}+2k\tilde\Delta}{\tilde \Delta_{\min}-2k\tilde\Delta}\leq \frac{3}{2}$ holds true, we have $\tilde\Delta\leq \frac{\tilde \Delta_{\min}}{10k}\leq \frac{\tilde D}{10k}$. We now have the following:
    \begin{align*}
        x_i^*&=\frac{det(M_{B^T,i})}{det(M_{B^T})}\\
        &\leq \frac{det(M_{\bar B^T,i})+2\tilde\Delta}{\tilde D_1-2\tilde\Delta}\tag{due to \Cref{det:lem}}\\
        &=\left(\frac{det(M_{\bar B^T,i})}{\tilde D_1}+\frac{2\tilde\Delta}{\tilde D_1}\right)\left(1-\frac{2\tilde\Delta}{\tilde D_1}\right)^{-1}\\
        &\leq\left(x_i'+\frac{2\tilde\Delta}{\tilde D_1}\right)\left(1+\frac{5\tilde\Delta}{2\tilde D_1}\right)\tag{as $(1-x)^{-1}\leq 1+\frac{5x}{4}$ for all $x\in[0,\frac{1}{5}]$}\\
        &= x_i'+\frac{2\tilde\Delta}{\tilde D_1}+x_i'\cdot \frac{5\tilde\Delta}{2\tilde D_1}+\frac{\tilde\Delta}{\tilde D_1}\cdot \frac{5\tilde\Delta}{\tilde D_1}\\
        &\leq x_i'+\frac{5\tilde \Delta}{\tilde D_1} \tag{as $\frac{\tilde \Delta}{\tilde D_1}\leq \frac{1}{10}$}
    \end{align*}
    Similarly we have that following:
    \begin{align*}
        x_i^*&=\frac{det(M_{B^T,i})}{det(M_{B^T})}\\
        &\geq \frac{det(M_{\bar B^T,i})-2\tilde\Delta}{\tilde D_1+2\tilde\Delta}\tag{due to \Cref{det:lem}}\\
        &=\left(\frac{det(M_{\bar B^T,i})}{\tilde D_1}-\frac{2\tilde\Delta}{\tilde D_1}\right)\left(1+\frac{2\tilde\Delta}{\tilde D_1}\right)^{-1}\\
        &\geq\left(x_i'-\frac{2\tilde\Delta}{\tilde D_1}\right)\left(1-\frac{2\tilde\Delta}{\tilde D_1}\right)\tag{as $(1+x)^{-1}\geq 1-x$ for all $x\geq 0$}\\
        &= x_i'-\frac{2\tilde\Delta}{\tilde D_1}-x_i'\cdot \frac{2\tilde\Delta}{\tilde D_1}+\frac{\tilde\Delta}{\tilde D_1}\cdot \frac{\tilde\Delta}{\tilde D_1}\\
        &\geq x_i'-\frac{4\tilde \Delta}{\tilde D_1} \tag{as $\frac{\tilde \Delta}{\tilde D_1}\leq \frac{1}{10}$}
    \end{align*}
    In an identical manner, we can show that for any $j\in[m]$, $|y^*_j-y'_j|\leq \frac{5\tilde\Delta}{\tilde D_2}$.
\end{proof}
\begin{lemma}\label{nxm:feasible}
    Consider a timestep $t$ when the condition in the line \ref{con:sub:line1} of the algorithm \ref{alg-nxm-instance} is satisfied. Consider the $k\times k$ submatrix $\bar B$ at the timestep $t$. Let $B$ be the submatrix of $A$ with same row indices and column indices as that of $\bar B$. Let $(x^*,y^*)$ and $(x',y')$ be the unique full support NE of $B$ and $\bar B$ respectively. If event $G$ holds, then for any $i\in[k]$, $x_i'-\frac{5k\tilde \Delta}{\tilde D}> 0$ and $x_i'+\frac{5k\tilde\Delta}{\tilde D_1}< 1$ where $\tilde\Delta=k\cdot k!\cdot\sqrt{\frac{2\log(nmT^2)}{t}}$ and $\tilde D_1=|\det(M_{\bar B^T})|$. 
\end{lemma}
\begin{proof}
    W.l.o.g let us assume that $det(M_{\bar B^T})>0$. Let us also assume that event $G$ holds. As the condition $1\leq \frac{\tilde \Delta_{\min}+2k\tilde\Delta}{\tilde \Delta_{\min}-2k\tilde\Delta}\leq \frac{3}{2}$ holds true, we have $\tilde\Delta\leq \frac{\tilde \Delta_{\min}}{10k}\leq \frac{\tilde D}{10k}$. For any $i\in[k]$, $\det(M_{\bar B^T,i})-5k\tilde\Delta\geq \det(M_{\bar B^T,i})-\frac{\tilde \Delta_{\min}}{2}\geq \frac{\det(M_{\bar B^T,i})}{2}>0$. Hence $x_i'-\frac{5k\tilde \Delta}{\tilde D_1}> 0$. Fix any $i\in[k]$. Now we have $x_i'+\frac{5k\tilde\Delta}{\tilde D_1}=1-\sum\limits_{j\in[k]\setminus\{i\}}x_j'+\frac{5k\tilde\Delta}{\tilde D_1}< 1-(k-1)\cdot \frac{5k\tilde \Delta}{\tilde D}+\frac{5k\tilde \Delta}{\tilde D_1}\leq 1 $ 
\end{proof}    

\begin{lemma}\label{lem:nxm:wrong}
    Consider a timestep $t$ when the condition in the line \ref{con:sub:line1} of the algorithm \ref{alg-nxm-instance} is satisfied. Consider the $k\times k$ submatrix $\bar B$ at the timestep $t$. Let $B$ be the submatrix of $A$ with same row indices and column indices as that of $\bar B$. If the event $G$ holds and $B$ is not the optimal sub-matrix of $A$, then $\tilde \Delta_g<\frac{5k\tilde\Delta}{\tilde D}+2\Delta$.
\end{lemma}
\begin{proof}
    Let $(x,y)$ and $(x',y')$ be the unique full support NE of $B$ and $\bar B$ respectively. Due to lemma \ref{nxm:game:suboptimal}, either $V_B^*-\max_{i\notin \supp(x)}\sum_{j=1}^my_j A_{i,j}<0$ or $\min_{j\notin \supp(y)}\sum_{i=1}^nx_i A_{i,j}-V_B^*<0$.  Now due to lemma \ref{nxm:game:deltag} and lemma \ref{nxm:deviation}, we have that either $V_{\bar B}^*-\max_{i\notin \supp(x')}\sum_{j=1}^my'_j \bar A_{i,j}<\frac{5k\tilde\Delta}{\tilde D}+2\Delta$ or $\min_{j\notin \supp(y')}\sum_{i=1}^nx'_i \bar A_{i,j}-V_B^*<\frac{5k\tilde\Delta}{\tilde D}+2\Delta$. Hence $\tilde \Delta_g<\frac{5k\tilde\Delta}{\tilde D}+2\Delta$.
\end{proof}

\begin{lemma}\label{nxm:deltag:result}
    Consider the timestep $t$ when the condition on the line \ref{con:sub:line2} of the algorithm \ref{alg-nxm-instance} gets satisfied. If the event $G$ holds, then $t\leq \frac{800(k^2\cdot k!)^2\log(nmT^2)}{\min\{1,\Delta_{\min}^2\}\cdot\min\{1,\Delta_g^2\}}$ and the optimal sub-matrix of $A$ is returned.
\end{lemma}
\begin{proof}
    Let $B$ be a sub-matrix of $\bar A$ with row indices $\supp(x^*)$ and column indices $\supp(y^*)$. Let $\tilde \Delta_{\min}$ be defined with respect to the matrix $B$.
    {Consider the timestep $t'=\frac{800(k^2\cdot k!)^2\log(nmT^2)}{\min\{1,\Delta_{\min}^2\}\cdot\min\{1,\Delta_g^2\}}$.} Let $\Delta=\sqrt{\frac{2\log(nm\cdot T^2)}{t'}}$. Let $\bar\Delta:=k^2\cdot k!\cdot\Delta$. First observe that $\bar\Delta=k\cdot \tilde \Delta=k^2\cdot k!\cdot\sqrt{\frac{2\log(nmT^2)}{t'}}\leq \frac{\min\{1,\Delta_{\min}\}\cdot\min\{1,\Delta_g\}}{20}\leq \frac{\Delta_{\min}}{20}$. Next observe that $\tilde \Delta_{\min}-2\bar\Delta\geq \Delta_{\min}-4\bar\Delta>0$. Hence we have $1\leq \frac{\tilde \Delta_{\min}+2\bar\Delta}{\tilde \Delta_{\min}-2\bar\Delta}$. Due to \Cref{det:lem}, we have $\frac{\tilde \Delta_{\min}+2\bar\Delta}{\tilde \Delta_{\min}-2\bar\Delta}\leq \frac{\Delta_{\min}+4\bar\Delta}{\Delta_{\min}-4\bar\Delta}\leq\frac{\Delta_{\min}+\Delta_{\min}/5}{\Delta_{\min}-\Delta_{\min}/5}=\frac{3}{2}$. 

    Consider the gap parameter $\tilde \Delta_g$ with respect to the sub-matrix $B$. Now we have the following:
    \begin{align*}
        4\Delta+\frac{10k\tilde \Delta}{\tilde D}&\leq \frac{4\min\{1,\Delta_{\min}\}\cdot\min\{1,\Delta_g\}}{20}+\frac{10\min\{1,\Delta_{\min}\}\cdot\min\{1,\Delta_g\}}{20\tilde \Delta_{\min}}\tag{as $\tilde D\leq \tilde\Delta_{\min}$}\\
        &\leq  \frac{4\Delta_g}{20}+\frac{12\Delta_g}{20} \tag{as $\frac{\Delta_{\min}}{\tilde\Delta_{\min}}\leq\frac{6}{5}$}\\
        &<\Delta_g
    \end{align*}
    Now using the above analysis and applying lemma \ref{nxm:game:deltag} and lemma \ref{nxm:deviation}, we get $\tilde\Delta_g\geq \Delta_g-\frac{5k\tilde \Delta}{\tilde D}-2\Delta\geq \frac{5k\tilde \Delta}{\tilde D}+2\Delta$.  Due to lemma \ref{nxm:game:unique-full} and lemma \ref{nxm:game:unique-overall}, $B$ is the only sub-matrix that satisfies the condition on the line \ref{con:sub:line1} of the algorithm \ref{alg-nxm-instance} at the timestep $t'$. Moreover, the condition on the line \ref{con:sub:line2} of the algorithm \ref{alg-nxm-instance} is also satisfied at the timestep $t'$.
    Hence due to the above analysis and due to lemma \ref{lem:nxm:wrong}, $t\leq \frac{800(k^2\cdot k!)^2\log(nmT^2)}{\min\{1,\Delta_{\min}^2\}\cdot\min\{1,\Delta_g^2\}}$ and  the optimal sub-matrix of $A$ is returned.

\end{proof}

Theorem \ref{thm:nxm:full} follows directly from lemma \ref{nxm:deltag:result} and corollary \ref{nxm:cor}.

\noindent
\textbf{Remark:} We get instance dependent poly-logarithmic regret due to Theorem \ref{thm:nxm:full} and Theorem \ref{thm:nxn:full}.

\subsection{Algorithm for $n\times m$ matrix with full row support}\label{appendix:nxn}
In this section, we describe the regret minimization algorithm (Algorithm \ref{alg-nxn-full}) for $n\times m$ matrix $A$ with unique Nash Equilibrium $(x^*,y^*)$ such that $\supp(x^*)=[n]$ and $\supp(y^*)=[n]$. Recall that $A\in[-1,1]^{n\times m}$ and $ \mathbf{A}_t\in [-1,1]^{n\times m}$. Moreover, for any matrix $M\in [-1,1]^{k\times k}$ we say a pair $(x',y')$ is a full support NE of $M$ if $|\supp(x')|=|\supp(y')|=k$. We translate and re-scale these matrices by adding $1$ to every entry and then dividing every entry by $2$. Hence for the rest of the section, we assume that  $A\in[0,1]^{n\times m}$ and $\mathbf{A}_t\in [0,1]^{n\times m}$. Note that 
the regret only changes by a factor of $2$.

\textbf{Equivalent closed form solution of NE.} In this section we work with an equivalent closed form solution of NE. First, we define a few matrices. For any $n\times n$ matrix $\tilde A$ with support NE, we define $M_{\tilde A}$ as follows--for all $i\in[n-1]$ and $j\in [n]$ $(M_{\tilde A})_{i,j}=\tilde A_{1,j}-\tilde A_{i+1,j}$ and for all $j\in[n]$ $(M_{\tilde A})_{n,j}=1$. For any $\ell\in[n]$, we also define $M_{\tilde A,\ell}$ as follows-- for $i\in[n]$ and $j\in[n]\setminus\{\ell\}$ $(M_{\tilde A,\ell})_{i,j}=(M_{\tilde A})_{i,j}$ and $(M_{\tilde A,\ell})_{i,\ell}=0$ if $i\neq n$ and $(M_{\tilde A,\ell})_{i,\ell}=1$ if $i=n$. Let $b=(0,\ldots,0,1)^\top$. Observe that we get $M_{\tilde A,i}$ by replacing the $i$-th column of $M_{\tilde A}$ with the column vector $b$.  If $\det(M_{\tilde A})\neq 0$, then by Cramer's rule, the system of linear equations $M_Az=b$ has a unique solution $z^*$ where  $z_i^*=\frac{\det(M_{\tilde A,i})}{\det(M_{\tilde A})}$. 

Now observe that $x^*$ and $y^*$ must satisfy the system of linear equations
$M_{\tilde A^\top}x=b$ and $M_{\tilde A}y=b$ respectively. This implies that $x_i^*=\frac{\det(M_{\tilde A^\top,i})}{\det(M_{\tilde A^\top})}$ and $y_j^*=\frac{\det(M_{\tilde A,j})}{\det(M_{\tilde A})}$. Observe that $\tilde A$ has a unique full support Nash equilibrium $(x^*,y^*)$ if and only if for all $i,j\in[n]$ ${\det(M_{\tilde A^\top,i})}\cdot{\det(M_{\tilde A^\top})}>0$ and ${\det(M_{\tilde A,j})}\cdot {\det(M_{\tilde A})}>0$.

\textbf{Important matrix-dependent parameters.} We now define two parameters $\Delta_{\min}$ and $D$ for the optimal $n\times n$ matrix $B$ as follows:
\begin{align*}
    %&\Delta_{\min}=\min\{|det(M_{ A^T})|,|det(M_{ A})|,\min_{i\in[n]}|det(M_{ A^T,i})|,\min_{j\in[n]}|det(M_{ A,j})|\}\\
    &\Delta_{\min}=\min\left\{\min_{i\in[n]}|\det(M_{ B^\top,i})|,\min_{j\in[n]}|\det(M_{ B,j})|\right\}; \quad D=\det(M_{ B^\top}).
\end{align*}

%\textbf{Remark:} Due to corollary \ref{nxn:cor}, Algorithm \ref{alg-nxn-full} has a trivial upper bound of $O(\sqrt{T\log(nT)})$ for any $n\times n$ input matrix.

\begin{algorithm}[t!]
\caption{Logarithmic regret algorithm for $n\times m$ games with unique Nash equilibrium with full row support}
\begin{algorithmic}[1]
\STATE \textbf{Given:} $\supp(x^*)=[n]$ and $\supp(y^*)=[n]$
\STATE Let $B$ be the sub-matrix of $A$ with row indices $\supp(x^*)$ and column indices $\supp(y^*)$.
\STATE $x_1\gets (1/n,\ldots,1/n)$.
\STATE $t_*\gets T+1$
\FOR{time step $t=1,2,\ldots,T$} 
\STATE Play $x_t$ and update the empirical means $\bar A_{ij}$.
\STATE Let $(x',y')$ be the Nash equilibrium of $\bar A$ and $x_{t+1}\gets x'$.
\STATE $\Delta \gets \sqrt{\frac{2\log(nm\cdot T^2)}{t}}$
\STATE $\tilde\Delta_{\min}\gets \min\{|\det(M_{\bar B^\top})|,|\det(M_{\bar B})|,\min_{i\in \supp(x^*)}|\det(M_{\bar B^\top,i})|,\min_{j\in\supp(y^*)}|\det(M_{\bar B,j})|\}$
\IF{$1\leq \frac{\tilde \Delta_{\min}+2n^2\cdot n!\cdot\Delta}{\tilde \Delta_{\min}-2n^2\cdot n!\cdot\Delta}\leq \frac{3}{2}$ and $t_*>T$}
\STATE $t_*\gets6.25\cdot t$
\ENDIF
\IF{$t\geq t_*$ and the matrix game $\bar B$ (empirical version of $B$) has a unique full support NE}\label{alg:nxn:con1}
\STATE $t_1\gets t$.
\WHILE{$t_1<T$}
\STATE Let $(x',y')$ be the NE of $\bar A$
\STATE $\widehat A\gets \bar A$, $\tilde D\gets |\det(M_{\bar B^\top})|$ and $\Delta \gets \sqrt{\frac{2\log(nm\cdot T^2)}{t_1}}$
\STATE $D_1\gets \frac{\tilde D}{5n\cdot n!}$, $T_1\gets\left(\frac{1}{\Delta}\right)^2$ and $T_2\gets \min\{t_1,T-t_1\}$
\STATE Run the Algorithm \ref{subroutine-nxn} with input parameters $x'$, $D_1$ $T_1$, $T_2$, $\widehat A$.\label{alg:nxn:st1}
\STATE Update the empirical means $\bar A_{ij}$.
\STATE $t_1\gets 2t_1$
\ENDWHILE
\STATE Terminate the Algorithm
\ELSIF{$t\geq t_*$ }\label{alg:nxn:con2} 
\FOR{time step $t'=t+1,t+2,\ldots,T$} 
\STATE Play $x_{t'}$ and update the empirical means $\bar A_{ij}$.
\STATE Let $(x',y')$ be the Nash equilibrium of $\bar A$ and $x_{t'+1}\gets x'$.
\ENDFOR
\STATE Terminate the algorithm
\ENDIF
\ENDFOR
\end{algorithmic}
\label{alg-nxn-full}
\end{algorithm}

\textbf{Main Nash regret result.} We establish the regret guarantee of Algorithm \ref{alg-nxn-full} by proving the following theorem.

\begin{theorem}\label{thm:nxn:full}
    Consider a matrix game on $A\in [0,1]^{n\times m}$ that has a unique full support Nash Equilibrium $(x^*,y^*)$. Then the regret $R(A,T)=T\cdot V_A^*-\sum_{t=1}^T\mathbb{E}[\langle x_t,Ay_t\rangle]$ incurred by the algorithm \ref{alg-nxn-full} is $\frac{c_1\cdot n^2\cdot n!\cdot\log(nmT)}{\Delta_{\min}}+\frac{c_2\cdot n^2\cdot n!\cdot(\log(nmT))^2}{|D|}$ where $c_1$ and $c_2$ are absolute constants.
\end{theorem}
\begin{proof}[Proof of Theorem \ref{thm:nxn:full}]
Let $\bar A_{ij,t}$ denote the $(i,j)$ element of $\bar A$ at the end of round $t$. Define an event $G$ as follows:
\begin{align*}
    G:=&\bigcap_{t=1}^T\bigcap_{i=1}^n \bigcap_{j=1}^n \{ |A_{ij}-\bar A_{ij,t}|\leq \sqrt{\tfrac{2\log(nmT^2)}{t}} \}.
\end{align*}
Observe that event $G$ holds with probability at least $1-\frac{8}{T}$. 

Let $R(T)=T\cdot V_A^*-\sum\limits_{t=1}^T\langle x_t,Ae_{j_t} \rangle$ denote the regret incurred where $x_t$ is the mixed strategy played by row player in round $t$ and $j_t$ is the column played by the column player. If we show that under event $G$, we have that $R(T)=O\left(\frac{n^2\cdot n!\cdot\log(nmT)}{\Delta_{\min}}+\frac{n^2\cdot n!\cdot(\log(nmT))^2}{|D|}\right)$, then due to \eqref{equivalent:regret}, we get that the regret incurred by the algorithm is \[O\left(\frac{n^2\cdot n!\cdot\log(nmT)}{\Delta_{\min}}+\frac{n^2\cdot n!\cdot(\log(nmT))^2}{|D|}\right).\]

Let us assume that event $G$ holds. First observe that due to Lemma \ref{nxn:t*} and Lemma \ref{nxn:t-condition}, it takes $O\left(\frac{(n^2\cdot n!)^2\cdot\log(nmT)}{\Delta_{\min}^2}\right)$ time steps to satisfy the condition in  Line \ref{alg:nxn:con1} of  Algorithm \ref{alg-nxn-full}. Due to Corollary \ref{nxn:cor}, we have that the regret incurred up to this time step is at most $O(\frac{n^2\cdot n!\cdot \log(nmT)}{\Delta_{\min}}).$  Each time we call the subroutine in algorithm \ref{subroutine-nxn}, we incur a regret of $O(\frac{n^2\cdot n!\cdot\log(nmT)}{|D|})$ (cf.~Lemma \ref{nxn:regret}). Due to the doubling trick, the subroutine gets called for at most $\log_2T$ times (cf.~Lemma \ref{nxn:logT}). Hence, the total regret $R(T)$ incurred is \[O\left(\frac{n^2\cdot n!\cdot\log(nmT)}{\Delta_{\min}}+\frac{n^2\cdot n!\cdot(\log(nmT))^2}{|D|}\right).\]
\end{proof}

\subsubsection{Consequential lemmas of Algorithm~\ref{alg-nxn-full}'s conditional statements}
For the sake of simplicity of presentation, for the rest of the section we assume that empirical matrix $\bar A$ satisfies $0\leq \bar A_{i,j}\leq 1$ for any $i,j$ and $n=m$. When $n=m$ we have $\bar A=\bar B$. The analysis for the case when $n\neq m$ is analogous to the case when $n=m$. We will be using \Cref{det:lem} repeatedly.

\begin{lemma}\label{nxn:t*}
    Let $t$ be the timestep when the condition $1\leq \frac{\tilde \Delta_{\min}+2n^2\cdot n!\cdot\Delta}{\tilde \Delta_{\min}-2n^2\cdot n!\cdot\Delta}\leq \frac{3}{2}$ holds true for the first time. If event $G$ holds, then $t\geq \frac{128\cdot (n^2\cdot n!)^2\cdot\log(n^2T^2)}{\Delta_{\min}^2}$.
\end{lemma}
\begin{proof}
    Let $\bar\Delta:=n^2\cdot n!\cdot\Delta$. If $1\leq \frac{\tilde \Delta_{\min}+2\bar\Delta}{\tilde \Delta_{\min}-2\bar\Delta}\leq \frac{3}{2}$, then $\bar\Delta\leq \frac{\tilde\Delta_{\min}}{10}$. Now observe that $\frac{\tilde \Delta_{\min}}{\Delta_{min}}\leq \frac{\tilde \Delta_{\min}}{\tilde \Delta_{\min}-2\bar\Delta}\leq \frac{\tilde \Delta_{\min}}{\tilde \Delta_{\min}-\tilde \Delta_{\min}/5}=\frac{5}{4}$. Hence we have $\bar\Delta\leq \frac{\Delta_{\min}}{8}$. As $\bar\Delta=n^2\cdot n!\cdot\sqrt{\frac{2\log(n^2T^2)}{t}}$, we have $t\geq\frac{128\cdot (n^2\cdot n!)^2\cdot\log(n^2T^2)}{\Delta_{\min}^2}$.
\end{proof}

\begin{lemma}\label{nxn:t-condition}
    Consider any timestep $t\geq \frac{800\cdot (n^2\cdot n!)^2\cdot\log(n^2T^2)}{\Delta_{\min}^2}$. If event $G$ holds, then condition $1\leq \frac{\tilde \Delta_{\min}+2n^2\cdot n!\cdot\Delta}{\tilde \Delta_{\min}-2n^2\cdot n!\cdot\Delta}\leq \frac{3}{2}$ holds true.
\end{lemma}
\begin{proof}
    Let $\bar\Delta:=n^2\cdot n!\cdot\Delta$. First observe that $\bar\Delta=n^2\cdot n!\cdot\sqrt{\frac{2\log(n^2T^2)}{t}}\leq \frac{\Delta_{\min}}{20}$. Next observe that $\tilde \Delta_{\min}-2\bar\Delta\geq \Delta_{\min}-4\bar\Delta>0$. Hence we have $1\leq \frac{\tilde \Delta_{\min}+2\bar\Delta}{\tilde \Delta_{\min}-2\bar\Delta}$. Now we have $\frac{\tilde \Delta_{\min}+2\bar\Delta}{\tilde \Delta_{\min}-2\bar\Delta}\leq \frac{\Delta_{\min}+4\bar\Delta}{\Delta_{\min}-4\bar\Delta}\leq\frac{\Delta_{\min}+\Delta_{\min}/5}{\Delta_{\min}-\Delta_{\min}/5}=\frac{3}{2}$.
\end{proof}

\begin{lemma}\label{nxn:D}
    Consider any timestep $t\geq \frac{800\cdot (n^2\cdot n!)^2\cdot\log(n^2T^2)}{\Delta_{\min}^2}$. If event $G$ holds, then $\frac{4}{5}\leq \frac{|D|}{\tilde D}\leq \frac{6}{5}$
\end{lemma}
\begin{proof}
    Let $\bar\Delta:=n^2\cdot n!\cdot\Delta$. As $t\geq \frac{800\log(n^2T^2)}{\Delta_{\min}^2}$, due to lemma \ref{nxn:t-condition} the condition $1\leq \frac{\tilde \Delta_{\min}+2\bar\Delta}{\tilde \Delta_{\min}-2\bar\Delta}\leq \frac{3}{2}$ holds true. Hence we have $\bar\Delta\leq \frac{\tilde \Delta_{\min}}{10}$. As $\tilde\Delta_{\min}\leq \tilde D$, we have $\bar\Delta\leq \frac{\tilde D}{10}$. Now we have that $\frac{|D|}{\tilde D}\leq \frac{\tilde D+2\bar\Delta}{\tilde D}\leq \frac{\tilde D+\tilde D/5}{\tilde D}\leq\frac{6}{5}$. Similarly we have that $\frac{|D|}{\tilde D}\geq \frac{\tilde D-2\bar\Delta}{\tilde D}\geq \frac{\tilde D-\tilde D/5}{\tilde D}\geq\frac{4}{5}$. 
\end{proof}

% \begin{lemma}
%     Let $j_t$ be the column played by $y$-player in round $t$. Let $(x',y')$ be the NE of the empirical matrix $\bar A$. If event $G$ holds, then $V_A^*-\langle x',Ae_{j_t}\rangle \leq 2\sqrt{\frac{2\log(n^2T^2)}{t}}$
% \end{lemma}
% \begin{proof}
%     Let $\Delta_t=\sqrt{\frac{2\log(n^2T^2)}{t}}$. Let $(x^*,y^*)$ be the NE of $A$. As event $G$ holds, we now have the following: 
%     \begin{align*}
%         \langle x', Ae_{j_t}\rangle &\geq \langle x', \bar Ae_{j_t}\rangle-\Delta_t\\
%         &\geq\langle x',\bar Ay'\rangle-\Delta_t\\
%         &\geq \langle x^*, \bar Ay'\rangle-\Delta_t\\
%         &\geq \langle x^*, Ay'\rangle-2\Delta_t\\
%         &\geq \langle x^*, Ay^*\rangle-2\Delta_t\\
%         &= V_A^*-2\Delta_t
%     \end{align*}
% \end{proof}

\begin{lemma}\label{nxn:nash}
    Let $j_t$ be the column played by $y$-player in round $t$. Let $(x',y')$ be the NE of the empirical matrix $\bar A_{t-1} = \frac{1}{t-1}\sum_{s=1}^{t-1} \mathbf{A}_s$. If event $G$ holds, then $V_A^*-\langle x',Ae_{j_t}\rangle \leq \min\{2\sqrt{\frac{2\log(n^2T^2)}{t-1}},1\}$
\end{lemma}
\begin{proof}
    As $V_A^*\leq 1$ and $\langle x',Ae_{j_t}\rangle\geq 0$, we have $V_A^*-\langle x',Ae_{j_t}\rangle \leq 1$. Let $\bar{A} = \bar{A}_{t-1}$ and $\Delta_t=\sqrt{\frac{2\log(n^2T^2)}{t-1}}$. Let $(x^*,y^*)$ be the NE of $A$. As event $G$ holds, we now have the following for any $t\geq 2$: 
    \begin{align*}
        \langle x', Ae_{j_t}\rangle &\geq \langle x', \bar Ae_{j_t}\rangle-\Delta_t\\
        &\geq\langle x',\bar Ay'\rangle-\Delta_t\\
        &\geq \langle x^*, \bar Ay'\rangle-\Delta_t\\
        &\geq \langle x^*, Ay'\rangle-2\Delta_t\\
        &\geq \langle x^*, Ay^*\rangle-2\Delta_t\\
        &= V_A^*-2\Delta_t
    \end{align*}
\end{proof}

\begin{corollary}\label{nxn:cor}
    If the algorithm \ref{alg-nxn-full} at time $t = \{2,\dots,t'\}$ plays the NE of the empirical matrix $\bar A_{t-1} = \frac{1}{t-1}\sum_{s=1}^{t-1} \mathbf{A}_s$ and event $G$ holds, then the regret incurred till round $t'$ is at most $4\sqrt{2t'\log(n^2T^2)}+1$
\end{corollary}
\begin{proof}
    For all $t\in\{2,\ldots,t'\}$, we have $V_A^*-\langle x',Ae_{j_t}\rangle \leq 2\sqrt{\frac{2\log(n^2T^2)}{t-1}}$ due to lemma \ref{nxn:nash}. Hence the regret incurred until round $t'$ is equal to $t ' \cdot V_A^*-\sum_{t=1}^{t'} \langle x',Ae_{j_t} \rangle\leq 1+2\sum_{t=2}^{t'}\sqrt{\frac{2\log(n^2T^2)}{t-1}}\leq 4\sqrt{2t'\log(n^2T^2)}+1$.
\end{proof}

\begin{lemma}\label{nxn:iff}
    Consider any timestep $t\geq \frac{800\cdot (n^2\cdot n!)^2\cdot\log(n^2T^2)}{\Delta_{\min}^2}$. If event $G$ holds, then the empirical matrix $\bar A$ has a unique full support NE if and only if the input matrix $A$ has a unique full support NE.
\end{lemma}
\begin{proof}
It suffices to show that $det(M_A)>0$ if and only if $det(M_{\bar A})>0$, $det(M_{A^T})>0$ if and only if $det(M_{\bar A^T})>0$, for any $i\in [n]$ $det(M_{A^T,i})> 0$ if and only if $det(M_{\bar A^T,i})\geq 0$, and for any $j\in [n]$, $det(M_{A,j})\geq 0$ if and only if $det(M_{\bar A,j})\geq 0$. Let $\bar\Delta:=n^2\cdot n!\cdot\Delta$. First observe that $\bar\Delta=n^2\cdot n!\cdot\sqrt{\frac{2\log(n^2T^2)}{t}}\leq \frac{\Delta_{\min}}{20}$. Let $\bar\Delta:=n^2\cdot n!\cdot\Delta$. Now observe that $\bar\Delta=n^2\cdot n!\cdot\sqrt{\frac{2\log(n^2T^2)}{t}}\leq \frac{\Delta_{\min}}{20}$. Fix an index $j\in[n]$. If $det(M_{A,j})> 0$, then due to lemma \ref{det:lem}, we have $det(M_{\bar A,j})\geq det(M_{A,j})-2\bar\Delta\geq det(M_{A,j})-\frac{\Delta_{\min}}{10}\geq \frac{9 det(M_{A,j})}{10}>0$.

As $t\geq \frac{800\log(n^2T^2)}{\Delta_{\min}^2}$, due to lemma \ref{nxn:t-condition} the condition $1\leq \frac{\tilde \Delta_{\min}+2\bar\Delta}{\tilde \Delta_{\min}-2\bar\Delta}\leq \frac{3}{2}$ holds true. Hence we have $\bar\Delta\leq \frac{\tilde \Delta_{\min}}{10}$. Fix an index $j\in[n]$. If $det(M_{\bar A,j})> 0$, then due to lemma \ref{det:lem}, we have $det(M_{A,j})\geq det(M_{\bar A,j})-2\bar\Delta\geq det(M_{\bar A,j})-\frac{\tilde\Delta_{\min}}{5}\geq \frac{4 det(M_{\bar A,j})}{5}>0$. In an indentical way, we can also show that $det(M_A)>0$ if and only if $det(M_{\bar A})>0$,  $det(M_{A^T})>0$ if and only if $det(M_{\bar A^T})>0$, and for any $i\in [n]$ $det(M_{A^T,i})> 0$ if and only if $det(M_{\bar A^T,i})\geq 0$.
\end{proof}

\begin{lemma}\label{nxn:deviation}
    Consider any timestep $t\geq \frac{800\cdot (n^2\cdot n!)^2\cdot\log(n^2T^2)}{\Delta_{\min}^2}$. Let the input matrix $A$ have a unique full support NE $(x^*,y^*)$. Let $(x',y')$ be the NE of the empirical matrix $\bar A$ at timestep $t$. If event $G$ holds, then for any $i\in[n]$, $|x^*_i-x'_i|\leq \frac{5\tilde\Delta}{\tilde D}$ where $\tilde\Delta=n\cdot n!\cdot\sqrt{\frac{2\log(n^2T^2)}{t}}$ and $\tilde D=|\det(M_{\bar A^T})|$. 
\end{lemma}
\begin{proof}
         W.l.o.g let us assume that $det(M_{A^T})>0$. As $t\geq \frac{800\cdot (n^2\cdot n!)^2\cdot\log(n^2T^2)}{\Delta_{\min}^2}$, due to lemma \ref{nxn:t-condition} the condition $1\leq \frac{\tilde \Delta_{\min}+2n\tilde\Delta}{\tilde \Delta_{\min}-2n\tilde\Delta}\leq \frac{3}{2}$ holds true. Hence we have $\tilde\Delta\leq \frac{\tilde \Delta_{\min}}{10n}\leq \frac{\tilde D}{10n}$. We now have the following:
    \begin{align*}
        x_i^*&=\frac{det(M_{A^T,i})}{det(M_{A^T})}\\
        &\leq \frac{det(M_{\bar A^T,i})+2\tilde\Delta}{\tilde D-2\tilde\Delta}\\
        &=\left(\frac{det(M_{\bar A^T,i})}{\tilde D}+\frac{2\tilde\Delta}{\tilde D}\right)\left(1-\frac{2\tilde\Delta}{\tilde D}\right)^{-1}\\
        &\leq\left(x_i'+\frac{2\tilde\Delta}{\tilde D}\right)\left(1+\frac{5\tilde\Delta}{2\tilde D}\right)\tag{as $(1-x)^{-1}\leq 1+\frac{5x}{4}$ for all $x\in[0,\frac{1}{5}]$}\\
        &= x_i'+\frac{2\tilde\Delta}{\tilde D}+x_i'\cdot \frac{5\tilde\Delta}{2\tilde D}+\frac{\tilde\Delta}{\tilde D}\cdot \frac{5\tilde\Delta}{\tilde D}\\
        &\leq x_i'+\frac{5\tilde \Delta}{\tilde D} \tag{as $\frac{\tilde \Delta}{\tilde D}\leq \frac{1}{10}$}
    \end{align*}
    Similarly we have that following:
    \begin{align*}
        x_i^*&=\frac{det(M_{A^T,i})}{det(M_{A^T})}\\
        &\geq \frac{det(M_{\bar A^T,i})-2\tilde\Delta}{\tilde D+2\tilde\Delta}\\
        &=\left(\frac{det(M_{\bar A^T,i})}{\tilde D}-\frac{2\tilde\Delta}{\tilde D}\right)\left(1+\frac{2\tilde\Delta}{\tilde D}\right)^{-1}\\
        &\geq\left(x_i'-\frac{2\tilde\Delta}{\tilde D}\right)\left(1-\frac{2\tilde\Delta}{\tilde D}\right)\tag{as $(1+x)^{-1}\geq 1-x$ for all $x\geq 0$}\\
        &= x_i'-\frac{2\tilde\Delta}{\tilde D}-x_i'\cdot \frac{2\tilde\Delta}{\tilde D}+\frac{\tilde\Delta}{\tilde D}\cdot \frac{\tilde\Delta}{\tilde D}\\
        &\geq x_i'-\frac{4\tilde \Delta}{\tilde D} \tag{as $\frac{\tilde \Delta}{\tilde D}\leq \frac{1}{10}$}
    \end{align*}
\end{proof}
\begin{lemma}\label{nxn:feasible}
    Consider any timestep $t\geq \frac{800\cdot (n^2\cdot n!)^2\cdot\log(n^2T^2)}{\Delta_{\min}^2}$. Let the input matrix $A$ have a unique full support NE $(x^*,y^*)$. Let $(x',y')$ be the NE of the empirical matrix $\bar A$ at timestep $t$. If event $G$ holds, then for any $i\in[n]$, $x_i'-\frac{5n\tilde \Delta}{\tilde D}> 0$ and $x_i'+\frac{5n\tilde\Delta}{\tilde D}< 1$ where $\tilde\Delta=n\cdot n!\cdot\sqrt{\frac{2\log(n^2T^2)}{t}}$ and $\tilde D=|\det(M_{\bar A^T})|$. 
\end{lemma}
\begin{proof}
    W.l.o.g let $\det(M_{\bar A^T})>0$. As $t\geq \frac{800\cdot (n^2\cdot n!)^2\cdot\log(n^2T^2)}{\Delta_{\min}^2}$, due to lemma \ref{nxn:t-condition} the condition $1\leq \frac{\tilde \Delta_{\min}+2n\tilde\Delta}{\tilde \Delta_{\min}-2n\tilde\Delta}\leq \frac{3}{2}$ holds true. Hence we have $\tilde\Delta\leq \frac{\tilde \Delta_{\min}}{10n}$. For any $i\in[n]$, $\det(M_{\bar A^T,i})-5n\tilde\Delta\geq \det(M_{\bar A^T,i})-\frac{\tilde \Delta_{\min}}{2}\geq \frac{\det(M_{\bar A^T,i})}{2}>0$. Hence $x_i'-\frac{5n\tilde \Delta}{\tilde D}> 0$. Fix any $i\in[n]$. Now we have $x_i'+\frac{5n\tilde\Delta}{\tilde D}=1-\sum\limits_{j\in[n]\setminus\{i\}}x_j'+\frac{5n\tilde\Delta}{\tilde D}< 1-(n-1)\cdot \frac{5n\tilde \Delta}{\tilde D}+\frac{5n\tilde \Delta}{\tilde D}\leq 1 $ 
\end{proof}
\begin{lemma}\label{nxn:regret}
    Let the input matrix $A$ have a unique NE $(x^*,y^*)$ that is not a PSNE. If event $G$ holds, then every call of the algorithm \ref{subroutine-nxn} with parameters mentioned in the line \ref{alg:nxn:st1} of algorithm \ref{alg-nxn-full} incurs a regret of at most $O\left(\frac{n^2\cdot n!\cdot\log(nT)}{|D|}\right)$.
\end{lemma}
\begin{proof}
    Let us assume that the algorithm \ref{subroutine-nxn} gets called immediately after the timestep $t_1$. Let $\Delta = \sqrt{\frac{2\log(n^2T^2)}{t_1}}$. Recall that $\widehat A=\bar A$, $T_1=(\frac{1}{\Delta})^2=\frac{t_1}{2\log(n^2T^2)}$, $D_1=\frac{\tilde D}{5n\cdot n!}$, and  $T_2=\min\{t_1,T-t_1\}$. As event $G$ holds, for any $i,j$ we have $|A_{i,j}-\widehat A_{i,j}|\leq \Delta=\frac{1}{\sqrt{T_1}}$. Due to lemma \ref{nxn:t*}, we have $t_1\geq \frac{800\cdot (n^2\cdot n!)^2\cdot\log(T^2)}{\Delta_{\min}^2}$. Recall that $(x',y')$ is the NE of $\bar A$. Due to lemma \ref{nxn:deviation}, we have $||x'-x^*||_\infty\leq  \frac{5n\cdot n!\cdot\Delta}{\tilde D}=\frac{1}{D_1\sqrt{T_1}}$. Due to lemma \ref{nxn:feasible}, we have $\min_i\min\{x_i',1-x_i'\}\geq \frac{5n^2n!\Delta}{\tilde D}\geq \frac{n-1}{D_1\sqrt{T_1}}$. Hence due to Theorem \ref{sub-nxn-thm}, the regret incurred is $\frac{c_1n}{D_1}+\frac{c_2nT_2}{D_1T_1}\leq\frac{5c_1n^2\cdot n!}{\tilde D}+\frac{10c_2n^2\cdot n!\cdot\log(n^2T^2)}{\tilde D}=O\left(\frac{n^2\cdot n!\cdot\log(nT)}{|D|}\right)$ where $c_1,c_2$ are some absolute constants. We get the last equality due to lemma \ref{nxn:D}.
\end{proof}
    
\begin{lemma}\label{nxn:logT}
    Let the input matrix $A$ have a unique NE $(x^*,y^*)$ that is not a PSNE. If event $G$ holds, the algorithm \ref{alg-nxn-full} calls algorithm \ref{subroutine-nxn} for at most $\log_2T$ times.
\end{lemma}
\begin{proof}
    Let $t_{(i)}$ denote the value of $t_1$ when algorithm \ref{subroutine-nxn} was called by algorithm \ref{alg-nxn-full} for the $i$-th time. Observe that $t_{(i+1)}=2t_{(i)}$. Recall that the algorithm \ref{alg-nxn-full} can call algorithm \ref{subroutine-nxn} until $t_1$ becomes greater than $T$. Hence the algorithm \ref{alg-nxn-full} calls algorithm \ref{subroutine-nxn} for at most $\log_2T$ times. 
\end{proof}
\subsection{Regret guarantee for $n\times m$ matrix with unique Nash equilibrium}\label{appendix:main-thm-guarantee}
Say we run \Cref{alg-nxm-instance} on the input matrix $A$ until it returns a $k\times k$ submatrix $B$ at a timestep $t'$. Then we remove the rows from $A$ that are not part of the submatrix $B$. Then we run \Cref{alg-nxn-full} on the modified matrix $A$ for the remaining time steps. Recall the definitions of the matrix-dependent parameter $\Delta_{\min},\Delta_g$ and $D$ from \Cref{appendix:nxm,appendix:nxn}. Our algorithmic procedure has the following regret guarantee due to \Cref{thm:nxm:full,thm:nxn:full}.
\begin{theorem}[Formal version of \Cref{main-theorem-informal}]
    Fix a matrix game on $A\in [-1,1]^{n\times m}$ with a unique Nash equilibrium $(x^*,y^*)$ of support size $k \leq \min\{n,m\}$. 
    Then the Nash regret is upper bounded as follows:
    \begin{align*}
        R^N(A,T)&=T\cdot V_A^*-\sum_{t=1}^T\mathbb{E}[\langle x_t,Ay_t\rangle]\leq \frac{c_1\cdot k^2\cdot k!\cdot\log(nmT)}{\min\{1,\Delta_{\min}\}\cdot\min\{1,\Delta_g\}}+\frac{c_2\cdot k^2\cdot k!\cdot(\log(nmT))^2}{|D|}\\
    \end{align*}
    where $c_1,c_2$ are some absolute constants.
\end{theorem}

\subsection{One possible way to improve the dependence on input dimensions}\label{appendix:condition-number-guarantee}
Let $B$ be a $k \times k$ matrix with a unique full-support Nash equilibrium. Recall that \Cref{det:lem} was repeatedly used in our analysis to bound the deviation of determinants of such matrices in the previous two sections. However, a key issue with this lemma is that the deviation scales exponentially with $k$, ultimately leading to an exponential dependence on $k$ in the Nash regret guarantee.  

Using perturbation theory, one can hope to eliminate this exponential dependence on $k$ as follows:  

\[
\left| \frac{\det(B+\Delta B)-\det(B)}{\det(B)} \right| \approx \operatorname{tr}(B^{-1} \Delta B) 
\leq \|B^{-1}\|_F \|\Delta B\|_F 
\leq n \cdot \kappa(B) \cdot \frac{\|\Delta B\|}{\|B\|}
\]

where $\kappa(B) = \|B\| \cdot \|B^{-1}\|$ is the condition number of $B \in \mathbb{R}^{k \times k}$,  
$\|B\|_F$ is the Frobenius norm, and $\|B\|$ is the operator norm.  

Thus, if the matrices considered in the previous two sections are well-conditioned, one can potentially eliminate the exponential dependence on the input dimension in our Nash regret guarantees.
\subsection{$n\times m$ matrix game with non-unique Nash Equilibrium}\label{appendix:nxm:non-unique}
In this section, we provide a high level-description on how to accommodate the case when the Nash Equilibrium is not unique. We skip the detailed analysis as it is analogous to the case when the Nash Equilibrium is unique. Recall the notations $M_A$ and $M_{A,\ell}$ from \Cref{appendix:nxn} for any matrix $A\in\mathbb{R}^{n\times n}$. We say a statistic $\hat a$ which is an empirical estimate of a fixed value $a$ is well concentrated if $|a-\hat a|<<|a|$. In this section we do a small notational abuse. If we talk about vector $v$ from a $p$-dimensional simplex  in the context of $d$-dimensional simplex where $p<d$, we are actually talking about the vector that we get by appending $d-p$ zeroes to $v$. 

Let us consider an arbitrary input matrix $A\in\mathbb{R}^{n\times m}$. Note that we don't make any assumption regarding the uniqueness of Nash equilibrium. Now we describe our procedure. In a round $t$, we try to identify a square sub-matrix $\bar B$ of $\bar A$ such the following holds simultaneously:
\begin{itemize}
    \item The statistics $\det(M_{\bar B^\top,i}),\det(M_{\bar B^\top}), \det(M_{\bar B,j} $ and $\det(M_{\bar B})$ are well concentrated. One can use lemma \ref{det:lem} to ensure this.
    \item $(x',y')$ is a full support Nash Equilibrium of $\bar B$ where $x_i'=\frac{\det(M_{\bar B^\top,i})}{\det(M_{\bar B^\top})}$ and $y_j'=\frac{\det(M_{\bar B,j})}{\det(M_{\bar B})}$.
    \item If there is an index $j\in [m]$ such that the gap $\langle x',\bar Ae_j\rangle -V_{\bar B}^*<0$, then this gap is not well concentrated yet. One can use lemma \ref{nxm:game:deltag} to ensure this.
    \item If there is an index $i\in [n]$ such that the gap $V_{\bar B}^*-\langle e_i,\bar Ay'\rangle <0$, then this gap is not well concentrated yet. One can use lemma \ref{nxm:game:deltag} to ensure this.
\end{itemize}
If such a sub-matrix exists, we appropriately invoke the sub-routine in \Cref{subroutine-nxn} and run it for the next $t$ rounds. Otherwise we play the empirical NE of $\bar A$ for the next $t$ rounds.

The main reason why such an approach works is due to the properties of basic solutions of zero-sum games (as introduced in \cite{bohnenblust1948mathematical}).   Let $\calX:=\{x^*\in\simplex_n: x^*\in\arg\max_{x\in\simplex_n}\min_{y\in\simplex_m}\langle x,Ay\rangle\}$. Let $\calY:=y^*\in\arg\min_{y\in\simplex_m}\max_{x\in\simplex_n}\langle x,Ay\rangle\}$. For any non-null compact convex-set $\calR$, let $\calR^*$ be the set of vectors $r^*\in \calR$ such that the following does not hold for two distinct vectors $r',r''\in \calR$:
\begin{equation*}
    r^*=\frac{r'+r''}{2}
\end{equation*}
It follows that $\calR^*$ is non-null and is the smallest set that spans $R$ in the sense of convex linear combinations. A basic solution is a Nash equilibrium $(x^*,y^*)$ such that $x^*\in\calX^*$ and $y^*\in\calY^*$. For any submatrix $\dot{A}$ of $A$, let $\dot{x}$ and $\dot{y}$ be the contractions of the vectors $x$ and $y$ obtained by deleting the components corresponding to the rows and columns of $A$ which are absent in $\dot{A}$. The following lemma describes an important property of basic solutions.
\begin{lemma}[\citet{bohnenblust1948mathematical}]\label{basic-solution}
   A Nash equilibrium $(x^*,y^*)$ is a basic solution if and only if there is a square sub-matrix $\dot{A}$ of $A$ such that 
    \begin{equation*}
        \dot{x}^\top=\frac{\mathbf{1}^\top adj(\dot{A})}{\mathbf{1}^\top adj(\dot{A})\mathbf{1}},\; \dot{y}=\frac{adj(\dot{A})\mathbf{1}}{\mathbf{1}^\top adj(\dot{A})\mathbf{1}},\; V_A^*=\frac{det(\dot{A})}{\mathbf{1}^\top adj(\dot{A})\mathbf{1}}
    \end{equation*}
\end{lemma}

Hence, we get an instance-dependent poly-logarithmic regret here where the instance dependent terms are a function of following terms:
\begin{itemize}
    \item determinants $\det(M_{B^\top,i}),\det(M_{\bar B^\top}), \det(M_{B,j}), \det(M_{B})$ corresponding to sub-matrices $B$ which correspond to the basic solutions
    \item the sub-optimality gaps $V_{ B}^*-\langle x', Ae_j\rangle$ and $\langle e_i,Ay'\rangle-V_{ B}^*$ of other square sub-matrices $B$ which have full-support Nash equilibrium $(x',y')$
\end{itemize}

\section{Algorithms and Analysis for $2\times 2$ matrix game under Bandit Feedback}\label{appendix:2x2}
%\textcolor{red}{Write theorem and its proof sketch}
%\textcolor{red}{Comment on the fact that $0\leq \bar A_{ij}\leq 1$.}

In this section, we formally present the algorithm (Algorithm \ref{alg-2x2-bandit}) for $2\times 2$ matrix game under Bandit Feedback. In section \ref{section:bandits:exploration}, we provide an exploration procedure. In section \ref{section:bandits:exploitation}, we provide provide an exploitation procedure. Initially we make the assumption that the Nash equilibrium is unique and it is for a fixed time horizon. We then relax these assumptions in \Cref{appendix:anytime-bandit,appendix:2x2:non-unique}. Note that, throughout this section, NE stands for Nash equilibrium.

We first provide a high level intuition of our algorithm. Unlike the full information case, we are not guaranteed that every element would be explored equally. Hence, in each round $t$, we have to play a strategy $x_t$ such that either every element would be explored sufficiently or we incur a negative regret. We achieve this goal by designing a complicated exploration procedure. Here we describe the key parts of the exploration procedure.

Consider a matrix $A=\begin{bmatrix}
    a & b\\
    c & d\\
\end{bmatrix}$. 
Our aim is to estimate the gaps $|a-b|,|a-c|,|d-c|,|d-b|$ up to a constant factor. 
% However, we do not have control over the actions of the column player. Hence, we design a scheme where the column player either helps us explore one of the gaps that we are interested in or we incur a non-positive regret. 
% First let us define the well-separated rows and column. Row $1$ (resp. row $2$) is well-separated if we have estimated the gap $|a-b|$ (resp. $|c-d|$) upto a constant factor. Similarly column $1$ (resp. column $2$) is well-separated if we have estimated the gap $|a-c|$ (resp. $|b-d|$) upto a constant factor.
We begin with an uniform exploration phase where we aim to estimate two such gaps up to a constant factor. First, we play $x_t=(1/2,1/2)$ in each round. We eventually estimate one of the gaps $|a-c|, |b-d|$ up to a constant factor due to the fact that one of the columns must be played at least half the time. Then we play one of the two strategies $(1/2,1/2)$ and $e_i$ ($i$ depends on the gap that we have estimated previously up to a constant factor) in each round. We eventually estimate one more gap up to a constant factor (again due to the fact that one of the columns must be played at least half the time). Let us assume that we have estimated the gaps $|a-b|$ and $|a-c|$ up to a constant factor by the end of the uniform exploration phase. Now we aim to estimate the gap $|b-d|$ up to a constant factor. We do this by playing a strategy $x^{(1)}$ that incurs a non-positive regret if column one is played by the column-player. Hence, the column-player is forced to play the 2nd column if it wants the row-player to incur positive regret. Hence, we end up estimating the gap $|b-d|$ up to a constant factor. Now, we aim to estimate the final gap $|c-d|$ up to a constant factor. We do this by finding a strategy $x^{(2)}$ that incurs a non-positive regret if column two is played by the column-player. Hence if the row-player plays $x^{(2)}$ the column-player is forced to play the 1st column if it wants the row-player to incur positive regret. Using the strategies $x^{(1)}$ and $x^{(2)}$ we estimate the final gap $|c-d|$ as follows. We play $x^{(1)}$ until the column two is played by the column player. We then switch to $x^{(2)}$ and play it until  column one is played by the column-player. We then switch back to $x^{(1)}$ and repeat this process until we estimate the gap $|c-d|$ up to a constant factor. This concludes our exploration procedure.
% We refer the reader to appendix \ref{section:bandits:exploration} for more details.

After this initial exploration phase is over, if we observe that matrix $A$ has a pure strategy Nash equilibrium $(i^*,j^*)$, then we play $e_{i^*}$ for the rest of the rounds. Otherwise, we invoke the sub-routine similarly to the full-information case. 
Putting it all together, we obtain our main result of Theorem~\ref{thm:2x2:bandit} for bandit feedback. The complete algorithm for bandit feedback is in Appendix \ref{section:bandits:exploration} and Appendix \ref{section:bandits:exploitation}.

We now establish the regret guarantee of both the exploration and exploitation procedure by proving Theorem \ref{thm:2x2:bandit}. 

\begin{proof}[Proof of Theorem \ref{thm:2x2:bandit}]
Let $\bar A_{ij,t}$ denote the element $(i,j)$ of $\bar A$ after the element $(i,j)$ is sampled for $t$ times. Let us now define an event $G$ as follows:
\begin{align*}
    G:=&\bigcap_{t=1}^T\bigcap_{i=1}^2 \bigcap_{j=1}^2 \{ |A_{ij}-\bar A_{ij,t}|\leq \sqrt{\tfrac{2\log(T^2)}{t}} \} \\
\end{align*}
Observe that event $G$ holds with probability at least $1-\frac{8}{T}$. 

Let $R(T)=T\cdot V_A^*-\sum\limits_{t=1}^T\langle x_t,Ae_{j_t} \rangle$ denote the regret incurred where $x_t$ is the mixed strategy played by row player in round $t$ and $j_t$ is the column played by the column player. If we show that under event $G$, $R(T)=O\left(\frac{\log(T)}{\Delta_{\min}^3}+\frac{ (\log(T))^2}{\Delta_{\min}^2}\right)$ with probability at least $1-\frac{1}{T^2}$, then due to Equation (\ref{equivalent:regret}), we get that the nash regret incurred by the algorithm is $O\left(\frac{\log(T)}{\Delta_{\min}^3}+\frac{(\log(T))^2}{\Delta_{\min}^2}\right)$. Hence, for the rest of the section, by regret we imply the expression $R(T)=T\cdot V_A^*-\sum\limits_{t=1}^T\langle x_t,Ae_{j_t} \rangle$. 

Let us assume that event $G$ holds. We divide our algorithm into an exploration phase (see section \ref{section:bandits:exploration}) and an exploitation phase (see section \ref{section:bandits:exploitation}). Exploration phase is further divided into two cases. Due to lemma \ref{bandit:lem:case1} we get that regret incurred in case 1 is $O(\frac{\log T}{\Delta_{\min}^3})$ with probability at least $1-\frac{1}{T^3}$. Similarly due to lemma \ref{bandit:lem:case2} we get that regret incurred in case 2 is $O(\frac{\log T}{\Delta_{\min}^3})$ with probability at least $1-\frac{1}{T^3}$. For more details we refer the reader to section \ref{section:bandits:exploration}.

Now we analyse the exploitation phase. Each time the algorithm \ref{alg-2x2-bandit} calls the subroutine in algorithm \ref{subroutine-nxn-bandit}, it incurs a regret of $O(\frac{\log(T)}{\Delta_{\min}^2})$ with probability at least $1-\frac{1}{T^4}$ (see lemma \ref{2x2:regret:bandit}). Due to the doubling trick, the subroutine gets called for at most $\log_2T$ times (see lemma \ref{2x2:logT:bandit}). Hence, in the exploitation phase, we incur a regret of $O(\frac{(\log(T))^2}{\Delta_{\min}^2})$ with probability at least $1-\frac{1}{T^3}$. Putting the two pieces together, we have with probability at least $1-\frac{1}{T^2}$, the total regret $R(T)$ incurred is $O\left(\frac{\log(T)}{\Delta_{\min}^3}+\frac{(\log(T))^2}{\Delta_{\min}^2}\right)$.
\end{proof}

 Recall that $A\in[-1,1]^{2\times 2}$ and $ \mathbf{A}_t\in [-1,1]^{2\times 2}$. We translate and re-scale these matrices by adding $1$ to every entry and then dividing every entry by $2$. Hence for the rest of the section, we assume that  $A\in[0,1]^{2\times 2}$ and $\mathbf{A}_t\in [0,1]^{2\times 2}$. Note that the regret only changes by a factor of $2$. We now define the following parameters:
 \begin{align*}
     \Delta_{\min}:=\min\{\min_i|A_{i,1}-A_{i,2}|,\min_j|A_{1,j}-A_{2,j}|\};\quad D=A_{1,1}-A_{1,2}-A_{2,1}+A_{2,2}
 \end{align*}

 Now we begin describing the algorithm for the bandit-feedback. 
\subsection{Procedure for Exploration}\label{section:bandits:exploration}
In this section, we describe a procedure to explore under Bandit Feedback. Let the input matrix be $A$. At each time step we maintain an empirical matrix $\bar A$. Let $n_{i,j}^{t}$ denote the number of times the element $(i,j)$ has been sampled up to the time step $t$. At a time step $t$, let $\delta_{i,j}:=\sqrt{\frac{2\log(T^2)}{n^t_{i,j}}}$. 

\paragraph{Well separated rows and columns:}At a timestep $t$, we say that a row $i$ is well separated if $1\leq \frac{\bar g_{r,i}+2\Delta_{r,i}}{\bar g_{r,i}-2\Delta_{r,i}}\leq \frac{3}{2}$ where $\bar g_{r,i}=|\bar A_{i,1}-\bar A_{i,2}|$ and $\Delta_{r,i}=\max\left\{\delta_{i,1},\delta_{i,2}\right\}$. Similarly at a timestep $t$, we say that a column $j$ is well separated if $1\leq \frac{\bar g_{c,j}+2\Delta_{c,j}}{\bar g_{c,j}-2\Delta_{c,j}}\leq \frac{3}{2}$ where $\bar g_{c,j}=|\bar A_{1,j}-\bar A_{2,j}|$ and $\Delta_{c,j}=\max\left\{\delta_{1,j},\delta_{2,j}\right\}$. Let $g_{r,i}=| A_{i,1}- A_{i,2}|$ and $g_{c,j}=|A_{1,j}-A_{2,j}|$. Note that when row $i$ is well separated and event $G$ holds, we have $10\Delta_{r,i}\leq\bar g_{r,i}$ and $g_{r,i}-2\Delta_{r,i}\leq\bar g_{r,i}\leq g_{r,i}+2\Delta_{r,i}$. Similarly when column $j$ is well separated and event $G$ holds, we have $10\Delta_{c,j}\leq\bar g_{c,j}$ and $g_{c,j}-2\Delta_{c,j}\leq\bar g_{c,j}\leq g_{c,j}+2\Delta_{c,j}$. This implies that when a row or a column is well separated, then the elements in the row or column have the same ordering in both the empirical matrix $\bar A$ and the input matrix $A$.

\paragraph{Uniform Exploration:}Now in each timestep $t$, we play $x_t=(1/2,1/2)$ until one of the columns is well separated. Let us assume that column $j_1$ gets well separated and $\bar A_{i_1,j_1}>\bar A_{i_2,j_1}$ where $i_1,i_2\in\{1,2\}$. Let $j_2\in\{1,2\}\setminus\{j_1\}$.  {Next we alternatively play two strategies $x^{(1)}$ and $x^{(2)}$ where $x_{i_1}^{(1)}=1$ and $x_{i_1}^{(2)}=1/2$} until the row $i_1$ gets well separated or the column $j_2$ gets well separated. We alternate from $x^{(1)}$ to $x^{(2)}$ when column $j_2$ gets played, and we alternate from $x^{(2)}$ to $x^{(1)}$ after a single round of play. Now we have the following two cases (namely Case 1 and Case 2).

In the flowcharts below, dotted lines denote that we have figured out gap between the entries involved up to a constant factor. The strategies mentioned between the matrices are the strategies we play to explore the matrix and figure out the pairwise ordering. 

The following flowchart shows how we reach Case 1:
\[
\begin{bmatrix}
    A_{1,1} & A_{1,2}\\
    A_{2,1} & A_{2,2}
\end{bmatrix}
\rightarrow
x_t=(1/2,1/2)
\rightarrow
\begin{bmatrix}
    A_{i_1,j_1} & A_{i_1,j_2}\\
    \vdots \\
    A_{i_2,j_1} & A_{i_2,j_2}
\end{bmatrix}
\rightarrow
x^{(1)}\leftrightarrow x^{(2)}
\rightarrow
\begin{bmatrix}
    A_{i_1,j_1} & A_{i_1,j_2}\\
    \vdots &\vdots\\
    A_{i_2,j_1} & A_{i_2,j_2}
\end{bmatrix}
\]
Analogously, the following flowchart shows how we reach Case 2:
\[
\begin{bmatrix}
    A_{1,1} & A_{1,2}\\
    A_{2,1} & A_{2,2}
\end{bmatrix}
\rightarrow
x_t=(1/2,1/2)
\rightarrow
\begin{bmatrix}
    A_{i_1,j_1} & A_{i_1,j_2}\\
    \vdots \\
    A_{i_2,j_1} & A_{i_2,j_2}
\end{bmatrix} %d
\rightarrow
x^{(1)}\leftrightarrow x^{(2)} 
\rightarrow
\begin{bmatrix}
    A_{i_1,j_1} \cdots & A_{i_1,j_2}\\
    \vdots & \\
    A_{i_2,j_1} & A_{i_2,j_2}
\end{bmatrix}
\]

\paragraph{Case 1: Column $j_2$ gets well separated.} If $\bar A_{i_1,j_2}>\bar A_{i_2,j_2}$, then we terminate the exploration. On the other hand, if  $\bar A_{i_1,j_2}<\bar A_{i_2,j_2}$, then let $\ell:=\arg\max_{\ell\in\{1,2\}}|\bar A_{i_\ell,j_1}-\bar A_{i_\ell,j_2}|$ and $\bar\ell=\{1,2\}\setminus\{\ell\}$. If $\bar A_{i_\ell,j_\ell}<\bar A_{i_\ell,j_{\bar\ell}}$, then we terminate the exploration. Otherwise, let $\delta_0=|\bar A_{i_\ell,j_1}-\bar A_{i_\ell,j_2}|/3$. {Now we alternatively play two strategies $x^{(3)}$ and $x^{(4)}$ where $x_{i_\ell}^{(3)}=1-\delta_0$ and $x_{i_\ell}^{(4)}=0$}    until the row $i_{\bar\ell}$ gets well-separated. That is, we alternate from $x^{(3)}$ to $x^{(4)}$ when the element $(i_{\bar\ell},j_{\bar \ell})$ gets played, and we alternate from $x^{(4)}$ to $x^{(3)}$ when the element $(i_{\bar\ell},j_{\ell})$ gets played. If $\bar A$ has a pure strategy Nash equilibrium, we terminate the exploration. Otherwise, let $\delta_1=|\bar A_{i_1,j_1}-\bar A_{i_1,j_2}|/3$ and $\delta_2=|\bar A_{i_2,j_1}-\bar A_{i_2,j_2}|/3$.  We now alternatively play two strategies $x^{(5)}$ and $x^{(6)}$ where $x^{(5)}_{i_1}=\delta_2$ and $x_{i_1}^{(6)}=1-\delta_1$ until the condition $1\leq\frac{\tilde \Delta_{\min}+2\Delta}{\tilde \Delta_{\min}-2\Delta}\leq \frac{3}{2}$ holds true where 
\[\Delta:=\max_{i,j}\delta_{i,j}\quad\text{and}\quad \tilde\Delta_{\min}:=\min\{|\bar A_{11}-\bar A_{12}|,|\bar A_{21}-\bar A_{22}|,|\bar A_{11}-\bar A_{21}|,|\bar A_{12}-\bar A_{22}|\}.\] 
We alternate from $x^{(5)}$ to $x^{(6)}$ when column $j_1$ gets played, and we alternate from $x^{(6)}$ to $x^{(5)}$ when column $j_2$ gets played. After this we terminate the exploration.

The following flowchart shows how we explore in Case 1:
\[
\begin{bmatrix}
    A_{i_1,j_1} & A_{i_1,j_2}\\
    \vdots &\vdots\\
    A_{i_2,j_1} & A_{i_2,j_2}
\end{bmatrix}
\rightarrow
\bar A_{i_1,j_2}<\bar A_{i_2,j_2}
\rightarrow
\begin{bmatrix}
    A_{i_\ell,j_\ell} \cdots & A_{i_\ell,j_{\bar \ell}}\\
    \vdots &\vdots \\
    A_{i_{\bar\ell},j_\ell} & A_{i_{\bar \ell},j_{\bar \ell}}
\end{bmatrix}
\rightarrow
x^{(3)}\leftrightarrow x^{(4)}
\rightarrow
\begin{bmatrix}
    A_{i_1,j_1} \cdots & A_{i_1,j_2}\\
    \vdots &\vdots\\
    A_{i_2,j_1} \cdots & A_{i_2,j_2}
\end{bmatrix}
\]
%\ljr{Should we add a flow chart for going between $x^{(5)}$ and $x^{(6)}$?}
%\arn{That thing is just an extra step. Not that important.}\ljr{ok cool. Just wanted it to be complete.}

\paragraph{Case 2: Row $i_1$ gets well separated.} If $\bar A_{i_1,j_1}<\bar A_{i_1,j_2}$, then we terminate the exploration. On the other hand, if  $\bar A_{i_1,j_1}>\bar A_{i_1,j_2}$, then let $\delta=|\bar A_{i_1,j_1}-\bar A_{i_1,j_2}|/3$. Now, we play the strategy $x^{(7)}$ where $x_{i_1}^{(7)}=1-\delta$   until the column $j_{2}$ gets well-separated. If $\bar A_{i_1,j_2}>\bar A_{i_2,j_2}$, then we terminate the exploration.  Otherwise, if $\bar A_{i_1,j_2}<\bar A_{i_2,j_2}$, then {we alternatively play two strategies $x^{(7)}$ and $x^{(8)}$ where $x_{i_1}^{(8)}=0$}    until the row $i_{2}$ gets well separated. We alternate from $x^{(7)}$ to $x^{(8)}$ when the element $(i_{2},j_{2})$ gets played, and we alternate from $x^{(8)}$ to $x^{(7)}$ when the element $(i_{2},j_{1})$ gets played. If $\bar A$ has a pure strategic Nash equilibrium, we terminate the exploration. Otherwise, let $\delta_3=|\bar A_{i_2,j_1}-\bar A_{i_2,j_2}|/3$.  We now alternatively play the two strategies $x^{(9)}$ and $x^{(7)}$ where $x_{i_1}^{(9)}=\delta_3$ until the condition $1\leq\frac{\tilde \Delta_{\min}+2\Delta}{\tilde \Delta_{\min}-2\Delta}\leq \frac{3}{2}$ holds true where
\[\Delta:=\max_{i,j}\delta_{i,j}\quad\text{and}\quad \tilde\Delta_{\min}:=\min\{|\bar A_{11}-\bar A_{12}|,|\bar A_{21}-\bar A_{22}|,|\bar A_{11}-\bar A_{21}|,|\bar A_{12}-\bar A_{22}|\}.\] We alternate from $x^{(9)}$ to $x^{(7)}$ when column $j_1$ gets played, and we alternate from $x^{(7)}$ to $x^{(9)}$ when column $j_2$ gets played. After this we terminate the exploration.

The following flowchart shows how we explore in Case 2:
\[
\begin{bmatrix}
    A_{i_1,j_1} \cdots & A_{i_1,j_2}\\
    \vdots &\\
    A_{i_2,j_1} & A_{i_2,j_2}
\end{bmatrix}
\rightarrow
x^{(7)}
\rightarrow
\begin{bmatrix}
    A_{i_1,j_1} \cdots & A_{i_1,j_{2}}\\
    \vdots &\vdots \\
    A_{i_{2},j_1} & A_{i_{2},j_{2}}
\end{bmatrix}
\rightarrow
x^{(7)}\leftrightarrow x^{(8)}
\rightarrow
\begin{bmatrix}
    A_{i_1,j_1} \cdots & A_{i_1,j_2}\\
    \vdots &\vdots\\
    A_{i_2,j_1} \cdots & A_{i_2,j_2}
\end{bmatrix}
\]
%\ljr{Should we add a flow chart for going between $x^{(7)}$ and $x^{(9)}$?}

\subsubsection{Analysis of the exploration procedure}
Now we analyse our exploration procedure. We break our analysis into various parts

\paragraph{Uniform Exploration:} Let us assume that event $G$ holds. Recall that initially in each round $t$, we play $x_t=(1/2,1/2)$ until column $j_1$ is well separated. In lemma \ref{bandit:lem:nj1-upper}, we show that column $j_1$ is played for at most $O(\frac{\log T}{\Delta_{\min}^2})$ times before it becomes well-separated. Using this fact and corollary \ref{bandit:cor:nj2}, we get that the regret incurred till the time-step when the column $j_1$ becomes well-separated is $O(\frac{\log T}{\Delta_{\min}^2})$. In lemma \ref{bandit:lem:nj1-lower}, we show that column $j_1$ is played for at least $\Omega(\frac{\log T}{g_{c,j_1}^2})$ times before it became well-separated. This fact will be crucial when we analyse Case 1. In lemma \ref{bandit:lem:Ai1j1}, we show that $A_{i_1,j_1}\geq V_A^*$. This implies that  playing the strategy $x^{(1)}$ incurs a non-positive regret whenever column $j_1$ is played. This forces the column player to play column $j_2$ if it wants the row player to incur positive regret. 

Now we prove our technical lemmas.
\begin{lemma}\label{bandit:lem:nj1-upper}
    Let $n_{j_1}$ denote the number of times column $j_1$ was played with the strategy $(1/2,1/2)$ before which it was well-separated. If event $G$ holds, then with probability at least $1-\frac{1}{T^4}$ we have $n_{j_1}\leq \frac{3200\log(T^2)}{\Delta_{\min}^2}$.
\end{lemma}
\begin{proof}
    If column $j_1$ was played for $\frac{3200\log(T^2)}{\Delta_{\min}^2}$ times with the strategy $(1/2,1/2)$, then with probability at least $1-\frac{1}{T^4}$ each element of the column $j_1$ is played for at least $k=\frac{800\log(T^2)}{\Delta_{\min}^2}$ times. Now observe that $\Delta_{c,j_1}\leq\sqrt{\frac{2\log(T^2)}{k}}\leq \frac{\Delta_{\min}}{20}$. Next observe that $\bar g_{c,j_1}-2\Delta_{c,j_1}\geq g_{c,j_1}-4\Delta_{c,j_1}\geq \Delta_{\min}-4\Delta_{c,j_1}>0$. Hence we have $1\leq \frac{\bar g_{c,j_1}+2\Delta_{c,j_1}}{\bar g_{c,j_1}-2\Delta_{c,j_1}}$. Now we have $\frac{\bar g_{c,j_1}+2\Delta_{c,j_1}}{\bar g_{c,j_1}-2\Delta_{c,j_1}}\leq \frac{g_{c,j_1}+4\Delta_{c,j_1}}{g_{c,j_1}-4\Delta_{c,j_1}}\leq  \frac{\Delta_{\min}+4\Delta_{c,j_1}}{\Delta_{\min}-4\Delta_{c,j_1}} \leq\frac{\Delta_{\min}+\Delta_{\min}/5}{\Delta_{\min}-\Delta_{\min}/5}=\frac{3}{2}$.
\end{proof}
\begin{corollary}\label{bandit:cor:nj2}
    Let $n_{j_2}$ denote the number of times column $j_2$ was played with the strategy $(1/2,1/2)$ before the column $j_1$ was well-separated. If event $G$ holds, then with probability at least $1-\frac{1}{T^4}$ we have $n_{j_2}\leq \frac{3200\log(T^2)}{\Delta_{\min}^2}$.
\end{corollary}
\begin{proof}
    Due to the proof of lemma \ref{bandit:lem:nj1-upper}, with probability at least $1-\frac{1}{T^4}$, we have $n_{j_2}\leq \frac{3200\log(T^2)}{\Delta_{\min}^2}$ otherwise column $j_2$ gets well separated first.
\end{proof}
\begin{lemma}\label{bandit:lem:nj1-lower}
    Let $t$ be the time-step when the column $j_1$ became well-separated. If event $G$ holds, then $\min_{i\in\{1,2\}}n_{i,j_1}^t\geq \frac{128\log(T^2)}{g_{c,j_1}^2} $
\end{lemma}
\begin{proof}
    As the column $j_1$ becomes well-separated at the time-step $t$, we have $\frac{\bar g_{c,j_1}+2\Delta_{c,j_1}}{\bar g_{c,j_1}-2\Delta_{c,j_1}}\leq \frac{3}{2}$. This implies that $\Delta_{c,j_1}\leq \frac{\bar g_{c,j_1}}{10}$. As event $G$ holds, we have $\bar g_{c,j_1}\leq g_{c,j_1}+2\Delta_{c,j_1}$. Hence, we have $\Delta_{c,j_1}\leq \frac{ g_{c,j_1}}{8}$. This implies that $\min_{i\in\{1,2\}}n_{i,j_1}^t\geq \frac{128\log(T^2)}{g_{c,j_1}^2}$.
\end{proof}

\begin{lemma}\label{bandit:lem:Ai1j1}
    If event $G$ holds, then $A_{i_1,j_1}\geq V_A^*$
\end{lemma}
\begin{proof}
    As event $G$ holds and $\bar A_{i_1,j_1}>\bar A_{i_2,j_1}$, we have $A_{i_1,j_1}>A_{i_2,j_1}$. If $A$ has a unique NE $(x^*,y^*)$ which is not a PSNE, we have $V_A^*=x_1^*A_{i_1,j_1}+x_2^*A_{i_2,j_1}<A_{i_1,j_1}$. If $(i_1,j_2)$ is a PSNE, then $V_A^*=A_{i_1,j_2}<A_{i_1,j_1}$. If $(i_2,j_2)$ is a PSNE, then $V_A^*=A_{i_2,j_2}<A_{i_2,j_1}<A_{i_1,j_1}$. The only other possibility is $(i_{1},j_{1})$ being a PSNE in which case $V_A^*=A_{i_{1},j_{1}}$.
\end{proof}
\paragraph{Case 1: Column $j_2$ gets well separated.}
In this case, we aim to prove the following lemma.
\begin{lemma}\label{bandit:lem:case1}
    If event $G$ and case 1 holds, then with probability at least $1-\frac{1}{T^3}$ the regret incurred is at most $\frac{c_1\log T}{\Delta_{\min}^3}$
\end{lemma}
% \begin{proof}
%     Using \textcolor{red}{lemmas \ref{bandit:lem:nj1-upper}, \ref{bandit:lem:nj2-upper}, \ref{bandit:lem:case1-noregret1}, \ref{bandit:lem:case1-njl}, \ref{bandit:lem:case1-noregret2} and \ref{bandit:lem:case1-explore}} we get that regret incurred in case 1 is $O(\frac{\log T}{\Delta_{\min}^3})$ with probability at least $1-\frac{1}{T^3}$. 
% \end{proof}

Let us assume that event $G$ holds. Recall that we alternated between $x^{(1)}$ and $x^{(2)}$ until column $j_2$ became well-separated. Consider the time steps in which the strategy $x^{(2)}$ was played.  In Lemma \ref{bandit:lem:nj2-upper}, we show that column $j_2$ is played for at most $O(\frac{\log T}{\Delta_{\min}^2})$ times before it becomes well-separated. In Lemma \ref{bandit:lem:nj2-lower}, we show that column $j_2$ is played for at least $\Omega(\frac{\log T}{g_{c,j_2}^2})$ times before it became well-separated. This fact will be crucial when we analyse Lemma \ref{bandit:lem:case1-ratio}. In Corollary \ref{bandit:lem:case1before:regret}, we show that the regret incurred till the time-step when the column $j_2$ got well-separated is $O(\frac{\log T}{\Delta_{\min}^2})$. We prove this corollary when we analyse Case 2.

Now, consider the time-step when the column $j_2$ got well-separated. If $\bar A_{i_1,j_2}>\bar A_{i_2,j_2}$, then the matrix game on $A$ has PSNE which lies in the row $i_1$. Hence, we terminate our exploration procedure. Let us instead assume that $\bar A_{i_1,j_2}<\bar A_{i_2,j_2}$. Recall that $\ell:=\arg\max_{\ell\in\{1,2\}}|\bar A_{i_\ell,j_1}-\bar A_{i_\ell,j_2}|$ and $\bar\ell=\{1,2\}\setminus\{\ell\}$. In lemma \ref{bandit:lem:case1-ratio}, we show that that the ratio  $\frac{\bar A_{i_\ell,j_\ell}-\bar A_{i_\ell,j_{\bar \ell}}}{ A_{i_\ell,j_\ell}- A_{i_\ell,j_{\bar \ell}}}$ is between $\frac{1}{2}$ and $\frac{3}{2}$. This fact is crucial for two reasons. First, it allows us to determine the ordering between  $A_{i_\ell,j_\ell}$ and $A_{i_\ell,j_{\bar \ell}}$ (which is same as that of $\bar A_{i_\ell,j_\ell}$ and $\bar A_{i_\ell,j_{\bar \ell}}$) without actually requiring us to make the row $i_\ell$ well-separated. Second, it allows us to find a strategy that incurs a non-positive regret whenever the column $j_\ell$ is played. If $\bar A_{i_\ell,j_\ell}<\bar A_{i_\ell,j_{\bar\ell}}$, then $(i_\ell,j_\ell)$ is the unique PSNE. Hence, we terminate our exploration procedure. Let us instead assume that $\bar A_{i_\ell,j_\ell}>\bar A_{i_\ell,j_{\bar\ell}}$

Recall that we alternated between $x^{(3)}$ and $x^{(4)}$ until the row $i_{\bar \ell}$ got well-separated. In Lemma \ref{bandit:lem:case1-noregret1}, we show that the strategy $x^{(3)}$ incurs a non-positive regret whenever the column player plays column $j_\ell$. This forces the column player to play column $j_{\bar \ell}$ if it wants the row player to incur positive regret. In lemma \ref{bandit:lem:Aijbarl}, we show that $A_{i_{\bar \ell},j_{\bar \ell}}\geq V_A^*$. This implies that  playing the strategy $x^{(4)}$ incurs a non-positive regret whenever column $j_{\bar \ell}$ is played. This forces the column player to play column $j_\ell$ if it wants the row player to incur positive regret. In lemma \ref{bandit:lem:case1-njl}, we show that we incur positive regret in at most $O(\frac{\log T}{\Delta_{\min}^3})$ time steps before the row $i_{\bar \ell}$ got well-separated.

If $\bar A$ has a PSNE, then $A$ also has a PSNE. Hence, we terminate our exploration procedure. Instead if $A$ does have a PSNE, then recall that we alternated between $x^{(5)}$ and $x^{(6)}$ until we met the stopping condition $1\leq\frac{\tilde \Delta_{\min}+2\Delta}{\tilde \Delta_{\min}-2\Delta}\leq \frac{3}{2}$. In Lemma \ref{bandit:lem:case1-noregret2}, we show that the strategy $x^{(5)}$ incurs a non-positive regret whenever the column player plays column $j_2$. This forces the column player to play column $j_{1}$ if it wants the row player to incur positive regret. In the same lemma, we also showed that the strategy $x^{(6)}$ incurs a non-positive regret whenever the column player plays column $j_1$. This forces the column player to play column $j_{2}$ if it wants the row player to incur positive regret.   In Lemma \ref{bandit:lem:case1-explore}, we show that we incur positive regret in at most $O(\frac{\log T}{\Delta_{\min}^3})$ time steps before we meet the stopping condition. As we will see in section \ref{section:bandits:exploitation}, meeting this stopping condition will enable us to exploit in the exploitation phase.

We now prove all the technical lemmas for Case 1.
\begin{lemma}\label{bandit:lem:nj2-upper}
    Let us assume that Case 1 holds. Let $n_{j_2}$ denote the number of times column $j_2$ was played with the strategy $(1/2,1/2)$ before which it was well-separated. If event $G$ holds, then with probability at least {$1-\frac{1}{T^4}$} we have $n_{j_2}\leq \frac{3200\log(T^2)}{\Delta_{\min}^2}$.
\end{lemma}
\begin{proof}
    If column $j_2$ was played for $\frac{3200\log(T^2)}{\Delta_{\min}^2}$ times with the strategy $(1/2,1/2)$, then with probability at least $1-\frac{1}{T^4}$ each element of the column $j_2$ is played for at least $k=\frac{800\log(T^2)}{\Delta_{\min}^2}$ times. Now observe that $\Delta_{c,j_2}\leq\sqrt{\frac{2\log(T^2)}{k}}\leq \frac{\Delta_{\min}}{20}$. Next observe that $\bar g_{c,j_2}-2\Delta_{c,j_2}\geq g_{c,j_2}-4\Delta_{c,j_2}\geq \Delta_{\min}-4\Delta_{c,j_2}>0$. Hence we have $1\leq \frac{\bar g_{c,j_2}+2\Delta_{c,j_2}}{\bar g_{c,j_2}-2\Delta_{c,j_2}}$. Now we have $\frac{\bar g_{c,j_2}+2\Delta_{c,j_2}}{\bar g_{c,j_2}-2\Delta_{c,j_2}}\leq \frac{g_{c,j_2}+4\Delta_{c,j_2}}{g_{c,j_2}-4\Delta_{c,j_2}}\leq  \frac{\Delta_{\min}+4\Delta_{c,j_2}}{\Delta_{\min}-4\Delta_{c,j_2}} \leq\frac{\Delta_{\min}+\Delta_{\min}/5}{\Delta_{\min}-\Delta_{\min}/5}=\frac{3}{2}$.
\end{proof}

\begin{lemma}\label{bandit:lem:nj2-lower}
     Let us assume that Case 1 holds. Let $t$ be the time-step when the column $j_2$ became well-separated. If event $G$ holds, then $\min_{i\in\{1,2\}}n_{i,j_2}^t\geq \frac{128\log(T^2)}{g_{c,j_2}^2} $
\end{lemma}
\begin{proof}
    As the column $j_2$ becomes well-separated at the time-step $t$, we have $\frac{\bar g_{c,j_2}+2\Delta_{c,j_2}}{\bar g_{c,j_2}-2\Delta_{c,j_2}}\leq \frac{3}{2}$. This implies that $\Delta_{c,j_2}\leq \frac{\bar g_{c,j_2}}{10}$. As event $G$ holds, we have $\bar g_{c,j_2}\leq g_{c,j_2}+2\Delta_{c,j_2}$. Hence, we have $\Delta_{c,j_2}\leq \frac{ g_{c,j_2}}{8}$. This implies that $\min_{i\in\{1,2\}}n_{i,j_2}^t\geq \frac{128\log(T^2)}{g_{c,j_2}^2}$.
\end{proof}

\begin{lemma}\label{bandit:lem:nash-property}
    Consider a matrix $M=\begin{bmatrix}
        a&b\\
        c&d\\
    \end{bmatrix}$. If $a>c$ and $ d<b$, then we have the following:
    \begin{enumerate}
        \item If $M$ has a unique Nash Equilibrium which is not a PSNE, then ${(a-c)+(d-b)}\geq \max\{|a-b|,|c-d|\}\geq \frac{(a-c)+(d-b)}{2}$.
        \item If $M$ has a unique Nash Equilibrium which is a PSNE, then $\max\{|a-b|,|c-d|\}\geq {(a-c)+(d-b)}$
    \end{enumerate}
\end{lemma}
\begin{proof}
    If $M$ has a unique Nash Equilibrium which is not a PSNE, then $a>b$ and $c<d$. Then $\max\{|a-b|,|c-d|\}\geq \frac{(a-b)+(d-c)}{2}=\frac{(a-c)+(d-b)}{2}$. Similarly, $\max\{|a-b|,|c-d|\}\leq {(a-b)+(d-c)}={(a-c)+(d-b)}$. If $M$ has a unique Nash Equilibrium which is a PSNE, then w.l.o.g let us assume that $(1,1)$ is the unique PSNE. Hence we have $a<b$. Now observe that $d-c=(d-b)+(b-a)+(a-c)$. Hence $(d-c)>(b-a)$ which implies that $\max\{|a-b|,|c-d|\}=(d-c)\geq {(a-c)+(d-b)}$
\end{proof}

\begin{lemma}\label{bandit:lem:case1-ratio}
    Let us assume that Case 1 holds and $A_{i_1,j_2}<A_{i_2,j_2}$.  Consider a time-step $t$ which is not before the timestep when the column $j_2$ became well-separated. If event $G$ holds, then at the time step $t$, $\frac{1}{2}\leq \frac{\bar A_{i_\ell,j_\ell}-\bar A_{i_\ell,j_{\bar \ell}}}{ A_{i_\ell,j_\ell}- A_{i_\ell,j_{\bar \ell}}}\leq \frac{3}{2}$.
\end{lemma}
\begin{proof}
    
    Due to Lemma \ref{bandit:lem:nj1-lower}, each element in the column $j_1$ is sampled for at least $\frac{128\log(T^2)}{g_{c,j_1}^2}$ times. Due to Lemma \ref{bandit:lem:nj2-lower}, each element in the column $j_2$ is sampled for at least $\frac{128\log(T^2)}{g_{c,j_2}^2}$ times. This implies that at the time step $t$, $\Delta_{c,j_1}\leq \frac{g_{c,j_1}}{8}$ and $\Delta_{c,j_2}\leq \frac{g_{c,j_2}}{8}$. Let $\ell'=\arg\max_{\ell'\in\{1,2\}}g_{r,i_{\ell'}}$.

    Now we repeatedly use lemma \ref{bandit:lem:nash-property} to prove lemma \ref{bandit:lem:case1-ratio}.
    
     First let us consider the case when $\ell=\ell'$. 
    If $A$ has a unique NE which is not a PSNE, then $g_{c,j_1}+g_{c,j_2}\geq g_{r,i_{\ell}} $ and  $g_{r,i_{\ell}}\geq \frac{g_{c,j_1}+g_{c,j_2}}{2}\geq 4(\Delta_{c,j_1}+\Delta_{c,j_2})$. As event $G$ holds, we have $\bar g_{r,i_{\ell}}\geq  g_{r,i_{\ell}}-(\Delta_{c,j_1}+\Delta_{c,j_2})\geq \frac{3g_{r,i_{\ell}}}{4}$. Similarly we have $\bar g_{r,i_{\ell}}\leq  g_{r,i_{\ell}}+(\Delta_{c,j_1}+\Delta_{c,j_2})\leq \frac{5g_{r,i_{\ell}}}{4}$. {If $A$ has a unique PSNE}, then $g_{r,i_\ell}\geq g_{c,j_1}+g_{c,j_2} \geq 8(\Delta_{c,j_1}+\Delta_{c,j_2})$. As event $G$ holds, we have $\bar g_{r,i_{\ell}}\geq  g_{r,i_{\ell}}-(\Delta_{c,j_1}+\Delta_{c,j_2})\geq \frac{7g_{r,i_{\ell}}}{8}$. Similarly we have $\bar g_{r,i_{\ell}}\leq  g_{r,i_{\ell}}+(\Delta_{c,j_1}+\Delta_{c,j_2})\leq \frac{9g_{r,i_{\ell}}}{8}$.

    Next let us consider the case when $\ell\neq \ell'$. If $A$ has a unique NE which is not a PSNE, then $g_{c,j_1}+g_{c,j_2}\geq g_{r,i_{\ell'}}\geq g_{r,i_{\ell}} $ and  $g_{r,i_{\ell'}}\geq \frac{g_{c,j_1}+g_{c,j_2}}{2}\geq 4(\Delta_{c,j_1}+\Delta_{c,j_2})$. Now we have $g_{r,i_{\ell}}+(\Delta_{c,j_1}+\Delta_{c,j_2})\geq\bar g_{r,i_{\ell}}\geq \bar g_{r,i_{\ell'}}\geq g_{r,i_{\ell'}}-(\Delta_{c,j_1}+\Delta_{c,j_2})$. This implies that $g_{r,i_{\ell}}\geq \frac{g_{r,i_{\ell'}}}{2}\geq 2(\Delta_{c,j_1}+\Delta_{c,j_2})$. As event $G$ holds, we have $\bar g_{r,i_{\ell}}\geq  g_{r,i_{\ell}}-(\Delta_{c,j_1}+\Delta_{c,j_2})\geq \frac{g_{r,i_{\ell}}}{2}$. Similarly we have $\bar g_{r,i_{\ell}}\leq  g_{r,i_{\ell}}+(\Delta_{c,j_1}+\Delta_{c,j_2})\leq \frac{3g_{r,i_{\ell}}}{2}$. {If $A$ has a unique PSNE}, then $g_{r,i_{\ell'}}\geq g_{c,j_1}+g_{c,j_2} \geq 8(\Delta_{c,j_1}+\Delta_{c,j_2})$. Now we have $g_{r,i_{\ell}}+(\Delta_{c,j_1}+\Delta_{c,j_2})\geq\bar g_{r,i_{\ell}}\geq \bar g_{r,i_{\ell'}}\geq g_{r,i_{\ell'}}-(\Delta_{c,j_1}+\Delta_{c,j_2})$. This implies that $g_{r,i_{\ell}}\geq \frac{3g_{r,i_{\ell'}}}{4}\geq 6(\Delta_{c,j_1}+\Delta_{c,j_2})$. As event $G$ holds, we have $\bar g_{r,i_{\ell}}\geq  g_{r,i_{\ell}}-(\Delta_{c,j_1}+\Delta_{c,j_2})\geq \frac{5g_{r,i_{\ell}}}{6}$. Similarly we have $\bar g_{r,i_{\ell}}\leq  g_{r,i_{\ell}}+(\Delta_{c,j_1}+\Delta_{c,j_2})\leq \frac{7g_{r,i_{\ell}}}{6}$. 
\end{proof}
     
\begin{lemma}\label{bandit:lem:case1-noregret1}
     Let us assume that Case 1 holds, $A_{i_1,j_2}<A_{i_2,j_2}$ and $A_{i_{\ell},j_{\ell}}>A_{i_{\ell},j_{\bar \ell}}$.  Let $t$ be the time-step when the column $j_2$ became well-separated.  At the time step $t$, let $\delta_0:=|\bar A_{i_\ell,j_1}-\bar A_{i_\ell,j_2}|/3$.  If event $G$ holds , then the strategy $x$ where $x_{i_\ell}^{(3)}=1-\delta_0$ incurs a non-positive regret for the column $j_\ell$. That is, the following holds true:
      \begin{equation*}
          V_A^*-\langle x^{(3)}, (A_{1,j_\ell},A_{2,j_\ell})\rangle \leq 0
      \end{equation*}
\end{lemma}
\begin{proof}
    First let us consider the case when $A$ has a unique NE $(x^*,y^*)$ which is not a PSNE. It suffices to show that $x_{i_\ell}^{(3)}\geq x^*_{i_\ell}$ as $A_{i_\ell,j_\ell}>A_{i_{\bar \ell},j_\ell}$. Now we have the following:
    \begin{align*}
        x^*_{i_\ell}&=1-\frac{|A_{i_\ell,1}-A_{i_\ell,2}|}{|D|}\\
        &\leq 1-\frac{2|\bar A_{i_\ell,1}-\bar A_{i_\ell,2}|}{3|D|}\tag{due to Lemma \ref{bandit:lem:case1-ratio}}\\
        & \leq 1-\frac{|\bar A_{i_\ell,1}-\bar A_{i_\ell,2}|}{3}\tag{since $|D|\leq 2$}\\
    \end{align*}

    Next let us consider the case when $A$ has a unique PSNE. In this case, $(i_{\bar\ell},j_{\bar\ell})$ is the only element that can be PSNE as $A_{i_{\ell},j_{\ell}}>A_{i_{\bar\ell},j_{\ell}}$, $A_{i_{\ell},j_{\ell}}>A_{i_{\ell},j_{\bar\ell}}$ and $A_{i_{\bar\ell},j_{\bar\ell}}>A_{i_{\ell},j_{\bar\ell}}$. This implies that $A_{i_{\bar\ell},j_{\bar\ell}}>A_{i_{\bar\ell},j_{\ell}}$ Now observe that any convex combination of the elements in the column $j_\ell$ will be greater than the element $A_{i_{\bar \ell},j_{\bar \ell}}$. Hence, the following holds true:
     \begin{equation*}
          V_A^*-\langle x^{(3)}, (A_{1,j_\ell},A_{2,j_\ell})\rangle \leq 0
      \end{equation*}

\end{proof}

\begin{lemma}\label{bandit:lem:Aijbarl}
    Let $A_{i_1,j_1}>A_{i_2,j_1}$, $A_{i_1,j_2}<A_{i_2,j_2}$ and $A_{i_{\ell},j_{\ell}}>A_{i_{\ell},j_{\bar \ell}}$.Then $A_{i_{\bar \ell},j_{\bar \ell}}\geq V_A^*$
\end{lemma}
\begin{proof}
    Note that $A_{i_{\bar\ell},j_{\bar\ell}}>A_{i_{\ell},j_{\bar\ell}}$.
    If $A$ has a unique NE $(x^*,y^*)$ which is not a PSNE, we have $V_A^*=x_{i_\ell}^*A_{i_\ell,j_{\bar\ell}}+x_{i_{\bar\ell}}^*A_{i_{\bar\ell},j_{\bar\ell}}<A_{i_{\bar\ell},j_{\bar\ell}}$. The only other possibility is $(i_{\bar \ell},j_{\bar\ell})$ being a PSNE in which case $V_A^*=A_{i_{\bar \ell},j_{\bar\ell}}$.
\end{proof}

\begin{lemma}\label{bandit:lem:case1-njl}
     Let us assume that Case 1 holds, $A_{i_1,j_2}< A_{i_2,j_2}$ and $A_{i_{\ell},j_{\ell}}>A_{i_{\ell},j_{\bar \ell}}$. Let $t$ be the time-step when the column $j_2$ became well-separated.  At the time step $t$, let $\delta_0:=|\bar A_{i_\ell,j_1}-\bar A_{i_\ell,j_2}|/3$. Consider two strategies $x^{(3)}$ and $x^{(4)}$ where $x_{i_\ell}^{(3)}=1-\delta_0$ and $x_{i_\ell}^{(4)}=0$. Let $n_{j_\ell}$ and $n_{j_{\bar\ell}}$ denote the number of times columns $j_\ell$ and $j_{\bar \ell}$ were played with the strategies $x^{(4)}$ and $x^{(3)}$ respectively before which the row $i_{\bar \ell}$ was well-separated. If event $G$ holds, then with probability $1-\frac{1}{T^4}$, we have the following:
     \begin{equation*}
         n_{j_{\bar\ell}}\leq \frac{9600\log(T^2)}{\Delta_{\min}^3} \quad\text{and}\quad n_{j_\ell}\leq \frac{1600\log(T^2)}{\Delta_{\min}^2}
     \end{equation*}
\end{lemma}
\begin{proof}
     First due to Lemma \ref{bandit:lem:case1-ratio}, we have $1/3\geq\delta_0\geq\frac{\Delta_{\min}}{6}$. If column $j_{\bar \ell}$ was played for $\frac{9600\log(T^2)}{\Delta_{\min}^3}$ times with the strategy $x^{(3)}$, then with probability at least $1-\frac{1}{T^4}$ the element $(i_{\bar\ell},j_{\bar\ell})$ is played for $k$ times where $\frac{800\log(T^2)}{\Delta_{\min}^2}<k<\frac{1600\log(T^2)}{\Delta_{\min}^2}$ times. Also observe that $k-1\leq n_{j_\ell}\leq k$. Now observe that $\Delta_{r,i_{\bar\ell}}\leq\sqrt{\frac{2\log(T^2)}{k-1}}\leq \frac{\Delta_{\min}}{20}$. Next observe that $\bar g_{r,i_{\bar\ell}}-2\Delta_{r,i_{\bar\ell}}\geq g_{r,i_{\bar\ell}}-4\Delta_{r,i_{\bar\ell}}\geq \Delta_{\min}-4\Delta_{r,i_{\bar\ell}}>0$. Hence we have $1\leq \frac{\bar g_{r,i_{\bar\ell}}+2\Delta_{r,i_{\bar\ell}}}{\bar g_{r,i_{\bar\ell}}-2\Delta_{r,i_{\bar\ell}}}$. Now we have $\frac{\bar g_{r,i_{\bar\ell}}+2\Delta_{r,i_{\bar\ell}}}{\bar g_{r,i_{\bar\ell}}-2\Delta_{r,i_{\bar\ell}}}\leq \frac{g_{r,i_{\bar\ell}}+4\Delta_{r,i_{\bar\ell}}}{g_{r,i_{\bar\ell}}-4\Delta_{r,i_{\bar\ell}}}\leq  \frac{\Delta_{\min}+4\Delta_{r,i_{\bar\ell}}}{\Delta_{\min}-4\Delta_{r,i_{\bar\ell}}} \leq\frac{\Delta_{\min}+\Delta_{\min}/5}{\Delta_{\min}-\Delta_{\min}/5}=\frac{3}{2}$.
\end{proof}

\begin{lemma}\label{bandit:lem:case1-noregret2}
     Let us assume that Case 1 holds and $A$ have a unique NE $(x^*,y^*)$ which is not a PSNE.  Let $t$ be the time-step when the row $i_{\bar \ell}$ became well-separated.  At the time step $t$, let $\delta_1=|\bar A_{i_1,j_1}-\bar A_{i_1,j_2}|/3$ and $\delta_2=|\bar A_{i_2,j_1}-\bar A_{i_2,j_2}|/3$.  If event $G$ holds , then the strategies $x^{(5)}$ and $x^{(6)}$ where $x_{i_1}^{(5)}=\delta_2$ and $x_{i_1}^{(6)}=1-\delta_1$ incur non-positive regret on columns $j_2$ and $j_1$ respectively.
\end{lemma}
\begin{proof}
    It suffices to show that $x_{i_1}^{(5)}\leq x^*_{i_1}$ and $x_{i_1}^{(6)}\geq x^*_{i_1}$ as $A_{i_1,j_1}>A_{i_2,j_1}$ and $A_{i_2,j_2}>A_{i_1,j_2}$. W.l.o.g let $\ell=1$. Now we have the following:
    \begin{align*}
        x^*_{i_1}&=1-\frac{|A_{i_1,1}-A_{i_1,2}|}{|D|}\\
        &\leq 1- \frac{2|\bar A_{i_1,1}-\bar A_{i_1,2}|}{3|D|}\tag{{due to lemma \ref{bandit:lem:case1-ratio}}}\\
        & \leq 1-\frac{|\bar A_{i_1,1}-\bar A_{i_1,2}|}{3}\tag{as $|D|\leq 2$}\\
    \end{align*}

    As the row $i_{\bar \ell}$ (which is the row $i_2$ as per our assumption) is well separated, we have $\frac{\bar g_{r,i_2}+2\Delta_{r,i_2}}{\bar g_{r,i_2}-2\Delta_{r,i_2}}\leq \frac{3}{2}$. This implies that $\Delta_{r,i_2}\leq \frac{\bar g_{r,i_2}}{10}$. As event $G$ holds, we have $\bar g_{r,i_2}\leq g_{r,i_2}+2\Delta_{r,i_2}$. This implies that $g_{r,i_2}\geq \frac{4\bar g_{r,i_2}}{5}$. Similarly if  event $G$ holds, we have $\bar g_{r,i_2}\geq g_{r,i_2}-2\Delta_{r,i_2}$. This implies that $\bar g_{r,i_2}\geq \frac{5 g_{r,i_2}}{6}$.
    
    Now we have the following:
    \begin{align*}
        x^*_{i_1}&=\frac{|A_{i_2,1}-A_{i_2,2}|}{|D|}\\
        &\geq \frac{4|\bar A_{i_2,1}-\bar A_{i_2,2}|}{5|D|}\tag{as $g_{r,i_2}\geq \frac{4\bar g_{r,i_2}}{5}$}\\
        & >\frac{|\bar A_{i_2,1}-\bar A_{i_2,2}|}{3}\tag{as $|D|\leq 2$}\\
    \end{align*}

\end{proof}

\begin{lemma}\label{bandit:lem:case1-explore}
     Let us assume that Case 1 holds and $A$ have a unique NE $(x^*,y^*)$ which is not a PSNE.  Let $t$ be the time-step when the row $i_{\bar \ell}$ became well-separated.  At the time step $t$, let $\delta_1=|\bar A_{i_1,j_1}-\bar A_{i_1,j_2}|/3$ and $\delta_2=|\bar A_{i_2,j_1}-\bar A_{i_2,j_2}|/3$. Consider two strategies $x^{(5)}$ and $x^{(6)}$ where $ x_{i_1}^{(5)}=\delta_2$ and $x_{i_1}^{(6)}=1-\delta_1$. Let $n_{j_1}$ and $n_{j_{2}}$ denote the number of times columns $j_1$ and $j_{2}$ were played with the strategies $x^{(5)}$ and $x^{(6)}$ respectively before the condition $1\leq\frac{\tilde \Delta_{\min}+2\Delta}{\tilde \Delta_{\min}-2\Delta}\leq \frac{3}{2}$ was satisfied. If event $G$ holds, then with probability $1-\frac{2}{T^4}$, we have the following:
     \begin{equation*}
         n_{j_{1}}\leq \frac{9600\log(T^2)}{\Delta_{\min}^3} \quad \text{and} \quad n_{j_2}\leq \frac{9600\log(T^2)}{\Delta_{\min}^3}
     \end{equation*}
\end{lemma}
\begin{proof} 
    First observe that $\delta_1=|\bar A_{i_1,j_1}-\bar A_{i_1,j_2}|/3\leq 1/3$ and $\delta_2=|\bar A_{i_2,j_1}-\bar A_{i_2,j_2}|/3\leq 1/3$. Next due to the proofs of lemma \ref{bandit:lem:case1-noregret2} and lemma \ref{bandit:lem:case1-ratio}, we have $\delta_1=|\bar A_{i_1,j_1}-\bar A_{i_1,j_2}|/3\geq | A_{i_1,j_1}-A_{i_1,j_2}|/6\geq \frac{\Delta_{\min}}{6}$ and $\delta_2=|\bar A_{i_2,j_1}-\bar A_{i_2,j_2}|/3\geq | A_{i_2,j_1}- A_{i_2,j_2}|/6\geq \frac{\Delta_{\min}}{6}$.

    Next observe that $n_{j_1}-1\leq n_{j_2}\leq n_{j_1}$. If column $j_{1}$ was played for $\frac{9600\log(T^2)}{\Delta_{\min}^3}$ times with the strategy $ x^{(5)}$, then with probability at least $1-\frac{1}{T^4}$ each element in the column $j_1$ is played for at least $k=\frac{800\log(T^2)}{\Delta_{\min}^2}$ times. Similarly if column $j_{2}$ was played for $\frac{9600\log(T^2)}{\Delta_{\min}^3}$ times with the strategy $x^{(6)}$, then with probability at least $1-\frac{1}{T^4}$ each element in the column $j_2$ is played for at least $k'=\frac{800\log(T^2)}{\Delta_{\min}^2}$ times. In this case we have $\Delta\leq\frac{\Delta_{\min}}{20}$. Next observe that $\tilde \Delta_{\min}-2\Delta\geq \Delta_{\min}-4\Delta>0$. Hence we have $1\leq \frac{\tilde \Delta_{\min}+2\Delta}{\tilde \Delta_{\min}-2\Delta}$. Now we have $\frac{\tilde \Delta_{\min}+2\Delta}{\tilde \Delta_{\min}-2\Delta}\leq \frac{\Delta_{\min}+4\Delta}{\Delta_{\min}-4\Delta}\leq\frac{\Delta_{\min}+\Delta_{\min}/5}{\Delta_{\min}-\Delta_{\min}/5}=\frac{3}{2}$.
\end{proof}
\paragraph{Case 2: Row $i_1$ gets well separated.}
In this case, we aim to prove the following lemma.
\begin{lemma}\label{bandit:lem:case2}
    If event $G$ and case 2 holds, then with probability at least $1-\frac{1}{T^3}$ the regret incurred is at most $\frac{c_1\log T}{\Delta_{\min}^3}$
\end{lemma}
Let us assume that event $G$ holds. Recall that we alternated between $x^{(1)}$ and $x^{(2)}$ until the row $i_1$ became well-separated. Consider the timesteps in which the strategy $x^{(2)}$ was played.  In lemma \ref{bandit:lem:case2-n1n2}, we show that we incur positive regret in at most $O(\frac{\log T}{\Delta_{\min}^2})$ timesteps before the row $i_1$ got well-separated.

Now consider the time-step when the row $i_1$ got well-separated. If $\bar A_{i_1,j_1}<\bar A_{i_1,j_2}$, then $(i_1,j_1)$ is the PSNE of $A$. Hence, we terminate our exploration procedure. Let us instead assume that $\bar A_{i_1,j_1}>\bar A_{i_1,j_2}$.  Recall that we played $x^{(7)}$ until the column $j_2$ got well-separated. In lemma \ref{bandit:lem:case2-noregret1}, we show that the strategy $x^{(7)}$ incurs a non-positive regret whenever the column player plays column $j_1$. This forces the column player to play column $j_{2}$ if it wants the row player to incur positive regret. In lemma \ref{bandit:lem:case2-nj2}, we show that we incur positive regret in at most $O(\frac{\log T}{\Delta_{\min}^3})$ timesteps before the column $j_2$ got well-separated.

Now consider the time-step when the column $j_2$ got well-separated. If $\bar A_{i_1,j_2}<\bar A_{i_2,j_2}$, then $(i_1,j_2)$ is the PSNE of $A$. Hence, we terminate our exploration procedure. Let us instead assume that $\bar A_{i_1,j_2}<\bar A_{i_2,j_2}$. Recall that we alternated between $x^{(7)}$ and $x^{(8)}$ until the row $i_{2}$ got well-separated. Recall that the strategy $x^{(7)}$ incurs a non-positive regret whenever the column player plays column $j_1$. This forces the column player to play column $j_{2}$ if it wants the row player to incur positive regret. In lemma \ref{bandit:lem:Aijbarl:case2}, we show that $A_{i_{2},j_{2}}\geq V_A^*$. This implies that  playing the strategy $x^{(8)}$ incurs a non-positve regret whenever column $j_{2}$ is played. This forces the column player to play column $j_1$ if it wants the row player to incur positive regret. In lemma \ref{bandit:lem:case2-nj1}, we show that we incur positive regret in at most $O(\frac{\log T}{\Delta_{\min}^3})$ timesteps before the row $i_{2}$ got well-separated.

If $\bar A$ has a PSNE, then $A$ also has a PSNE. Hence, we terminate our exploration procedure. Instead if $A$ does have a PSNE, then recall that we alternated between $x^{(7)}$ and $x^{(9)}$ until we met the stopping condition $1\leq\frac{\tilde \Delta_{\min}+2\Delta}{\tilde \Delta_{\min}-2\Delta}\leq \frac{3}{2}$. Recall that the strategy $x^{(7)}$ incurs a non-positive regret whenever the column player plays column $j_1$. This forces the column player to play column $j_{2}$ if it wants the row player to incur positive regret.  In lemma \ref{bandit:lem:case2-noregret2}, we show that the strategy $x^{(9)}$ incurs a non-positive regret whenever the column player plays column $j_2$. This forces the column player to play column $j_{1}$ if it wants the row player to incur positive regret. In lemma \ref{bandit:lem:case2-explore}, we show that we incur positive regret in at most $O(\frac{\log T}{\Delta_{\min}^3})$ timesteps before we meet the stopping condition. As we will see in section \ref{section:bandits:exploitation}, meeting this stopping condition will enable us to exploit in the exploitation phase.

We now prove all the technical lemmas for Case 2.
\begin{lemma}\label{bandit:lem:case2-n1n2}
    Let us assume that Case 2 holds. Let $n_1$ denote the number of times column $j_2$ was played with strategy $x^{(1)}$ where $x_{i_1}^{(1)}=1$ before the row $i_1$ got well-separated. Let $n_2$ denote the number of times column $j_1$ was played with the strategy $x^{(2)}=(1/2,1/2)$ before the row $i_1$ got well-separated. Let $n_3$ denote the number of times column $j_2$ was played with strategy $x^{(2)}$ before the row $i_1$ got well-separated. If event $G$ holds, then with probability at least $1-\frac{2}{T^4}$ we have the following:
    \begin{equation*}
        n_{1}\leq \frac{6400\log(T^2)}{\Delta_{\min}^2} \quad \text{and} \quad n_{2}\leq \frac{3200\log(T^2)}{\Delta_{\min}^2}\quad \text{and} \quad n_3<\frac{3200\log(T^2)}{\Delta_{\min}^2}
    \end{equation*}
\end{lemma}
\begin{proof}
    
    Due to the proof of lemma \ref{bandit:lem:nj2-upper}, with probability at least $1-\frac{1}{T^4}$, we have $n_3<\frac{3200\log(T^2)}{\Delta_{\min}^2}$ otherwise column $j_2$ gets well separated first. Next due to the rules of alternating between $x^{(1)}$ and $x^{(2)}$ we have $n_2+n_3\leq n_1\leq n_2+n_3+1$.
    Now if column $j_1$ was played for $\frac{3200\log(T^2)}{\Delta_{\min}^2}$ times with the strategy $(1/2,1/2)$, then with probability at least $1-\frac{1}{T^4}$ each element of the column $j_1$ is played for at least $k=\frac{800\log(T^2)}{\Delta_{\min}^2}$ times. Now observe that $\Delta_{r,i_1}\leq\sqrt{\frac{2\log(T^2)}{k}}\leq \frac{\Delta_{\min}}{20}$. Next observe that $\bar g_{r,i_1}-2\Delta_{r,i_1}\geq g_{r,i_1}-4\Delta_{r,i_1}\geq \Delta_{\min}-4\Delta_{r,i_1}>0$. Hence we have $1\leq \frac{\bar g_{r,i_1}+2\Delta_{r,i_1}}{\bar g_{r,i_1}-2\Delta_{r,i_1}}$. Now we have $\frac{\bar g_{r,i_1}+2\Delta_{r,i_1}}{\bar g_{r,i_1}-2\Delta_{r,i_1}}\leq \frac{g_{r,i_1}+4\Delta_{r,i_1}}{g_{r,i_1}-4\Delta_{r,i_1}}\leq  \frac{\Delta_{\min}+4\Delta_{r,i_1}}{\Delta_{\min}-4\Delta_{r,i_1}} \leq\frac{\Delta_{\min}+\Delta_{\min}/5}{\Delta_{\min}-\Delta_{\min}/5}=\frac{3}{2}$.
\end{proof}
\begin{corollary}\label{bandit:lem:case1before:regret}
    Let us assume that Case 1 holds. Let $n_1$ denote the number of times column $j_2$ was played with strategy $x^{(1)}$ where $x_{i_1}^{(1)}=1$ before the column $j_2$ got well-separated. Let $n_2$ denote the number of times column $j_1$ was played with the strategy $x^{(2)}=(1/2,1/2)$ before the column $j_2$ got well-separated. Let $n_3$ denote the number of times column $j_2$ was played with strategy $x^{(2)}$ before the column $j_2$ got well-separated. If event $G$ holds, then with probability at least $1-\frac{2}{T^4}$ we have the following:
    \begin{equation*}
        n_{1}\leq \frac{6400\log(T^2)}{\Delta_{\min}^2} \quad \text{and} \quad n_{2}\leq\frac{3200\log(T^2)}{\Delta_{\min}^2}\quad \text{and} \quad n_3\leq\frac{3200\log(T^2)}{\Delta_{\min}^2}
    \end{equation*}
\end{corollary}
\begin{proof}
    Due to lemma \ref{bandit:lem:nj2-upper}, with probability at least $1-\frac{1}{T^4}$, we have $n_3\leq\frac{3200\log(T^2)}{\Delta_{\min}^2}$. Due to the proof of lemma \ref{bandit:lem:case2-n1n2}, with probability at least $1-\frac{1}{T^4}$, we have $n_1\leq \frac{6400\log(T^2)}{\Delta_{\min}^2}$ and $n_{2}\leq\frac{3200\log(T^2)}{\Delta_{\min}^2}$ otherwise row $i_1$ gets well-separated first.
\end{proof}

\begin{lemma}\label{bandit:lem:case2-noregret1}
     Let us assume that Case 2 holds and $A_{i_1,j_1}>A_{i_1,j_2}$.  Let $t$ be the time-step when the row $i_{1}$ became well-separated.  At the time step $t$, let $\delta=|\bar A_{i_1,j_1}-\bar A_{i_1,j_2}|/3$ .  If event $G$ holds , then the strategy $x^{(7)}$ where $x_{i_1}^{(7)}=1-\delta$ incurs a non-positive regret on the column $j_1$.
\end{lemma}
\begin{proof}
    First let us consider the case when $A$ has a unique NE $(x^*,y^*)$ which is not a PSNE. Then it suffices to show that $x_{i_1}^{(7)}\geq x^*_{i_1}$. As the row $i_1$ is well separated, we have $\frac{\bar g_{r,i_1}+2\Delta_{r,i_1}}{\bar g_{r,i_1}-2\Delta_{r,i_1}}\leq \frac{3}{2}$. This implies that $\Delta_{r,i_1}\leq \frac{\bar g_{r,i_1}}{10}$. As event $G$ holds, we have $\bar g_{r,i_1}\leq g_{r,i_1}+2\Delta_{r,i_1}$. This implies that $g_{r,i_1}\geq \frac{4\bar g_{r,i_1}}{5}$. Similarly  as event $G$ holds, we have $\bar g_{r,i_1}\geq g_{r,i_1}-2\Delta_{r,i_1}$. This implies that $\bar g_{r,i_1}\geq \frac{5g_{r,i_1}}{6}$. Now we have the following:
    \begin{align*}
        x^*_{i_1}&=1-\frac{|A_{i_1,1}-A_{i_1,2}|}{|D|}\\
        &\leq 1-\frac{4|\bar A_{i_1,1}-\bar A_{i_1,2}|}{5|D|}\tag{as $g_{r,i_1}\geq \frac{4\bar g_{r,i_1}}{5}$}\\
        & <1-\frac{|\bar A_{i_1,1}-\bar A_{i_1,2}|}{3}\tag{as $|D|\leq 2$}\\
    \end{align*}

    Now let us consider the case when $A$ has a unique PSNE $(i^*,j^*)$. If $i^*=i_2$ and $j^*=j_2$, then $V_A^*=A_{i_2,j_2}$ and $A_{i_1,j_1}>A_{i_2,j_1}>A_{i_2,j_2}>A_{i_2,j_1}$. This implies that the strategy $x$ incurs a negative regret on the column $j_1$. If $i^*=i_1$ and $j^*=j_2$, then $V_A^*=A_{i_1,j_2}$. Now we have the following:
    \begin{align*}
        \langle x, (A_{1,j_1},A_{2,j_1})\rangle &\geq (1-\delta)\cdot A_{i_1,j_1}\tag{as $A_{i_2,j_1}\geq 0$}\\
        &\geq A_{i_1,j_1} -\delta \tag{as $A_{i_1,j_1}\leq 1$}\\
        &> A_{i_1,j_1}-g_{r,i_1} \tag{as $g_{r,i_1}\geq \frac{4\bar g_{r,i_1}}{5}$ }\\
        &= A_{i_2,j_1}
    \end{align*}
\end{proof}

\begin{lemma}\label{bandit:lem:case2-nj2}
     Let us assume that Case 2 holds, $A_{i_1,j_2}< A_{i_2,j_2}$ and $A_{i_{1},j_{1}}>A_{i_{1},j_{2}}$. Let $t$ be the time-step when the row $i_{1}$ became well-separated.  Let $n_{j_2}$ denote the number of times column $j_{2}$ was played with the strategy $x^{(7)}$ (recall that $x_{i_1}^{(7)}=1-\delta$) before the column $j_2$ got well separated. If event $G$ holds, then with probability $1-\frac{1}{T^4}$, we have the following:
     \begin{equation*}
        n_{j_2}\leq \frac{9600\log(T^2)}{\Delta_{\min}^3}
     \end{equation*}
\end{lemma}
\begin{proof}
    {First due to the proof of lemma \ref{bandit:lem:case2-noregret1} we have $1/3\geq\delta\geq\frac{\Delta_{\min}}{6}$}. If column $j_{2}$ was played for $\frac{9600\log(T^2)}{\Delta_{\min}^3}$ times with the strategy $x^{(7)}$, then with probability at least $1-\frac{1}{T^4}$ each element in the column $j_2$ is played for at least $k=\frac{800\log(T^2)}{\Delta_{\min}^2}$ times. In this case we have $\Delta_{c,j_2}\leq\frac{\Delta_{\min}}{20}$.  Next observe that $\bar g_{c,j_2}-2\Delta_{c,j_2}\geq g_{c,j_2}-4\Delta_{c,j_2}\geq \Delta_{\min}-4\Delta_{c,j_2}>0$. Hence we have $1\leq \frac{\bar g_{c,j_2}+2\Delta_{c,j_2}}{\bar g_{c,j_2}-2\Delta_{c,j_2}}$. Now we have $\frac{\bar g_{c,j_2}+2\Delta_{c,j_2}}{\bar g_{c,j_2}-2\Delta_{c,j_2}}\leq \frac{g_{c,j_2}+4\Delta_{c,j_2}}{g_{c,j_2}-4\Delta_{c,j_2}}\leq  \frac{\Delta_{\min}+4\Delta_{c,j_2}}{\Delta_{\min}-4\Delta_{c,j_2}} \leq\frac{\Delta_{\min}+\Delta_{\min}/5}{\Delta_{\min}-\Delta_{\min}/5}=\frac{3}{2}$.
\end{proof}

\begin{lemma}\label{bandit:lem:Aijbarl:case2}
    Let $A_{i_1,j_1}>A_{i_2,j_1}$, $A_{i_1,j_2}<A_{i_2,j_2}$ and $A_{i_{1},j_{1}}>A_{i_{1},j_{2}}$. Then $A_{i_{2},j_{2}}\geq V_A^*$
\end{lemma}
\begin{proof}
    If $A$ has a unique NE $(x^*,y^*)$ which is not a PSNE, we have $V_A^*=x_{i_1}^*A_{i_1,j_{2}}+x_{i_{2}}^*A_{i_{2},j_{2}}<A_{i_{2},j_{2}}$. The only other possibility is $(i_{2},j_{2})$ being a PSNE in which case $V_A^*=A_{i_{2},j_{2}}$.
\end{proof}

\begin{lemma}\label{bandit:lem:case2-nj1}
     Let us assume that Case 2 holds, $A_{i_1,j_2}< A_{i_2,j_2}$ and $A_{i_{1},j_{1}}>A_{i_{1},j_{2}}$. Consider the two strategies $x^{(7)}$ and $x^{(8)}$ (recall that $x_{i_1}^{(7)}=1-\delta$ and $x_{i_1}^{(8)}=0$). Let $n_{j_1}$ and $n_{j_{2}}$ denote the number of times columns $j_1$ and $j_{2}$ were played with the strategies $x^{(8)}$ and $x^{(7)}$ respectively before which the row $i_{2}$ was well-separated. If event $G$ holds, then with probability $1-\frac{1}{T^4}$, we have the following:
     \begin{equation*}
         n_{j_{1}}\leq \frac{9600\log(T^2)}{\Delta_{\min}^3}\quad\text{and}\quad n_{j_2}\leq \frac{9600\log(T^2)}{\Delta_{\min}^3}
     \end{equation*}
\end{lemma}
\begin{proof}
     Recall that we have $1/3\geq\delta\geq\frac{\Delta_{\min}}{6}$. If column $j_{2}$ was played for $\frac{9600\log(T^2)}{\Delta_{\min}^3}$ times with the strategy $x^{(7)}$, then with probability at least $1-\frac{1}{T^4}$ the element $(i_{2},j_{2})$ is played for at $k$ times where $\frac{800\log(T^2)}{\Delta_{\min}^2}<k\leq \frac{9600\log(T^2)}{\Delta_{\min}^3}$ times. Also observe that $k-1\leq n_{j_1}\leq k$. Now observe that $\Delta_{r,i_{2}}\leq\sqrt{\frac{2\log(T^2)}{k-1}}\leq \frac{\Delta_{\min}}{20}$. Next observe that $\bar g_{r,i_{2}}-2\Delta_{r,i_{2}}\geq g_{r,i_{2}}-4\Delta_{r,i_{2}}\geq \Delta_{\min}-4\Delta_{r,i_{2}}>0$. Hence we have $1\leq \frac{\bar g_{r,i_{2}}+2\Delta_{r,i_{2}}}{\bar g_{r,i_{2}}-2\Delta_{r,i_{2}}}$. Now we have $\frac{\bar g_{r,i_{2}}+2\Delta_{r,i_{2}}}{\bar g_{r,i_{2}}-2\Delta_{r,i_{2}}}\leq \frac{g_{r,i_{2}}+4\Delta_{r,i_{2}}}{g_{r,i_{2}}-4\Delta_{r,i_{2}}}\leq  \frac{\Delta_{\min}+4\Delta_{r,i_{2}}}{\Delta_{\min}-4\Delta_{r,i_{2}}} \leq\frac{\Delta_{\min}+\Delta_{\min}/5}{\Delta_{\min}-\Delta_{\min}/5}=\frac{3}{2}$.
\end{proof}

\begin{lemma}\label{bandit:lem:case2-noregret2}
     Let us assume that Case 2 holds and $A$ have a unique NE $(x^*,y^*)$ which is not a PSNE.  Let $t$ be the time-step when the row $i_{2}$ became well-separated.  At the time step $t$, let $\delta_3=|\bar A_{i_2,j_1}-\bar A_{i_2,j_2}|/3$.  If event $G$ holds , then the strategy $x^{(9)}$ where $x_{i_1}^{(9)}=\delta_3$ incurs a non-positive regret on column $j_2$.
\end{lemma}
\begin{proof}
    It suffices to show that $x_{i_1}^{(9)}\leq x^*_{i_1}$ as  $A_{i_2,j_2}>A_{i_1,j_2}$. As the row $i_2$ is well separated, we have $\frac{\bar g_{r,i_2}+2\Delta_{r,i_2}}{\bar g_{r,i_2}-2\Delta_{r,i_2}}\leq \frac{3}{2}$. This implies that $\Delta_{r,i_2}\leq \frac{\bar g_{r,i_2}}{10}$. As event $G$ holds, we have $\bar g_{r,i_2}\leq g_{r,i_2}+2\Delta_{r,i_2}$. This implies that $g_{r,i_2}\geq \frac{4\bar g_{r,i_2}}{5}$. Similarly, as event $G$ holds, we have $\bar g_{r,i_2}\geq g_{r,i_2}-2\Delta_{r,i_2}$. This implies that $\bar g_{r,i_2}\geq \frac{5 g_{r,i_2}}{6}$. Now we have the following:
    \begin{align*}
        x^*_{i_1}&=\frac{|A_{i_2,1}-A_{i_2,2}|}{|D|}\\
        &\geq \frac{4|\bar A_{i_2,1}-\bar A_{i_2,2}|}{5|D|}\tag{as $g_{r,i_2}\geq \frac{4\bar g_{r,i_2}}{5}$}\\
        & >\frac{|\bar A_{i_2,1}-\bar A_{i_2,2}|}{3}\tag{as $|D|\leq 2$}\\
    \end{align*}

\end{proof}

\begin{lemma}\label{bandit:lem:case2-explore}
     Let us assume that Case 2 holds and $A$ have a unique NE $(x^*,y^*)$ which is not a PSNE.  Let $t$ be the time-step when the row $i_{2}$ became well-separated. Consider the two strategies $x^{(7)}$ and $x^{(9)}$ where $x_{i_1}^{(7)}=1-\delta$ and $x_{i_1}^{(9)}=\delta_3$. Let $n_{j_1}$ and $n_{j_{2}}$ denote the number of times columns $j_1$ and $j_{2}$ were played with the strategies $x^{(9)}$ and $x^{(7)}$ respectively before the condition $1\leq\frac{\tilde \Delta_{\min}+2\Delta}{\tilde \Delta_{\min}-2\Delta}\leq \frac{3}{2}$ was satisfied. If event $G$ holds, then with probability $1-\frac{2}{T^4}$, we have the following:
     \begin{equation*}
         n_{j_{1}}\leq \frac{9600\log(T^2)}{\Delta_{\min}^3} \quad \text{and} \quad n_{j_2}\leq \frac{9600\log(T^2)}{\Delta_{\min}^3}
     \end{equation*}
\end{lemma}
\begin{proof}
    Recall that we have $\frac{1}{3}\geq \delta\geq\frac{\Delta_{\min}}{6}$. Next due to the proof of lemma \ref{bandit:lem:case2-noregret2} we have $\frac{1}{3}\geq\delta_3\geq\frac{\Delta_{\min}}{6}$. Next observe that $n_{j_1}-1\leq n_{j_2}\leq n_{j_1}$. If column $j_{1}$ was played for $\frac{9600\log(T^2)}{\Delta_{\min}^3}$ times with the strategy $x^{(9)}$, then with probability at least $1-\frac{1}{T^4}$ each element in the column $j_1$ is played for at least $k=\frac{800\log(T^2)}{\Delta_{\min}^2}$ times. Similarly if column $j_{2}$ was played for $\frac{9600\log(T^2)}{\Delta_{\min}^3}$ times with the strategy $x^{(7)}$, then with probability at least $1-\frac{1}{T^4}$ each element in the column $j_2$ is played for at least $k'=\frac{800\log(T^2)}{\Delta_{\min}^2}$ times. In this case we have $\Delta\leq\frac{\Delta_{\min}}{20}$. Next observe that $\tilde \Delta_{\min}-2\Delta\geq \Delta_{\min}-4\Delta>0$. Hence we have $1\leq \frac{\tilde \Delta_{\min}+2\Delta}{\tilde \Delta_{\min}-2\Delta}$. Now we have $\frac{\tilde \Delta_{\min}+2\Delta}{\tilde \Delta_{\min}-2\Delta}\leq \frac{\Delta_{\min}+4\Delta}{\Delta_{\min}-4\Delta}\leq\frac{\Delta_{\min}+\Delta_{\min}/5}{\Delta_{\min}-\Delta_{\min}/5}=\frac{3}{2}$.
\end{proof}

\subsection{Exploitation Algorithm for Bandit Feedback}\label{section:bandits:exploitation}

Let $n_{i,j}^{t}$ denote the number of times the element $(i,j)$ has been sampled till the timestep $t$. We now refer the reader to Algorithm \ref{alg-2x2-bandit} which contains our exploitation procedure.

\begin{algorithm}[t!]
\caption{Logarithmic regret algorithm for $2\times 2$ games}
\begin{algorithmic}[1]
\STATE Explore using the procedure mentioned in section \ref{section:bandits:exploration}.
\STATE Let $t_1$ be the number of timesteps after which the exploration procedure terminated
\STATE If there is a row $i$ of the empirical matrix $\bar A$ that strongly dominates the other row, then play $x_t$ for rest of the rounds $t$ where $(x_t)_i=1$.
\STATE If there is a PSNE $(i,j)$ of $\bar A$, then play $x_t$ for rest of the rounds $t$ where $(x_t)_i=1$.
\STATE $t_2\gets \min_{i,j}n_{i,j}^{t_1}$.
\WHILE{$t_1<T$}
\STATE Let $(x',y')$ be the NE of $\bar A$
\STATE $\tilde D\gets |\bar A_{11}-\bar A_{12}-\bar A_{21}+\bar A_{22}|$ and $\Delta \gets \max_{i,j}\sqrt{\frac{2\log(T^2)}{n_{i,j}^{t_1}}}$
\STATE $D_1\gets \frac{\tilde D}{2}$, $T_1\gets\left(\frac{1}{\Delta}\right)^2$, $T_2\gets t_2,\widehat A\gets \bar A$ 
\STATE {Run the Algorithm \ref{subroutine-nxn-bandit} with input parameters $x'$, $D_1$, $T_1$, $T_2$, $\widehat A$ and terminate if the total number of time steps becomes $T$.} \label{2x2:line1:bandit}
\STATE Let $t_1$ denote the timestep after which the Algorithm \ref{subroutine-nxn-bandit} ended.
\STATE Update the empirical means $\bar A_{ij}$.
%\STATE $t_2\gets 2t_2$
\STATE $t_2\gets \min_{i,j}n_{i,j}^{t_1}$.
\ENDWHILE
\end{algorithmic}
\label{alg-2x2-bandit}
\end{algorithm}

For the rest of this subsection, we analyse the case where $A$ has a unique NE $(x^*,y^*)$ which is not a PSNE.
\begin{lemma}\label{2x2:tmin:bandit}
Let $t_1$ be the time-step when the exploration sub-routine terminates. If event $G$ holds, then $\min_{i,j}n_{i,j}^{t_1}\geq \frac{128\log(T^2)}{\Delta_{\min}^2} $
\end{lemma}
\begin{proof}
    Recall that the terminating condition was $1\leq\frac{\tilde \Delta_{\min}+2\Delta}{\tilde \Delta_{\min}-2\Delta}\leq \frac{3}{2}$. This implies that $\Delta\leq \frac{\tilde \Delta_{\min}}{10}$. As event $G$ holds, we have $\tilde \Delta_{\min}\leq \Delta_{\min}+2\Delta$. Hence, we have $\Delta\leq \frac{ \Delta_{\min}}{8}$. This implies that $\min_{i,j}n_{i,j}^{t_1}\geq \frac{128\log(T^2)}{\Delta_{\min}^2} $.
\end{proof}

\begin{lemma}\label{2x2:dmin:bandit}
    Let $t_1$ be the time-step when the exploration sub-routine terminates. Consider any timestep $t\geq t_1$. Let $\Delta=\max_{i,j}\sqrt{\frac{2\log(T^2)}{n_{ij}^{t}}}$. If event $G$ holds, then $\frac{3}{4}\leq \frac{\tilde\Delta_{\min}}{ \Delta_{\min}}\leq \frac{5}{4}$ and $\Delta\leq \tilde \Delta_{\min}/6$.
\end{lemma}
\begin{proof}
    Let us assume that event $G$ holds.
     Due to lemma \ref{2x2:tmin:bandit}, we have $\min_{i,j}n_{i,j}^{t_1}\geq \frac{128\log(T^2)}{\Delta_{\min}^2} $. 
     Now observe that $\Delta\leq \frac{\Delta_{\min}}{8}$. Now we have $\tilde \Delta_{\min}\geq \Delta_{\min}-2\Delta\geq \frac{3\Delta_{\min}}{4}$. Similarly we have $\tilde \Delta_{\min}\leq \Delta_{\min}+2\Delta\leq \frac{5\Delta_{\min}}{4}$. This also implies that $\Delta\leq \frac{\tilde \Delta_{\min}}{6}$
\end{proof}

\begin{lemma}\label{2x2:D:bandit}
    Let $t_1$ be the time-step when the exploration sub-routine terminates. Consider any timestep $t\geq t_1$. If event $G$ holds, then $\frac{2}{3}\leq \frac{|D|}{\tilde D}\leq \frac{4}{3}$.
\end{lemma}
\begin{proof}
      Let $\Delta=\max_{i,j}\sqrt{\frac{2\log(T^2)}{n_{ij}^{t}}}$.  Due to lemma \ref{2x2:dmin:bandit} we have $\Delta\leq \frac{\tilde \Delta_{\min}}{6}$. As $2\tilde\Delta_{\min}\leq \tilde D$, we have $\Delta\leq \frac{\tilde D}{12}$. Now we have that $\frac{|D|}{\tilde D}\leq \frac{\tilde D+4\Delta}{\tilde D}\leq \frac{\tilde D+\tilde D/3}{\tilde D}\leq\frac{4}{3}$. Similarly we have that $\frac{|D|}{\tilde D}\geq \frac{\tilde D-4\Delta}{\tilde D}\geq \frac{\tilde D-\tilde D/3}{\tilde D}\geq\frac{2}{3}$. 
\end{proof}

\begin{lemma}\label{2x2:deviation:bandit}
    Let $t_1$ be the time-step when the exploration sub-routine terminates. Consider any timestep $t\geq t_1$. Let $(x',y')$ be the NE of the empirical matrix $\bar A$ at timestep $t$. If event $G$ holds, then $|x^*_1-x'_1|\leq \frac{2\Delta}{\tilde D}$ where $\Delta=\max_{i,j}\sqrt{\frac{2\log(T^2)}{n_{ij}^{t}}}$ and $\tilde D=|\bar A_{11}-\bar A_{12}-\bar A_{21}+\bar A_{22}|$. 
\end{lemma}
\begin{proof}
         W.l.o.g let us assume that $D:=A_{11}-A_{12}-A_{21}+A_{22}>0$. As $t\geq t_1$, due to lemma \ref{2x2:tmin:bandit} we have $\Delta\leq \frac{\tilde \Delta_{\min}}{6}\leq \frac{\tilde D}{12}$. Let $\Delta_{1}:=(A_{22}-A_{21})-(\bar A_{22}-\bar A_{21})$ and $\Delta_{2}:=(A_{11}-A_{12})-(\bar A_{12}-\bar A_{12})$. Observe that if $0<a<b$ then $\frac{a+x}{b+x}$ is an increasing function when $x\geq 0$. Also observe that if $0<a<b$ then $\frac{a-x}{b-x}$ is a decreasing function when $ x\leq a$. Using the previous two observations, we now have the following:
    \begin{align*}
        x^*_1&=\frac{A_{22}-A_{21}}{D}\\
        &=\frac{\bar A_{22}-\bar A_{21}+\Delta_1}{\tilde D+\Delta_1+\Delta_2}\\
        &\leq \frac{\bar A_{22}-\bar A_{21}+2\Delta}{\tilde D+2\Delta+\Delta_2}\tag{as $\Delta_1\leq 2\Delta$}\\
        &\leq \frac{\bar A_{22}-\bar A_{21}+2\Delta}{\tilde D}\tag{as $\Delta_2\geq-2\Delta$}\\
        &=x'+\frac{2\Delta}{\tilde D}
    \end{align*}
    Similarly we have that following:
    \begin{align*}
        x^*_1&=\frac{A_{22}-A_{21}}{D}\\
        &=\frac{\bar A_{22}-\bar A_{21}+\Delta_1}{\tilde D+\Delta_1+\Delta_2}\\
        &\geq \frac{\bar A_{22}-\bar A_{21}-2\Delta}{\tilde D-2\Delta+\Delta_2}\tag{as $\Delta_1\geq -2\Delta$}\\
        &\geq \frac{\bar A_{22}-\bar A_{21}-2\Delta}{\tilde D}\tag{as $\Delta_2\leq2\Delta$}\\
        &=x'-\frac{2\Delta}{\tilde D}
    \end{align*}
\end{proof}

\begin{lemma}\label{2x2:feasible:bandit}
    Let $t_1$ be the time-step when the exploration sub-routine terminates. Consider any timestep $t\geq t_1$. Let $(x',y')$ be the NE of the empirical matrix $\bar A$ at timestep $t$. If event $G$ holds, then for any $i\in[2]$, $x_i'-\frac{2 \Delta}{\tilde D}\geq\frac{2\Delta_{\min}}{3|D|}$ and $x_i'+\frac{2\Delta}{\tilde D}< 1$ where $\Delta=\max_{i,j}\sqrt{\frac{2\log(T^2)}{n_{ij}^{t}}}$ and $\tilde D=|\bar A_{11}-\bar A_{12}-\bar A_{21}+\bar A_{22}|$. 
\end{lemma}
\begin{proof}
    W.l.o.g let $\tilde D>0$. As $t\geq t_1$, due to lemma \ref{2x2:dmin:bandit} we have $\tilde\Delta_{\min}\geq \frac{3\Delta_{\min}}{4}$ and $\Delta\leq \frac{\tilde \Delta_{\min}}{6}$. Similarly, due to lemma \ref{2x2:D:bandit} we have $\frac{2\tilde D}{3}\leq |D|$
    
    $(\bar A_{22}-\bar A_{21})-2\Delta\geq (\bar A_{22}-\bar A_{21})-\frac{\tilde \Delta_{\min}}{3}\geq \frac{2(\bar A_{22}-\bar A_{21})}{3}\geq\frac{2\tilde\Delta_{\min}}{3}\geq \frac{\Delta_{\min}}{2}$. Hence $x_1'-\frac{2 \Delta}{\tilde D}\geq\frac{\Delta_{\min}}{\tilde D}\geq \frac{2\Delta_{\min}}{3|D|}$. Similarly, we can show that $x_2'-\frac{2 \Delta}{\tilde D}\geq\frac{2\Delta_{\min}}{3|D|}$. Next we have $x_1'+\frac{2\Delta}{\tilde D}=1-x_2'+\frac{2\Delta}{\tilde D}< 1-\frac{2\Delta}{\tilde D}+\frac{2 \Delta}{\tilde D}= 1 $. Similarly, we can show that $x_2'+\frac{2\Delta}{\tilde D}< 1$. 
\end{proof}

\begin{lemma}\label{2x2:t2:bandit}
    Consider a call of the algorithm \ref{subroutine-nxn-bandit} with parameters mentioned in the line \ref{2x2:line1:bandit} of algorithm \ref{alg-2x2-bandit}. If event $G$ holds, then with probability at least $1-\frac{1}{T^4}$ one of the columns is played for at most $\frac{3t_2|D|}{\Delta_{\min}}$ times
\end{lemma}
\begin{proof}
    Due to lemma \ref{2x2:tmin:bandit}, $t_2\geq \frac{128\log(T^2)}{\Delta_{\min}^2}$. Note that due to lemma \ref{2x2:feasible:bandit} when column $j$ is played, each entry of that column is played with probability at least $\frac{2\Delta_{\min}}{3|D|}$.  Hence, if both columns are played for at least $\frac{3t_2|D|}{\Delta_{\min}}$ times, then each element is played for at least $T_2=t_2$ times with probability at least $1-\frac{1}{T^4}$ due to Chernoff bound.
\end{proof}

\begin{lemma}\label{2x2:regret:bandit}
     If event $G$ holds, then with probability at least $1-\frac{1}{T^4}$, a call of the algorithm \ref{subroutine-nxn-bandit} with parameters mentioned in the line \ref{2x2:line1:bandit} of algorithm \ref{alg-2x2-bandit} incurs a regret of at most $O\left(\frac{\log T}{\Delta_{\min}^2}\right)$.
\end{lemma}
\begin{proof}
Let us assume that the algorithm \ref{subroutine-nxn-bandit} gets called immediately after the timestep $t_1$. Let $\Delta = \max_{i,j}\sqrt{\frac{2\log(T^2)}{n_{i,j}^{t_1}}}$. Recall that $\widehat A=\bar A$, $T_1=(\frac{1}{\Delta})^2$, $T_2=t_2$, $D_1=\frac{\tilde D}{2}$. As event $G$ holds, for any $i,j$ we have $|A_{i,j}-\widehat A_{i,j}|\leq \Delta=\frac{1}{\sqrt{T_1}}$. Due to lemma \ref{2x2:tmin:bandit}, we have $\min_{i,j}n_{i,j}^{t_1}\geq \frac{128\log(T^2)}{\Delta_{\min}^2}$.  Recall that $(x',y')$ is the NE of $\bar A$. Due to lemma \ref{2x2:deviation:bandit}, we have $|x'_1-x^*_1|\leq  \frac{2\Delta}{\tilde D}=\frac{1}{D_1\sqrt{T_1}}$. Due to lemma \ref{2x2:feasible:bandit}, we have $\min_i\min\{x_i',1-x_i'\}\geq \frac{2\Delta}{\tilde D}= \frac{1}{D_1\sqrt{T_1}}$. Moreover, due to lemma \ref{2x2:t2:bandit} , one of the columns is played for at most $\frac{t_2|D|}{\Delta_{\min}}$ times with probability at least $1-\frac{1}{T^4}$.  Hence due to Theorem \ref{2x2:subroutine:bandit}, the regret incurred is $\frac{c_1}{D_1}+\frac{c_2T_3}{D_1T_1}=O\left(\frac{\log T}{\Delta_{\min}^2}\right)$ where $c_1,c_2$ are absolute constants. We get the last equality as $T_1=\Theta(\frac{t_2}{\log T})$, $D_1=\Theta(|D|)$, and  $T_3=O(\frac{t_2|D|}{\Delta_{\min}^2}+\frac{\sqrt{T_1}}{\Delta_{\min}})$ (recall the definition of $T_3$ in the Theorem \ref{2x2:subroutine:bandit}).
\end{proof}
    
\begin{lemma}\label{2x2:logT:bandit}
    Let the input matrix $A$ have a unique NE $(x^*,y^*)$ that is not a PSNE. If event $G$ holds, the algorithm \ref{alg-2x2-bandit} calls algorithm \ref{2x2:subroutine:bandit} for at most $\log_2T$ times.
\end{lemma}
\begin{proof}
    Let $t_{(1,k)}$ denote the value of $t_1$ when algorithm \ref{2x2:subroutine:bandit} was called by algorithm \ref{alg-2x2-bandit} for the $k$-th time. 
    Similarly let $t_{(2,k)}$ denote the value of $t_2$ when algorithm \ref{2x2:subroutine:bandit} was called by algorithm \ref{alg-2x2-bandit} for the $k$-th time. 
    Observe that $t_{(2,k+1)}\geq \min_{i,j}n_{i,j}^{t_{(1,k)}}+t_{(2,k)}=2t_{(2,k)}$. Recall that the algorithm \ref{alg-2x2-bandit} can call algorithm \ref{subroutine-nxn-bandit} until $t_1$ becomes greater than $T$. And $t_1\geq t_2$. Hence the algorithm \ref{alg-2x2-bandit} calls algorithm \ref{2x2:subroutine:bandit} for at most $\log_2T$ times. 
\end{proof}

\subsection{Anytime algorithm with $\sqrt{T}$ minimax guarantee}\label{appendix:anytime-bandit}
We provide a brief description on how to design an anytime algorithm that has a poly-logarithmic instance dependent regret and $\sqrt{T}$ minimax regret. 

First we make a small modification to our bandit feedback algorithm to ensure that it has a $\sqrt{T}$ minimax regret. During the exploration phase, count the number of steps in which the bandit feedback algorithm can potentially get a positive regret. If the number of such steps becomes equal to $\sqrt{T}$, we terminate the algorithm and instead execute the UCB algorithm by \cite{o2021matrix} for the rest of the time horizon. This ensures that our bandit feedback algorithm will have the poly-logarithmic instance dependent regret and $\sqrt{T}$ minimax regret. 

Let $t$ be initialized to some large constant. Now repeatedly do the following:
\begin{itemize}
    \item Initialize all the statistics to 0 and execute the modified algorithm for bandit feedback with time horizon $T=t$.
    \item Double the variable $t$, that is $t\gets 2t$.
\end{itemize}

It is easy to observe that this standard doubling trick leads to only an additional log-factor in our regret.

\subsection{$2\times 2$ matrix game under Bandit Feedback with non-unique Nash Equilibrium }\label{appendix:2x2:non-unique}
In this section, we present an algorithm for the case when Nash equilibrium is not unique. Consider an input matrix $A\in[0,1]^{2\times 2}$ that has non-unique Nash equilibrium . 

At each time step we maintain an empirical matrix $\bar A$. Let $n_{i,j}^{t}$ denote the number of times the element $(i,j)$ has been sampled up to the time step $t$. At a time step $t$, let $\delta_{i,j}:=\sqrt{\frac{2\log(T^2)}{n^t_{i,j}}}$.

We now describe our algorithm. Recall the terminology of well-separated rows and columns from section \ref{appendix:2x2}. Also recall that $g_{r,i}=| A_{i,1}- A_{i,2}|$ and $g_{c,j}=|A_{1,j}-A_{2,j}|$. 

\paragraph{Uniform Exploration:}Now in each timestep $t$, we play $x_t=(1/2,1/2)$ until one of the columns is well separated. Let us assume that column $j_1$ gets well separated and $\bar A_{i_1,j_1}>\bar A_{i_2,j_1}$ where $i_1,i_2\in\{1,2\}$. Let $j_2\in\{1,2\}\setminus\{j_1\}$.  {Next we alternatively play two strategies $x^{(1)}$ and $x^{(2)}$ where $x_{i_1}^{(1)}=1$ and $x_{i_1}^{(2)}=1/2$} until either the column $j_2$ gets well separated or the following holds true:
\begin{equation}
   |\bar A_{i_1,j_2}-\bar A_{i_2,j_1}|\geq \frac{\delta_{i_1,j_2}+\delta_{i_2,j_1}}{2} \label{eq:non-unique-NE}
\end{equation}

We alternate from $x^{(1)}$ to $x^{(2)}$ when column $j_2$ gets played, and we alternate from $x^{(2)}$ to $x^{(1)}$ after a single round of play. Now we have the following two scenarios (namely Scenario 1 and Scenario 2).

\paragraph{Scenario 1: Column $j_2$ gets well separated.} If $\bar A_{i_1,j_2}>\bar A_{i_2,j_2}$, then we play $x^{(1)}$ for the rest of the rounds. On the other hand, if  $\bar A_{i_1,j_2}<\bar A_{i_2,j_2}$, then let $\ell:=\arg\max_{\ell\in\{1,2\}}|\bar A_{i_\ell,j_1}-\bar A_{i_\ell,j_2}|$ and $\bar\ell=\{1,2\}\setminus\{\ell\}$. If $\bar A_{i_\ell,j_\ell}<\bar A_{i_\ell,j_{\bar\ell}}$, then we alternatively play two strategies $x^{(1)}$ and $x^{(3)}$ where $x_{i_2}^{(3)}=1$ until the equation \ref{eq:non-unique-NE} holds true. We alternate from $x^{(1)}$ to $x^{(3)}$ when column $j_2$ gets played, and we alternate from $x^{(3)}$ to $x^{(1)}$  when column $j_1$ gets played. If $\bar A_{i_1,j_2}<\bar A_{i_2,j_1}$ then we play $x^{(3)}$ for the rest of the rounds. Instead if $\bar A_{i_1,j_2}>\bar A_{i_2,j_1}$ then we play $x^{(1)}$ for the rest of the rounds.

%Let it be the case that equation \ref{eq:non-unique-NE} gets satisfied first. If $\bar A_{i_2,j_1}<\bar A_{i_1,j_2}$ then we play the strategy $x^{(1)}$ for the rest of the rounds. If $\bar A_{i_2,j_1}>\bar A_{i_1,j_2}$ then we play the strategy $x^{(3)}$ where $x_{i_2}^{(1)}=1$ for the rest of the rounds.
\paragraph{Scenario 2: Equation \ref{eq:non-unique-NE} is satisfied.}  If $\bar A_{i_1,j_2}<\bar A_{i_2,j_1}$ then we play $x^{(3)}$ for the rest of the rounds. Instead if $\bar A_{i_1,j_2}>\bar A_{i_2,j_1}$ then we play $x^{(1)}$ for the rest of the rounds.

Now we break down our analysis into multiple cases and provide the regret guarantee for each case. 

\textbf{Case 1:} Let us assume that $A_{1,1}=A_{2,1}$ and $A_{2,1}=A_{2,2}$. It is easy to observe than any strategy will incur a non-positive regret as $V_A^*=\max\{A_{1,1},A_{2,1}\}$.

\textbf{Case 2:}  Let us assume that column $j_1$ gets well separated and $\bar A_{i_1,j_1}>\bar A_{i_2,j_1}$. Using Hoeffding's inequality and Chernoff bound, one can show that with probability at least $1-\frac{1}{T^4}$ it takes $O\left(\frac{\log T}{g_{c,j_1}^2}\right)$ for the column $j_1$ to get well separated and we can conclude that $A_{i_1,j_1}>A_{i_2,j_1}$. Let $\Delta_1:=|A_{i_1,j_2}-A_{i_2,j_1}|$ and $\Delta_2:=|A_{i_1,j_2}-A_{i_1,j_2}|$. Now we divide the analysis of this case into multiple sub-cases.

\textbf{Sub-case (a):} Here we consider the case when $A_{i_1,j_2}=A_{i_2,j_1}$. In this case we also have $A_{i_1,j_2}=A_{i_2,j_1}$ otherwise $A$ has a unique Nash equilibrium. This implies that $V_A^*=A_{i_1,j_2}$. Hence, in this sub-case we incur a non-positive regret.

\textbf{Sub-case (b):} Here we consider the case when $A_{i_1,j_1}=A_{i_1,j_2}>A_{i_2,j_2}$. Now observe we incur zero-regret for playing $x^{(1)}$ as $V_A^*=A_{i_1,j_1}$. If during the run of the algorithm scenario 1 holds, then using an analysis similar to that of lemma \ref{bandit:lem:case2-n1n2} one can show that with probability at least $1-\frac{1}{T^4}$ we incur positive regret in at most $O(\frac{\log T}{\Delta_{1}^2}+\frac{\log T}{\Delta_{2}^2})$ rounds before the column $j_2$ gets well separated. As column $j_2$ is well-separated, we have $\bar A_{i_1,j_2}>\bar A_{i_2,j_2}$. Now for the rest of the rounds, we play $x^{(1)}$ and therefore incur non-positive regret. 

Instead if scenario 2 holds, in that case too using an analysis similar to that of lemma \ref{bandit:lem:case2-n1n2} one can show that with probability at least $1-\frac{1}{T^4}$ we incur positive regret in at most $O(\frac{\log T}{\Delta_{1}^2}+\frac{\log T}{\Delta_{2}^2})$ rounds before equation \ref{eq:non-unique-NE} gets satisfied. As equation \ref{eq:non-unique-NE} gets satisfied, we have $\bar A_{i_1,j_2}>\bar A_{i_2,j_1}$ with probability at least $1-\frac{1}{T^4}$ due to hoeffding's inequality. Now for the rest of the rounds, we play $x^{(1)}$ and therefore incur non-positive regret.

\textbf{Sub-case (c):} Here we consider the case when $A_{i_1,j_1}=A_{i_1,j_2}=A_{i_2,j_2}$. Now observe we incur zero-regret for playing $x^{(1)}$ and also when column $j_2$ is played as $V_A^*=A_{i_1,j_1}$. In this case, during the run of the algorithm scenario 2 holds. Using an analysis similar to that of lemma \ref{bandit:lem:case2-n1n2} one can show that with probability at least $1-\frac{1}{T^4}$ we incur positive regret in at most $O(\frac{\log T}{\Delta_{1}^2})$ rounds before equation \ref{eq:non-unique-NE} gets satisfied. As equation \ref{eq:non-unique-NE} gets satisfied, we have $\bar A_{i_1,j_2}>\bar A_{i_2,j_1}$ with probability at least $1-\frac{1}{T^4}$ due to hoeffding's inequality. Now for the rest of the rounds, we play $x^{(1)}$ and therefore incur non-positive regret.

\textbf{Sub-case (d):} Here we consider the case when $A_{i_1,j_1}=A_{i_1,j_2}<A_{i_2,j_2}$. Now observe we incur non-positive regret for playing $x^{(1)}$ and also when column $j_2$ is played as $V_A^*=A_{i_1,j_1}$. If during the run of the algorithm scenario 1 holds, then using an analysis similar to that of lemma \ref{bandit:lem:case2-n1n2} one can show that with probability at least $1-\frac{1}{T^4}$ we incur positive regret in at most $O(\frac{\log T}{\Delta_{1}^2}+\frac{\log T}{\Delta_{2}^2})$ rounds before the column $j_2$ gets well separated. As column $j_2$ is well-separated, we have $\bar A_{i_1,j_2}<\bar A_{i_2,j_2}$. Next observe that we incur non-positive regret whenever either entry $A_{i_1,j_1}$ is played or entry $A_{i_2,j_2}$ is played. Using hoeffding inequality, one can show that with probability at least $1-\frac{1}{T^4}$ we incur positive regret in at most $O(\frac{\log T}{\Delta_{1}^2}+\frac{\log T}{\Delta_{2}^2})$ rounds before equation \ref{eq:non-unique-NE} gets satisfied. As equation \ref{eq:non-unique-NE} gets satisfied, we have $\bar A_{i_1,j_2}>\bar A_{i_2,j_1}$ with probability at least $1-\frac{1}{T^4}$ due to hoeffding's inequality. Now for the rest of the rounds, we play $x^{(1)}$ and therefore incur non-positive regret.

Instead if scenario 2 holds, in that case too using an analysis similar to that of lemma \ref{bandit:lem:case2-n1n2} one can show that with probability at least $1-\frac{1}{T^4}$ we incur positive regret in at most $O(\frac{\log T}{\Delta_{1}^2}+\frac{\log T}{\Delta_{2}^2})$ rounds before equation \ref{eq:non-unique-NE} gets satisfied. As equation \ref{eq:non-unique-NE} gets satisfied, we have $\bar A_{i_1,j_2}>\bar A_{i_2,j_1}$ with probability at least $1-\frac{1}{T^4}$ due to hoeffding's inequality. Now for the rest of the rounds, we play $x^{(1)}$ and therefore incur non-positive regret.

\textbf{Sub-case (e):} Here we consider the case when $A_{i_1,j_1}>A_{i_1,j_2}>A_{i_2,j_1}$. In this case we have $A_{i_1,j_2}=A_{i_2,j_2}$ otherwise the matrix has a unique Nash equilibrium. Now observe we incur non-positive regret for playing $x^{(1)}$ and also when column $j_2$ is played as $V_A^*=A_{i_1,j_2}$. Next observe that in this sub-case, during the run of the algorithm scenario 2 holds. Using an analysis similar to that of lemma \ref{bandit:lem:case2-n1n2} one can show that with probability at least $1-\frac{1}{T^4}$ we incur positive regret in at most $O(\frac{\log T}{\Delta_{1}^2})$ rounds before equation \ref{eq:non-unique-NE} gets satisfied. As equation \ref{eq:non-unique-NE} gets satisfied, we have $\bar A_{i_1,j_2}>\bar A_{i_2,j_1}$ with probability at least $1-\frac{1}{T^4}$ due to hoeffding's inequality. Now for the rest of the rounds, we play $x^{(1)}$ and therefore incur non-positive regret.

\textbf{Sub-case (f):} Here we consider the case when $A_{i_1,j_1}>A_{i_1,j_2}=A_{i_2,j_1}$. In this case we have $A_{i_1,j_2}=A_{i_2,j_2}$ otherwise the matrix has a unique Nash equilibrium. Now observe that any strategy incurs a non-positive regret as $V_A^*=A_{i_1,j_2}$.

\textbf{Sub-case (g):} Here we consider the case when $A_{i_1,j_2}<A_{i_2,j_1}$ and $A_{i_2,j_2}<A_{i_2,j_1}$. In this case we have $A_{i_1,j_2}=A_{i_2,j_2}$ otherwise the matrix has a unique Nash equilibrium. Now observe that any strategy incurs a non-positive regret as $V_A^*=A_{i_1,j_2}$.

\textbf{Sub-case (h):} Here we consider the case when $A_{i_1,j_2}<A_{i_2,j_1}\leq A_{i_2,j_2}$. In this case we have $A_{i_2,j_1}=A_{i_2,j_2}$ otherwise the matrix has a unique Nash equilibrium. Now observe that we incur non-positive regret whenever we play the entries $A_{i_1,j_1}, A_{i_2,j_1}$ and $A_{i_2,j_2}$ as $V_A^*=A_{i_2,j_2}$. If during the run of the algorithm scenario 1 holds, then using an analysis similar to that of lemma \ref{bandit:lem:case2-n1n2} one can show that with probability at least $1-\frac{1}{T^4}$ we incur positive regret in at most $O(\frac{\log T}{\Delta_{1}^2}+\frac{\log T}{\Delta_{2}^2})$ rounds before the column $j_2$ gets well separated. As column $j_2$ is well-separated, we have $\bar A_{i_1,j_2}<\bar A_{i_2,j_2}$. Using hoeffding inequality, one can show that with probability at least $1-\frac{1}{T^4}$ we incur positive regret in at most $O(\frac{\log T}{\Delta_{1}^2}+\frac{\log T}{\Delta_{2}^2})$ rounds before equation \ref{eq:non-unique-NE} gets satisfied. As equation \ref{eq:non-unique-NE} gets satisfied, we have $\bar A_{i_1,j_2}<\bar A_{i_2,j_1}$ with probability at least $1-\frac{1}{T^4}$ due to hoeffding's inequality. Now for the rest of the rounds, we play $x^{(3)}$ and therefore incur non-positive regret.

Instead if scenario 2 holds, in that case too using an analysis similar to that of lemma \ref{bandit:lem:case2-n1n2} one can show that with probability at least $1-\frac{1}{T^4}$ we incur positive regret in at most $O(\frac{\log T}{\Delta_{1}^2}+\frac{\log T}{\Delta_{2}^2})$ rounds before equation \ref{eq:non-unique-NE} gets satisfied. As equation \ref{eq:non-unique-NE} gets satisfied, we have $\bar A_{i_1,j_2}<\bar A_{i_2,j_1}$ with probability at least $1-\frac{1}{T^4}$ due to hoeffding's inequality. Now for the rest of the rounds, we play $x^{(3)}$ and therefore incur non-positive regret.